\documentclass{article} 
\usepackage{iclr2026_conference,times}

\usepackage{amsmath,amsfonts,bm}

\def\eqref#1{equation~\ref{#1}}

\def\1{\bm{1}}

\DeclareMathAlphabet{\mathsfit}{\encodingdefault}{\sfdefault}{m}{sl}
\SetMathAlphabet{\mathsfit}{bold}{\encodingdefault}{\sfdefault}{bx}{n}

\renewcommand{\eqref}[1]{(\ref{#1})}

\usepackage{hyperref}
\usepackage{url}

\usepackage[american]{babel}

\usepackage{mathtools} 
\usepackage[table]{xcolor}
\usepackage{booktabs} 
\usepackage{tikz} 
\usepackage{multirow} 
\usepackage{algorithm}
\usepackage{algpseudocode}
\usepackage{graphicx}
\usepackage{mathtools}
\usepackage{multirow}
\usepackage{amsthm}
\usepackage{amssymb}
\usepackage{todonotes}
\usepackage{pdflscape}
\usepackage{rotating}
\usepackage{multicol}
\usepackage{float}
\usepackage{xcolor}
\usepackage{subcaption}
\usepackage{wrapfig}
\usepackage[inline]{enumitem} 
\usepackage[many]{tcolorbox}

\usepackage{environ}
\NewEnviron{contributions}{\subsubsection*{Author Contributions}\BODY}
\NewEnviron{acknowledgements}{\subsubsection*{Acknowledgments}\BODY}

\newtheorem{theorem}{Theorem}
\tcolorboxenvironment{theorem}{
  enhanced,
  boxrule=0pt,
  frame hidden,
  colback=gray!12,                       
  sharp corners,
  top=3mm, bottom=3mm, left=3mm, right=3mm
}
\newtheorem{lemma}{Lemma}

\tcolorboxenvironment{lemma}{
  enhanced,
  boxrule=0pt,
  frame hidden,
  colback=gray!12,                        
  sharp corners,
  top=3mm, bottom=3mm, left=3mm, right=3mm
}

\title{Learning  Implicit Causal World Models from Multi-Agent Demonstrations}

\author{Jasorsi Ghosh \\
    Department of Computer Science\\
    Purdue University\\
    West Lafayette, Indiana, USA \\
    \texttt{ghosh117@purdue.edu}
}

\iclrfinalcopy 
  
\begin{document}
\maketitle
\begin{abstract}
In model-based reinforcement learning, world models exist as internal simulators, but their training often conflates statistical correlations with causal mechanisms. This problem is exacerbated in multi-agent systems where physical transitions are intertwined with strategic agent intents, causing world models to fail under distribution shift.  We introduce Implicit Causal World Models to recover environmental dynamics from offline demonstrations without requiring pre-defined causal graphs. By incorporating policy variance, we render world models discoverable via the sequential backdoor condition. Evaluations across coordination tasks (Two-Door, Navigation, and Giveway) demonstrate that these models provide interpretable causal representations under both full and partial observability, with model accuracy scaling directly with interventional strength.

\end{abstract}
\vspace{-1.3 em}
\section{Introduction}\label{sec:intro}
\vspace{-1.0 em}
In the evolution of model-based reinforcement learning \citep{moerland2022modelbasedreinforcementlearningsurvey}, world models have progressed from early predictive systems \citep{10.1145/122344.122377} to sophisticated generative architectures \citep{hafner2024masteringdiversedomainsworld,hafner2019learninglatentdynamicsplanning,lecun2022path} that utilize latent state-space models to simulate future trajectories. We frame these world models as an agent’s "internal simulators," representing its belief about the environment. While world models improve sample efficiency \citep{https://doi.org/10.5281/zenodo.1207631} in controlled settings , they remain fundamentally limited by their reliance on statistical correlations rather than causal underlying mechanisms \citep{kipf2020contrastivelearningstructuredworld}. In high-stakes environments, this leads to the spurious correlation problem, where models incorrectly attribute state transitions to non-causal features present during expert demonstrations. 
\begin{wrapfigure}{l}{0.5\linewidth}
  \centering
  \includegraphics[trim={0pt 0pt 00pt 50pt}, clip,width=\linewidth]{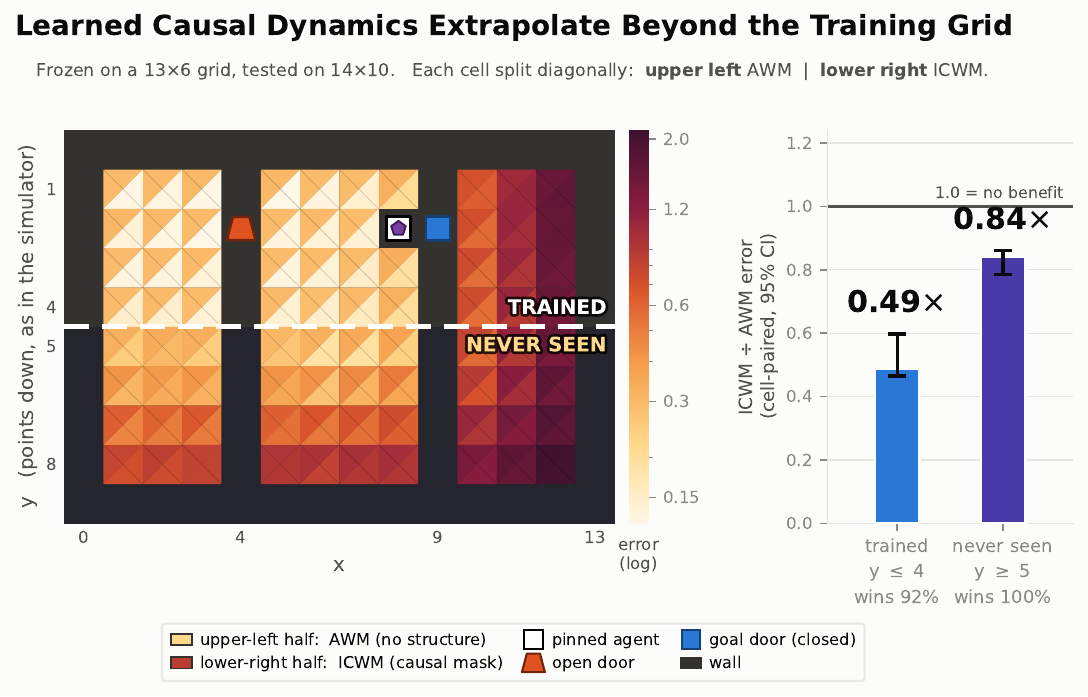}
  \caption{\textbf{A learned causal dynamics transfers to states the model has never
  seen.} The region below the dashed line is unseen in training. Full details in Appendix \ref{app:structure-correlation}}
  \label{fig:hero}
\end{wrapfigure}
This challenge is compounded in multi-agent systems, where environmental physics are intertwined \citep{lowe2020multiagentactorcriticmixedcooperativecompetitive,dasgupta2019causalreasoningmetareinforcementlearning} with the strategic, often confounded, intents of other agents. For example, an autonomous vehicle model might incorrectly learn to accelerate because an adjacent car moves, rather than because the traffic light turned green. In this scenario, the physics of the green light correlates perfectly with another agent’s intent to move, leading to causal confusion.

Causal World Models \citep{ahmed2020causalworldroboticmanipulationbenchmark,nam2026causal} redefine internal simulators by replacing black-box transition functions with a structured network of causal links representing the environment’s governing mechanisms. While prior work \citep{richens2024robustagentslearncausal} highlights the theoretical robustness of agents that learn causal models, this structural approach explicitly resolves distorted dynamics \citep{xia2022neuralcausalmodelscounterfactual} introduced by selection bias or unobserved confounders in offline data. By utilizing structured world models , this grounding in environmental physics ensures that transition dynamics remain invariant \citep{peters2015causalinferenceusinginvariant}, even under distributional shifts \citep{deng2025energybasedtransferreinforcementlearning,arjovsky2020invariantriskminimization} in the observational space.

\begin{figure*}[t]
  \centering
  \includegraphics[trim={0pt 370pt 300pt 0pt}, clip, width=1.0\linewidth]{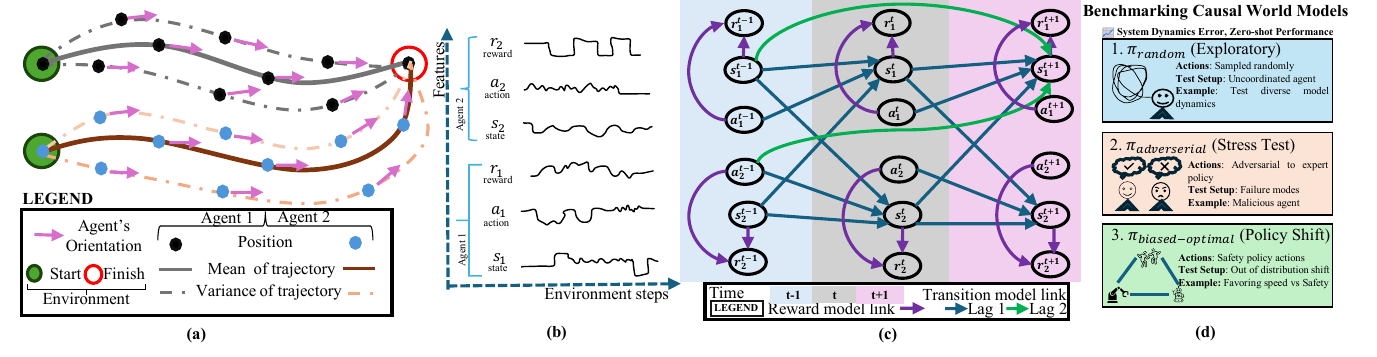}

  \caption{Overview of the ICWM Framework: (a) \textbf{Stochastic Demonstrations}: Expert multi-agent trajectories collected via soft interventions. (b) \textbf{Feature Extraction}: Temporal features $(r, a, s)$ processed for causal discovery. (c) \textbf{Causal Structure Discovery}: Learned world model showing \textcolor{violet}{reward} and transition physics across temporal lags (\textcolor{blue}{Lag 1: $t-1 \rightarrow t$}, \textcolor{green!50!black}{ Lag 2: $t-2 \rightarrow t$}). (d) \textbf{Benchmarking}: Stress testing under exploratory ($\pi_{random}$), adversarial ($\pi_{adversarial}$), and distributional ($\pi_{biased-optimal}$) policy shifts.}
  \label{fig:method}
\end{figure*}

We propose Implicit Causal World Models (ICWM), which leverage neural architectures to approximate structural causal models \citep{schölkopf2021causalrepresentationlearning,pawlowski2020deepstructuralcausalmodels} without requiring the hand-engineered directed acyclic graphs often rendered intractable by the combinatorial complexity of multi-agent systems. Discovering true  enviornment dynamics from policy rollouts requires overcoming temporal dependencies that violate the independent and identically distributed assumptions central to standard causal discovery \citep{Peters2017ElementsOC}. We address this by framing sequential trajectories as a vector autoregressive process \citep{Runge2019InferringCF} and utilizing the PCMCI+ algorithm \citep{runge2022discoveringcontemporaneouslaggedcausal}  to account for both time-lagged and contemporaneous dependencies.

Adhering to the Pearl Causal Hierarchy \citep{10.1145/3241036}, achieving the identifiability of a true causal world model requires Layer 2 (Intervention) understanding. However, active hard interventions in multi-agent systems are often computationally or prohibitively costly \citep{kallus2020confoundingrobustpolicyevaluationinfinitehorizon}. We propose bypassing this constraint by collecting demonstrations with varying degrees of stochastic interventional strength applied exclusively to agent actions via the expert's policy. This method, parallels the implications of soft interventions described in the literature \citep{NEURIPS2020_7b497aa1}, allows the Causal World Model to discover the true causal parents of environment dynamics from the spurious confounders present in the demonstrations.

In the framework of Causal Imitation Learning \citep{kumor2022sequentialcausalimitationlearning}, agent actions in multi-agent systems are treated as sequential interventions upon the environment by the expert policy. While behavioral cloning \citep{NIPS1988_812b4ba2} is the standard approach for imitating expert policies, it fails to discover the underlying environmental physics when demonstrations are influenced by unobserved confounders \citep{dehaan2019causalconfusionimitationlearning}.  For the learned world model to be identifiable, the sequential backdoor condition \citep{Zhang2016MarkovDP} must be satisfied. This requires that the observed state history effectively d-separates agent actions from future states along all non-causal backdoor paths. We present a motivation study on the board game \texttt{Tic-Tac-Toe} in Appendix  \ref{sec::motivation}.

In this paper (Figure \ref{fig:method}), we address the limitations of model-based reinforcement learning in multi-agent settings by introducing a framework for discovering causal mechanisms from biased offline data. Our contributions as follows: 
\begin{enumerate}[leftmargin=*,noitemsep,topsep=0pt]
\item \textbf{Stochastic Interventional Data Collection (C1)\label{c1}}: We introduce a demonstration collection strategy that applies soft interventions through policy variance, satisfying the Pearl Causal Hierarchy requirements for identifiability without costly hard interventions.

\item \textbf{Multi-Agent Sequential Backdoor Condition (C2)\label{c2}}: We formalize and prove the Multi-Agent Sequential Backdoor Condition, establishing that interventional dynamics are recoverable from observational multi-agent data despite unobserved strategic confounding, and use it to decouple environmental physics from agent intent.

\item \textbf{Interventional Strength Recovers Causal Structure (C3)\label{c3}}: Across three coordination domains (Two-Door \citep{gym_multigrid}, Navigation \citep{bettini2022vmasvectorizedmultiagentsimulator}, and Giveway \citep{bettini2022vmasvectorizedmultiagentsimulator}), under both full and partial observability and scaling to $2/3/4$ agents, we show that discovery error falls monotonically as the interventional strength $\sigma$ grows.

\item \textbf{Generalization Error Bound (C4)\label{c4}}: We establish, and confirm empirically, that the out-of-distribution policy evaluation error of the learned world model scales inversely with the interventional strength, falling by up to two orders of magnitude under adversarial policies.

\item \textbf{Intervention, Not Graph Quality, Is the Lever (C5)\label{c5}}: Pooling $1{,}240$ trained models, we show that the association between recovered-graph quality and predictive accuracy is a common-cause artefact of $\sigma$: conditioning on $\sigma$ removes it, and a graph-free model matches the masked model at matched data. The causal mask earns its keep specifically under partial observation and layout extrapolation, and as a structural prior whose corruption is worse than no prior at all.

\item \textbf{Calibrated Out-of-Distribution Uncertainty (C6)\label{c6}}: We show that the learned models localize their predictive uncertainty to a minority of state dimensions under distribution shift, rather than degrading uniformly.
\end{enumerate}
Impact of this work (Figure \ref{fig:impact}) seen as advances Model-Based Reinforcement Learning toward out-of-distribution robustness and structural interpretability, ensuring autonomous systems remain reliable even when surrounding agent behaviors deviate from training distributions. This framework facilitates verifiable safety, fault attribution, agent alignment, and making precise credit assignment to agent actions feasible.
\vspace{-1.3 em}
\section{Related Work}\label{sec:Related_Work}
\vspace{-1. em}
In our literature review, we identify practical bottlenecks in related problems and describe the gaps.

\textbf{Limitations of Associative World Models}: Model-Based Reinforcement Learning often relies on associative world models (DreamerV3, PlaNet, MuZero \citep{Schrittwieser_2020}) that predict future states by matching statistical distributions. Operating at Layer 1 of the Pearl Causal Hierarchy, these models are prone to causal confusion \citep{dehaan2019causalconfusionimitationlearning}, failing to distinguish between functional environment transitions and spurious correlations.

In multi-agent systems world models \citep{zhu2026shareversemultiagentconsistentvideo,zhang2025revisitingmultiagentworldmodeling,zhang2025combocompositionalworldmodels}, this manifest as an objective mismatch where models conflate invariant physical laws with the strategic intents of other agents. When dynamics are "intertwined," the world model develops a policy-dependent bias; consequently, shifting from expert to exploratory or adversarial policies causes catastrophic failure because the model expects environment physics to follow the expert's specific behavior. We address this by treating agent intent as a formal confounder $U$, adopting an interventional framework \citep{schölkopf2021causalrepresentationlearning} to ensure out-of-distribution robustness.

\textbf{FOCUS} \citep{zhu2022offlinereinforcementlearningcausal} prunes transitions using Kernel-based Conditional Independence tests \citep{zhang2012kernelbasedconditionalindependencetest}. These methods assume causal sufficiency. Multi-agent systems violate this assumption. Unobserved agent intents act as persistent confounders. Observational discovery yields fully-connected graphs in dense settings. We address this by treating intent as a formal confounder $U$. We use an interventional framework \citep{schölkopf2021causalrepresentationlearning}.

\textbf{Scalability of Neural Causal Architectures}: GNNs \citep{kipf2020contrastivelearningstructuredworld,10.1109/TNN.2008.2005605}, C-VAEs \citep{yang2023causalvaestructuredcausaldisentanglement}, and Neural SCMs \citep{pawlowski2020deepstructuralcausalmodels} approximate causal structures. Their message-passing requirements scale as $O(N^2)$. This requirement makes them computationally prohibitive for dense multi-agent dynamics. Furthermore, true counterfactual reasoning remains intractable \citep{schölkopf2021causalrepresentationlearning}.

\textbf{Temporal Gaps in Causal Discovery}:
Causal discovery is divided into constraint-based \citep{Spirtes2000} methods, which use statistical independence tests, and score-based \citep{chickering2002optimal} methods, which maximize a fit criterion \citep{10.5555/647233.719736}. Both paradigms are complicated by agent demonstrations that violate i.i.d. assumptions. While PCMCI+ \citep{runge2022discoveringcontemporaneouslaggedcausal} addresses autocorrelated and contemporaneous dependencies in time series, it has not been integrated into reinforcement learning transition models. We bridge this gap by using these techniques to inform the structural constraints of our world model.

{\textbf{The Intervention-Observation Paradox}:
The reliance on costly hard interventions bottlenecks causal discovery \citep{10.5555/3023476.3023496} and causal reinforcement learning algorithms \citep{Zhang2016MarkovDP}. We bypass this by demonstrating that multi-agent policy variance provides sufficient interventional strength for Layer 2 identifiability. This parallels Causal Imitation Learning (CIL) theory, where the sequential backdoor condition \citep{Zhang2016MarkovDP} (SBC) is used to identify robust policies. We shift this focus, applying the SBC directly to the world model to ensure transition dynamics are identifiable even when agent intents act as unobserved confounders.

\textbf{Comparison to Existing MAS Models}:
Current multi-agent system frameworks \citep{ahmed2020causalworldroboticmanipulationbenchmark,Madumal2019ExplainableRL} often focus on emergent representations or communication protocols. Our framework differentiates itself by explicitly discovering environmental physics from agent strategy. Models that treat transitions as deterministic finite automata \citep{10.1145/568438.568455} (DFA) to learn the system transition \citep{li2024emergentworldrepresentationsexploring,vafa2024evaluatingworldmodelimplicit,vafa2025foundationmodelfoundusing}, and our work overlaps with these approaches in how structural constraints are integrated and learned alongside a generative model
Generative 

\textbf{Generative Video World Models}: High-fidelity video models use diffusion and transformer processes to synthesize visually coherent states \citep{ha2018recurrentworldmodelsfacilitate, bruce2024geniegenerativeinteractiveenvironments}. By prioritizing visual plausibility over causal grounding, these architectures are prone to physical hallucinations under policy shifts \citep{li2024emergentworldrepresentationsexploring, zhang2025morpheusbenchmarkingphysicalreasoning}. Over pixel space based world modeling \citep{assran2023selfsupervisedlearningimagesjointembedding,zhou2025dinowmworldmodelspretrained}, our framework instead prioritizes the discovery of invariant structural mechanisms to ensure robustness against spurious correlations inherent in training demonstrations.
\vspace{-1.3 em}
\section{background}
\vspace{-1. em}
\subsection{Structural Causal Model for MAS}
\vspace{-0.7 em}
We formalize multi-agent transition dynamics as an Structural Causal Model \citep{pearl2009causality} (SCM) $\mathcal{M} = \langle \mathbf{U}, \mathbf{V}, \mathcal{F}, P(\mathbf{U}) \rangle$, with exogenous variables $\mathbf{U}$ (environment noise $u_{env}$, latent agent intents $u^i_{intent}$) and endogenous variables $\mathbf{V}$ (state $s_t$, joint actions $\mathbf{a_t}$, reward $r_t$, successor $s_{t+1}$). The structural mappings $\mathcal{F}$ are:

    \vspace{-1.5em} 
    \begin{align}
        \textbf{Environment Physics:} \quad & s_{t+1} \leftarrow f_s(s_t, \mathbf{a}_t, u_{env}) \label{eq::physics} \\
        \textbf{Reward Model:} \quad & r_t \leftarrow f_r(s_t, \mathbf{a}_t, u_{env}) \label{eq::reward} \\
        \textbf{Agent Policies:} \quad & a^i_t \leftarrow f_{a^i}(s_t, u^i_{intent}) \label{eq::policy}
    \end{align}

In this formulation, $f_s$ represents the invariant physical laws we aim to recover. The dependence of actions $a^i_t$ on latent intents $u^i_{intent}$ introduces persistent confounding, as $u^i_{intent}$ typically correlates with $s_t$ in expert demonstrations, necessitating an interventional approach for discovery.  Equation 3 has basis in prior work \citep{lowe2020multiagentactorcriticmixedcooperativecompetitive,dasgupta2019causalreasoningmetareinforcementlearning} (See Appendix \ref{section::scm} for a detailed breakdown of these components.)
\vspace{-0.7 em}
\subsection{Longitudinal Causal Decision Graph for MAS} 
\vspace{-0.7 em}
We represent sequential multi-agent system using a Causal decision graph $\mathcal{G}$ (Figure \ref{fig:cdg}) \citep{Zhang_Bareinboim_2021}, which encodes the structural dependencies of Eqs. \ref{eq::physics}, \ref{eq::reward}, \ref{eq::policy} over time. (See Appendix \ref{section::cdg} for a detailed analysis of these topological structures.)

In multi-agent demonstrations, unobserved strategic intent acts as a confounder $U$ that induces non-causal backdoor paths. Under full observability, introducing $U$ makes it a common cause of both actions and transitions, producing causal confusion; under partial observability the same intent-driven actions are additionally filtered through an observation bottleneck (Fig. \ref{fig:cdg}a–d). Appendix \ref{section::cdg} analyzes all four topological regimes.
\vspace{-0.7 em}
\subsection{Stochastic Soft Interventions}
\vspace{-0.7 em}
Hard interventions, such as $do(\mathbf{a}_t = \hat{\mathbf{a}})$ or $do(\mathbf{s}_t = \hat{\mathbf{s}})$, isolate causal effects by entirely severing incoming edges to a target variable \citep{pearl2009causality}. We execute stochastic soft interventions. We perturb the underlying mechanisms by injecting controlled variance directly into the behavioral data (Appendix \ref{sec::bandit}).
\subsubsection{Stochastic Soft Intervention via Policy}\label{def:si}In observational data, latent intent $u^i_{intent}$ persistently confounds agent actions. We break this confounding by applying a soft intervention to the agent’s policy (\textbf{claim C\ref{c1}}). We replace the original (Eqn. \ref{eq::policy}) conditional distribution $P(a^i_t \mid s_t, u_{intent})$ with an interventional distribution $P^\sigma(a^i_t \mid s_t)$. Let $\pi^i_{exp}$ denote the deterministic expert policy yielding demonstrations $\mathcal{D}$. We enforce a stationary mixture distribution:\begin{equation} \label{eq:mixture}P^\sigma(a^i_t \mid s_t) = (1 - \sigma)\pi^i_{exp}(a^i_t \mid s_t) + \sigma \zeta(a^i_t)\end{equation}where $\sigma \in [0, 1]$ dictates the interventional strength, and $\zeta$ is a strictly independent stochastic process (e.g., a uniform distribution over $\mathcal{A}^i$). Structurally, this operates as a definitive causal soft intervention \citep{10.5555/3023476.3023496} by fundamentally altering the data-generating mechanism (see Appendix \ref{sec::ssi}).
\vspace{-0.7 em}
\subsubsection{Interventional Signal via Policy Entropy}\label{def::pe}
\vspace{-0.7 em}
We quantify this interventional capacity via the Shannon Entropy of $P^\sigma$ at state $s$:\begin{equation} \label{eq:pe}H(P^\sigma \mid s) = -\int_{\mathcal{A}^i} P^\sigma(a \mid s) \log P^\sigma(a \mid s) da\end{equation}The identifiability of the transition function $f_s$ demands aggregate entropy in $\mathcal{D}$. Deterministic demonstrations ($\sigma=0$, $H=0$) permanently conflate invariant physics with hidden intent. High entropy forces natural experiments. By varying actions in identical states, we break backdoor paths and isolate the environmental physics.
\vspace{-0.7 em}
\subsection{Causal Identifiability and the Sequential Backdoor Condition}
\vspace{-0.7 em}
\label{sec::sbc_defined}
A causal query $P(y \mid do(x))$ is identifiable if it is computable from the observed joint distribution given the causal graph \citep{pearl2009causality}. We apply this to sequential decision-making to extract invariant dynamics from biased demonstrations.

Under the manipulated causal decision graph, when the Sequential Backdoor Criterion (SBC) \citep{Zhang2016MarkovDP} is satisfied, we estimate the causal effect of the action history $\mathbf{a}_{1:t}$ on the subsequent state $s_{t+1}$ as $P(s_{t+1} \mid do(\mathbf{a}_{1:t}))$. is reducible to an observational expression. Therefore, the causal identifiability of the world model is asserted directly from the demonstrations.While this theorem guarantees identifiability for an isolated agent, the introduction of contemporaneous agents structurally violates the standard SBC. We detail the foundational criteria, the sequential graph surgery, and the strict necessity for a multi-agent extension in Appendix \ref{sec::sbc}.
\vspace{-0.7 em}
\subsection{Recovering a Causal World model}
\vspace{-0.7 em}
Following \citet{vafa2024evaluatingworldmodelimplicit}, we formalize environmental physics as a DFA $\mathcal{W} = (Q, \Sigma, \delta, q_0, F)$, where $Q$ is the state space, $\Sigma$ is the action alphabet, and $\delta: Q \times \Sigma \to Q$ is the transition function. We define $q_{\text{reject}}$ as a sink state for impossible transitions, with $F = Q \setminus \{q_{\text{reject}}\}$ as the set of valid states. A generative model (proxy for world model) $m(\cdot): \Sigma^* \to \Delta(\Sigma)$ recovers $\mathcal{W}$ if it satisfies exact next-token prediction for any sequence $s \in S(q)$ leading to state $q$:
\begin{equation} \label{eq::dfa}
   \forall q \in F, \forall s \in S(q), \forall a \in \Sigma : m(a \mid s) > 0 \iff \delta(q, a) \neq q_{\text{reject}} 
\end{equation}
While DFA recovery ensures a model honors environmental rules via their support, it does not guarantee recovery of the invariant causal mechanism $f_s$ in MAS, where  $u_{intent}$ acts as a confounder. 

To bridge this, we assert that robust recovery requires satisfying the SBC. By conditioning on a history $H_k$ that d-separates the current action $\mathbf{a}_k$ from the intent $u_{intent}$ within the manipulated graph $\mathcal{G}_{\underline{\mathbf{a}_{1:k-1}} \overline{\mathbf{a}_{k+1:t}}}$, the model identifies the transition physics $f_s$ independently of the expert policy or demonstrations.
\vspace{-1.3 em}
\section{Framework and Methodology}
\vspace{-1. em}
\subsection{The Multi-Agent Sequential Backdoor Condition}
\vspace{-0.7 em}
To establish the structural identifiability of our Implicit Causal World Model (\textbf{claim C\ref{c2}}), we must demonstrate that the interventional distribution of the environmental transition can be uniquely computed from observational multi-agent data. Let $H_k = (\mathbf{s}_{0:k}, \mathbf{a}_{0:k-1})$ represent the observed history, $U$ the unobserved latent intent, $A^i_t$ the action of the target agent, and $A^{-i}_t$ the joint actions of all other agents.

We formalize the Multi-Agent Sequential Backdoor Condition (MAS-SBC) under the premise that agents operate via decentralized policies and transitions are Markovian with respect to the state-action space. 

\begin{theorem}[MAS-SBC]\label{theorem:mas_sbc}
Given the multi-agent causal decision graph $\mathcal{G}$, the interventional transition dynamics $P(s_{t+1} \mid do(a^i_t), H_k)$ can be resolved strictly through observational distributions via the sequential backdoor adjustment:
    \begin{equation}   
    P(s_{t+1} \mid do(a^i_t), H_k) = \sum_{a^{-i}_t} P(s_{t+1} \mid a^i_t, a^{-i}_t, H_k) P(a^{-i}_t \mid H_k)
    \end{equation}
\end{theorem}
The complete formal proof ({\textbf{claim C5}) using Pearl’s do-calculus on the mutilated graph $\mathcal{G}_{\overline{A_{1:k-1}^i}\underline{A_{k+1:t}^i}}$ is provided in Appendix \ref{sec::mas_sbc_proof}.
\vspace{-0.7 em}
\subsection{Identifiability via the MAS-SBC}
\vspace{-0.7 em}
Decoupling $f_s$ from $U$ requires two things. \textbf{Structural identifiability}: Theorem \ref{theorem:mas_sbc} extends sequential adjustment by marginalizing over contemporaneous peer actions, blocking the multi-agent backdoor path $A^i_k \leftarrow U \rightarrow A^{-i}_k \rightarrow s_{k+1}$. \textbf{Causal support}: the adjustment cannot be evaluated on deterministic expert data, where $a^i_t$ is a deterministic function of $U$ and the two are collinear. Stochastic soft interventions ($\sigma > 0$) break this collinearity and guarantee the positivity needed to compute interventional queries. In practice we apply these interventions independently and simultaneously to all agents, with $\zeta(a^i_t) \perp \zeta(a^j_t)$.
\vspace{-0.7 em}
\subsection{Generalization Error Bound}
\vspace{-0.7 em}
To formalize the necessity of stochastic soft interventions, we establish a theoretical bound on the out-of-distribution policy evaluation error. We mathematically prove that the generalization error of the learned causal world model scales inversely with the square root of the interventional strength ($\sigma$) applied during data collection. The formal theorem (\textbf{claim C4}), the step-by-step derivation using the Simulation Lemma, and the discussion of its theoretical implications are provided in Appendix \ref{sec:generalization_bound_appendix}.
\vspace{-0.7 em}
\subsection{Causal Structure Discovery}
\vspace{-0.7 em}
\label{sec::csd}
\textbf{Assumptions.} Our MAS formulation violates causal sufficiency by construction, since $u_{intent}$ is unobserved. We therefore assume: \textbf{(i) Strict temporal lag} \label{assum:temporal_lag}: the physical mechanism $f_s$ acts with lag $\tau \ge 1$, so no true causal interaction is instantaneous (Figure \ref{fig:cdg}); and \textbf{(ii) Contemporaneous confounding isolation} \label{assum:contemp_confounding}: $u_{intent}$ coordinates actions strictly at time $t$ and has no direct lagged effect on $s_{t+1}$ outside $\mathbf{a}_t$.

\textbf{Structure Discovery.} Under above assumptions, we treat the MAS trajectories in $\mathcal{D}_\sigma$ as a Vector Autoregressive (VAR) process. Since the sampled histories $\mathcal{H}_{t}^{\tau}$ are almost IID across the dataset, we define the conditioning set as the state-action-reward history $(\mathcal{H}_{t}^{\tau})$ up to a maximum lag $\tau$:
\begin{equation}
    \mathcal{H}_{t}^{\tau} = \{ (s_{t-k}, \mathbf{a}_{t-k}, r_{t-k}) \mid 1 \le k \le \tau \}
\end{equation}
We utilize the PCMCI+ algorithm to evaluate the Conditional Mutual Information (CMI) \citep{WYNER197860} and identify the directed edge set $E$ of the skeleton $\mathcal{G}_{full}$. The discovery yields two sub-graphs: Lagged discovery ($E_{lag}$: edges $V_{t-k} \to V_{t}$ spanning time steps) and contemporaneous discovery ($E_{contemp}$: edges $V_{i,t} \to V_{j,t}$ within the same step). Under Assumption \textbf{(i)}, invariant physical transitions lie in $E_{lag}$. Assumption \textbf{(ii)} restricts spurious intent correlations to $E_{contemp}$. We recover the unconfounded causal skeleton $\mathcal{G}_{skeleton} = (V, E_{skeleton})$ by discarding contemporaneous dependencies:$$E_{skeleton} = \{ (V_{t-k} \to V_{t}) \in E_{lag} \mid 1 \le k \le \tau \}$$
Conditional independence tests in PCMCI+ assign a $p$-value to each candidate edge. This yields an adjacency tensor $\mathcal{C} \in \mathbb{R}^{|V| \times |V| \times \tau}$ encoding these significance values (see Appendix \ref{sec:appendix_causal_discovery}):
$$\mathcal{D}_{\sigma} \xrightarrow{\text{PCMCI}^{+}} \mathcal{G}_{skeleton} \implies \mathcal{C} \in \mathbb{R}^{|V| \times |V| \times \tau}$$

\vspace{-0.7 em}
\subsection{Causal Masked World Model Architecture}
\vspace{-0.7 em}
We instantiate the Implicit Causal World Model (ICWM) by translating $\mathcal{G}_{skeleton}$ edges into a tensor mask $\mathbf{C}$ that structurally enforces the learned SCM. 

\textbf{Neural Masking.} Let transition function $f_s$ be parameterized by neural network $\Phi$ with parameters $\theta$. Given input tensor $\mathbf{x}_t = \mathcal{H}^\tau_t$ (history up to lag $\tau$), we satisfy the \texttt{SBC/MAS-SBC }by applying mask $\mathbf{C}$ via the Hadamard product:
$\hat{s}_{t+1} = \Phi( \mathbf{x}_t \odot \mathbf{C}; \theta )$. 
This restricts the receptive field so each prediction $\hat{s}_{i, t+1}$ depends strictly on its unconfounded parents $Pa(s_{i, t+1}) \in \mathcal{G}_{skeleton}$, isolating invariant physics.
\vspace{-1.3 em}
\section{Experimentation}
\vspace{-1. em}
\label{sec::experimentation}
\begin{figure*}[t]
  \centering
  \includegraphics[width=\linewidth]{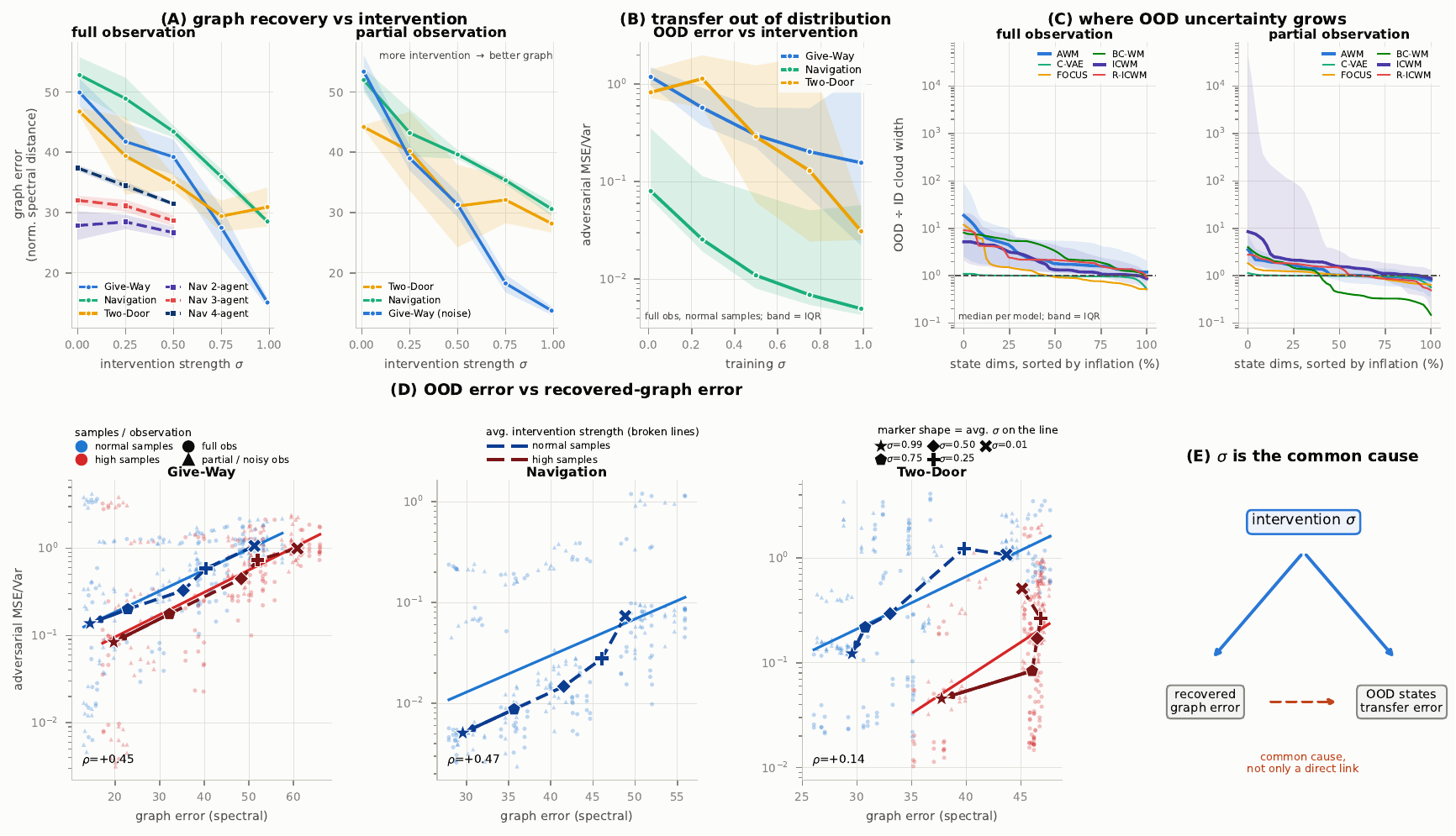}
    \caption{\textbf{Interventional strength drives both causal-structure recovery and out-of-distribution transfer.}
    \textbf{(A)} Graph error vs $\sigma$ (dashed: navigation $2/3/4$ agents).
    \textbf{(B)} Adversarial error (MSE/Var) vs $\sigma$.
    \textbf{(C)} Per-axis OOD/ID density-width inflation, dimensions sorted.
    \textbf{(D)} Adversarial error vs graph error; broken lines trace the $\sigma\!=\!0.99{\to}0.01$ trajectory.
    \textbf{(E)} $\sigma$ lowers both errors together (common cause, not a direct link).
    Details in Appendix \ref{app:structure-correlation}, \ref{app:mask-ablation}, Figure \ref{fig:app-nd-metrics}.}
    \label{fig:results}
\end{figure*}

\begin{figure*}[t]
  \centering
  \includegraphics[width=\linewidth]{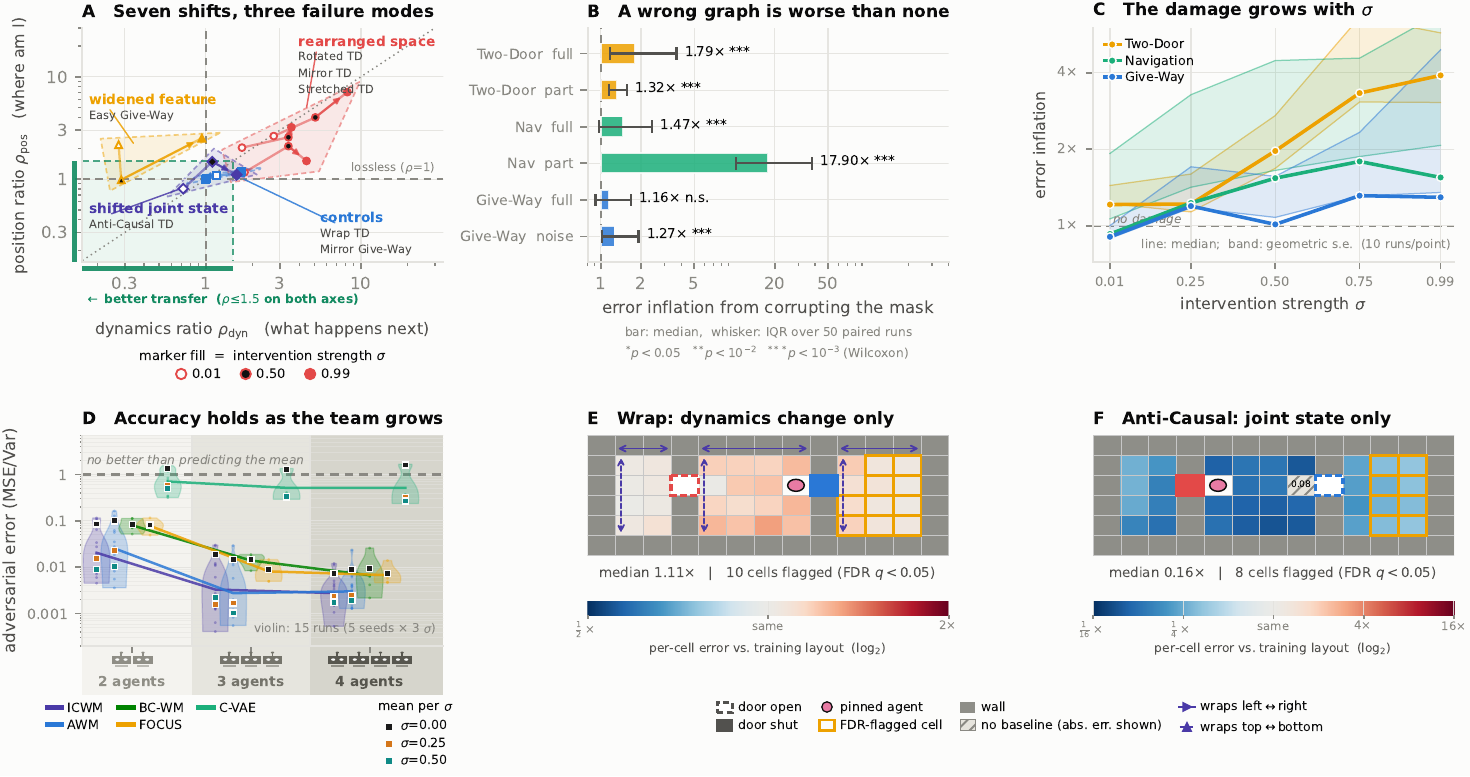}
    \caption{\textbf{What actually breaks a world model under shift, and what the mask is worth.}
    \textbf{(A)} All seven shifted layouts on the dynamics/position plane, scored against
    each frozen model's own error on the layout it trained on, so smaller is better.
    $\rho\!=\!1$ is lossless and the green region is $\rho\!\leq\!1.5$ on both axes.
    Marker shape and colour give the kind of change, the tinted band groups layouts that
    probe the same question, and marker fill goes hollow, black, then family-colour as
    $\sigma$ rises from $0.01$ to $0.99$.
    \textbf{(B)} Corrupting the causal mask inflates error in all six environment groups
    (bar: median, whisker: IQR over $50$ paired runs).
    \textbf{(C)} That damage grows with $\sigma$ (line: median, band: geometric s.e.).
    \textbf{(D)} Adversarial error on $2/3/4$-agent navigation (log axis). Coloured
    squares mark the mean per intervention level; the line follows each model's geometric
    mean across team sizes. ICWM and AWM stay near $10^{-3}$, well below C-VAE.
    \textbf{(E,F)} Per-cell error of the best ICWM against the training layout, for the
    two shifts that hold geometry fixed, drawn in the environment's own frame. Each map
    has its own log$_2$ scale. Wrap stays near parity ($1.11\times$); Anti-Causal is
    easier than the source ($0.16\times$) because shutting the red door confines the
    agent. Appendix \ref{app:shifted-generalization} tables every condition,
    Appendix \ref{app:mask-ablation} the ablation, Appendix \ref{app:shifted-envs} the layouts.}
    \label{fig:shift-taxonomy}
\end{figure*}


\textbf{Setup.} We evaluate ICWM against five baselines (Appendix \ref{sec::baselines_appendix}): the associative world model AWM, the causal-pruning method FOCUS, a behavior-cloning world model BC-WM, a recurrent variant R-ICWM, and a generative C-VAE. AWM is the only \emph{graph-free} architecture; every other model, ICWM included, receives the same soft PCMCI mask, a differentiable gate that weighs each input by the significance of its discovered causal links (Appendix \ref{sec::soft_masking}). Throughout, $\sigma$ is the \emph{interventional strength} of a dataset: $\sigma\!=\!0$ is purely observational expert data, $\sigma\!\to\!1$ is near-random exploration. We use two errors (Appendix \ref{sec::metricWM}): \emph{prediction error}, the next-state MSE, divided by the next-state variance (MSE/Var) whenever we compare across environments; and \emph{graph error}, the normalized spectral distance between the discovered and ground-truth causal graphs.

\textbf{Domains.} We use three multi-agent coordination domains (Figure \ref{fig:env}) with high inter-agent confounding:
\begin{itemize}[leftmargin=*,noitemsep,topsep=0pt]
    \item \textbf{Two-Door}: a room with a red and a blue door; agents must open the red door then the blue door, in order, and are rewarded for fast joint completion.
    \item \textbf{Navigation}: multi-agent pathfinding where collision avoidance and reaching shared goals couple the agents strategically.
    \item \textbf{Giveway}: two agents on opposite ends of a corridor too narrow to pass; they must negotiate a recessed notch to reach opposing goals without colliding.
\end{itemize}
Models train on datasets $\mathcal{D}^N_\sigma$ of $N$ episodes at interventional strength $\sigma$ (Appendix \ref{sec:appendix_cardinality}, \ref{sec::expert_training}), and are evaluated on held-out biased-optimal ($\sigma\!=\!0.5$), random ($\sigma\!=\!0.9$), and adversarial policies.


We organize the results around five questions. Figure \ref{fig:results} answers the first three: \emph{(i)} do soft interventions recover the true causal structure (\textbf{claim C\ref{c3}}), \emph{(ii)} does the learnt model transfer out of distribution (\textbf{claim C\ref{c6}}), and \emph{(iii)} what is responsible for that transfer (\textbf{claim C\ref{c5}})? Figure \ref{fig:shift-taxonomy} and Figure \ref{fig:ricwm-edge} answers the remaining two: \emph{(iv)} where does the lever stop working and what exactly breaks?, and \emph{(v)} does memory help, and where?

\textbf{Q$(i)$.} Figure \ref{fig:results}A scores PC and PCMCI discovery against the ground-truth graph as $\sigma$ grows. Graph error falls steadily in \emph{every} environment, under both full and partial observation, and the trend survives scaling: with $2$, $3$, and $4$ navigation agents the recovered graphs improve in the same way. In simple terms, adding randomness to the demonstrations is what turns observational data into an identifiable causal graph; without it ($\sigma\!\approx\!0$) discovery is at its worst everywhere. Detailed per-metric results are in Appendix \ref{app:structure-correlation}.

\textbf{Q$(ii)$.} We evaluate every trained model on an \emph{adversarial} policy far outside the demonstration distribution (Figure \ref{fig:results}B). On all three environments the variance-normalized error falls monotonically as training $\sigma$ grows, by up to two orders of magnitude on navigation, matching our bound that OOD error scales inversely with intervention. The causal world models (ICWM, AWM) beat the generative C-VAE by a wide margin throughout.

The models also \emph{reason} sensibly under this shift. Fitting a predictive density around each frozen model and comparing its per-axis widths OOD against in-distribution (Figure \ref{fig:results}C), uncertainty inflates on only a minority of state dimensions while most stay at their in-distribution width: the models localize what they no longer know instead of becoming uniformly uncertain. The full cloud-width study is in Figure \ref{fig:app-nd-metrics}.

\textbf{Q$(iii)$.} Across every environment and both sample regimes, models handed better graphs also predict better out of distribution (Figure \ref{fig:results}D; pooled within-cell Spearman $\rho\!=\!+0.62$). It is tempting to read this as ``better graph $\Rightarrow$ better model.'' The data say otherwise (Figure \ref{fig:results}E): $\sigma$ drives graph error and OOD error down \emph{together}, and once we condition on $\sigma$ the correlation vanishes ($\rho\!\approx\!0$), with the graph-free AWM matching ICWM at matched data. The average-intervention trajectories in Figure \ref{fig:results}D make this concrete: moving from $\sigma\!=\!0.01$ to $\sigma\!=\!0.99$ slides each model down and to the left, improving graph error and OOD error jointly. The interventional richness of the demonstrations is the lever; the recovered graph is its by-product. The graph still earns its keep where information is missing, under partial observation and layout extrapolation (Appendix \ref{app:structure-correlation}), and as a prior that must not be wrong: corrupting it inflates error in every environment group (Appendix \ref{app:mask-ablation}).

\textbf{Q$(iv)$.} Where does the lever stop? Here we move the \emph{environment} instead of the policy. We run the frozen models, without retraining, on seven altered layouts (Appendix \ref{app:shifted-envs}) and score each against its own error on the training layout, where $\rho\!=\!1$ means lossless transfer (Figure \ref{fig:shift-taxonomy}A). The headline now inverts. On the three layouts that rearrange space, $\rho_{\mathrm{dyn}}$ \emph{rises} with intervention, from a median of $1.9\times$ at $\sigma\!=\!0.01$ to $5.3\times$ at $\sigma\!=\!0.99$ ($n\!=\!96$ paired conditions, $93\%$ worse, Wilcoxon $p\!<\!10^{-14}$), and this holds for every architecture, ICWM and graph-free AWM alike. The reason is simple: exploratory data widens coverage only \emph{within} the training geometry, so it buys nothing once the geometry itself moves. Three shifts isolate what breaks, each holding the geometry fixed and changing one factor. \texttt{Wrap} changes only the transition rule and stays at $\approx\!1.0\times$ for every $\sigma$ (Mann--Whitney $p\!<\!10^{-6}$ against the space family). This is the decisive control: it shows the large ratios come from unfamiliar \emph{coordinates}, not from a fragile dynamics model. \texttt{Easy Give-Way} lifts position error while its dynamics error \emph{falls} ($\rho_{\mathrm{pos}}/\rho_{\mathrm{dyn}}\!=\!4.2$), and \texttt{Anti-Causal} leaves both near $1.0$ ($p\!<\!10^{-24}$ between the two); it is in fact \emph{easier} than the source, because shutting the red door confines the agent (Figure \ref{fig:shift-taxonomy}E,F). 
\begin{wrapfigure}{l}{0.50\textwidth}
  \centering
  \includegraphics[width=0.48\textwidth]{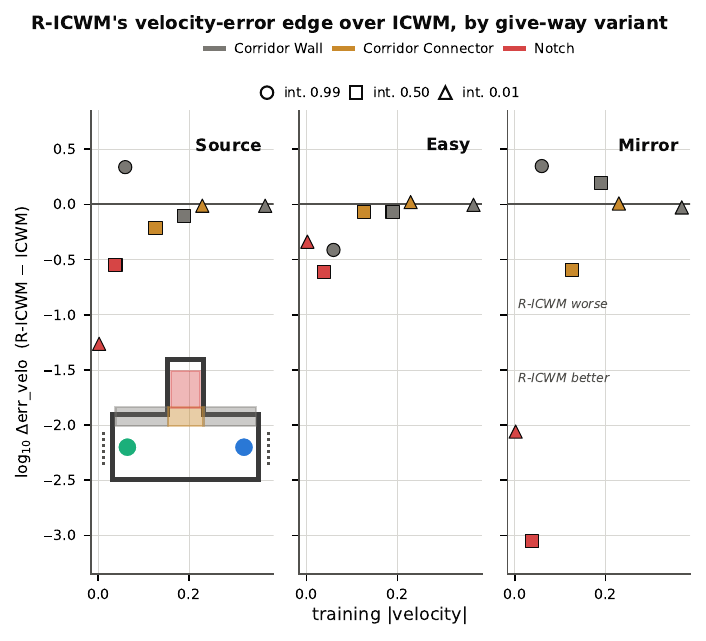}
  \caption{Velocity error, R-ICWM minus ICWM (log$_{10}$). Below $0$ means R-ICWM is
  better. Its edge grows as the trained speed of a region drops, in all three layouts (Source, Easy, Mirror).}
  \label{fig:ricwm-edge}
\end{wrapfigure}
Being off-manifold is not automatically harder. What matters is whether the changed mechanism is one the model has to predict. The mask still earns its keep here: corrupting it inflates error in all six groups (median $1.16$ to $17.9\times$, worst as $\sigma\!\to\!1$; Figure \ref{fig:shift-taxonomy}B,C). Scale is not the obstacle either. On $2/3/4$-agent navigation, ICWM and AWM hold a $\sim\!10^{-3}$ adversarial error as agents are added (Figure \ref{fig:shift-taxonomy}D), orders below C-VAE. The honest boundary is this: soft intervention identifies dynamics \emph{within} a geometry, but not spatial extrapolation.

\textbf{Q$(v)$.} R-ICWM is ICWM plus a recurrent state, so it
should pay off where a single frame says least. We split the give-way corridor into three
regions and measure how fast agents actually moved there in training: the \texttt{Wall} is
fastest, the \texttt{Connector} at the bay's mouth slower, and the \texttt{Notch} inside the
bay slowest, since agents only pause there to let the other pass. Comparing the two models
on the velocity part of the error alone (Figure \ref{fig:ricwm-edge}), the gap tracks speed,
not the layout: in the fast \texttt{Wall} the models are even, and R-ICWM's edge grows
through the \texttt{Connector} to the slow \texttt{Notch}, where it is one to three orders
of magnitude better. The pattern survives both give-way shifts, and is largest under
\texttt{Mirror}. Memory buys little when motion is fast and smooth, and a lot where it is
slow and stop-start. This is visible only in the velocity channel; the mixed dynamics error
averages it away (Appendix \ref{app:velocity-companion}).
\vspace{-1. em}
\section{CONCLUSION AND LIMITATIONS}
\vspace{-1.0 em}
We introduce ICWM, which learns world models from multi-agent demonstrations collected under soft interventions. Our central finding is that \emph{interventional strength} $\sigma$, small purposeful randomness in agent actions, is the lever on both causal-structure recovery and out-of-distribution accuracy: as it grows, discovered graphs sharpen and adversarial error falls by orders of magnitude, persisting as agents are added. The graph pays off where information is missing: on the stretched two-door layout (partial observation), ICWM reaches $0.53\times$ the graph-free AWM error, BC-WM $0.76\times$, and FOCUS $0.80\times$, while a corrupted graph is worse than none (Appendices \ref{app:structure-correlation}, \ref{app:mask-ablation}). The lever has a precise boundary: rearranging the layout makes transfer error \emph{grow} with $\sigma$ ($1.9\times$ to $5.3\times$), while changing only the transition rule transfers at $\approx\!1.0\times$, so $\sigma$ identifies dynamics within a geometry, not spatial extrapolation (Figure \ref{fig:shift-taxonomy}, Appendix \ref{app:shifted-generalization}). Two limitations follow. First, the requirement for random actions can break the very cooperation the model needs to observe in tightly coordinated settings (most visibly in Giveway), motivating targeted rather than uniform interventions (Appendix \ref{sec::bandit}). Second, our guarantees do not extend to layout changes, and the model alone cannot tell a practitioner which regime a new deployment falls into.

\textbf{Future work.} Connecting these justified world models to causal imitation and causal Q-learning \citep{MndezMolina2020CausalBQ}, and to the steerability gap of associative models \citep{vafa2025whatsproduciblereachablemeasuring}, is a natural next step. Beyond this, causally grounded world models are a prerequisite for \emph{runtime verification} and \emph{self-correcting agentic systems}: a model correct about which variables drive a transition can both screen proposed actions before execution, in the spirit of shielding \citep{alshiekh2017safereinforcementlearningshielding}, and supply the error signal for gradient-free self-revision \citep{shinn2023reflexionlanguageagentsverbal}.






\section*{Reproducibility Statement}
All of our results come from simulated environments with seeded pipelines. We
describe the environments and their state and action representations in
Appendices~\ref{app:environment}, \ref{app:state_rep},
and~\ref{app:action_space}, and the expert policies used to generate
demonstrations in Appendix~\ref{sec::expert_training}. We give the causal
discovery procedure and its theoretical grounding in
Appendix~\ref{sec:appendix_causal_discovery}, and the soft-mask rule in
Appendix~\ref{sec::soft_masking}. We describe the architectures and all
baselines in Appendix~\ref{sec::baselines_appendix}, and our evaluation metrics
in Appendix~\ref{sec::metricWM}. Every reported number is averaged over multiple
seeds. For the cell-paired comparisons behind our main claims, we hold the
model, environment, observability, $\sigma$, and seed fixed and vary only the
factor we are studying. We specify the shifted layouts in
Appendix~\ref{app:shifted-envs} and the corresponding transfer results in
Appendix~\ref{app:shifted-generalization}, and the mask-corruption procedure in
Appendix~\ref{app:mask-ablation}.
\section*{Ethics Statement}
This work studies learned
world models in simulated multi-agent environments. It involves no human
subjects, no personally identifiable information, and no real-world data, since
all demonstrations are generated programmatically in open-source simulators.
One point is worth stating clearly. Our method collects data under soft
interventions, which means deliberately adding randomness to an agent's actions.
This is harmless in simulation, but in a robotic or safety-critical setting it
would mean perturbing a system whose failures carry real cost. We therefore do
not recommend applying uniform interventions outside simulation without a
separate safety layer, and we treat reducing the required intervention budget as
a safety concern and not only a question of efficiency. We declare no conflicts
of interest or sponsorship that influenced this work.

\section*{LLM Usage}
We used large language models only as assistive tools, for editing prose that we
wrote ourselves and for help with implementation and plotting code. We did not
use them to generate research ideas, design experiments, produce or analyse
results, or find citations. We reviewed all text and code written with LLM
assistance. Every reported number comes from the executed pipeline, and we
checked every citation against its primary source. The authors take full
responsibility for the contents of this paper.

\bibliographystyle{iclr2026_conference}
\bibliography{uai2026-template}

\newpage


\begin{center}
{\large\bf Learning Implicit Causal World Models from Multi-Agent Demonstrations\\(Supplementary Material)}
\end{center}
\appendix
\section*{Index of Contents (Supplementary Material)}
This appendix provides extended formalisms, mathematical proofs, and detailed experimental configurations supporting the findings in the main text.
\begin{itemize}[leftmargin=1.5em, noitemsep]

    \item \textbf{Appendix \ref{sec::impact}: Impact of the Paper}

    \item \textbf{Appendix \ref{section::scm}: Structural Causal Model Formalization}
    \begin{itemize}
        \item Exogenous and Endogenous Partitioning
        \item Structural Functions and Dependencies
        \item The Role of $P(U)$
    \end{itemize}

    \item \textbf{Appendix \ref{section::cdg}: Causal Decision Graph}
    \begin{itemize}
        \item MDP Regimes (Full Observability)
        \item POMDP Regimes (Partial Observability)
        \item The Observation Space and Information Bottleneck
    \end{itemize}

    \item \textbf{Appendix \ref{sec::ssi}: Formalizing Policy Mixtures as Causal Interventions}
    \begin{itemize}
        \item Mapping the Mixture to a Structural Mechanism Shift
        \item The Role of Entropy in Structural Discovery
    \end{itemize}

    \item \textbf{Appendix \ref{sec::sbc}: Theoretical Foundations of the Sequential Backdoor Condition}
    \begin{itemize}
        \item The Manipulated Causal Decision Graph
        \item Formal Definition of the Sequential Backdoor Criterion
        \item The Structural Violation in Multi-Agent Environments
    \end{itemize}

    \item \textbf{Appendix \ref{sec::mas_sbc_proof}: Formal Proof of the MAS Sequential Backdoor Condition}
    \begin{itemize}
        \item Structural Axioms of the Multi-Agent System
        \item Step 1: Graphical d-Separation via Induction
        \item Step 2: Algebraic Reduction via do-Calculus
    \end{itemize}

    \item \textbf{Appendix \ref{sec::motivation}: Motivation Study in Tic-Tac-Toe}

    \item \textbf{Appendix \ref{sec:generalization_bound_appendix}: Generalization Error Bound of the ICWM}
    \begin{itemize}
        \item Discussion of Proven Implications
    \end{itemize}

    \item \textbf{Appendix \ref{sec:appendix_causal_discovery}: Theoretical Guarantees and Algorithmic Selection for Causal Discovery}
    \begin{itemize}
        \item Completeness of the Lagged Causal Skeleton in PCMCI
        \item Algorithmic Selection: PCMCI vs.\ Latent Discovery (LPCMCI)
    \end{itemize}

    \item \textbf{Appendix \ref{sec::bandit}: Contextual Bandit Formulation of Decision-Making}
    \begin{itemize}
        \item The Intractability of Expected Information Gain
        \item Surrogate Objectives
    \end{itemize}

    \item \textbf{Appendix \ref{app:environment}: MAS Environment Visualization}

    \item \textbf{Appendix \ref{app:state_rep}: State Representation of Different Environments}
    \begin{itemize}
        \item Two-Door Coordination
        \item Giveway Corridor
        \item Multi-Agent Navigation
    \end{itemize}

    \item \textbf{Appendix \ref{app:action_space}: Action Space Representation}
    \begin{itemize}
        \item Continuous Force-Based Action Space (VMAS)
        \item Discrete Categorical Action Space (MiniGrid)
        \item Comparative Summary
        \item Shifted Environments
    \end{itemize}

    \item \textbf{Appendix \ref{app:gt_adjacency}: Ground Truth Causal Adjacency Matrices}

    \item \textbf{Appendix \ref{sec::baselines_appendix}: Neural Architecture Design and Baselines}
    \begin{itemize}
        \item Summary of Model Properties
        \item Causal Structural Masking and AMLS
        \item Baseline Architectures
        \item Optimization and Loss Functions
    \end{itemize}

    \item \textbf{Appendix \ref{sec::metricWM}: Evaluation Metrics}
    \begin{itemize}
        \item Toy Study: Validating $\Delta_{\text{variance}}$
        \item Toy Evaluation: Density Metrics vs.\ the Variance Metric
        \item Causal Metric Stress Tests
    \end{itemize}

    \item \textbf{Appendix \ref{sec::expert_training}: MAPPO Training and Evaluation Performance of Expert Policy}

    \item \textbf{Appendix \ref{sec:appendix_cardinality}: Dataset Cardinality and Evaluation}
    \begin{itemize}
        \item Scaling Dynamics and Cardinality Bounds
        \item Evaluation Dataset Configuration
        \item Quality of the Collected Datasets ($\mathcal{D}^N_\sigma$)
        \item Evaluation Dataset Configuration: Two, Three, Four Agent Navigation
    \end{itemize}

    \item \textbf{Appendix \ref{sec::discovery}: Causal Discovery Setup and Results}
    \begin{itemize}
        \item Causal Discovery Experimental Parameters
        \item Evaluation Results: Three Primary Environments
        \item Evaluation Results: High Sample Complexity Experiments
        \item Evaluation Results: Two, Three, Four Agent Navigation
    \end{itemize}

    \item \textbf{Appendix \ref{app:wm_results}: World Model Setup and Results}
    \begin{itemize}
        \item Evaluation Results: Three Primary Environments (\ref{app:normal3-wm-sweep})
        \item Evaluation Results: High Sample Complexity (\ref{app:highsample-wm-sweep})
        \item Evaluation Results: Two, Three, Four Agent Navigation (\ref{app:multiagent-lidar-wm})
    \end{itemize}

    \item \textbf{Appendix \ref{app:mask-ablation}: Causal-Mask Ablation Study}

    \item \textbf{Appendix \ref{app:neural-density}: Uncertainty and Neural Density Measurement}
    \begin{itemize}
        \item Grid World: Two-Door
        \item Continuous Control: Give-Way
        \item Two-Agent Navigation
        \item Multi-Agent Navigation without LIDAR
        \item Quantitative Uncertainty Metrics (\ref{app:nd-metrics})
    \end{itemize}

    \item \textbf{Appendix \ref{app:error-density}: World-Model Spatial Error Density}
    \begin{itemize}
        \item Setup, Figures, Quantitative Summary, Discussion
    \end{itemize}

    \item \textbf{Appendix \ref{app:shifted-generalization}: Generalization under Layout Shift}
    \begin{itemize}
        \item Setup, Figures, Quantitative Summary, Discussion
        \item Where the Shift Actually Bites: Subtracting the Source Error Surface (\ref{app:notch-distortion})
        \item Splitting the Velocity Channel: Where Recurrence Pays Off (\ref{app:velocity-companion})
        \item Overall Takeaway
    \end{itemize}

    \item \textbf{Appendix \ref{app:structure-correlation}: Towards Pearl's Framework for Counterfactual Queries}
    \begin{itemize}
        \item The Pooled Correlation is Confounded
        \item Intervention Strength is the Common Cause
        \item The Null Replicates under Layout Shift
        \item A Localized, Conditional Exception
        \item Why Stretch Benefits and Rotation Does Not
    \end{itemize}

\end{itemize}
\section*{Notation Glossary}
\label{app:notation}

The following tables summarize the mathematical notation used throughout the paper, categorized by their respective methodological sections (see Tables~\ref{tab:eval_metrics}, \ref{tab:causal_discovery}, \ref{tab:soft_interventions}, \ref{tab:uncertainty_shift}, and \ref{tab:scm_decision_graph}).

\definecolor{lightblue}{RGB}{220, 230, 255}
\definecolor{lightgreen}{RGB}{220, 255, 220}
\definecolor{lightyellow}{RGB}{255, 255, 220}
\definecolor{lightorange}{RGB}{255, 235, 200}
\definecolor{lightgray}{RGB}{230, 230, 230}
\definecolor{lightpurple}{RGB}{235, 225, 255}

\begin{table*}[h]
\centering
\begin{tabular}{@{}ll@{}}
\toprule
\textbf{Symbol} & \textbf{Definition} \\ \midrule
\rowcolor{lightgray} \multicolumn{2}{l}{\textbf{Evaluation Metrics and Generalization Bounds}} \\
\rowcolor{lightorange} $D_{dim}$ & Dimensionality of state vector \\
\rowcolor{lightorange} $\mathcal{D}_\sigma$ & Demonstration dataset \\
\rowcolor{lightorange} $\text{MSE}$ & Mean Squared Error \\
\rowcolor{lightorange} $\text{RSE}$ & Relative System Error \\
\rowcolor{lightorange} $\text{SHD}$ & Structural Hamming Distance \\
\rowcolor{lightorange} $\Delta_{inv}$ & Policy Invariance Gap \\
\rowcolor{lightorange} $\Delta_{spec}$ & Spectral Distance \\
\rowcolor{lightorange} $\Delta_{\text{variance}}$ & Variance-normalized relative error, $\text{MSE}/\sigma^2_{\mathcal{D}}$ \\
\rowcolor{lightorange} $\sigma^2_{\mathcal{D}}$ & Target-variance floor \\
\rowcolor{lightorange} $\bar{s}$ & Input-agnostic mean successor state \\
\rowcolor{lightorange} $R^2$ & Coefficient of determination \\
\rowcolor{lightorange} $\mathcal{C}_{pred}, \mathcal{C}_{true}$ & Predicted and ground-truth causal tensors \\
\rowcolor{lightorange} $\lambda_i(\cdot)$ & $i$-th eigenvalue of an adjacency representation \\
\rowcolor{lightorange} $\widehat{\mathrm{MMD}}^2$ & Unbiased squared MMD (RBF kernel) \\
\rowcolor{lightorange} $k(u,v)$ & RBF kernel, $\exp(-\gamma\lVert u-v\rVert^2)$ \\
\rowcolor{lightorange} $\gamma^{-1}$ & MMD kernel bandwidth (median heuristic) \\
\rowcolor{lightorange} $\hat{p}(z)$ & Kernel density estimate at query point $z$ \\
\rowcolor{lightorange} $h$ & KDE bandwidth \\
\rowcolor{lightorange} $\mathcal{G}$ & Query grid for density evaluation \\
\rowcolor{lightorange} $\hat{p}_{\text{norm}}, \hat{p}_{\text{adv}}$ & Normal and adversarial next-state densities \\
\rowcolor{lightorange} $\text{Dist}_{TV}$ & Total Variation distance \\
\rowcolor{lightorange} $J$ & Complexity constant \\
\rowcolor{lightorange} $\epsilon_{env}$ & Irreducible stochasticity / error \\
\bottomrule
\end{tabular}
\caption{Evaluation Metrics and Generalization Bounds}
\label{tab:eval_metrics}
\end{table*}

\begin{table*}[h]
\centering
\begin{tabular}{@{}ll@{}}
\toprule
\textbf{Symbol} & \textbf{Definition} \\ \midrule
\rowcolor{lightgray} \multicolumn{2}{l}{\textbf{Causal Discovery and Neural Masking}} \\
\rowcolor{lightyellow} $\tau$ & Temporal lag limit \\
\rowcolor{lightyellow} $E_{lag}, E_{contemp}$ & Discovered edge sets \\
\rowcolor{lightyellow} $\mathcal{G}_{skeleton}$ & Unconfounded causal skeleton \\
\rowcolor{lightyellow} $\mathcal{C}$ & Adjacency Matrix with Lagged Structure \\
\rowcolor{lightyellow} $p_{ij}$ & Causal discovery $p$-value \\
\rowcolor{lightyellow} $\alpha$ & Significance threshold \\
\rowcolor{lightyellow} $\psi_{ij}$ & Log-Significance Transform value \\
\rowcolor{lightyellow} $\kappa$ & Temperature parameter \\
\rowcolor{lightyellow} $\Gamma$ & Differentiable soft mask tensor \\
\rowcolor{lightyellow} $M$ & Binary causal mask tensor \\
\rowcolor{lightyellow} $\Phi, \theta$ & Neural transition model parameters \\
\bottomrule
\end{tabular}
\caption{Causal Discovery and Neural Masking}
\label{tab:causal_discovery}
\end{table*}

\begin{table*}[h]
\centering
\begin{tabular}{@{}ll@{}}
\toprule
\textbf{Symbol} & \textbf{Definition} \\ \midrule
\rowcolor{lightgray} \multicolumn{2}{l}{\textbf{Soft Interventions and Sequential Backdoor}} \\
\rowcolor{lightgreen} $\sigma$ & Interventional strength ($\sigma \in [0,1]$) \\
\rowcolor{lightgreen} $\pi^i_{exp}$ & Deterministic expert policy for agent $i$ \\
\rowcolor{lightgreen} $\zeta$ & Independent stochastic process \\
\rowcolor{lightgreen} $P^\sigma$ & Interventional stationary mixture policy \\
\rowcolor{lightgreen} $H_k$ & Observed historical trajectory up to step $k$ \\
\rowcolor{lightgreen} $do(\cdot)$ & Pearl's intervention operator \\
\rowcolor{lightgreen} $A^{-i}_t, a^{-i}_t$ & Joint actions excluding agent $i$ \\
\rowcolor{lightgreen} $\mathcal{G}_{\overline{A} \underline{A}}$ & Mutilated causal graph \\
\bottomrule
\end{tabular}
\caption{Soft Interventions and Sequential Backdoor}
\label{tab:soft_interventions}
\end{table*}

\begin{table*}[h]
\centering
\begin{tabular}{@{}ll@{}}
\toprule
\textbf{Symbol} & \textbf{Definition} \\ \midrule
\rowcolor{lightgray} \multicolumn{2}{l}{\textbf{Uncertainty, Error Density, and Shift}} \\
\rowcolor{lightpurple} $\sigma_{\text{ID}}, \sigma_{\text{OOD}}$ & Predictive cloud width per dimension \\
\rowcolor{lightpurple} $\sigma_{\text{OOD}}/\sigma_{\text{ID}}$ & Dimensionless OOD width-inflation ratio \\
\rowcolor{lightpurple} $\overline{\mathrm{err}}_{\mathrm{dyn}}$ & Mean squared residual on dynamic dimensions \\
\rowcolor{lightpurple} $\overline{\mathrm{err}}_{\mathrm{pos}}$ & Mean squared residual on position dimensions \\
\rowcolor{lightpurple} $\rho$ & Shift-transfer ratio \\
\rowcolor{lightpurple} $\mathcal{D}$ & Set of dynamic state dimensions \\
\rowcolor{lightpurple} $\Delta_{\text{spec}}^{\text{norm}}$ & Normalized binary spectral distance \\
\rowcolor{lightpurple} $\mathcal{D}^N_\sigma$ & Training dataset of $N$ episodes \\
\rowcolor{lightpurple} $\mathcal{D}^\ddagger_{\sigma}$ & Held-out evaluation dataset \\
\bottomrule
\end{tabular}
\caption{Uncertainty, Error Density, and Shift}
\label{tab:uncertainty_shift}
\end{table*}

\begin{table*}[h]
\centering
\begin{tabular}{@{}ll@{}}
\toprule
\textbf{Symbol} & \textbf{Definition} \\ \midrule
\rowcolor{lightgray} \multicolumn{2}{l}{\textbf{SCM and Decision Graph}} \\
\rowcolor{lightblue} $\mathcal{M}$ & SCM tuple $\langle \mathbf{U}, \mathbf{V}, \mathcal{F}, P(\mathbf{U}) \rangle$ \\
\rowcolor{lightblue} $\mathbf{U}, \mathbf{V}$ & Sets of exogenous/endogenous variables \\
\rowcolor{lightblue} $u_{env}$ & Unobserved environmental noise \\
\rowcolor{lightblue} $u^i_{intent}$ & Unobserved latent intent of agent $i$ \\
\rowcolor{lightblue} $s_t, S_t$ & Environment state at time $t$ \\
\rowcolor{lightblue} $\mathbf{a}_t, \mathbf{A}_t$ & Joint action of all agents at time $t$ \\
\rowcolor{lightblue} $a^i_t, A^i_t$ & Specific action of agent $i$ at time $t$ \\
\rowcolor{lightblue} $r_t$ & Reward received at time $t$ \\
\rowcolor{lightblue} $o_t, O_t$ & Partial observation at time $t$ \\
\rowcolor{lightblue} $f_s, f_r, f_{a^i}$ & Deterministic structural functions \\
\bottomrule
\end{tabular}
\caption{SCM and Decision Graph}
\label{tab:scm_decision_graph}
\end{table*}
\newpage
\section{Impact of the paper}
\label{sec::impact}
\begin{figure}[!htb]
  \centering
  \includegraphics[trim={0pt 150pt 150pt 0pt}, clip, width=0.9\linewidth]{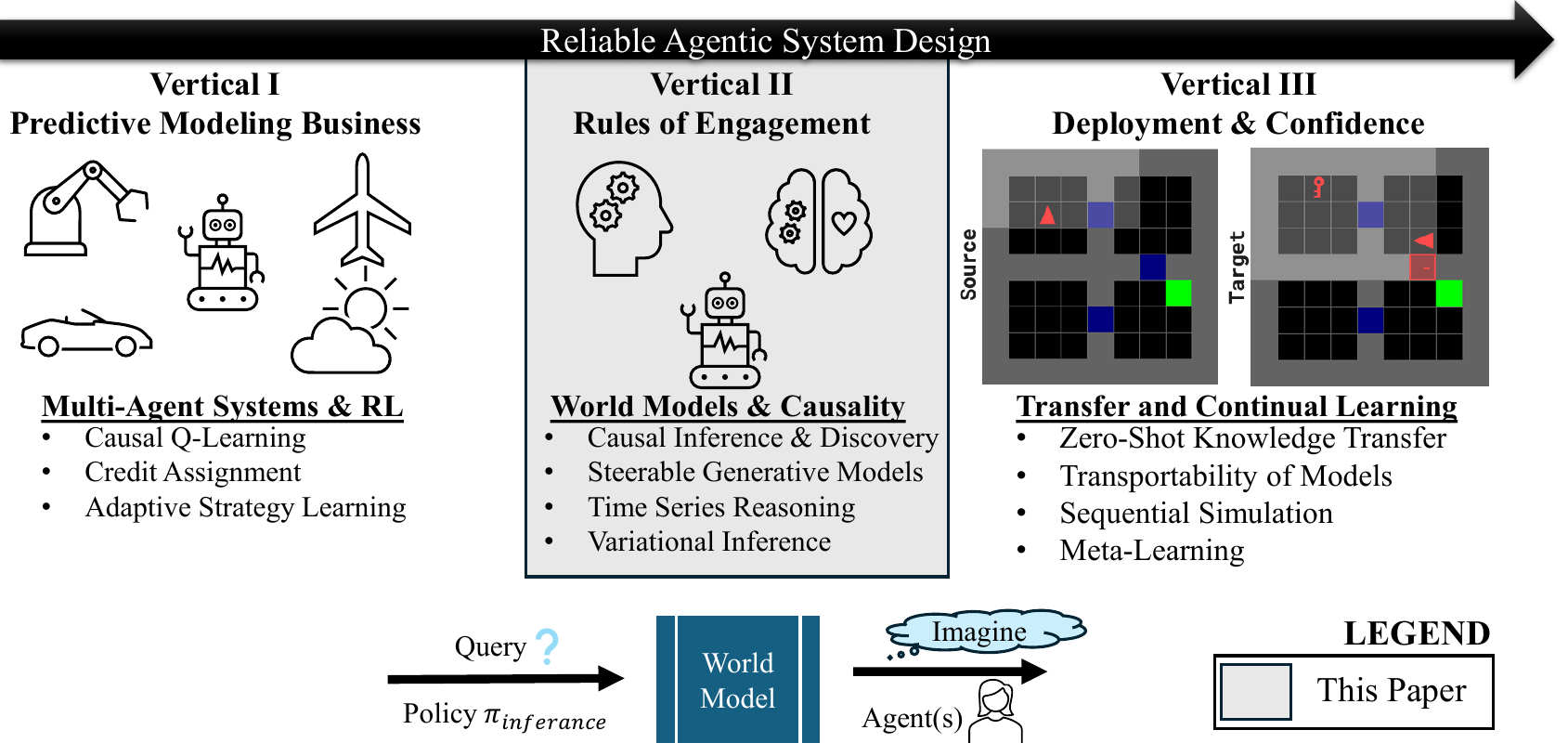}
  \caption{\textbf{The Synergistic Vertical Framework for Reliable Agentic Systems}. The architecture partitions system design into three critical pillars: behavioral coordination (Vertical I), structural understanding (Vertical II), and cross-domain scalability (Vertical III). This paper addresses Vertical II (Rules of Engagement), establishing the deconfounded causal world model as the essential physical foundation. By resolving invariant environment mechanisms from strategic agent intent, our work provides the identifiable internal simulator required for precise credit assignment in Vertical I and zero-shot knowledge transportability in Vertical III.}\label{fig:impact}
\end{figure}

\section{Structural Causal Model  Formalization}
\label{section::scm}
The transition dynamics of a Multi-Agent System are governed by the interplay between invariant environment rules and adaptive agent strategies. We formalize this via an SCM $\mathcal{M} = \langle \mathbf{U}, \mathbf{V}, \mathcal{F}, P(\mathbf{U}) \rangle$:
\subsection{Exogenous and Endogenous Partitioning}
The variables in our system are partitioned into two sets:
\begin{enumerate}
    \item \textbf{Exogenous Variables ($\mathbf{U}$)}: These are variables determined outside the model. In our framework, we distinguish between Environmental Stochasticity ($u_{env} \sim U_{env}$), which captures irreducible noise in physics (e.g., sensor noise or friction or stochastic in environment properties), and Unobserved Strategic Intents for agent $i$, ($u^i_{intent} \sim U_{intent}$), which represent the high-level goals or coordination signals that drive agent behaviors but are not explicitly captured in the state space.
    \item \textbf{Endogenous Variables ($\mathbf{V}$)}: These are the observable variables determined by the structural functions. For a single transition step $(t)$, these include the current environment state $s_t \in \mathcal{S}= \mathcal{S}^1 \times \dots \times \mathcal{S}^n$, the joint action  $a_t \in \mathcal{A} = \mathcal{A}^1 \times \dots \times \mathcal{A}^n$, the reward  $r_t \in \mathcal{R} = \mathcal{R}^1 \times \dots \times \mathcal{R}^n$, and the resulting next state $s_{t+1} \in \mathcal{S}$. Where $\mathcal{S}$ is the state space, $\mathcal{S}^i$ is the state space for agent $i$, $\mathcal{A}$ is the action space, $\mathcal{A}^i$ is the action space for agent $i$, $\mathcal{R}$ is the reward space, $\mathcal{R}^i$ is the reward space for agent $i$, and $n$ number of agents in the multi agent system.
\end{enumerate}
\subsection{Structural Functions and Dependencies}
The set $\mathcal{F}$ $ = \{f_s, f_r f_{a^1}, \dots, f_{a^n}\}$ defines the causal mechanisms of the world. By expressing these as deterministic functions of their parents and noise, we can model the system as a directed graph where edges represent functional dependencies:
\begin{enumerate}
    \item \textbf{Environment Physics ($f_s$)}: This function is the primary target of our world model. It maps the current state and joint actions to the next state. Critically, $f_s$ is policy-invariant; it represents the underlying "rules of the game" that remain constant regardless of whether the agents are acting optimally, randomly, or adversarially.
    \begin{itemize}
        \item Quantity: $P(s_{t+1} \mid s_t, \mathbf{a}_t) $ 
        \item Relationship: $s_{t+1} \leftarrow f_s(s_t, \mathbf{a}_t, u_{env})$
    \end{itemize}
    \item \textbf{Reward Model ($f_r$)}: This maps the state-action pair to a feedback signal. In MAS, rewards are often highly confounded because expert agents only receive high rewards by following specific, intent-driven trajectories.
        \begin{itemize}
        \item Quantity: $P(r_{t+1} \mid s_t, \mathbf{a}_t) $ 
        \item Relationship: $r_{t+1} \leftarrow f_r(s_t, \mathbf{a}_t, u_{env})$
    \end{itemize}
    \item \textbf{Agent Policies ($f_{a^i}$)}: Unlike $f_s$, the policy functions are agent-dependent. They are conditioned on the state $s_t$ and the latent intent $u^i_{intent}$. This structure is the root of causal confusion: in expert data, the correlation between $s_t$, $u^i_{intent}$, and $s_{t+1}$ is often so strong that standard associative models cannot determine if $s_{t+1}$ was caused by the physics of $s_t$ or the specific intent-driven choice of $a^i_t$.
        \begin{itemize}
        \item Quantity: $\pi^i(a^i | s_t)$, where $\pi^i$ is $i^{th}$ agents policy.  
        \item Relationship: $a^i_t \leftarrow f_{a^i}(s_t, u^i_{intent})$, the following relation is estabished from prior work \citep{lowe2020multiagentactorcriticmixedcooperativecompetitive,dasgupta2019causalreasoningmetareinforcementlearning}
    \end{itemize}
\end{enumerate}
\subsection{The Role of $P(U)$}
The joint probability distribution $P(\mathbf{U})$ induces the statistical distributions in the demonstration dataset. We distinguish between these exogenous variables as follows:
\begin{enumerate}
    \item Environmental Stochasticity ($u_{env}$): Represents irreducible noise (e.g., friction, jitter, weather conditions). We assume $u_{env} \perp \! \! \perp u_{intent}$. This term dictates predictive variance (kernel thickness or uncertainty in the model's output), which we do not explicitly deconfound. We assume $u_{env}$ is captured by dataset variance and offer no solution for shifts in its distribution, focusing solely on structural invariance. Our position is that under a significant distribution shift of $u_{env}$, the world dynamics themselves have likely shifted so fundamentally that knowledge reuse from the previous model would yield diminishing returns.
    \item Strategic Confounding ($u^i_{intent}$): Represents unobserved coordination signals that correlate with $s_t$ in expert data, creating backdoor paths ($s_t \leftarrow u^i_{intent} \rightarrow a^i_t \rightarrow s_{t+1}$) that cause associative models to mistake agent choice for physical law. We address this by leveraging policy variance (soft interventions) to satisfy the Sequential Backdoor Condition discovering true transition rules $f_s$. 
\end{enumerate}

\section{Causal decision graph}
\label{section::cdg}
\begin{figure}[!htb]
  \centering
  \includegraphics[trim={0pt 0pt 375pt 0pt}, clip, width=0.7\linewidth]{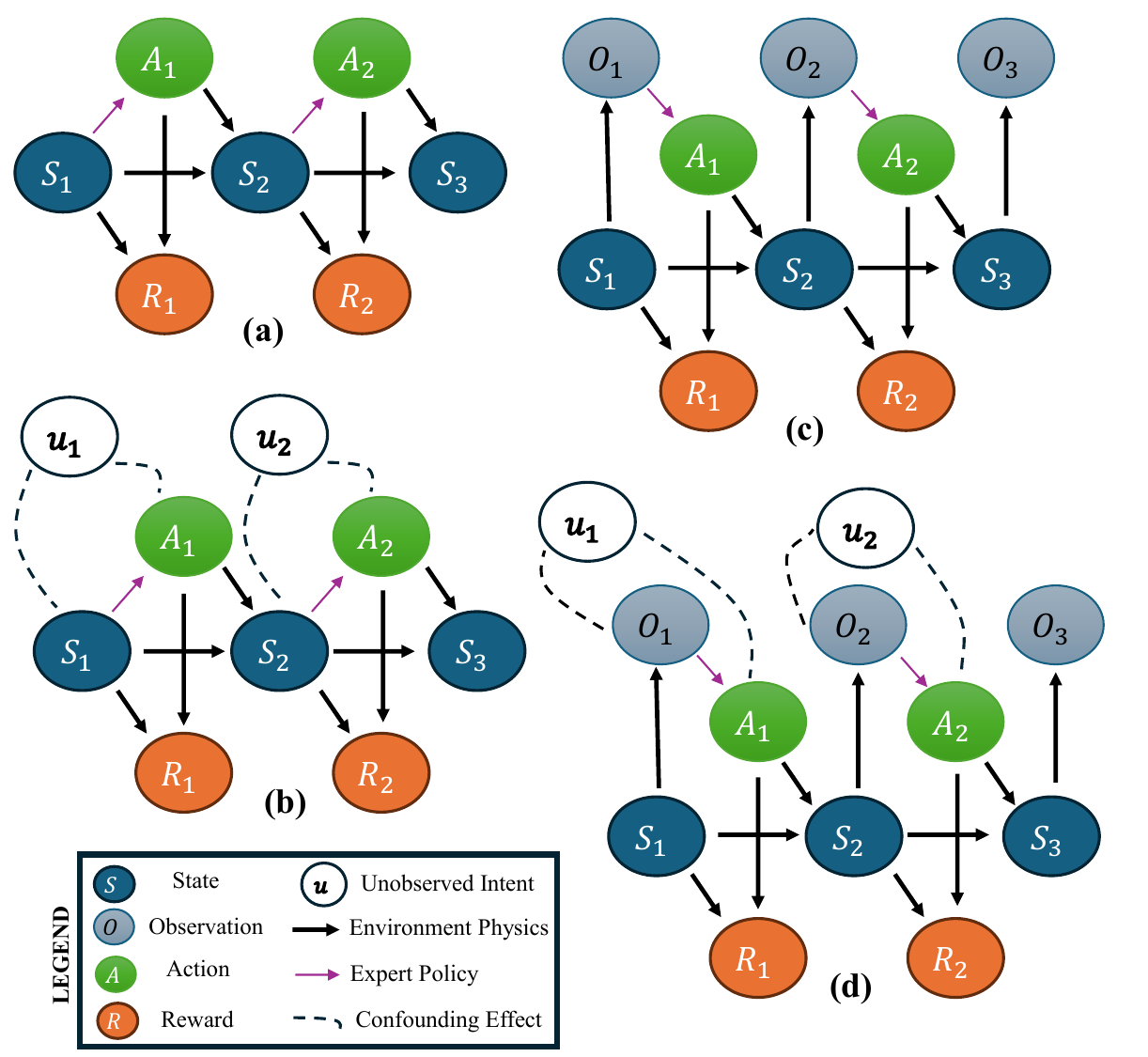}
  \caption{Causal Decision Graph for MAS (a) MDP (b) MDP with Confounding (c) Partially Observable MDP (d) Partially Observable MDP with Confounding}\label{fig:cdg}
\end{figure}

The Causal Decision Graph (CDG) allows us to visualize how structural dependencies evolve and where unobserved variables interfere with transition discernibility.

\subsection{MDP Regimes (Full Observability)}
\begin{itemize}
    \item \textbf{Standard MDP} (Fig. \ref{fig:cdg}a):  \citep{sutton2018reinforcement} : Actions are a direct functional child of the observed state ($S_t \to A_t$). The transition $S_{t+1} \leftarrow f_s(S_t, A_t, u_{env})$ is identifiable because no unobserved common causes exist between $A_t$ and $S_{t+1}$.
    \item \textbf{Confounded MDP} (Fig. \ref{fig:cdg}b): We introduce unobserved intent $U$. As shown by the dashed lines, $U$ influences $A_t$ while being correlated with the state sequence. This creates a backdoor path ($S_t \leftarrow U \to A_t \to S_{t+1}$). Associative models fail here because they cannot distinguish if $S_{t+1}$ resulted from the physical action $A_t$ or the latent intent $U$.
\end{itemize}

\subsection{POMDP Regimes (Partial Observability)}
\begin{itemize}
    \item \textbf{Standard POMDP} (Fig. \ref{fig:cdg}c): \citep{KAELBLING199899} : The state $S_t$ is latent. Agents perceive $O_t$, creating an information bottleneck ($S_t \to O_t \to A_t$). While the physics $f_s$ still depend on $S_t$, the policy is constrained by the observation.
    \item \textbf{Confounded POMDP} (Fig. \ref{fig:cdg}d): This represents the most realistic MAS scenario. The unobserved intent $U$ influences the action $A_t$, which is already limited by the partial observation $O_t$. Recovering the invariant physics $f_s$ in this regime requires ICWM to utilize the Sequential Backdoor Condition to d-separate the interventional influence of actions from the latent confounding of $U$.
\end{itemize}

\subsection{The Observation Space and Information Bottleneck}
In partially observable settings, we formally introduce the Observation Space $\mathcal{O}$. This space acts as a non-invertible, lossy mapping from the latent state:
\begin{itemize}
    \item Formal Mapping: The observation $o_t \in \mathcal{O}$ is generated by an observation function $f_o$:
\begin{equation} \label{eq:obs}
o_t \leftarrow f_o(s_t, u_{env})
\end{equation}
    \item The Information Bottleneck: This regime imposes a critical constraint: while the environment physics $f_s$ remain dependent on the full physical configuration $s_t$, the agent's policy is restricted to the information contained in $o_t$. This creates a bottleneck where the agent lacks the state-transparency needed to determine its own causal influence.
    \item Structural Challenge: Because $f_o$ is often a many-to-one mapping, multiple physical states $s_t$ can map to the same $o_t$. This aliasing exacerbates causal confusion, as the world model must not only deconfound strategic intent but also infer the underlying state $s_t$ from the observation history to correctly isolate the physical laws of the transition.
\end{itemize}

\section{Formalizing Policy Mixtures as Causal Interventions}
\label{sec::ssi}

\subsection{Mapping the Mixture to a Structural Mechanism Shift}
In a Structural Causal Model ($f$ as process equation), variables ($A$) are deterministic functions of causal parents ($PA_A$) and exogenous noise ($U_A$): $A \leftarrow f(PA_A, U_A)$. A hard intervention replaces $f$ with a constant. A soft intervention modifies the equation to a new mechanism $f'$ without fully severing parental dependence. 

Our stochastic policy mixture operates strictly as a soft intervention. Probabilistically, we define this mixture as the interventional distribution:
\begin{equation} \label{eq:mixture_appendix}
P^\sigma(a^i_t \mid s_t) = (1 - \sigma)\pi^i_{exp}(a^i_t \mid s_t) + \sigma \zeta(a^i_t)
\end{equation}
where $\pi^i_{exp}$ is the deterministic expert policy, $\sigma \in [0, 1]$ dictates the interventional strength, and $\zeta$ is an independent stochastic process defined over the agent's discrete or continuous action space $\mathcal{A}^i$.

Structurally, in purely observational data, actions follow the expert mechanism: $a^i_t \leftarrow f_{a^i}(s_t, u^i_{intent})$. The mixture replaces $f_{a^i}$ with a modified function $f'_{a^i}$ by introducing an independent exogenous noise variable $u_{rand} \sim \zeta$:
\begin{equation} \label{eq:soft_intervention}
a^i_t \leftarrow f'_{a^i}(s_t, u^i_{intent}, u_{rand}) = 
\begin{cases} 
f_{a^i}(s_t, u^i_{intent}) & \text{with probability } 1 - \sigma \\
u_{rand} & \text{with probability } \sigma 
\end{cases}
\end{equation}
Because $u_{rand} \perp \!\!\! \perp u_{intent}$, this piecewise assignment redefines the data-generating mechanism. The target variable $a^i_t$ retains its functional dependence on $s_t$, fulfilling the soft intervention criteria, while the structural influence of the confounding parent $u^i_{intent}$ degrades proportionally to $\sigma$.

\subsection{The Role of Entropy in Structural Discovery}

\begin{figure}[!htb]
\centering
\includegraphics[width=0.95\linewidth]{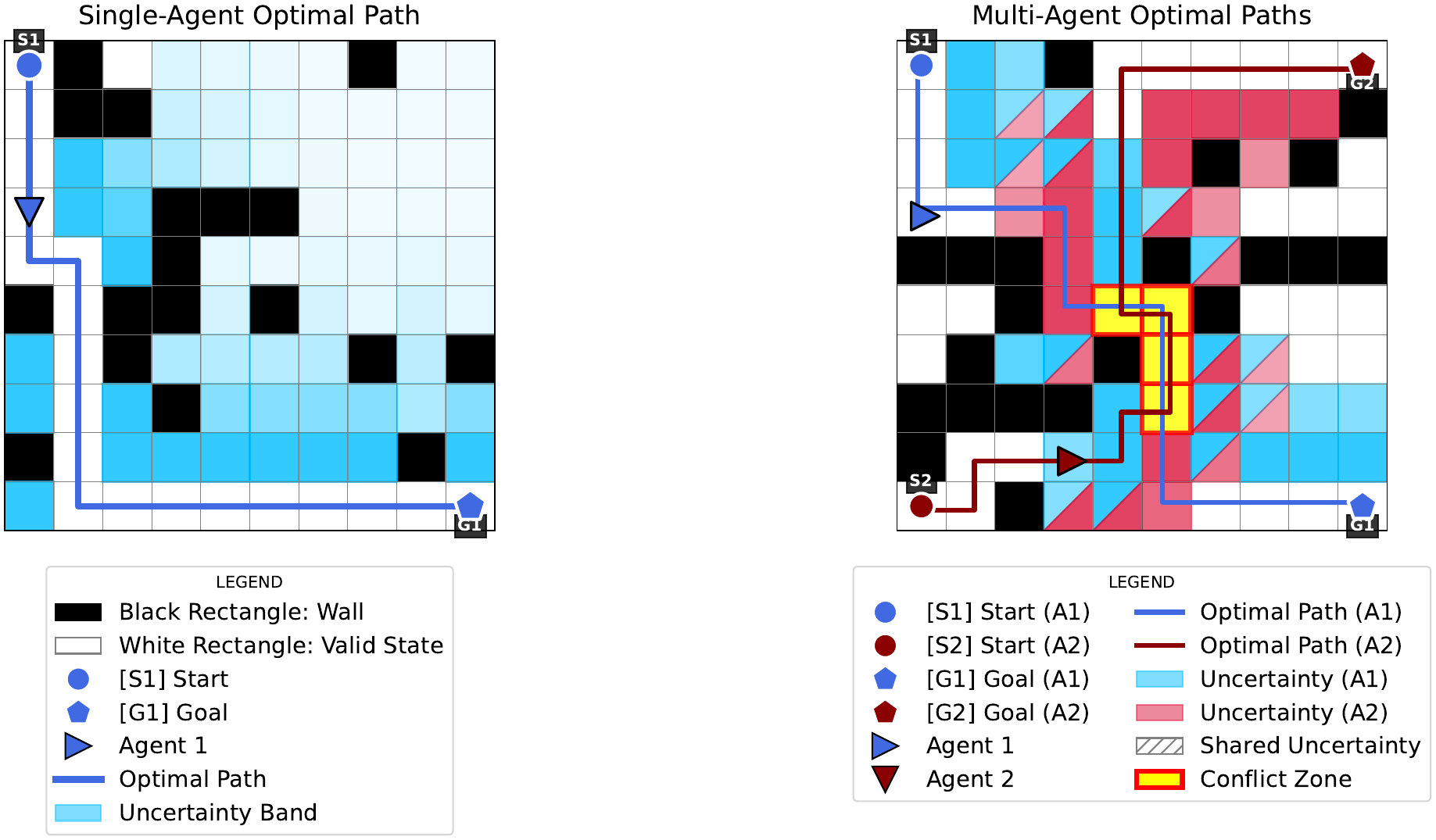}
\caption{Navigation trajectories from source to goal for a single agent (left) and two agents (right). Optimal path and  uncertainty bands are highlighted separately. Only deterministic optimal actions create spurious correlations between agent trajectories, and this confounding effect is strictly higher in multi-agent settings.}\label{fig:maze}
\end{figure}

Figure \ref{fig:maze} illustrates the formal research problem of structural discovery in maze navigation games \citep{stern2019multiagentpathfindingdefinitionsvariants,pacman1980,gotcha1973,mouseinthemaze1959}. The Shannon entropy $H$ of the policy quantifies the interventional strength introduced by the mixture parameter $\sigma$, defining the identifiability of the transition function $f_s$. In a purely deterministic expert dataset ($\sigma = 0$, $H = 0$), the absolute lack of action diversity perfectly aligns the unobserved intent $u^i_{intent}$ with the selected action $a^i_t$. This zero-entropy regime creates a persistent backdoor path ($s_t \leftarrow u^i_{intent} \rightarrow a^i_t \rightarrow s_{t+1}$) where environment dynamics remain entirely conflated with agent strategy, rendering $f_s$ structurally unidentifiable. As $\sigma$ increases toward 1, the injected independent variance raises the dataset entropy. This variance guarantees causal positivity by ensuring all legally valid actions retain a non-zero probability of selection, which inherently generates the natural experiments required to observe alternative state transitions.

The primary motivation for resolving this zero-entropy confounding is to enable the extraction of invariant environmental laws from biased observational data. While a partial policy mixture ($0 < \sigma < 1$) does not completely sever the structural backdoor path, the induced entropy provides the unconfounded state and action support strictly necessary to compute unbiased sequential backdoor adjustments. By achieving causal positivity, the model can structurally decouple the agents' strategic intents from the physical environment. Isolating the true physical gradient of the transition function $f_s$ ensures that the learned internal simulator remains accurate and robust even when surrounding agents deviate from their optimal training distributions.

\section{Theoretical Foundations of the Sequential Backdoor Condition}
\label{sec::sbc}
To establish the formal requirement for our Multi-Agent Sequential Backdoor Condition (MAS-SBC) in Theorem \ref{theorem:mas_sbc}, we first detail the structural mechanisms of the baseline single-agent SBC framework.
\begin{figure}[!htb]
  \centering
  \includegraphics[trim={0pt 0pt 375pt 0pt}, clip, width=0.8\linewidth]{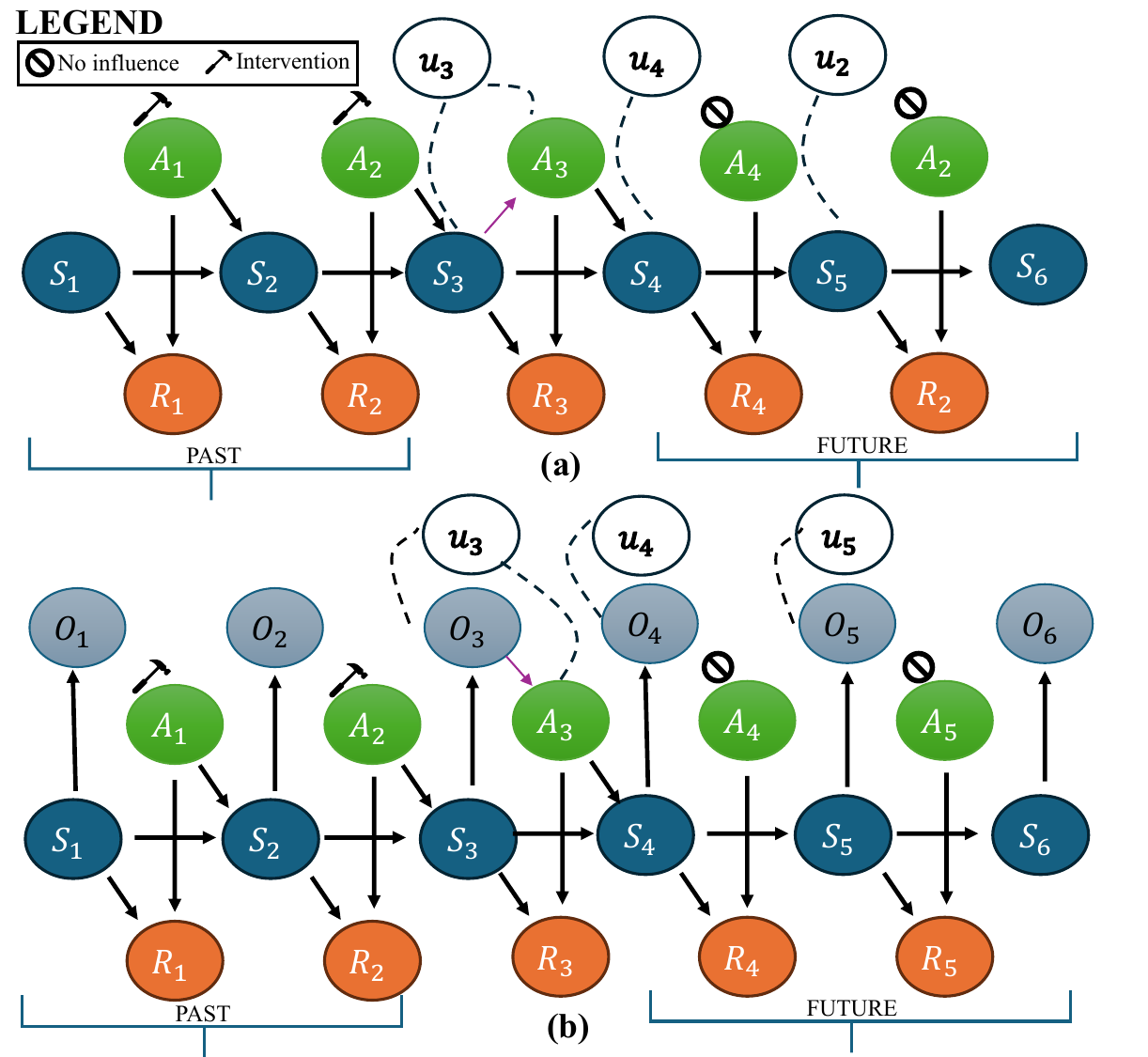}
  \caption{Manipulated Causal Decision Graph for a single agent (a) MDP with SBC at $k=3$ (b) Partially Observable MDP with SBC at $k=3$}\label{fig:sbc}
\end{figure}
\subsection{The Manipulated Causal Decision Graph}
To evaluate the causal impact of actions over a trajectory, we perform graph surgery of the causal decision graph $\mathcal{G}$ to create a manipulated graph $\mathcal{G}_{\overline{\mathbf{a}}_{1:k-1}, \underline{\mathbf{a}}_{k+1:t}}$ (Figure \ref{fig:sbc}). For a specific time step $k$, the graph is modified as follows:
\begin{itemize}
    \item \textbf{Past Actions ($\overline{\mathbf{a}}_{1:k-1}$)}: The incoming edges to previous actions are removed. This isolates past actions from their historical confounders (such as latent intent $U$), simulating a strict sequence of prior hard interventions.
    \item \textbf{Future Actions ($\underline{\mathbf{a}}_{k+1:t}$)}: The outgoing edges from future actions are removed. This renders future decisions independent of their causal descendants in the graph, preventing downstream planned actions from statistically polluting the current transition gradient.
\end{itemize}
This graph surgery isolates the forward-directed causal effect of the current action on the environment from backward-flowing statistical correlations.
\subsection{Formal Definition of the Sequential Backdoor Criterion} \label{sec::sbc_criteria}
The Sequential Backdoor Criterion (SBC) provides the exact graph-theoretic test to determine if an interventional query can be resolved using purely observational data. The interventional distribution for an outcome $Y \in \{s_{t+1}, r_t\}$ given a sequence of actions $do(\mathbf{a}_1, \dots, \mathbf{a}_t)$ is identifiable if, for every $k \in \{1, \dots, t\}$, there exists a set of covariates $\mathbf{Z}_k \subseteq \mathbf{V}$ satisfying:
\begin{itemize}
    \item $\mathbf{Z}_k$ are non-descendants of any future action in $\{\mathbf{a}_k, \dots, \mathbf{a}_t\}$.
    \item  $(Y \perp \mathbf{a}_k \mid \mathbf{a}_1, \dots, \mathbf{a}_{k-1}, \mathbf{Z}_1, \dots, \mathbf{Z}_k)$ in the manipulated graph $\mathcal{G}_{\overline{\mathbf{a}}_{1:k-1}, \underline{\mathbf{a}}_{k+1:t}}$.
\end{itemize}
To bridge this foundational theory to our specific world model formulation, the outcome $Y$ represents our fundamental learning targets: the physical state transition ($s_{t+1}$) and the environment reward ($r_t$). The theoretical sequence of covariates $\mathbf{Z}_{1:k}$ maps directly to our observable trajectory history $H_k = (\mathbf{s}_{0:k}, \mathbf{a}_{0:k-1})$.When these structural conditions hold, the interventional distribution for the outcome mechanism is mathematically identified by:

\begin{equation} \label{eq:sbc_adjustment}
P(Y \mid do(\mathbf{a}_{1:t})) = \sum_{\mathbf{z}_{1:t}} P(Y \mid \mathbf{z}_{1:t}, \mathbf{a}_{1:t}) \prod_{k=1}^t P(\mathbf{z}_k \mid \mathbf{z}_{1:k-1}, \mathbf{a}_{1:k-1})
\end{equation}
This formulation confirms that estimating the causal effect of the action history $\mathbf{a}_{1:t}$ on the transition to state $s_{t+1}$ or reward $r_t$ is possible. The history $H_k$ acts as a sufficient statistic to block confounding paths from the unobserved intent $U$, rendering the physical mechanisms structurally identifiable for a single agent.

Intuitively, the Sequential Backdoor Condition ensures we have observed enough history ($\mathbf{Z}_k$) to decouple an agent's actions from their hidden strategic intent. Equation \ref{eq:sbc_adjustment} then acts as a re-weighting formula that uses this history to simulate a ``fair experiment'' \citep{schlegel2019importanceresamplingoffpolicyprediction}. This allows us to calculate the environment's true response as if we were directly controlling the agents, rather than simply observing their coordinated behavior. The conditions of the SBC are illustrated in Figure \ref{fig:sbc}, which demonstrates that past actions are fixed and future actions cannot be influenced by the current decision.
\subsection{The Structural Violation in Multi-Agent Environments}
\begin{figure}[!htb]
  \centering
  \includegraphics[trim={0pt 250pt 450pt 0pt}, clip, width=0.7\linewidth]{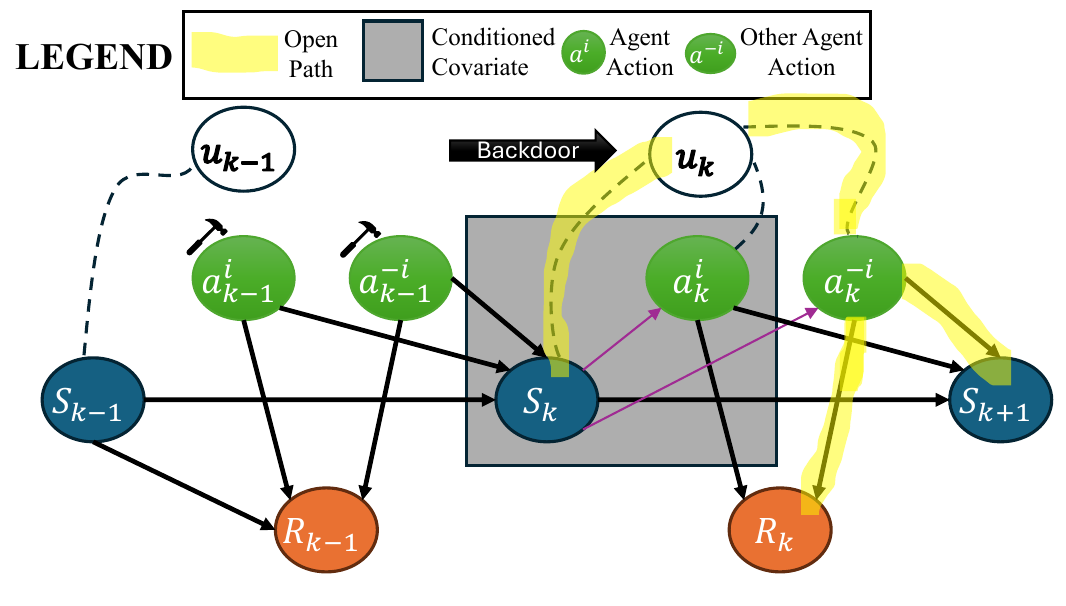}
  \caption{The dynamic confounding path in multi-agent systems. Contemporaneous actions ($a^{-i}_k$) driven by shared intent ($U$) create an unblocked backdoor path ($a^i_k \leftarrow U \rightarrow a^{-i}_k \rightarrow Y$). The standard historical conditioning set ($H_k$) fails to d-separate the target action from the latent intent, necessitating the MAS-SBC.}\label{fig:fail_sbc}
\end{figure}
The baseline SBC fundamentally assumes an isolated decision process. When applied to a multi-agent system, the structural conditions formally fail (Figure \ref{fig:fail_sbc}). The actions of other agents $a^{-i}_k$ exist contemporaneously with the target agent's action $a^i_k$. Both sets of actions are driven by a shared or correlated latent intent $U$. This topology creates a secondary, dynamic backdoor path: $a^i_k \leftarrow U \rightarrow a^{-i}_k \rightarrow Y$. Because the contemporaneous actions $a^{-i}_k$ occur exactly at time $k$, they cannot be included in the historical covariate set $\mathbf{Z}_k$ (our history $H_k$) prior to the action $a^i_k$. Consequently, the observed history fails to d-separate $a^i_k$ from $U$, directly violating Condition 2 of the SBC definition within the manipulated graph $\mathcal{G}_{\overline{\mathbf{a}}_{1:k-1}, \underline{\mathbf{a}}_{k+1:t}}$.

The path through $U$ and $a^{-i}_k$ remains open, meaning the interventional distribution fails to reduce to the observational expression. The environment physics remain permanently entangled with the intent of the surrounding agents. This mathematical failure establishes the strict necessity for extending the framework to the Multi-Agent Sequential Backdoor Condition (MAS-SBC), which executes the backdoor adjustment over the complete contemporaneous action profile.

\label{sec::discovery}

%

\section{Formal Proof of the MAS Sequential Backdoor Condition}
\label{sec::mas_sbc_proof}

\begin{figure}[!htb]
  \centering
  \includegraphics[trim={0pt 250pt 440pt 0pt}, clip, width=0.7\linewidth]{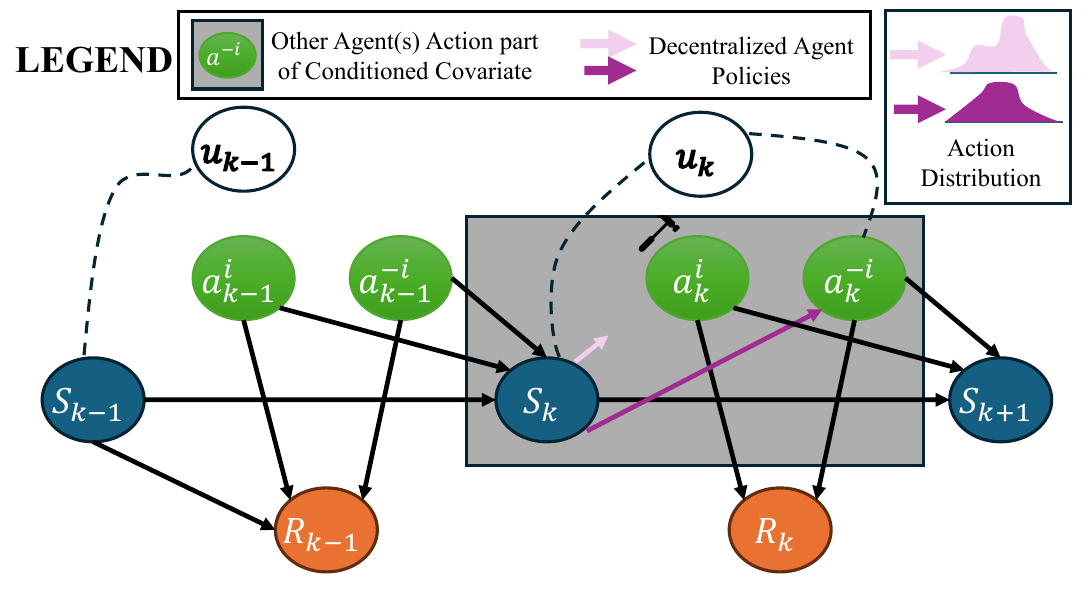}
  \caption{\textbf{Mutilated Causal Decision Graph $\mathcal{G}_{\overline{A^i}}$ for MAS-SBC Identifiability}. The hammer indicates the intervention $do(a^i_k)$, which severs incoming edges from $H_k$ and $U$. Conditioning on the intact peer actions ($a^{-i}_k$) blocks the contemporaneous backdoor path.}
    \label{fig::mas_sbc_surgery}
\end{figure}

To establish the structural identifiability of the world model, we utilize Pearl’s do-calculus to evaluate the causal dynamics within the mutilated graph $\mathcal{G}_{\overline{A_{1:k-1}^i}\underline{A_{k+1:t}^i}}$, isolating the specific intervention on $A^i_t$ by severing its relevant sequential dependencies. 

\subsection{Structural Axioms of the Multi-Agent System}
We establish the true Data Generating Process using three strict structural axioms:
\begin{itemize}[leftmargin=*,noitemsep,topsep=0pt]
    \item \textbf{Axiom 1: Transition Markovianity:} $S_{t+1} \perp \!\!\! \perp U \mid (H_k, A^i_t, A^{-i}_t)$. The physical transition mechanism $f_s$ depends exclusively on the observed history and realized joint actions. 
    \item \textbf{Axiom 2: Decentralized Policy:} $A^i_t \perp \!\!\! \perp A^{-i}_t \mid (H_k, U)$. Agents act simultaneously without direct causal influence on one another ($A^i_t \nrightarrow A^{-i}_t$). The only connections between contemporaneous actions are backdoor paths through shared history $H_k$ and overarching strategic intent $U$.
    \item \textbf{Axiom 3: Confounding Structure:} $U \rightarrow A^i_t$ and $U \rightarrow A^{-i}_t$. The latent intent acts as an unobserved contemporaneous confounder, coordinating all agent actions.
\end{itemize}

\subsection{Step 1: Graphical d-Separation via Induction}
To apply do-calculus, we must first prove that our chosen covariate set $Z_k = \{H_k, a^{-i}_k\}$ successfully blocks all non-causal paths pointing into $a^i_k$. We require the assumption of \textbf{Sufficiency of History}: The observed history $H_k$ acts as a sufficient statistic for the latent intent $U$, meaning $s_{k+1} \perp \!\!\! \perp U \mid (H_k, a^i_k, a^{-i}_k)$.

\begin{proof}[Inductive Proof of MAS d-Separation]
\textbf{Base Case (Single-Agent):} Assume peer agents $a^{-i}_k$ are removed. The only backdoor path is historical: $a^i_k \leftarrow U \rightarrow s_{k+1}$. Under the sufficiency of history, conditioning on $Z_k = \{H_k\}$ screens off $U$ and blocks this path, satisfying the single-agent Sequential Backdoor Criterion.

\textbf{Inductive Step (Multi-Agent Extension):} Reintroducing peer agents $a^{-i}_k$ creates a new, contemporaneous backdoor path: $a^i_k \leftarrow U \rightarrow a^{-i}_k \rightarrow s_{k+1}$. The base case set $\{H_k\}$ is historical and cannot block this instantaneous path. We expand the set to $Z_k = \{H_k, a^{-i}_k\}$. By Axiom 2, peer action $a^{-i}_k$ acts as a chain node on the confounding path, not a collider. Therefore, explicitly conditioning on $a^{-i}_k$ mathematically blocks the new multi-agent backdoor path. 

\textbf{Conclusion:} Conditioning on $a^{-i}_k$ reduces the unblocked topology back to the single-agent base case. Conditioning on $H_k$ subsequently clears the remaining historical path. Because no unblocked paths remain, the conditional independence $(s_{k+1} \perp \!\!\! \perp a^i_k \mid H_k, a^{-i}_k)$ holds in the mutilated graph $\mathcal{G}_{\overline{a^i_k}}$.
\end{proof}

\subsection{Step 2: Algebraic Reduction via do-Calculus}
With graphical d-separation proven, we resolve the interventional query to purely observational data. 

\begin{proof}[Proof of Theorem \ref{theorem:mas_sbc}]
We expand the target interventional distribution $P(s_{t+1} \mid do(a^i_t), H_k)$ by marginalizing over the peer actions $A^{-i}_t$ and the unobserved confounder $U$:
$$= \sum_{a^{-i}_t} \sum_{U} P(s_{t+1} \mid do(a^i_t), a^{-i}_t, H_k, U) P(a^{-i}_t \mid do(a^i_t), H_k, U) P(U \mid do(a^i_t), H_k)$$

Applying the rules of do-calculus based on our established axioms and d-separation:
\begin{enumerate}[leftmargin=*,noitemsep,topsep=0pt]
    \item \textbf{Isolating the Intervention:} $do(a^i_t)$ removes edges into $A^i_t$. It cannot backward-influence $U$ or $H_k$. Thus, $P(U \mid do(a^i_t), H_k) = P(U \mid H_k)$.
    \item \textbf{Applying Axiom 2:} $A^{-i}_t$ is not a descendant of $A^i_t$. Intervening on $A^i_t$ does not trigger a simultaneous reaction. Thus, $P(a^{-i}_t \mid do(a^i_t), H_k, U) = P(a^{-i}_t \mid H_k, U)$.
    \item \textbf{Applying Axiom 1:} The transition is fully determined by observed variables. Because $A^i_t$ is now a fixed parent, we drop the do-operator. $S_{t+1}$ is independent of $U$ given the actions. Thus, $P(s_{t+1} \mid do(a^i_t), a^{-i}_t, H_k, U) = P(s_{t+1} \mid a^i_t, a^{-i}_t, H_k)$.
\end{enumerate}

Substituting these reductions back yields:
$$= \sum_{a^{-i}_t} P(s_{t+1} \mid a^i_t, a^{-i}_t, H_k) \sum_{U} P(a^{-i}_t \mid H_k, U) P(U \mid H_k)$$

By the Law of Total Probability, the inner summation over $U$ reduces exactly to the marginal observational probability of the other agents' actions:
$$P(s_{t+1} \mid do(a^i_t), H_k) = \sum_{a^{-i}_t} P(s_{t+1} \mid a^i_t, a^{-i}_t, H_k) P(a^{-i}_t \mid H_k)$$
This concludes the proof, demonstrating structural identifiability.
\end{proof}

\section{Motivation Study in Tic-Tac-Toe}
\label{sec::motivation}
This study establishes the necessity of causal world models by analyzing failure modes within a controlled Tic-Tac-Toe environment, demonstrating that standard architectures collapse under interventional uncertainty. We show that integrating causal discovery and normalizing flows provides the structural alignment required for resilient state recovery beyond the optimal training distribution.
\begin{figure}[h]
  \centering
  \includegraphics[trim={0pt 260pt 270pt 20pt}, clip, width=\linewidth]{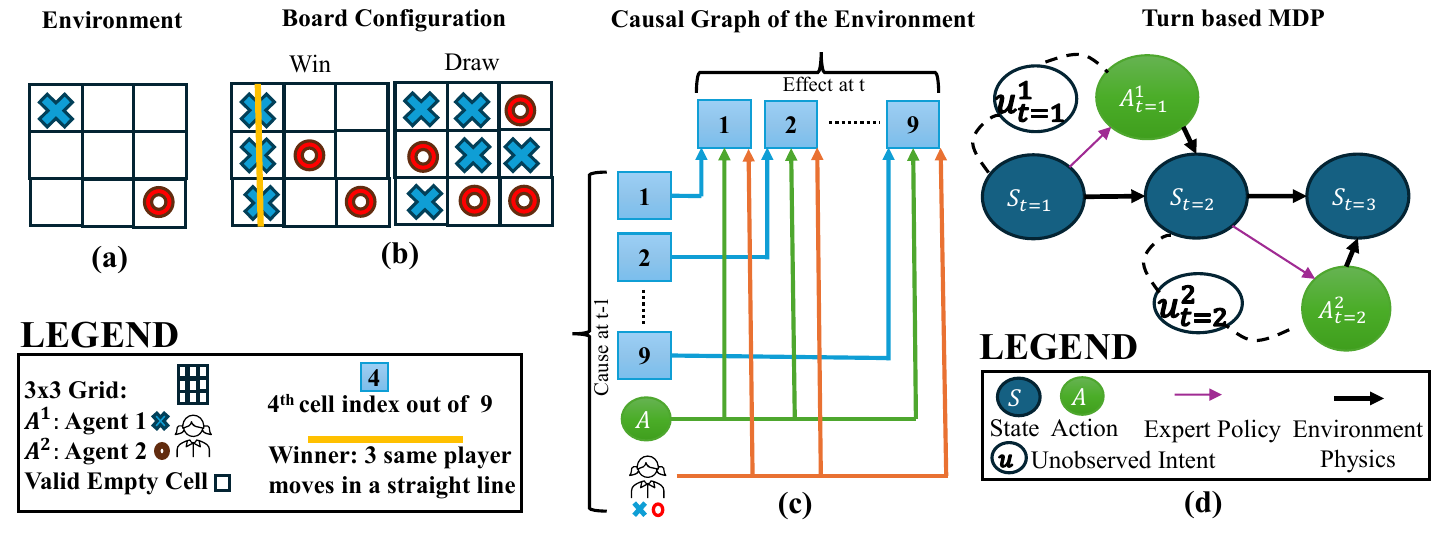}
  \caption{Motivation study in Tic-Tac-Toe. (a) Environment representation with initial agent moves. (b) A winning configuration for Agent 1 and a Draw. (c) The ground-truth causal graph defining state transitions and dependencies between moves. (d) The environment modeled as a turn-based Markov Decision Process, illustrating state transitions, expert policy actions, and unobserved agent intent.  }\label{fig:motivation_env}
\end{figure}
Tic-Tac-Toe \citep{terry2021pettingzoo} serves as our motivation study environment. This domain provides a controlled setting where we can formally define the causal relationships governing game transitions. The environment dynamics are fully observable and computationally tractable. Figure \ref{fig:motivation_env} illustrates the core components of this study. Panel (a) and (b) show the grid configuration and a terminal winning state for Agent 1 including another draw configuration by both agents. Panel (c) presents the ground-truth causal graph, detailing the dependencies between board cells, agent actions, and state effects at time $t$. A cell from the $3\times3$ grid depends on its value from the previous time step or can be altered by a player's action. Panel (d) formalizes this as a turn-based Markov Decision Process. This MDP formulation explicitly captures state transitions, agent policies, and latent intent, allowing us to isolate causal influence from environmental noise. Because the game is a turn-based MDP, it satisfies the Markov property and the structural requirements for causal inference (including SBC from Section \ref{sec::sbc_defined}).

\begin{figure}[htbp]
    \centering
    \begin{minipage}[c]{0.45\textwidth}
        \centering
        \includegraphics[width=\linewidth, keepaspectratio]{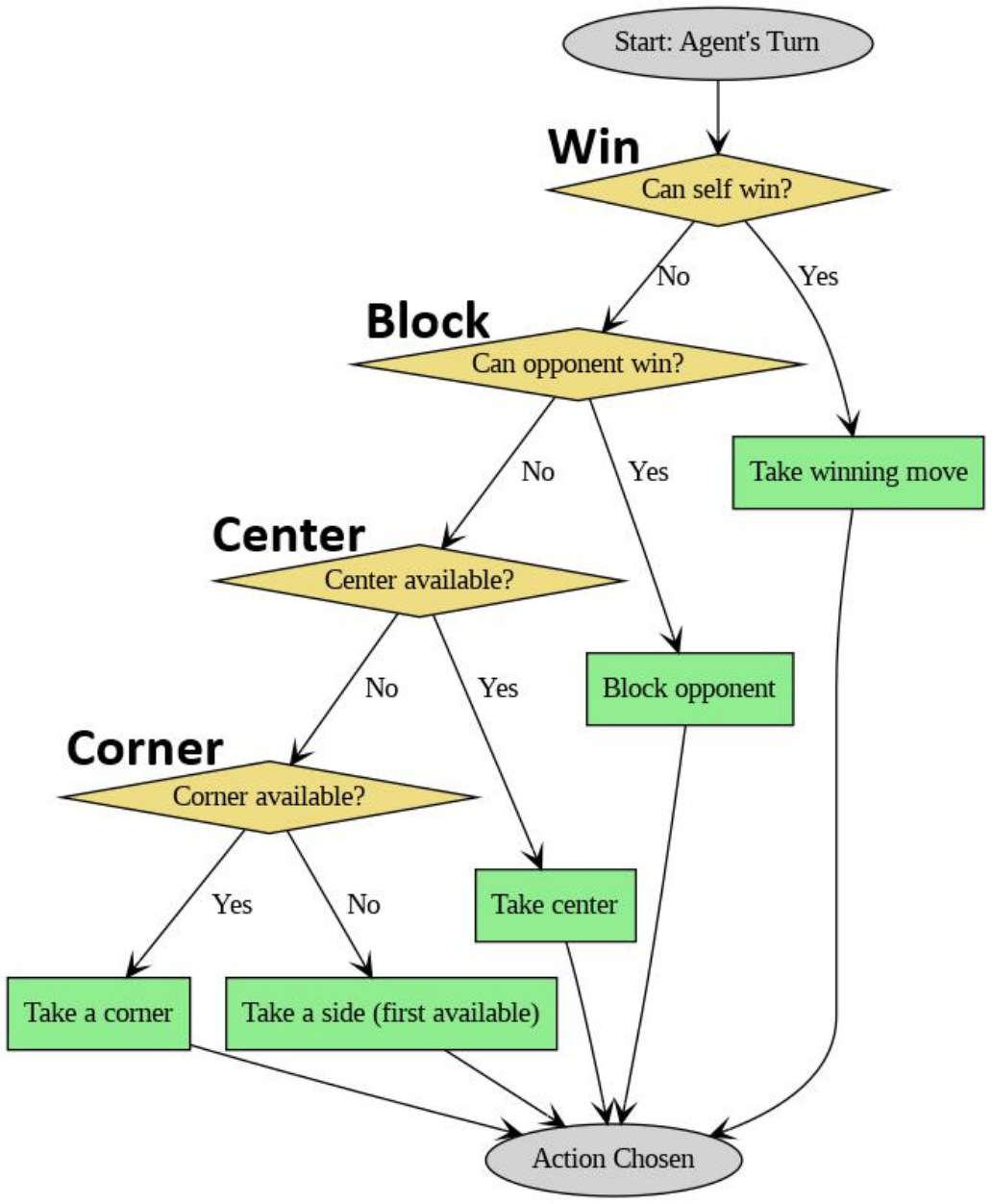}
        \caption{Logic flowchart for Tic-Tac-Toe agent decision-making.}
        \label{fig:decision_flowchart}
    \end{minipage}
    \hfill
    \begin{minipage}[c]{0.50\textwidth}
        \centering
        \hrule
        \vspace{0.3em}
        \captionof{algorithm}{Optimal Tic-Tac-Toe Move Selection}
        \label{algo::tic_tac_toe}
        \vspace{-0.5em}
        \hrule
        \vspace{0.4em}
        \begin{algorithmic}[1]
            \State \textbf{Input:} Board state $S$, Player ID $P$
            \State \textbf{Output:} Best action $a \in \text{LegalActions}(S)$
            \If{$\exists a \in \text{LegalActions}(S)$ completes line for $P$}
                \State \Return $a$ \Comment{Win immediately}
            \EndIf
            \If{$\exists a$ completes line for opponent}
                \State \Return $a$ \Comment{Block opponent}
            \EndIf
            \If{Center cell available}
                \State \Return CenterCell
            \EndIf
            \If{Corner cell available}
                \State \Return CornerCell
            \EndIf
            \State \Return Remaining cell
        \end{algorithmic}
        \vspace{0.3em}
        \hrule
    \end{minipage}
\end{figure}

The optimal policy for Tic-Tac-Toe from Algorithm \ref{algo::tic_tac_toe} relies on a strict hierarchy of heuristics that ensures a non-losing outcome (inspiration Game Theory \citep{vonneumann1944theory}). The strategy prioritizes immediate tactical gains, beginning with a winning move if a line can be completed. If no win is possible, the agent shifts to a defensive stance by blocking any move that would complete a line for the opponent. Lacking immediate threats or opportunities, the agent selects strategic board positions based on their geometric utility, starting with the center cell, which offers the highest number of potential lines. Corner cells serve as the next priority due to their involvement in multiple potential configurations. Consequently, figure \ref{fig:decision_flowchart} highlights this hierarchical decision making. 

\begin{figure}[htbp]
  \centering
  \includegraphics[trim={0pt 00pt 100pt 0pt}, clip, width=\linewidth]{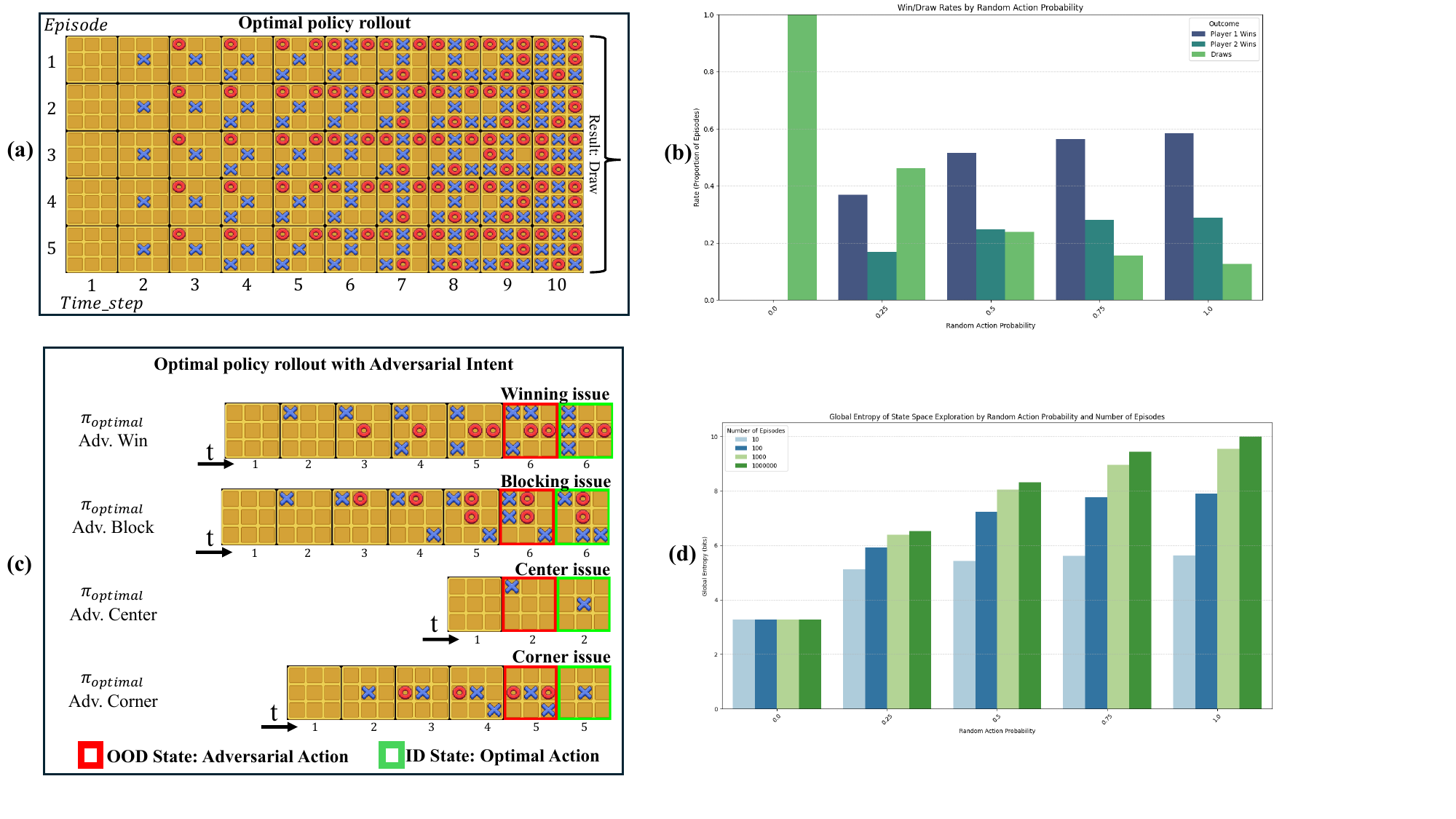}
  \caption{Analysis of Optimal Policy Performance and Exploration Entropy. (a) Optimal policy rollout showing consistent draw results across five episodes. (b) Win/Draw rates measured against varying random action probabilities (c) Optimal policy rollouts under conditions of adversarial intent, highlighting "OOD State" (adversarial action) and "ID State" (optimal action) transitions across specific game issues. (d) Global entropy of state space exploration categorized by random action probability and the total number of episodes.  }\label{fig:motivation_env2}
\end{figure}

We demonstrate in Figure \ref{fig:motivation_env2} panel (a) that when both agents in Tic-Tac-Toe play according to the deterministic policy in Algorithm \ref{algo::tic_tac_toe}, the final result is consistently a draw with no diversity in agent outcomes. However, introducing a probability for each agent to take a random action versus an action from Algorithm \ref{algo::tic_tac_toe} results in diversity in win and loss rates. Panel (b) shows simulated results over 1 million game episodes, where the introduction of exploration causes agent win rates to shift. The results indicate that Player 1 maintains a winning edge at a random action probability of 1.0, likely due to the first-mover advantage inherent in Tic-Tac-Toe. This plot acts as a strong motivation for an agent to perform adversarial actions to win over their counterpart. Panel (c) highlights four separate episode rollouts under the optimal policy from Algorithm \ref{algo::tic_tac_toe}, specifically identifying instances marked by red squares where an agent deviates from the optimal policy hierarchy, while green squares highlight the correct optimal action. A closer introspection of the red squares reveals these states are unreachable if both agents strictly follow the optimal policy; therefore, these are identified as Out-of-Distribution (OOD) states with respect to the optimal policy's state visitation probability. Drawing reference from panel (b), OOD states appear to be important pathways in adversarial or random play. A world model trained exclusively on optimal trajectory datasets cannot handle these states due to insufficient coverage. Panel (d) displays simulated rollouts of a mixture of optimal and random policies over 10, 100, 1000, and 1000000 episodes, where the y-axis represents the global entropy of the explored state space. Theoretically, the maximum bits required to accurately capture the variation of these discrete states is $\log_2(5,478) \approx 12.419$ bits, practically rounded to 13 bits, as each of the 9 cells can be colored in 3 ways under the constraint that both players can have equal number of moves on the board or player 1 may have exactly one move more than the  player 2. Careful introspection reveals that even with 1 million environment episodes, the agent does not explore the entire state space unless total random actions are employed. This plot motivates the argument that noise injected into agent actions is required for diversity and soft interventions, rather than achieving full state space coverage, which is fundamentally limited in offline algorithmic settings.
\begin{figure}[htbp]
  \centering
  \includegraphics[trim={0pt 270pt 150pt 0pt}, clip, width=\linewidth]{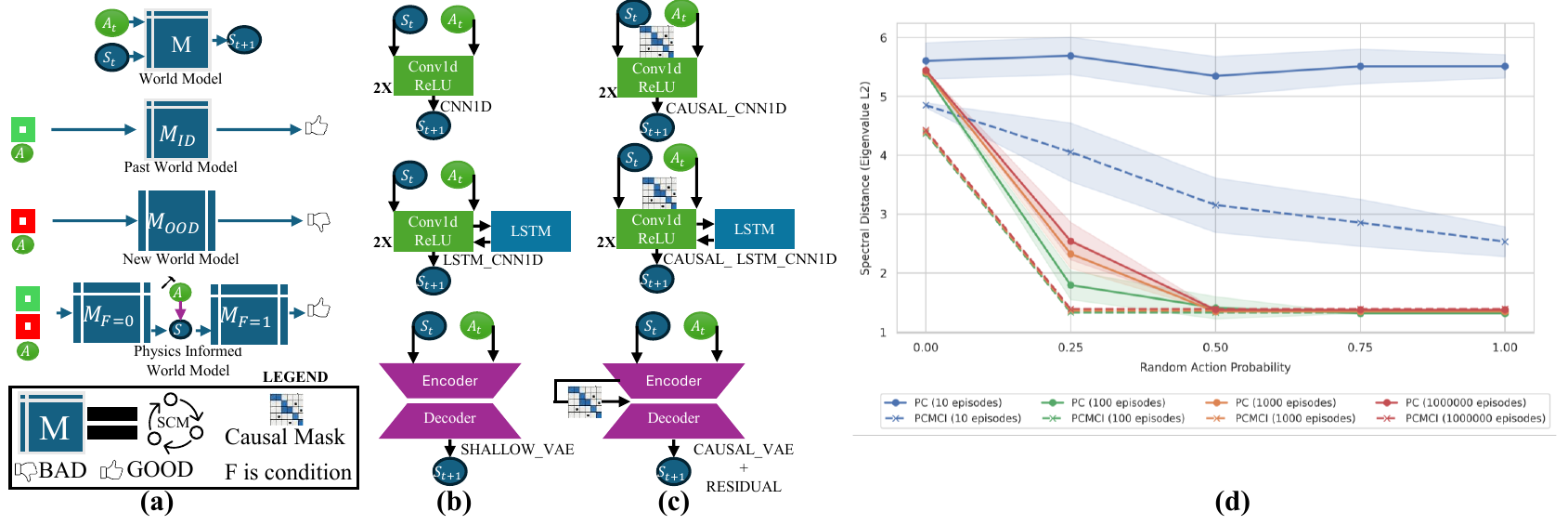}
 \caption{Model architectures and evaluation. (a) Schematic representation of various world models and causal conditions. (b) Architecture of standard models (CNN1D, LSTM\_CNN1D, and Shallow VAE). (c) Architecture of causal models incorporating causal masks (CAUSAL\_CNN1D, CAUSAL\_LSTM\_CNN1D, and CAUSAL\_VAE + Residual). (d) Spectral distance (Eigenvalue L2) performance comparison between standard PC and causal PCMCI models across varying random action probabilities and dataset size in-terms of episodes over 5 seeds.}\label{fig:motivation_env3}
\end{figure}
We evaluate the performance of various world model architectures in recovering transition dynamics, specifically targeting the OOD failure modes identified (specially curated 52 transitions over 4 scenarios) in Figure \ref{fig:motivation_env2}. As demonstrated in Figure \ref{fig:motivation_env2} (c), models trained solely on optimal trajectories fail to generalize to states where agents deviate from the optimal policy (red squares). To mitigate this, we introduce a physics-informed world model framework that consumes current states and actions to predict the transition function. The reliability of these models depends on their ability to recover accurate transition dynamics, even following interventional state changes.

As shown in Figure \ref{fig:motivation_env3} (a), the physics-informed model allows for reliable transition prediction under intervention modeling them as condition covariate. Panels (b) and (c) contrast standard baseline architectures with those utilizing causal masking, which filters non-causal dependencies. Figure \ref{fig:motivation_env3}(d) presents the quantitative results for running causal discovery on the said dataset trajectories: the PCMCI approach is superior, with error converging rapidly even with 100 episodes at a 25\% random action probability. In contrast, the standard PC algorithm, which relies on the i.i.d. assumption, exhibits poor convergence in low-data regimes. For 10 episodes, the PC algorithm fails to converge, and even with 1M episodes at 100\% random action probability, the residual error remains non-zero.

We evaluate the structural accuracy of recovered causal graphs using the PCMCI soft mask against the baseline PC algorithm. As shown in Figure \ref{fig:motivation_env4}, the PCMCI approach consistently recovers a more rank-sufficient representation of the feature space. While the PC algorithm is highly sensitive to the sample size and environmental entropy, the PCMCI method utilizes probabilistic weighting to regularize graph reconstruction in our implementation. This robustness is critical for maintaining structural integrity in regimes where the baseline PC algorithm produces sparse or noisy projections due to the collapse of conditional independence tests.

The performance discrepancy highlights the sensitivity of causal discovery to the interplay between data complexity and random action probability. As environmental exploration decreases toward a random action probability of 0.0, the data becomes increasingly deterministic, posing significant challenges for traditional i.i.d.-based discovery methods. PCMCI, by contrast, effectively mitigates the uncertainty associated with these low-exploration regimes, allowing for more stable graph identification. The ground truth causal graph can be found by the reader in figure \ref{fig:motivation_env} (c) for Tic-Tac-Toe environment.

These findings have direct implications for modeling real-world complex systems, where data collection is often bounded by the high cost of exploration. We identify two representative data-scarce regimes: a Targeted-Efficiency Regime consisting of 100 episodes with 25\% random action probability, reflecting scenarios with limited, directed interaction opportunities; and a Sparse-Exploration Regime consisting of 10 episodes with 100\% random action probability, reflecting high-cost environments requiring maximum state-space coverage from minimal interactions.

The PCMCI soft mask proves superior in these regimes, ensuring that even under restricted data budgets, the model recovers reliable causal pathways necessary for downstream decision-making. Where the PC algorithm’s reliance on large-sample assumptions leads to incomplete or unreliable graph recovery, the PCMCI method maintains structural fidelity, confirming its utility for complex environments where exhaustive data collection is prohibitive.

\begin{figure}[htbp]
  \centering
  \includegraphics[trim={0pt 0pt 0pt 30pt}, clip, width=\linewidth]{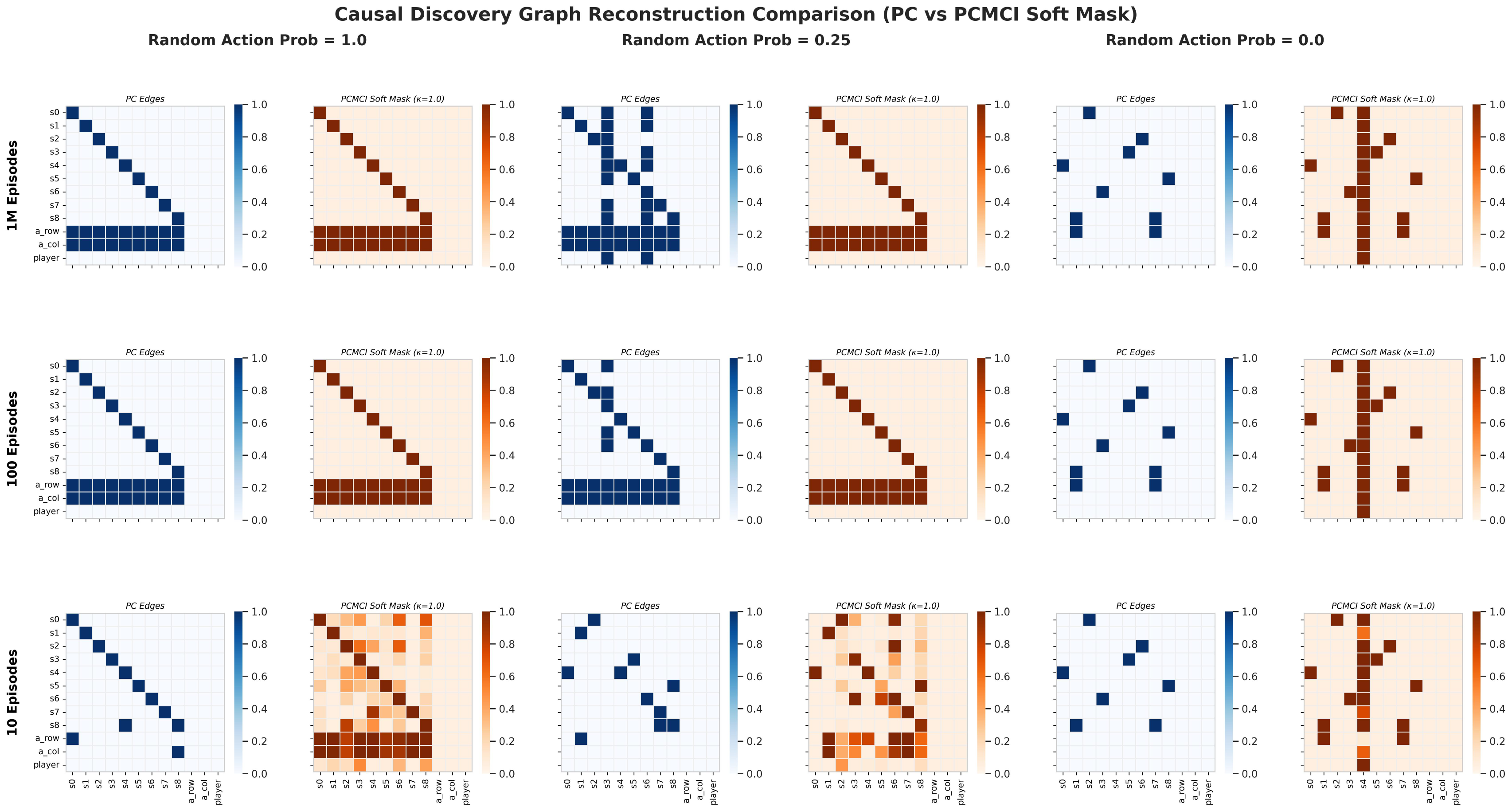}
  \caption{Causal discovery graph reconstruction comparison between the PC algorithm (blue, left column in each block) and PCMCI soft mask (orange, right column in each block). Rows represent varying amounts of training data (1M, 100, and 10 episodes), while columns represent different levels of exploration (Random Action Probability 1.0, 0.25, and 0.0). The plots demonstrate that the PCMCI soft mask maintains more robust and accurate graph structure recovery, particularly in low-data (10 episodes) and low-exploration (0.25 Random Action Probability) regimes compared to the baseline PC algorithm.}\label{fig:motivation_env4}
\end{figure}

We evaluate the out-of-distribution robustness of world models by incorporating ground-truth causal masks and applying a novel "Spine and Cloud" architecture. Figure \ref{fig:motivation_env5} presents the OOD test MSE for various architectures alongside the schematic of our uncertainty modeling pipeline.

As shown in Figure \ref{fig:motivation_env5}(a), constraining models with ground-truth causal masks significantly enhances performance in high-data regimes (1M episodes). Causal architectures, particularly \texttt{Causal\_CNN1D} and \texttt{Causal\_LSTM\_CNN1D}, demonstrate a consistent reduction in error as random action probability increases. This indicates that these architectures effectively leverage exploratory data to refine their transition functions, avoiding the spurious correlations that plague baseline models. The reduced variance, evidenced by the tighter error bands, confirms that causal guidance stabilizes learning even under interventional shifts. We also acknowledge the optimization for \texttt{VAE} and \texttt{Causal\_residual\_VAE} remains an open question under varying discrete data distributions.

To explicitly model residual uncertainty \citep{lu2023characterizing}, we utilize the "Spine and Cloud" normalizing flow \citep{papamakarios2021normalizing,elgazzar2024generative,margossian2023amortized} architecture shown in Figure \ref{fig:motivation_env5}(b). This framework is applied exclusively to regressor networks (world models), separating the deterministic transition dynamics from the stochastic residuals. The Spine Network is frozen, serving as the base model, while the Cloud Network is trained as a normalizing flow to map residual errors ($r = y_{gt} - \text{Spine}(x_{context})$) to a distribution characterized by shift in mean and variance $(\Delta\mu$ and $\Delta\sigma)$ of the neural density. The Spine and Cloud normalizing flow models the transition uncertainty as a push-forward measure by mapping the residual error distribution through an injective flow, which simultaneously enables the modeling of inverse dynamics due to the flow's volume-preserving and invertible properties.  

A high-fidelity model is identified by the Cloud Network's ability to minimize the negative log-likelihood (NLL) of these residuals under a Gaussian assumption. The pipeline proceeds in three stages: 
\begin{enumerate}
    \item \textbf{Input:} The system consumes a context tuple ($x_{context}$) consisting of the current state and action.
    \item \textbf{Training:} With the Spine Network fixed, the Cloud Network learns to characterize the distribution of residual errors.
    \item \textbf{Inference:} The model samples from the learned distribution to produce a predictive output ($y_{sample}$), allowing the agent to quantify its uncertainty when faced with interventional OOD states.
\end{enumerate}
This distribution-based approach allows for robust decision-making, as the agent can explicitly recognize and adapt to scenarios where its transition predictions are inherently uncertain.

\begin{figure}[htbp]
  \centering
  \includegraphics[trim={0pt 260pt 270pt 0pt}, clip, width=\linewidth]{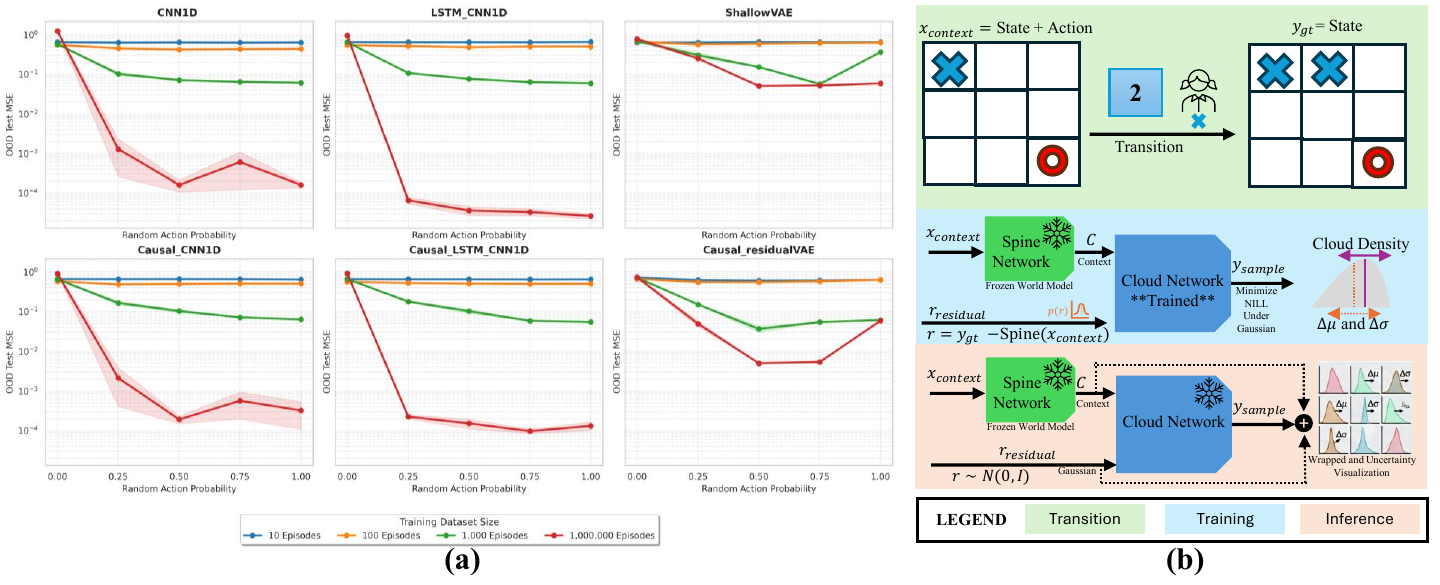}
  \caption{Out-of-distribution (OOD) performance analysis and model architecture. (a) OOD Test Mean Squared Error (MSE) comparison across various standard and causal model architectures (Ground-Truth environment dynamics causal mask has been used), evaluated against random action probability and training dataset size over 5 seeds. (b) Spine and Cloud normalizing flow architecture designed to evaluate world models under interventional uncertainty. The pipeline integrates a frozen Spine Network with a trainable Cloud Network to estimate residual transition uncertainty via conditional distribution modeling.}\label{fig:motivation_env5}
\end{figure}

To quantify the efficacy of our causal discovery pipeline, Figure \ref{fig:motivation_env6} provides a distributional comparison of OOD Test MSE between models guided by ground-truth causal masks and those utilizing masks discovered by our custom PC and PCMCI algorithms. The split violin plots visualize both the mean performance and the variance of predictive error, offering a granular view of how discovered masks influence model stability under interventional shifts. A key observation is the convergence of performance distributions in high-data regimes (1,000,000 episodes). In these cases, the custom-discovered PCMCI masks exhibit a performance profile nearly indistinguishable from the ground-truth masks, suggesting that our causal discovery process successfully recovers the essential graph structure required for accurate transition modeling. 

Conversely, the divergence observed in low-data regimes (10–100 episodes) serves as an empirical diagnostic for the causal discovery process. The increased variance and mean MSE in the custom halves of the violin plots highlight regimes where the algorithm struggles to resolve structural causal dependencies. This alignment gap indicates that while the PCMCI approach is robust, its structural accuracy remains data-dependent, particularly in environments where exploration is highly restricted. Consequently, these results validate the PCMCI soft mask as a high-fidelity proxy for ground-truth dynamics, while simultaneously identifying the data-efficiency bounds within which the discovered graph maintains predictive structural integrity. Architectural Sensitivity: Note that some architectures (e.g., \texttt{residualVAE}) exhibit different sensitivity to causal mask inaccuracies compared to \texttt{CNN1D}. This sensitivity implies that the choice of backbone architecture fundamentally influences the utility of the discovered causal graph, as different models may either amplify or attenuate errors within the mask structure during the transition prediction process.

\begin{figure}[htbp]
  \centering
  \includegraphics[trim={0pt 0pt 0pt 50pt}, clip, width=0.9\textwidth]{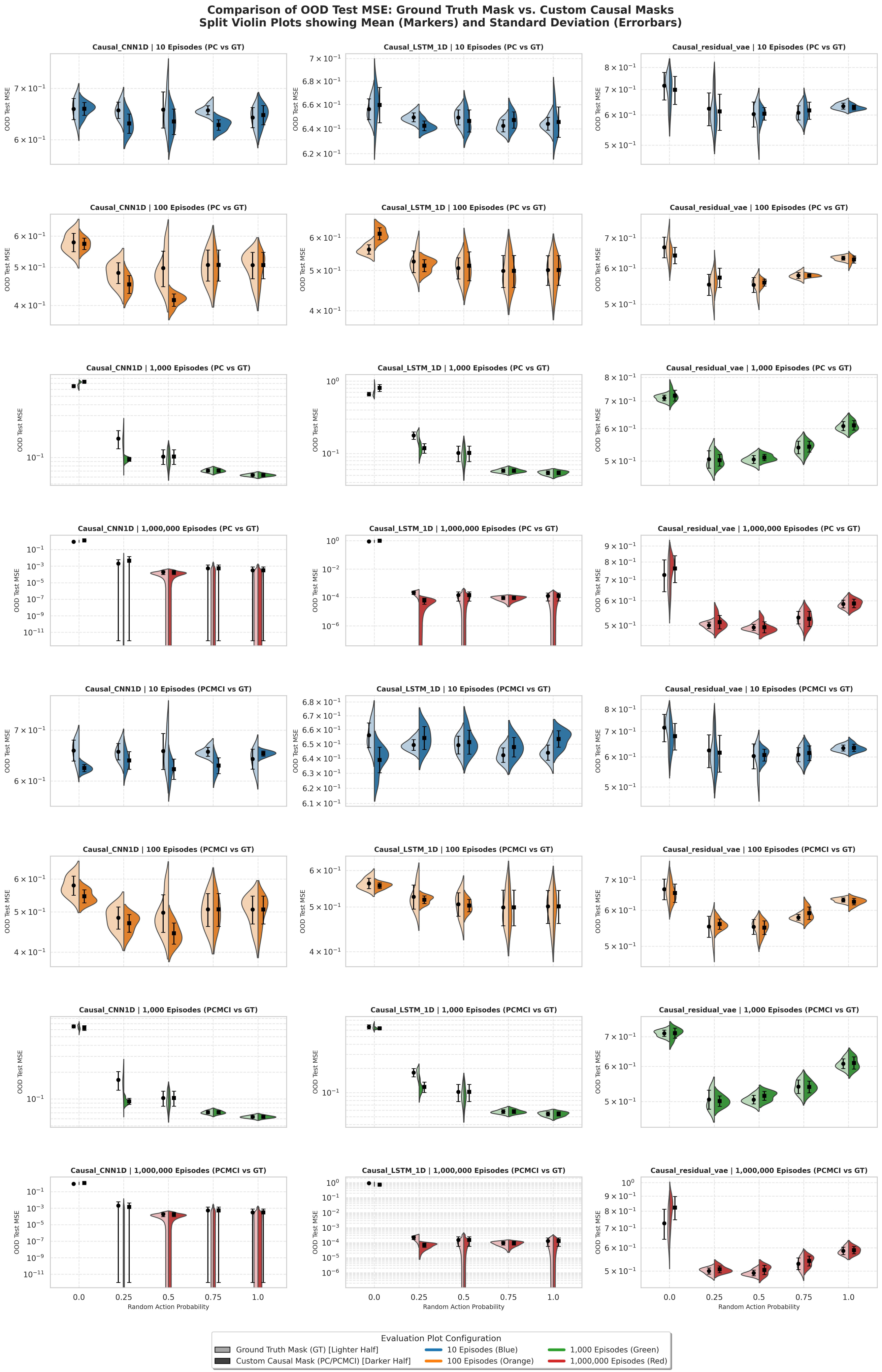}
  \caption{Split violin plots comparing OOD Test MSE between ground truth causal masks (lighter half) and custom causal masks (darker half, including PC and PCMCI). The rows organize results by model architecture and training data size (10 to 1,000,000 episodes), while the x-axis represents varying random action probabilities over 5 seeds. This comparison highlights the distribution, mean, and variance in predictive error, demonstrating the alignment or divergence of custom-discovered masks from the ground truth under different regimes of training data and environmental exploration.}\label{fig:motivation_env6}
\end{figure}

\begin{figure}[htbp]
  \centering
  \includegraphics[trim={00pt 00pt 0pt 80pt}, clip, width=\linewidth, keepaspectratio]{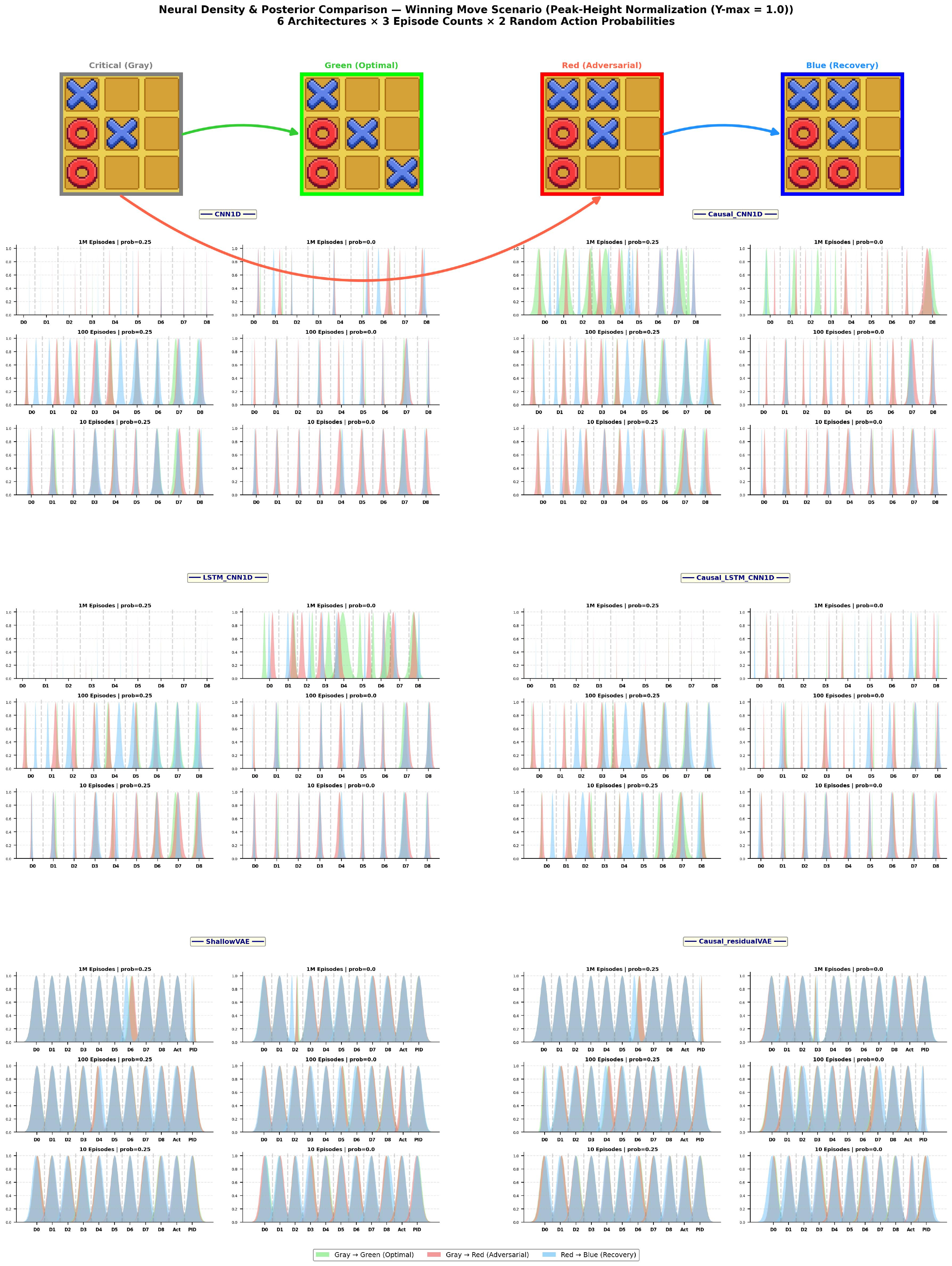}
  \caption{Failure mode: Winning and recovery. Density peaks under an adversarial intervention that denies an immediate win, together with the recovery dynamics back to the ground-truth transition. Colour coding follows the convention used throughout this section: \textcolor{green}{green} for optimal \textcolor{green}{ID} transitions, \textcolor{red}{red} for \textcolor{red}{OOD} adversarial interventions, and \textcolor{blue}{blue} for \textcolor{blue}{recovery}.}
  \label{fig:motivation_env_win}
\end{figure}

\begin{figure}[htbp]
  \centering
  \includegraphics[trim={00pt 00pt 0pt 80pt}, clip, width=\linewidth, keepaspectratio]{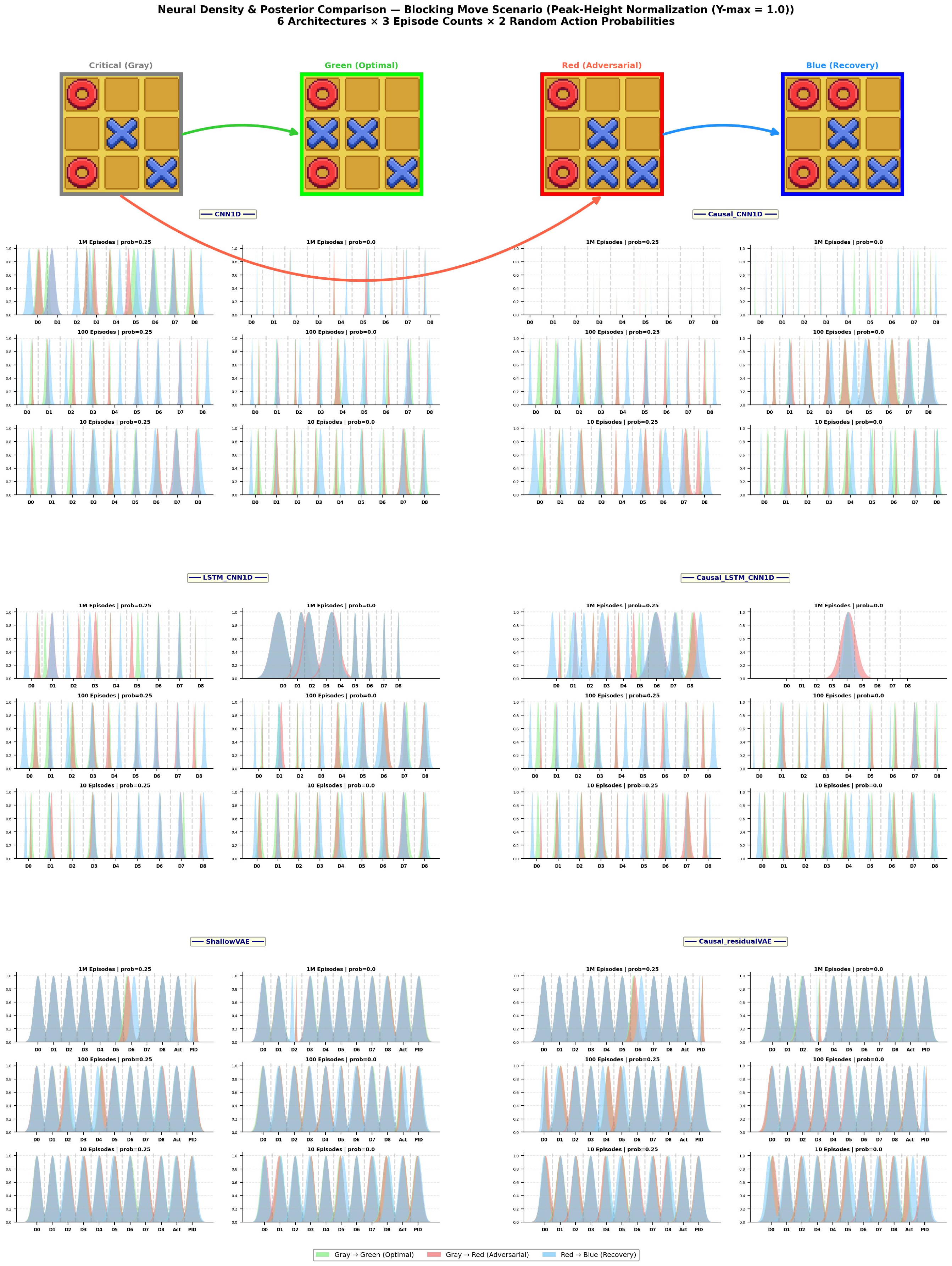}
  \caption{Failure mode: Blocking and recovery. Density peaks when the agent is forced away from a necessary blocking move, with the subsequent recovery trajectory. Colour coding as in Figure~\ref{fig:motivation_env_win}.}
  \label{fig:motivation_env_block}
\end{figure}

\begin{figure}[htbp]
  \centering
  \includegraphics[trim={00pt 00pt 0pt 80pt}, clip, width=\linewidth, keepaspectratio]{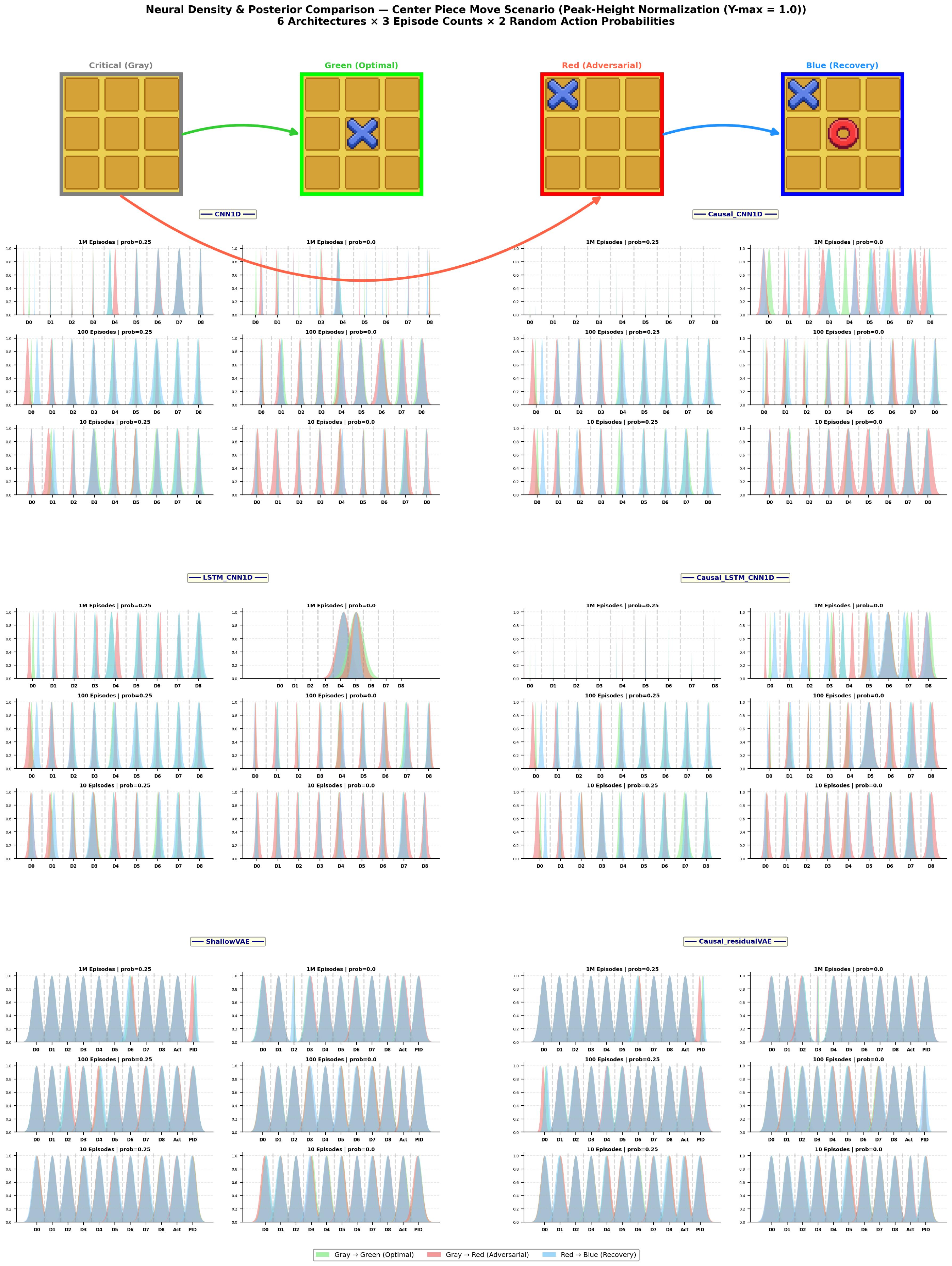}
  \caption{Failure mode: Center control and recovery. Density peaks under an intervention on central-cell occupancy, a structurally sensitive region of the causal mask. Colour coding as in Figure~\ref{fig:motivation_env_win}.}
  \label{fig:motivation_env_center}
\end{figure}

\begin{figure}[htbp]
  \centering
  \includegraphics[trim={00pt 00pt 0pt 80pt}, clip, width=\linewidth, keepaspectratio]{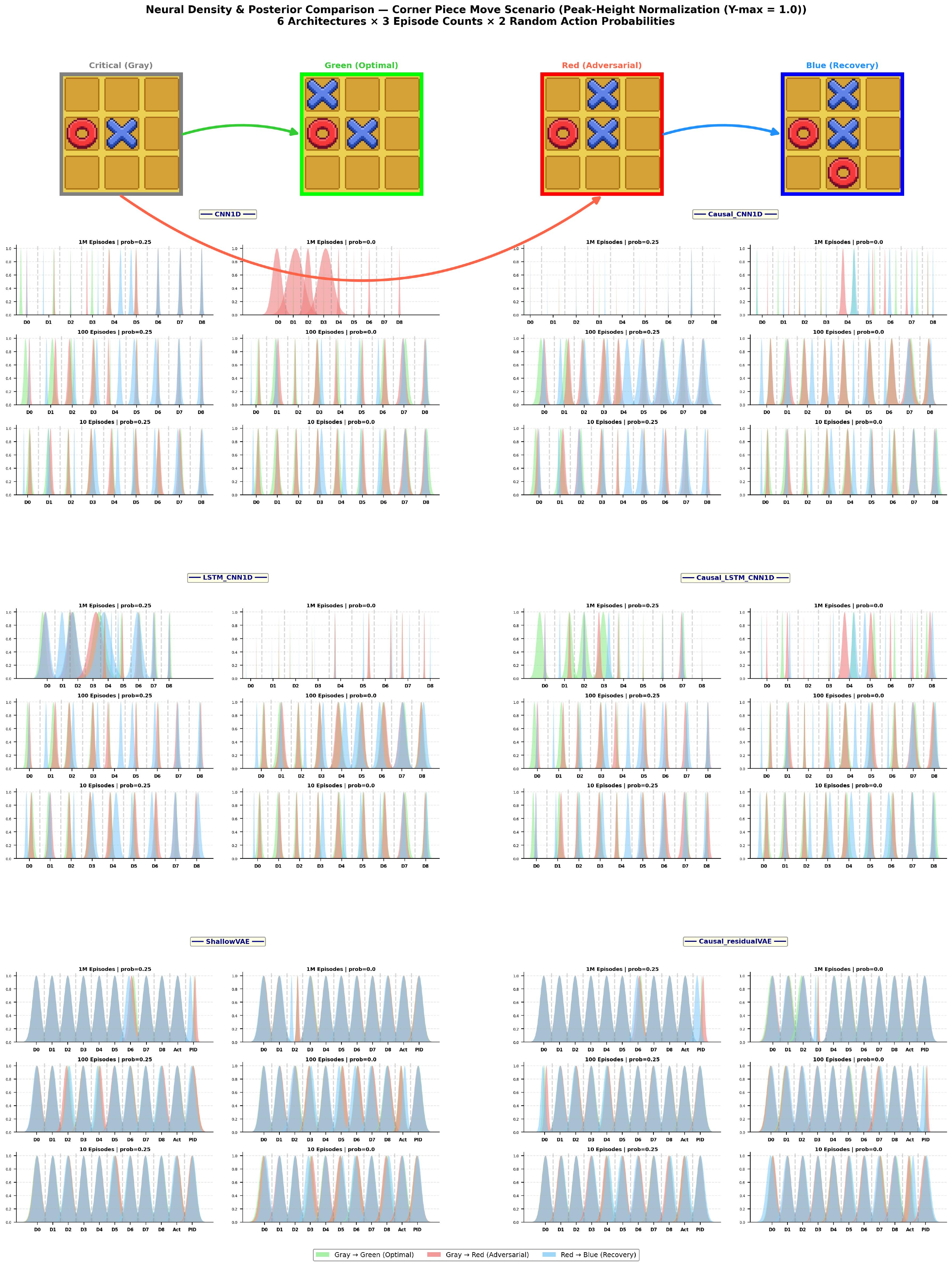}
  \caption{Failure mode: Corner occupancy and recovery. Density peaks under an intervention on corner-cell occupancy, together with the recovery dynamics. Colour coding as in Figure~\ref{fig:motivation_env_win}. Across Figures~\ref{fig:motivation_env_win}--\ref{fig:motivation_env_corner}, the normalizing flows capture uncertainty at the output layer of regressors and at the latent layer of VAEs, visualizing how density distributions shift with strategic intent and recovery capability within the game environment.}
  \label{fig:motivation_env_corner}
\end{figure}

Leveraging the Spine and Cloud normalizing flow framework introduced previously, we investigate the predictive uncertainty of our world models when subjected to varied transition dynamics. This analysis allows us to characterize the model's reliability not merely through aggregate performance metrics, but through the lens of distributional density estimation under duress. By systematically imposing interventional failure modes, we rigorously evaluate how different architectures internalize epistemic uncertainty and recovery dynamics when pushed beyond their training distribution.

We derive four distinct tactical failure modes: Winning (Figure~\ref{fig:motivation_env_win}), Blocking (Figure~\ref{fig:motivation_env_block}), Center control (Figure~\ref{fig:motivation_env_center}), and Corner occupancy (Figure~\ref{fig:motivation_env_corner}) by introducing adversarial interventions (\textcolor{red}{OOD} states) that force the agent to deviate from optimal play. Each failure mode is presented as a separate figure so that the per-architecture density panels remain legible. As visualized in Figures~\ref{fig:motivation_env_win}--\ref{fig:motivation_env_corner}, we employ a consistent color-coding scheme: \textcolor{green}{green} represents optimal \textcolor{green}{ID} state transitions, \textcolor{red}{red} denotes \textcolor{red}{OOD} adversarial interventions, and \textcolor{blue}{blue} illustrates the \textcolor{blue}{recovery} dynamics back to the ground truth. We operationalize uncertainty modeling based on the underlying architecture: for regressor-based world models (\texttt{CNN1D}, \texttt{LSTM\_CNN1D}), we visualize the final output layer as a permutation-invariant proxy for transition uncertainty; conversely, for VAE-based models (e.g., \texttt{Causal\_Residual\_VAE}), we analyze the latent layer, allowing for a more granular interpretation of how structural causal features are encoded.

A robust uncertainty model is evaluated against a specific rubric: a \textit{High-Fidelity Model} exhibits a sharp, singular density peak for \textcolor{green}{ID} transitions, with uncertainty variance expanding selectively on dimensions that underwent intervention, ensuring optimal disentanglement; in contrast, a \textit{Model Collapse} manifests as diffuse, multimodal uncertainty, or sharp, impulse-like spikes signaling overfitting. Empirically, \textcolor{green}{green} bands must be smooth and well-calibrated; \textcolor{red}{red} bands should exhibit variance along action dimensions and the intervention grid ID; and \textcolor{blue}{recovery} bands should demonstrate wider uncertainty specifically on intervened cells that deviate from the optimal policy's induced distribution.

Under reasonable data complexities (100--1000 episodes), our evaluation identifies two distinct winner architectures, each excelling in a specific diagnostic domain. The \texttt{Causal\_CNN1D} emerges as the winner for \textit{global diagnostic stability}. Its permutation-invariant architecture provides a robust, low-latency estimate of transition uncertainty across the entire grid, proving exceptionally reliable for rapid detection of \textcolor{red}{OOD} deviations where broad situational awareness is required. Conversely, the \texttt{Causal\_Residual\_VAE} is the winner for \textit{structural diagnostic precision}. While the regressor architectures provide consistent global peaks, the VAE’s latent space excels in complex failure modes like \textit{Center} and \textit{Corner} occupancy. It localizes uncertainty to the specific causal mask inaccuracies, providing a structural map of the model's internal causal alignment that global regressors aggregate away. 

These results demonstrate that the choice of architecture defines the diagnostic utility: the \texttt{Causal\_CNN1D} provides the robust view of transition uncertainty, while the \texttt{Causal\_Residual\_VAE} serves as the surgical diagnostic tool for isolating local causal failures. Together, they form a comprehensive uncertainty modeling pipeline, confirming that neither architecture is universally superior, but rather that their integration provides the most resilient defense against interventional uncertainty in world models.

The following table \ref{tab:train_config}, table  \ref{tab:causal_config} are the experimental details of the motivation section: 
\begin{table}[htbp]
\centering
\begin{minipage}[t]{0.48\textwidth}
    \centering
    \small
    \caption{Training Hyperparameters}
    \label{tab:train_config}
    \begin{tabular}{l l}
    \toprule
    \textbf{Parameter} & \textbf{Value} \\
    \midrule
    Exp. Rate ($p$) & $\{0.0, 0.25, 0.5, 0.75, 1.0\}$ \\
    Dataset Size ($N$) & $\{10, 100, 1000, 10^6\}$ \\
    Batch Size & 128 \\
    Optimizer & Adam ($1 \times 10^{-3}$) \\
    Epochs & 20 (Det), 50 (Var) \\
    \bottomrule
    \end{tabular}
\end{minipage}
\hfill
\begin{minipage}[t]{0.48\textwidth}
    \centering
    \small
    \caption{PC/PCMCI Parameters}
    \label{tab:causal_config}
    \begin{tabular}{l l}
    \toprule
    \textbf{Parameter} & \textbf{Value} \\
    \midrule
    Lag Range ($\tau$) & $\tau \in \{1\}$ \\
    Significance ($\alpha$) & 0.05 \\
    Sigmoid ($\kappa$) & 1.0 \\
    Epsilon ($\epsilon$) & $10^{-8}$ \\
    Cond. Limit & Py: 3, Px: 0 \\
    \bottomrule
    \end{tabular}
\end{minipage}
\end{table}

Input vector $\mathbf{x} \in \mathbb{R}^{11}$ (board state, action, player ID) mapped to target $\mathbf{y} \in \mathbb{R}^{9}$.

\begin{itemize}
    \item \textbf{CNN1D / LSTM\_CNN1D:} Feature extraction via Conv1D ($32 \rightarrow 16$ channels). LSTM variants apply recurrent processing prior to feature extraction.
    \item \textbf{ShallowVAE:} Linear Encoder ($11 \rightarrow 64 \rightarrow 22$) $\rightarrow$ Latent ($11$) $\rightarrow$ Decoder ($11 \rightarrow 64 \rightarrow 9$). Loss: MSE + KLD.
    \item \textbf{Causal Variants:} \texttt{Causal\_CNN1D} and \texttt{Causal\_LSTM\_CNN1D} utilize a ground-truth transpose adjacency mask $\mathbf{A}$ for input modulation.
    \item \textbf{CausalResidualVAE:} VAE integrated with MLP residual blocks; utilizes a linear warm-start scheduler ($0.25$ ratio) to impose the causal prior.
\end{itemize}

\section{Generalization Error Bound of the ICWM}
\label{sec:generalization_bound_appendix}

To formalize the necessity of stochastic soft interventions, we establish a theoretical bound on the out-of-distribution policy evaluation error. Generalization error is driven by the divergence between the true transition model $\mathcal{M}^*$ and the learned model $\mathcal{M}_\theta$. Because $\mathcal{M}_\theta$ relies on the MAS-SBC adjustment to decouple invariant physics from strategic intent, its accuracy is bottlenecked by the causal support in the demonstration dataset $\mathcal{D}_\sigma$, quantified by the interventional strength $\sigma \in (0, 1]$.

\begin{lemma}[Transition Estimation Error]
\label{lemma:transition_error}
Let $\mathcal{D}_\sigma$ be a dataset of $N$ transitions collected under the soft interventional policy mixture $P^\sigma$. The worst-case total variation (TV) distance ($\text{Distance}_{TV}$) between the learned causal transition $\mathcal{M}_\theta$ and the true mechanism $\mathcal{M}^*$ is bounded by:
\begin{equation}
\max_{s, \mathbf{a}} \text{Distance}_{TV}(\mathcal{M}_\theta(\cdot \mid s, \mathbf{a}) \parallel \mathcal{M}^*(\cdot \mid s, \mathbf{a})) \le \frac{J}{\sqrt{N \sigma}} + \epsilon_{env}
\end{equation}
where $J$ is a constant dependent on the state-action space complexity, and $\epsilon_{env}$ represents irreducible environmental stochasticity and neural network approximation error.
\end{lemma}

\begin{proof}
The learning of $\mathcal{M}_\theta$ relies on the finite observational dataset $\mathcal{D}_\sigma$ generated by the mixture policy $P^\sigma(\mathbf{a} \mid s) = (1 - \sigma)\pi_{exp}(\mathbf{a} \mid s) + \sigma \zeta(\mathbf{a})$. We assume the stochastic process $\zeta$ establishes a uniform distribution over the joint discrete action space $\mathcal{A}$. 

We first prove that the minimum probability of observing any specific joint action $\mathbf{a} \in \mathcal{A}$ in any visited state is strictly bounded from below. Since probability distributions are non-negative, the expert policy's contribution is $(1 - \sigma)\pi_{exp}(\mathbf{a} \mid s) \ge 0$. Because $\zeta$ is a uniform distribution over the joint action space, $\zeta(\mathbf{a}) = \frac{1}{|\mathcal{A}|}$ for all valid actions. Substituting these limits into the mixture policy yields:
\begin{equation}
\min_{\mathbf{a}, s} P^\sigma(\mathbf{a} \mid s) \ge 0 + \sigma \left(\frac{1}{|\mathcal{A}|}\right) = \frac{\sigma}{|\mathcal{A}|}
\end{equation}

Let $N(s, \mathbf{a})$ represent the exact number of visits to a specific state and joint-action tuple. To ensure the dataset $\mathcal{D}_\sigma$ has a non-zero probability of covering physically reachable states, we require the underlying Markov Decision Process to be ergodic. An ergodic MDP guarantees that every state is visited infinitely often in the limit, establishing a strictly positive stationary state distribution $d(s)$. We define this lower bound as $d_{min} = \min_{s} d(s) > 0$. Under this ergodicity condition, the expected dataset support for any valid transition is bounded by \citep{kakade2003sample}:
\begin{equation} \label{eq::expectation}
\mathbb{E}[N(s, \mathbf{a})] \ge N d_{min} \left( \frac{\sigma}{|\mathcal{A}|} \right)
\end{equation}

To bound the divergence of the empirical distribution $\hat{P}$ from the true distribution $\mathcal{M}^*$, we apply Weissman's $L_1$ concentration inequality for discrete probability distributions \citep{Weissman2003InequalitiesFT}. For a probability distribution defined over a state space of size $|\mathcal{S}|$, the probability that the $L_1$ deviation exceeds a threshold $\epsilon$ is bounded by:
\begin{equation}
P(|| \hat{P} - \mathcal{M}^* ||_1 \ge \epsilon) \le (2^{|\mathcal{S}|} - 2) \exp\left(-\frac{N(s, \mathbf{a}) \epsilon^2}{2}\right)
\end{equation}

Converting the $L_1$ norm to Total Variation distance ($\text{Dist}_{TV} = \frac{1}{2} || \cdot ||_1$) and applying a union bound over all $|\mathcal{S}||\mathcal{A}|$ state-action pairs, the maximum $\text{Dist}_{TV}$ with probability at least $1 - \delta$ is:
\begin{equation}
\max_{s, \mathbf{a}} \text{Dist}_{TV}(\hat{P}(\cdot \mid s, \mathbf{a}) \parallel \mathcal{M}^*(\cdot \mid s, \mathbf{a})) \le \sqrt{\frac{1}{2 N(s, \mathbf{a})} \log\left(\frac{2^{|\mathcal{S}|} |\mathcal{S}| |\mathcal{A}|}{\delta}\right)}
\end{equation}

Substituting the expected support $\mathbb{E}[N(s, \mathbf{a})]$ (Equation \ref{eq::expectation}) into the concentration bound yields the purely statistical error. To account for optimization and capacity limitations of the neural network $\mathcal{M}_\theta$, we use the triangle inequality:

\begin{equation}
\text{Dist}_{TV}(\mathcal{M}_\theta \parallel \mathcal{M}^*) \le 
\underbrace{\text{Dist}_{TV}(\mathcal{M}_\theta \parallel \hat{P})}_{\substack{\text{Neural Architecture} \\ \text{\& Optimization Error}}} 
+ 
\underbrace{\text{Dist}_{TV}(\hat{P} \parallel \mathcal{M}^*)}_{\substack{\text{Empirical Dataset vs.} \\ \text{True Model Distance}}}
\end{equation}

Defining $\epsilon_{env} = \text{Dist}_{TV}(\mathcal{M}_\theta \parallel \hat{P})$ as the irreducible approximation error, and aggregating the dimensionality constraints and confidence bounds into a complexity constant $J = \sqrt{\frac{|\mathcal{A}|}{2 d_{min}} \log\left(\frac{2^{|\mathcal{S}|} |\mathcal{S}| |\mathcal{A}|}{\delta}\right)}$, we recover the final bound:
\begin{equation}
\max_{s, \mathbf{a}} \text{Dist}_{TV}(\mathcal{M}_\theta(\cdot \mid s, \mathbf{a}) \parallel \mathcal{M}^*(\cdot \mid s, \mathbf{a})) \le \frac{J}{\sqrt{N \sigma}} + \epsilon_{env}
\end{equation}
\end{proof}

\begin{theorem}[ICWM Generalization Bound]
\label{theorem:generalization}
Given a reward function bounded by $R_{max}$ and discount factor $\gamma \in [0, 1)$, the error in evaluating any arbitrary target policy $\pi_{test}$ using the learned ICWM $\mathcal{M}_\theta$ compared to the true environment $\mathcal{M}^*$ is bounded by:
\begin{equation}
\left| V_{\pi_{test}}^{\mathcal{M}_\theta} - V_{\pi_{test}}^{\mathcal{M}^*} \right| \le \frac{\gamma R_{max}}{(1-\gamma)^2} \left( \frac{J}{\sqrt{N \sigma}} + \epsilon_{env} \right)
\end{equation}
\end{theorem}

\begin{proof}
We define the value function $V_{\pi}^{\mathcal{M}}(s)$ as the expected discounted cumulative reward starting from state $s$ and following policy $\pi$ under transition dynamics $\mathcal{M}$:
\begin{equation}
V_{\pi}^{\mathcal{M}}(s) = \mathbb{E}_{\pi, \mathcal{M}} \left[ \sum_{t=0}^\infty \gamma^t r_t \mid s_0 = s \right]
\end{equation}

We bound the difference in the expected discounted return of an arbitrary target policy $\pi_{test}$ evaluated in the true environment $\mathcal{M}^*$ versus the learned causal world model $\mathcal{M}_\theta$. We invoke the Simulation Lemma \citep{Kearns2002NearOptimalRL}, which systematically bounds multi-step value function discrepancies using the maximum single-step transition divergence.

By applying the Performance Difference Lemma \citep{10.5555/645531.656005,kakade2003sample}, the difference in value functions can be expanded as a telescoping sum over the state-action distribution induced by the policy:
\begin{equation}
V_{\pi_{test}}^{\mathcal{M}_\theta}(s) - V_{\pi_{test}}^{\mathcal{M}^*}(s) = \frac{1}{1-\gamma} \mathbb{E}_{(s, \mathbf{a}) \sim d^{\mathcal{M}_\theta}_{\pi_{test}}} \left[ \gamma \sum_{s'} \left( \mathcal{M}_\theta(s' \mid s, \mathbf{a}) - \mathcal{M}^*(s' \mid s, \mathbf{a}) \right) V_{\pi_{test}}^{\mathcal{M}^*}(s') \right]
\end{equation}

Because the maximum single-step reward is bounded by $R_{max}$, the maximum possible accumulated value of any state over an infinite horizon is the geometric sum $R_{max} \sum_{t=0}^{\infty} \gamma^t$. For $\gamma \in [0, 1)$, this sum converges exactly to $\frac{R_{max}}{1-\gamma}$. We bound the inner summation using this maximum value constraint and the Total Variation distance (where $\sum_{s'} |P(s') - Q(s')| \cdot V(s') \le 2 \text{Dist}_{TV}(P \parallel Q) \max V$):
\begin{equation}
\left| V_{\pi_{test}}^{\mathcal{M}_\theta} - V_{\pi_{test}}^{\mathcal{M}^*} \right| \le \frac{\gamma R_{max}}{(1-\gamma)^2} \max_{s, \mathbf{a}} \text{Dist}_{TV}(\mathcal{M}_\theta(\cdot \mid s, \mathbf{a}) \parallel \mathcal{M}^*(\cdot \mid s, \mathbf{a}))
\end{equation}

Substituting the bounded transition error established in Lemma \ref{lemma:transition_error} directly into this inequality yields:
\begin{equation}
\left| V_{\pi_{test}}^{\mathcal{M}_\theta} - V_{\pi_{test}}^{\mathcal{M}^*} \right| \le \frac{\gamma R_{max}}{(1-\gamma)^2} \left( \frac{J}{\sqrt{N \sigma}} + \epsilon_{env} \right)
\end{equation}
This concludes the proof.
\end{proof}

\subsection{Discussion of Proven Implications}
The mathematical derivations in Theorem \ref{theorem:generalization} establish three core theoretical implications regarding the identifiability of causal world models:

\begin{enumerate}[leftmargin=*,noitemsep,topsep=0pt]
    \item \textbf{Unidentifiability in Deterministic Regimes:} The presence of $\sigma$ in the denominator of the bound dictates that as $\sigma \to 0$ (a purely deterministic expert), the estimation error mathematically approaches infinity. In this zero-entropy regime, expert actions are perfectly collinear with the unobserved latent intent $U$, rendering the true physical transition function $f_s$ structurally unidentifiable due to the lack of counterfactual support.
    \item \textbf{Bounding the Planning Snowball Effect:} Multi-step planning algorithms compound single-step transition errors over time, represented by the $\frac{\gamma}{(1-\gamma)^2}$ multiplier. By ensuring $\sigma > 0$, the soft intervention guarantees sufficient causal positivity across the joint state-action space, successfully constraining the maximum single-step deviation $\text{Dist}_{TV}$ and capping the compounding multi-step error.
    \item \textbf{Diminishing Returns on Stochasticity:} The bound shrinks proportionally to the square root of the interventional strength ($\frac{1}{\sqrt{\sigma}}$) but is ultimately floored by $\epsilon_{env}$. This indicates that while small amounts of initial policy variance provide massive reductions in generalization error by breaking causal confounding, adding excessive noise eventually yields diminishing returns against the irreducible approximation limits of the neural architecture.
\end{enumerate}

These theoretical implications are directly validated by our empirical results in Section \ref{sec::experimentation}. Specifically, the observed system error rapidly decreases as $\sigma$ is introduced but plateaus or degrades at extreme values (e.g., $\sigma \ge 0.75$), confirming the mathematical dynamics established in this bound.

\section{Theoretical Guarantees and Algorithmic Selection for Causal Discovery}
\label{sec:appendix_causal_discovery}

In this section \ref{sec::csd}, we provide the theoretical justifications for our causal structure discovery methodology. First, we provide a formal proof demonstrating that the elimination of contemporaneous edges ($E_{contemp}$) is a principled operation that does not violate the completeness guarantees of the true physical dynamics. Second, we rigorously justify the selection of PCMCI+ over latent causal discovery algorithms such as LPCMCI \citep{reiser2022causaldiscoverytimeseries}, grounding this decision in the architectural constraints of our Implicit Causal World Model (ICWM).

\subsection{Completeness of the Lagged Causal Skeleton in PCMCI+}
\label{sec:pcmci_completeness}

In Section \ref{sec::csd}, we define the unconfounded causal skeleton $\mathcal{G}_{skeleton}$ by strictly discarding the contemporaneous edge set $E_{contemp}$ generated by the PCMCI+ algorithm. A critical theoretical concern is whether the exclusion of $E_{contemp}$ inadvertently drops true physical dependencies, particularly if the algorithm's minimality principle prioritizes contemporaneous correlations over lagged physical mechanisms. We demonstrate that under the structural assumptions of our multi-agent framework, the lagged structural edges ($E_{lag}$) are preserved and independently satisfy the completeness guarantee for the invariant transition function $f_s$.

\subsubsection{Theoretical Setup and the MCI Test}
Let $\mathbf{V}$ be the set of observable endogenous variables evolving as a Vector Autoregressive (VAR) process. PCMCI+ identifies the causal graph through the Momentary Conditional Independence (MCI) test. For a putative lagged edge $V^i_{t-\tau} \to V^j_t$ with $\tau \ge 1$, the MCI test evaluates:
\begin{equation} \label{eq:mci_test}
V^i_{t-\tau} \perp \!\!\! \perp V^j_t \mid \widehat{Pa}(V^j_t) \setminus \{V^i_{t-\tau}\} \cup \widehat{Pa}(V^i_{t-\tau})
\end{equation}
where $\widehat{Pa}(\cdot)$ denotes the set of preliminary parents identified in the initial PC phase.

In our MAS formulation, the true Data Generating Process or SCM $\mathcal{M}$ is governed by two distinct mechanisms:
\begin{enumerate}[leftmargin=*,noitemsep,topsep=0pt]
    \item \textbf{Physics ($f_s$):} Governs invariant, lagged structural edges ($\tau \ge 1$).
    \item \textbf{Strategic Intent ($u_{intent}$):} Acts as an unobserved contemporaneous confounder ($\tau = 0$), inducing spurious correlations between $V^i_t$ and $V^j_t$.
\end{enumerate}

\subsubsection{Proof of Lagged Skeleton Completeness}

\textbf{Claim:} If a true physical mechanism $V^i_{t-\tau} \to V^j_t$ exists, the presence of the unobserved contemporaneous confounder $u_{intent}$ will not cause PCMCI+ to d-separate and drop the lagged edge, ensuring $E_{lag}$ is structurally complete for $f_s$.

\begin{proof}
Suppose a true invariant physical edge $V^i_{t-\tau} \to V^j_t$ exists. For PCMCI+ to incorrectly drop this edge, it must find a separating set $\mathbf{S} \subset \mathbf{V}$ such that $V^i_{t-\tau} \perp \!\!\! \perp V^j_t \mid \mathbf{S}$.

Under the minimality principle of causal discovery, if a contemporaneous variable $V^k_t$ perfectly mediated the relationship (i.e., $V^i_{t-\tau} \to V^k_t \to V^j_t$), conditioning on $V^k_t$ would correctly d-separate the lagged cause from the effect. However, we must evaluate the structural paths created by the unobserved intent.

By Assumption \ref{assum:contemp_confounding}, $u_{intent}$ coordinates actions strictly at time $t$. This creates a spurious contemporaneous path: $V^i_t \leftarrow u_{intent} \rightarrow V^j_t$. Because $u_{intent}$ is unobserved, PCMCI+ detects this correlation and instantiates a contemporaneous edge in $E_{contemp}$ (e.g., $V^i_t \leftrightarrow V^j_t$ or $V^i_t \to V^j_t$). We must determine if conditioning on this contemporaneous relationship d-separates the lagged physical edge. It does not, due to two structural axioms:
\begin{enumerate}[leftmargin=*,noitemsep,topsep=0pt]
    \item \textbf{Independence of Physics and Intent:} The physical transition $V^i_{t-\tau} \to V^j_t$ is an autonomous mechanism governed by $f_s$. The strategic intent $u_{intent}$ at time $t$ does not alter the physical laws defined at time $t-\tau$. Therefore, the conditional mutual information $I(V^i_{t-\tau} ; V^j_t \mid \mathbf{S})$ remains strictly positive for any subset $\mathbf{S}$ containing contemporaneous variables.
    \item \textbf{Collider Bias Avoidance:} If $V^i_{t-\tau}$ and $u_{intent}$ both causally influence $V^j_t$, then $V^j_t$ acts as a collider. Conditioning on $V^j_t$ (or its contemporaneous siblings) does not block the path from $V^i_{t-\tau}$; rather, it opens spurious paths. The MCI test (Equation \ref{eq:mci_test}) is mathematically designed to condition on the parents of both the cause and effect, isolating the specific structural link and preventing collider bias.
\end{enumerate}

Because no subset of contemporaneous variables generated by $u_{intent}$ can d-separate the true lagged physical mechanism, the MCI test will return a significant dependency. The true physical edge is strictly preserved in $E_{lag}$.
\end{proof}

\subsubsection{A Principled Heuristic Backed by Causal Decision Graphs}
\begin{figure}[!htb]
  \centering
  \includegraphics[trim={0pt 400pt 600pt 0pt}, clip, width=0.9\linewidth]{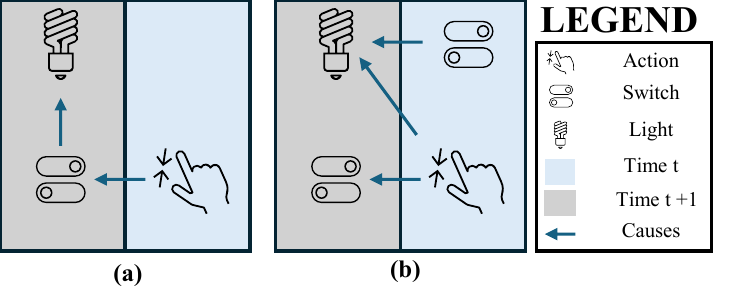}
  \caption{Resolving spurious contemporaneous correlation into temporally extended structural mechanisms. (a) A naive causal representation exhibiting a contemporaneous dependency ($E_{contemp}$). The switch appears to instantaneously cause the light to illuminate within the same time step $t$, a structural artifact typical of confounded multi-agent coordination. (b) The correct structural formulation derived via our CDG-backed heuristic. The invariant physical mechanism ($f_s$) operates with a strict temporal lag ($\tau \ge 1$), demonstrating that the state of the light at $t+1$ is fully determined by the parent state (the switch) at time $t$. This physical grounding illustrates why true environmental dynamics are strictly captured within the lagged edge set ($E_{lag}$), justifying the elimination of contemporaneous edges to recover the unconfounded causal skeleton.}\label{fig:augmented_cause}
\end{figure}
Because the true physical edges are mathematically guaranteed to remain in $E_{lag}$ (Figure \ref{fig:augmented_cause}), the contemporaneous edges $E_{contemp}$ identified by PCMCI+ contain strictly spurious correlations driven by $u_{intent}$ (as well as unresolvable Markov equivalence structures caused by the unobserved confounding). 

We formalize the removal of $E_{contemp}$ not as arbitrary graph surgery, but as a principled heuristic backed by Causal Decision Graphs (CDGs). In the CDG framework applied to multi-agent systems, environmental mechanisms ($f_s$) are defined by temporally extended structural equations, whereas joint policy execution (driven by $u_{intent}$) induces contemporaneous dependencies exclusively between decision nodes at time $t$. 

Therefore, defining the final causal skeleton as $E_{skeleton} = \{ (V_{t-k} \to V_{t}) \in E_{lag} \mid 1 \le k \le \tau \}$ is a domain-informed heuristic derived directly from the CDG of the multi-agent environment. This heuristic safely eliminates the unblocked backdoor paths associated with joint intent without violating the completeness guarantee of the VAR process underlying the physical environment.

\subsection{Algorithmic Selection: PCMCI+ vs. Latent Discovery (LPCMCI)}
\label{sec:pcmci_vs_lpcmci}

While the standard PCMCI+ algorithm operates under the assumption of causal sufficiency, the multi-agent system explicitly contains an unobserved confounder. In standard causal discovery literature, the Latent PCMCI (LPCMCI) algorithm is typically deployed to handle such hidden variables. However, we actively choose PCMCI+ combined with strict structural assumptions over LPCMCI. This decision is driven by the strict architectural requirements of our neural world model and the necessity of forward reasoning using known physical priors.

\subsubsection{Architectural Infeasibility of Partial Ancestral Graphs (PAGs)}
The Sequential Backdoor Condition (SBC) requires a deterministic, fully resolved directed graph to execute causal adjustments. Our Implicit Causal World Model (ICWM) instantiates this graph as a causal tensor mask $\mathbf{C}$ to restrict information flow in the neural network. 

LPCMCI outputs a Partial Ancestral Graph (PAG), which explicitly models latent confounding using bidirected edges ($\leftrightarrow$) and marks Markov equivalence ambiguities with circle edges (e.g., $o \rightarrow$). Translating an ambiguous PAG \citep{10.1214/aos/1031689015} into a deterministic neural architecture is mathematically infeasible without relying on arbitrary, heuristic tie-breaking for unresolved edges. By utilizing PCMCI+ and filtering out $E_{contemp}$, we obtain a strictly directed graph ($E_{lag}$) that maps flawlessly into our causal tensor mask.

\subsubsection{Continuous Edge Integration via Fuzzy Masking}
To operationalize the resolved directed graph ($E_{lag}$) within the neural architecture, we adjust the retrieved edges using a fuzzy mask. Rather than enforcing a rigid binary constraint, we instantiate the causal tensor $\mathbf{C}$ with continuous values $\in [0, 1]$. This fuzzy masking allows the ICWM to softly condition on the discovered causal parents. It preserves the full differentiability of the neural network during training while natively accommodating the statistical uncertainty of the dependencies extracted by PCMCI+.

\section{Contextual Bandit Formulation of Decision-Making}
\label{sec::bandit}
We frame decision-making as a contextual bandit \citep{kumar2008mortal} task, where the environment state constitutes the context and the available legal actions are the arms of the bandit. Unlike traditional reinforcement learning, which optimizes for cumulative reward, our agent selects actions to maximize information gain regarding the environment’s transition dynamics. By treating each move as a local intervention, the agent probes the causal structure of the environment. The SBC (Appendix \ref{sec::sbc}) or MAS-SBC (Appendix \ref{sec::mas_sbc_proof}) then acts as a regulatory constraint, such that the effect of the agent's actions can be identified with trajectory data . This prevents selection bias, allowing the world model to learn the true environmental transition function rather than a policy-dependent approximation.

To transition from uniform random exploration to an information-maximizing framework within our mixture policy $\pi_{\text{behavior}} = \sigma \pi_{\text{int}} + (1-\sigma) \pi_{\text{expert}}$ equation \ref{eq:mixture_appendix}, we redefine $\pi_{\text{int}}$ as an active learning agent \citep{bellemare2016unifying,pathak2017curiosity}.

\subsection{The Intractability of Expected Information Gain}
Directly maximizing the Expected Information Gain (EIG) objective,
\begin{equation} \label{eq::active_agent_objective}
    \pi_{\text{int}}(S_t) = \arg\max_{A_t \in \mathcal{A}} \mathbb{E}_{S_{t+1} \sim P(\cdot|S_t, A_t)} \left[ \text{KL} \left( P(\theta_{world\_model} | S_t, A_t, S_{t+1}) \| P(\theta_{world\_model}) \right) \right],
\end{equation} 
where $\theta_{world\_model}$ are world model parameter.
Equation \ref{eq::active_agent_objective} is computationally intractable for complex world models due to two primary factors:
\begin{itemize}
    \item \textbf{Posterior Inference:} Bayesian updating over high-dimensional neural parameters is analytically impossible \citep{tong2001active} and prohibitively expensive via MCMC \citep{andrieu2003introduction} in real-time training.
    \item \textbf{Forward Simulation:} Estimating the expectation requires accurate future rollouts, creating a circular dependency where a reliable model is required to learn the model.
\end{itemize}

\subsection{Surrogate Objectives}
We can (as future work) approximate EIG using tractable epistemic uncertainty metrics:

\begin{itemize}
    \item \textbf{Predictive Variance (Uncertainty Sampling):} Utilizing our Spine and Cloud flows, we define 
    \begin{equation}
        \pi_{\text{int}}(S_t) = \arg\max_{A_t} \text{Var}_{\text{Cloud}}(S_t, A_t).
    \end{equation} 
    High residual variance indicates model confusion; prioritizing these actions forces data collection in regions where the model’s internal representation is weakest.
    
    \item \textbf{Ensemble Disagreement (Query-by-Committee):} Maintaining an ensemble of $N$ models, we define 
    \begin{equation}
        \pi_{\text{int}}(S_t) = \arg\max_{A_t} \frac{1}{N} \sum_{i=1}^N \| \mu_i(S_t, A_t) - \bar{\mu}(S_t, A_t) \|^2.
    \end{equation} 
    High disagreement across ensemble \citep{freund1992information} predictions highlights frontier states where the environment's causal structure is not yet resolved.
\end{itemize}


\section{MAS Environment Visualization}
\label{app:environment}
The Following figure \ref{fig:env} represents the three important domains chosen for evaluation. 

\begin{figure}[h]
  \centering
  \includegraphics[trim={0pt 220pt 350pt 0pt}, clip, width=\linewidth]{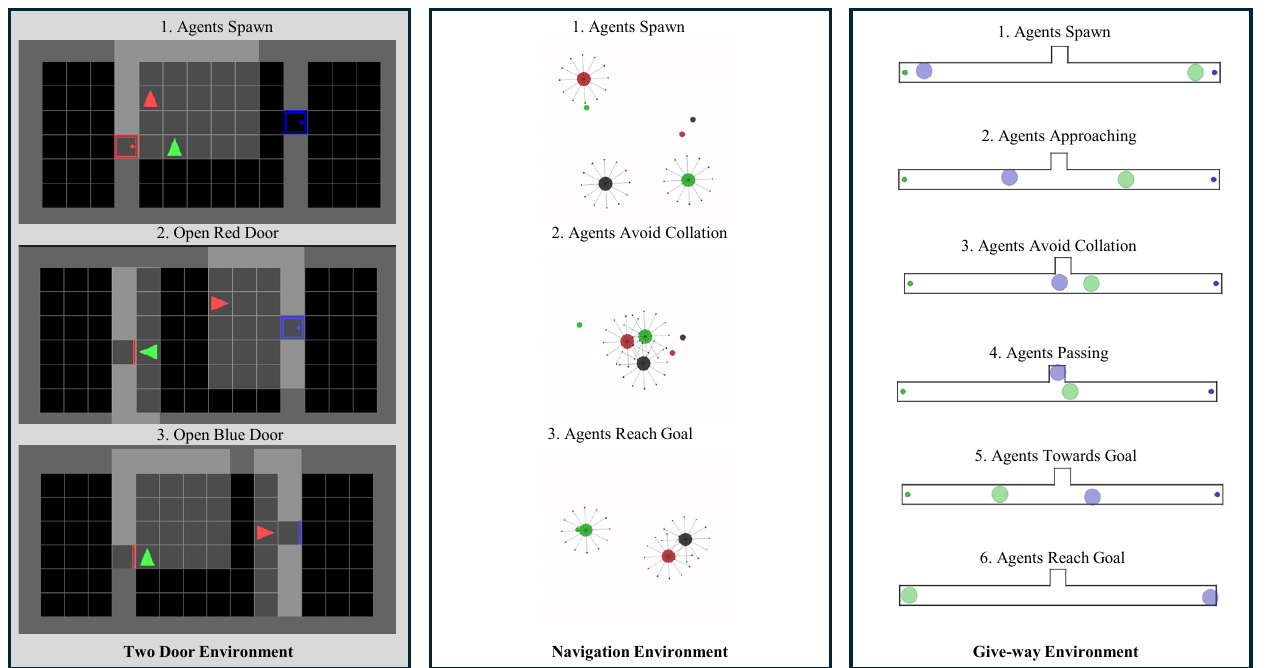}
  \caption{Environments used for empirical experiments.}\label{fig:env}
\end{figure}

\section{State Representation of Different environments}
\label{app:state_rep}
\subsection{Two-Door Coordination}
\begin{itemize}[leftmargin=*,noitemsep,topsep=0pt]
    \item \textbf{Variables}: Let $P_i = [x_i, y_i, d_i]$ be the state of agent $i \in \{R, G\}$ and $D_j = [s_j, x_j, y_j]$ be the state of door $j \in \{Red, Blue\}$, where $s_j \in \{0, 1\}$ denotes door status.
    
    \item \textbf{Full Observability (MDP)}: The state tensor $V_{\text{full}}$ is a direct concatenation of all global variables, providing the model with the complete physical configuration:
    $V_{\text{full}} = [P_R, D_{Red}, D_{Blue}, P_G]$

    \item \textbf{Visibility Logic}: Information is gated by a visibility indicator $\mathbb{1}_{i \to j}$ based on a $3 \times 3$ directional Field of View (FOV):
    $$\mathbb{1}_{i \to j} = \begin{cases} 1 & \text{if } (x_j, y_j) \in \text{FOV}(P_i, d_i) \\ 0 & \text{otherwise} \end{cases}$$

    \item \textbf{Partial Observability (POMDP)}: The observation tensor for the Red agent $V_{\text{part}, R}$ masks the peer agent's features when outside the FOV:
    $V_{\text{part}, R} = [P_R, D_{Red}, D_{Blue}, P_G \cdot \mathbb{1}_{R \to G}]$
    
\end{itemize}

\subsection{Giveway Corridor}
\begin{itemize}[leftmargin=*,noitemsep,topsep=0pt]\item \textbf{Variables}: Let $P_i = [x_i, y_i]$ and $V_i = [\dot{x}_i, \dot{y}_i]$ represent the absolute position and velocity of agent $i \in \{0, 1\}$. Let $C= [ |x_1 -x_2|,|y_1 - y_2|  ]$ define relative separation between agents.\item \textbf{Full Observability (MDP)}: The state tensor $V_{\text{full}}$ is a direct concatenation of kinematic data for all agents in the corridor:
$$V_{\text{full}} = [P_0, V_0, P_1, V_1,C]$$

\item \textbf{Logic for Noise}: Observational uncertainty is modeled as an additive exogenous variable $\epsilon$, where noise is sampled from a uniform distribution $\mathcal{U}$ scaled by the noise magnitude $n=0.03$:
$$\epsilon \sim \mathcal{U}(-n, n)$$

\item \textbf{Noisy Observability}: The state tensor $V_{\text{noise}}$ incorporates these perturbations across all continuous observation channels to test model robustness:
$$V_{\text{noise}} = [P_0 + \epsilon_p, V_0 + \epsilon_v, P_1 + \epsilon_p, V_1 + \epsilon_v]$$
\end{itemize}
\subsection{Multi-Agent Navigation}
\begin{itemize}[leftmargin=*,noitemsep,topsep=0pt]\item \textbf{Variables}: Let $P_i = [x_i, y_i]$ and $V_i = [\dot{x}_i, \dot{y}_i]$ be the absolute position and velocity of agent $i$. Let $G_i = [x_{gi}, y_{gi}]$ be the static goal coordinates. Let $L_i \in \mathbb{R}^{12}$ be the 12-ray Lidar distance measurements representing proximity to other collidable entities.

\item \textbf{Full Observability (MDP)}: The state tensor $V_{\text{full}}$ provides the global fleet configuration, including all agent kinematics, goals, and the full Lidar array for obstacle/agent detection:
$$V_{\text{full}} = [\{P_i, V_i, G_i, L_i\}_{i=1}^n]$$

\item \textbf{Logic for Partial Observability}: This setup implements a **Goal-Masking Information Bottleneck**. While the global state $s_t$ contains all coordinates, the observation $o_{i,t}$ for agent $i$ is restricted. Specifically, the model sees the physical kinematics ($P, V$) and spatial constraints ($L$) of all agents, but the strategic targets (goals) of peer agents $j \neq i$ are obscured.

\item \textbf{Partial Observability (POMDP)}: The observation tensor $V_{\text{part}, i}$ is constructed to isolate agent $i$'s intent while treating peer agents as moving obstacles with hidden objectives. The peer goals $G_j$ are replaced by a null or masked value:
$$V_{\text{part}, i} = [P_i, V_i, G_i, \{P_j, V_j\}_{j \neq i}, L_i]$$

\end{itemize}

\section{Action Space Representation}
\label{app:action_space}
The scenarios evaluated in this study utilize two distinct action space paradigms: the continuous force-based control used in VMAS (Navigation and Give-Away) and the discrete, symbolic control used in MiniGrid (2-Door Navigation).

\subsection{Continuous Force-Based Action Space (VMAS)}
In VMAS, the action space enables precise, physics-based control. For each agent $i$ at time $t$, the action vector $\mathbf{u}_{i,t} \in \mathbb{R}^2$ is defined as:

\begin{equation}
    \mathbf{u}_{i,t} = \begin{bmatrix} u_{x} \\ u_{y} \end{bmatrix}
\end{equation}

where $u_x$ and $u_y$ represent the continuous force components applied along the global Cartesian axes. These vectors allow the policy to modulate both the direction and magnitude of movement in the Euclidean plane.

The control output is bounded by a norm constraint $\| \mathbf{u}_{i,t} \|_2 \leq u_{\text{range}}$. If the policy outputs a vector exceeding this threshold, the environment applies a radial clipping operator:

\begin{equation}
    \hat{\mathbf{u}}_{i,t} = \begin{cases} 
    \mathbf{u}_{i,t} & \text{if } \| \mathbf{u}_{i,t} \|_2 \leq u_{\text{range}} \\ 
    u_{\text{range}} \frac{\mathbf{u}_{i,t}}{\| \mathbf{u}_{i,t} \|_2} & \text{if } \| \mathbf{u}_{i,t} \|_2 > u_{\text{range}} 
    \end{cases}
\end{equation}

\subsection{Discrete Categorical Action Space (MiniGrid)}
In the MiniGrid 2-Door navigation environment, agents operate in a discrete, symbolic action space. These actions correspond to specific, predefined behavioral primitives. The action space $\mathcal{A}_{\text{MiniGrid}}$ is a categorical set, where the agent selects a single integer $a \in \{0, \dots, N-1\}$. The required dimension is effectively 1D, representing the chosen categorical index:

\begin{table}[h]
\centering
\begin{tabular}{lll}
\hline
\textbf{Action Index} & \textbf{Label} & \textbf{Description} \\ \hline
0 & turn\_left & Rotate the agent 90$^\circ$ counter-clockwise. \\
1 & turn\_right & Rotate the agent 90$^\circ$ clockwise. \\
2 & forward & Move the agent one cell forward. \\
3 & pickup & Pick up an object in front. \\
4 & drop & Drop a held object. \\
5 & toggle & Open doors or activate switches. \\
6 & done & Signal task completion. \\ \hline
\end{tabular}
\end{table}

\subsection{Comparative Summary}
\begin{itemize}
    \item \textbf{Dimensionality:} VMAS requires an $\mathbb{R}^2$ continuous vector (2 dimensions), whereas MiniGrid requires a categorical index $\mathbb{Z}_7$ (effectively 1 dimension with 7 discrete values).
    \item \textbf{Control Philosophy:} VMAS actions represent \textit{force-modulated control}, where the agent optimizes energy expenditure $E \propto \| \mathbf{u} \|_2^2$. MiniGrid actions represent \textit{intent-based navigation}, where state transitions are deterministic and independent of force magnitude. This difference fundamentally changes the learning objective from navigating a continuous energy manifold to optimizing a sequence of symbolic decision-making steps.
\end{itemize}

\section{Shifted environments}
\label{app:shifted-envs}
\begin{figure}[htbp]
  \centering
  \includegraphics[trim={0pt 50pt 30pt 0pt}, clip, width=\textwidth]{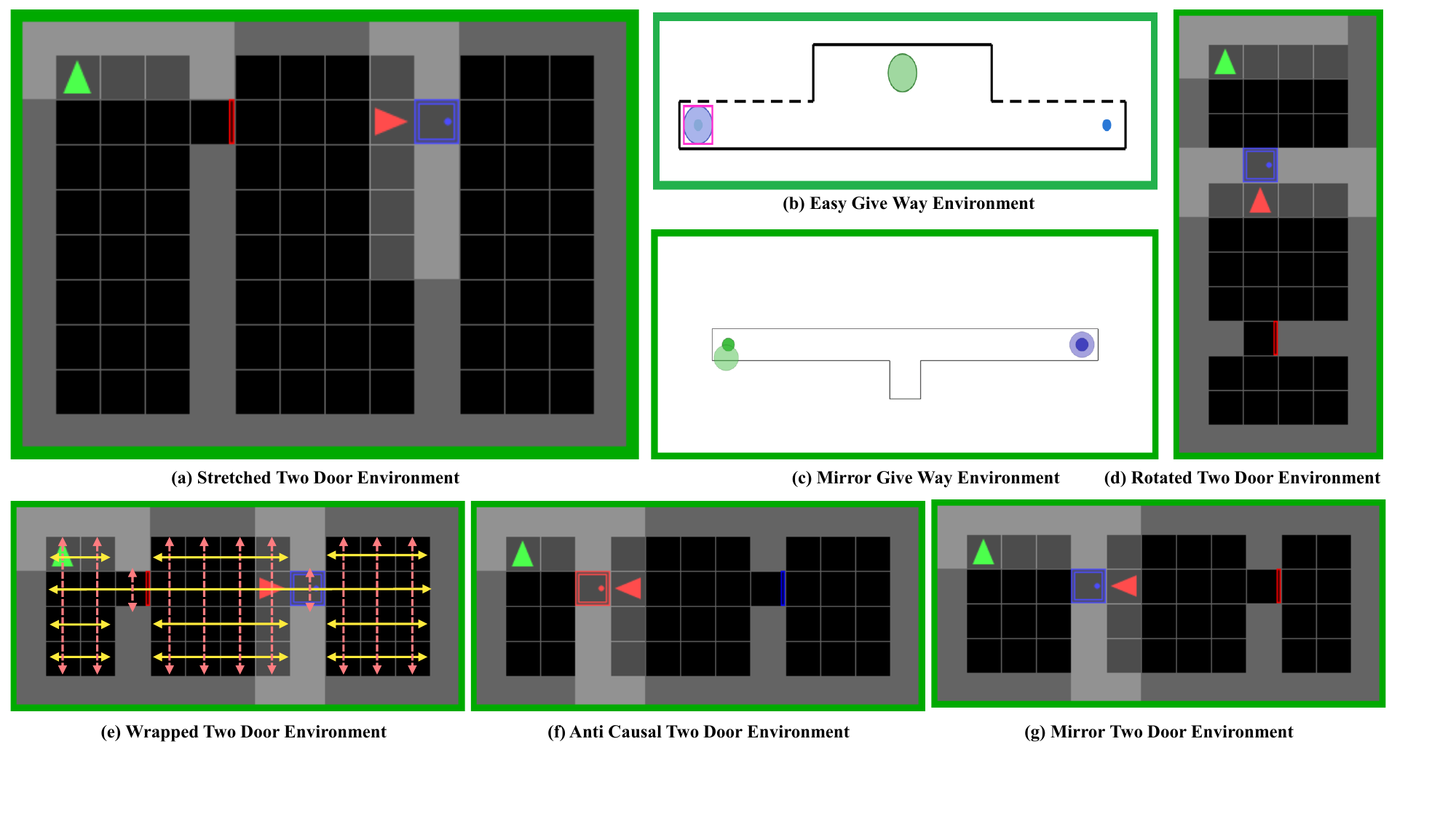}
  \caption{Example states from the seven shifted environments. Each keeps the dynamics of its default environment but changes the layout (or, for the anti-causal variant, the cue--dynamics mapping), so all seven are out of distribution for a model trained only on the defaults. Panels: (a) Stretched Two Door, (b) Easy Give Way, (c) Mirror Give Way, (d) Rotated Two Door, (e) Wrapped Two Door, (f) Anti-Causal Two Door, (g) Mirror Two Door.}
  \label{fig:shifted-envs}
\end{figure}
To test whether a world model has learned the true dynamics rather than the specific
layout it was trained on, we build a set of \emph{shifted} environments. Each one keeps
the same agents, the same physics, and the same task, but changes the geometry of the
arena. A model that understands how the agents move should transfer to these layouts,
while a model that has memorised the training layout should break. We train only on the
default environments and evaluate on the shifted ones, so every shifted layout is
strictly out of distribution.

Figure~\ref{fig:shifted-envs} shows one example state from each shifted environment. The
green and red triangles are the two agents in the grid (two-door) environments, the green
and blue circles are the two agents in the give-way environments, the blue box marks the
goal door, and the red bar marks the closed door.

\paragraph{How each layout differs from its default.}
The differences are summarised below and can be checked against
Fig.~\ref{fig:shifted-envs}. Panels are listed in the order they appear in the figure.

\begin{itemize}
  \item \textbf{Stretched Two Door} (Fig.~\ref{fig:shifted-envs}(a)). The two-door room
  is scaled wider, so the two corridors and the central room are longer than in the
  default. The doors, the walls, and the agents are unchanged; only the distances the
  agents must travel grow. This tests whether the model handles the same behaviour played
  out over a larger space.

  \item \textbf{Easy Give Way} (Fig.~\ref{fig:shifted-envs}(b)). The give-way corridor is
  widened and the passing bay is enlarged, so the two agents have more room to resolve the
  give-way manoeuvre than in the default. The task and the physics are unchanged; only the
  spatial margins grow. This tests whether the model preserves the give-way dynamics when
  the coordination is made geometrically easier.

  \item \textbf{Mirror Give Way} (Fig.~\ref{fig:shifted-envs}(c)). The give-way notch is
  flipped from above the corridor to below it. The corridor, the two goals, the agent
  radius, and the action range are all unchanged, so the only difference is that an agent
  now steps down into the bay instead of up. This tests whether the model has tied the
  give-way behaviour to a specific side of the corridor.

  \item \textbf{Rotated Two Door} (Fig.~\ref{fig:shifted-envs}(d)). The two-door room is
  rotated counter-clockwise by $90^\circ$. The wide, short default becomes tall and narrow, and the corridors
  that ran left and right now run top and bottom. Positions and motions that were horizontal
  in the default are vertical here, so the model must apply the same dynamics along a
  different axis.

  \item \textbf{Wrapped Two Door} (Fig.~\ref{fig:shifted-envs}(e)). The two-door arena is
  made toroidal: an agent that exits one edge re-enters from the opposite edge (indicated
  by the wrap-around arrows), so the room's boundaries connect rather than block. The room,
  doors, and agents are otherwise those of the default. This tests whether the model has
  learned local motion rules that still hold when the global connectivity of the arena
  changes.

  \item \textbf{Anti-Causal Two Door} (Fig.~\ref{fig:shifted-envs}(f)). The layout matches
  the default two-door room, but the cue-to-dynamics mapping is inverted: the relationship
  between the observed signal and the resulting transition is reversed relative to training.
  This is a dynamics shift rather than a geometric one, and it tests whether the model has
  learned the true causal mechanism rather than a cue that only correlates with it in the
  default distribution.

  \item \textbf{Mirror Two Door} (Fig.~\ref{fig:shifted-envs}(g)). The whole layout is
  reflected left to right. The room, the two doors, and the corridors move to the opposite
  side of the arena, so a fixed agent that sat on the right in the default now sits on the
  left. The geometry is identical to the default up to a mirror, which isolates whether
  the model has learned a direction-specific rule rather than the true symmetric dynamics.
\end{itemize}

\section{Ground Truth Causal Adjacency Matrices}
\label{app:gt_adjacency}
Figures \ref{fig:adj1}, \ref{fig:adj2}, \ref{fig:adj3} represent the ground truth SCM or system governing dynamics for our experimental setup. This can be sourced from their environment's code-base and transition dynamics. We redirect the reader to figure \ref{fig:cdg}, \ref{fig::mas_sbc_surgery}, \ref{fig:augmented_cause} for parallel comparison. Please note that agent actions have no incoming edges as they are external stimuli from the perspective of the world model. 
\begin{figure}[htbp]
  \centering
  \includegraphics[trim={0pt 00pt 0pt 20pt}, clip, width=\linewidth]{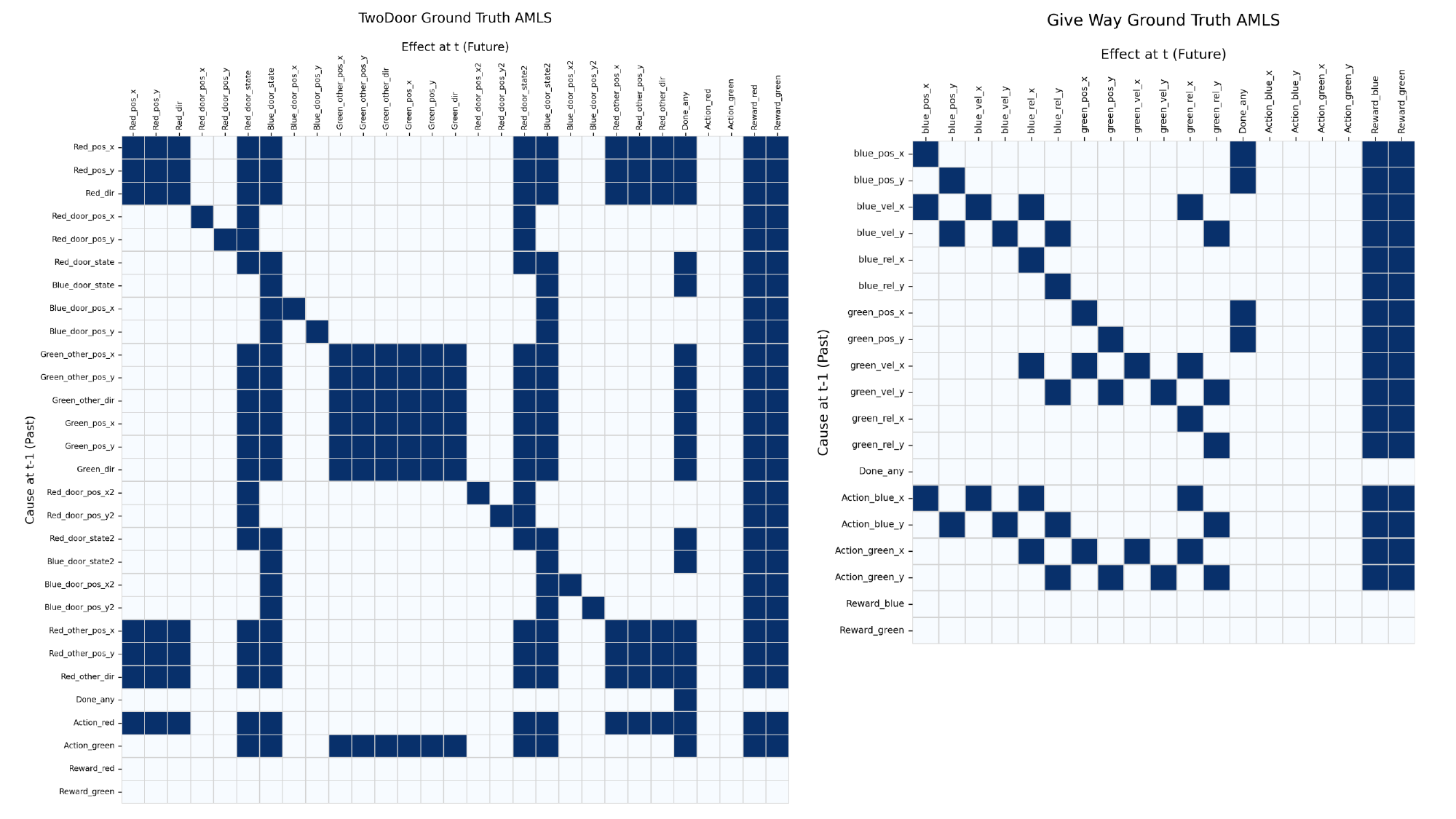}
  \caption{Ground truth causal adjacency matrices describing cause-effect relationship across time lag for environments (Left) Two-Door and (Right) Give-way under both full and partial observability. The blue square signifies a presence of  relationship. }\label{fig:adj1}
\end{figure}

\begin{figure}[htbp]
  \centering
  \includegraphics[trim={0pt 00pt 400pt 10pt}, clip, width=\linewidth]{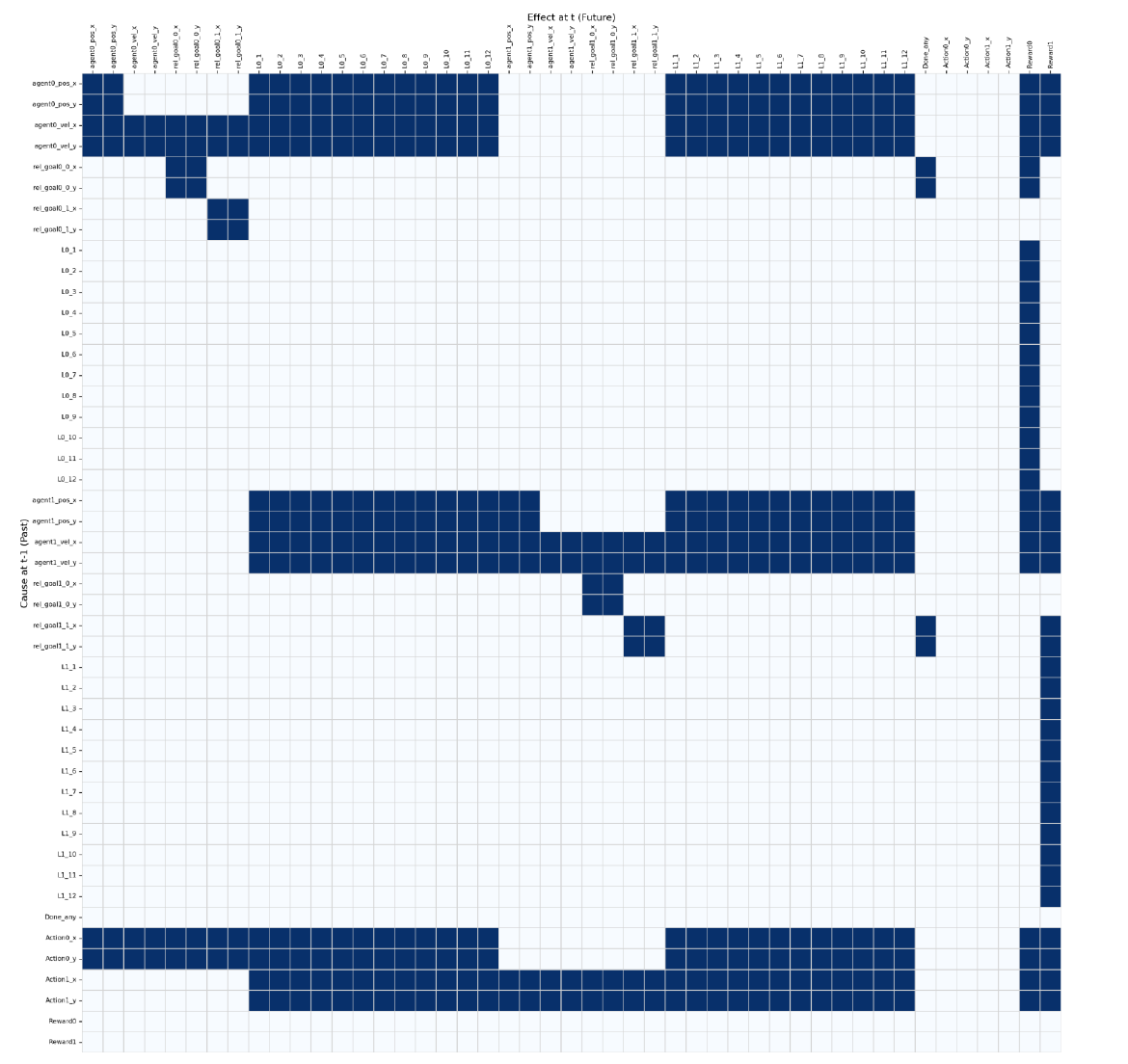}
  \caption{Ground truth causal adjacency matrices describing cause-effect relationship across time lag for environments Navigation under full observability. The blue square signifies a presence of  relationship. }\label{fig:adj2}
\end{figure}

\begin{figure}[htbp]
  \centering
  \includegraphics[trim={0pt 00pt 400pt 10pt}, clip, width=\linewidth]{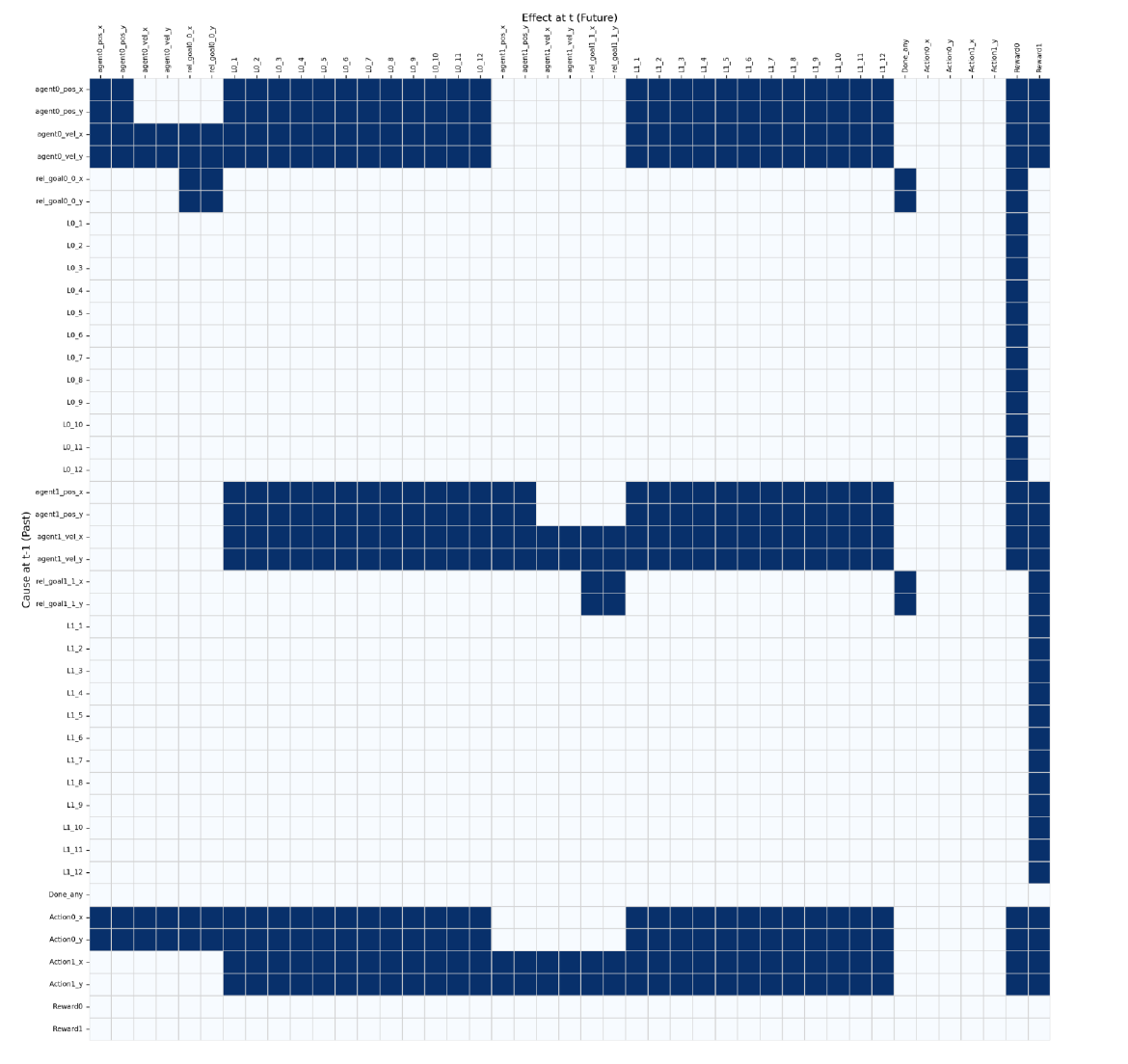}
  \caption{Ground truth causal adjacency matrices describing cause-effect relationship across time lag for environments Navigation under partial observability. The blue square signifies a presence of  relationship. }\label{fig:adj3}
\end{figure}

\section{Neural Architecture Design and Baselines}
\label{sec::baselines_appendix}

The following sections detail the architectural formalisms and baseline configurations utilized in this study. To ensure a rigorous evaluation of causal robustness, we distinguish between dense associative models and structural models constrained by causal discovery.

\subsection{Summary of Model Properties}

Table \ref{tab:model_properties} categorizes the world models by their discovery mechanism, interventional strength ($\sigma$), and masking strategy.
\definecolor{colorPCMCI}{RGB}{220, 230, 255}    
\definecolor{colorZero}{RGB}{255, 230, 230}      
\definecolor{colorSigma}{RGB}{230, 255, 230}     
\definecolor{colorDense}{RGB}{255, 240, 200}     
\definecolor{colorHadamard}{RGB}{240, 220, 255}  
\definecolor{colorLatent}{RGB}{220, 255, 255}    

\begin{table}[htbp]
\centering
\caption{Baseline configurations for causal transition modeling.}
\label{tab:model_properties}
\begin{tabular}{llll}
\toprule
\textbf{Abbreviation} & \textbf{Discovery Method} & \textbf{Data Entropy ($\sigma$)} & \textbf{Masking Operator} \\ \midrule
$\text{AWM}_{\sigma}$ & \cellcolor{gray!20} N/A & \cellcolor{colorSigma} $\sigma$ & \cellcolor{colorDense} Dense (None) \\
$\text{FOCUS}^{\odot}_{\sigma=0}$ & \cellcolor{gray!20} PC Algorithm & \cellcolor{colorZero} $0.0$ & \cellcolor{colorHadamard} Hadamard ($\odot$) \\
$\text{BC-WM}^{\odot}_{\sigma=0}$ & \cellcolor{colorPCMCI} PCMCI+ & \cellcolor{colorZero} $0.0$ & \cellcolor{colorHadamard} Hadamard ($\odot$) \\
$\text{ICWM}^{\odot}_{\sigma}$ & \cellcolor{colorPCMCI} PCMCI+ & \cellcolor{colorSigma} $\sigma$ & \cellcolor{colorHadamard} Hadamard ($\odot$) \\
$\text{R-ICWM}^{\odot}_{\sigma}$ & \cellcolor{colorPCMCI} PCMCI+ & \cellcolor{colorSigma} $\sigma$ & \cellcolor{colorHadamard} Hadamard ($\odot$) \\
$\text{C-VAE}^{\text{inj}}_{\sigma}$ & \cellcolor{colorPCMCI} PCMCI+ & \cellcolor{colorSigma} $\sigma$ & \cellcolor{colorLatent} Latent Injection \\ \bottomrule
\end{tabular}
\end{table}

\subsection{Causal Structural Masking and AMLS}
\label{sec::soft_masking}
The fundamental structural constraint in our architecture is the \textit{Adjacency Matrix with Lagged Structure} (AMLS) tensor, denoted as $\mathcal{C}$. For a system with history $\tau=1$, $\mathcal{C}$ represents the dependencies between variables at time $t$ and state transitions at $t+1$. 

To bridge discrete discovery with continuous neural optimization in our proposed framework, we implement \textit{Statistical Soft Masking} via a differentiable gating mechanism \citep{jang2017categoricalreparameterizationgumbelsoftmax,maddison2017concretedistributioncontinuousrelaxation,zheng2018dagstearscontinuousoptimization}. This modulates input features through a Log-Significance Transform ($\psi_{ij}$) relative to a significance threshold $\alpha$:
\begin{equation}
    \psi_{ij} = \log\left(\frac{\alpha}{p_{ij} + \epsilon}\right)
\end{equation}

The functional mask weight $\Gamma_{ij}$ is then computed using a sigmoid transformation scaled by a unique temperature parameter $\kappa$:
\begin{equation}
    \Gamma_{ij} = \sigma\left(\frac{\psi_{ij}}{\kappa}\right) = \frac{1}{1 + \exp(-\psi_{ij}/\kappa)}
\end{equation}

Specifically, $\psi_{ij} > 0$ (where $p_{ij} < \alpha$) indicates significant causal edges, driving $\Gamma_{ij} \to 1$ to permit information flow. Conversely, $\psi_{ij} < 0$ indicates non-causal edges, driving $\Gamma_{ij} \to 0$ to enforce \textit{Causal Shielding}. The temperature $\kappa$ controls the ``hardness'' of the gate; as $\kappa \to 0$, the mask converges to a discrete binary shutter.

\subsection{Baseline Architectures}

All MLP-based transition functions share a consistent depth of 3 hidden layers and a width of 32 neurons per agent in the framework.

\begin{enumerate}
    \item \textbf{Associative World Model ($\text{AWM}_{\sigma}$):} A standard 3-layer MLP with dense connectivity. Every input feature in the state-action vector can influence every output state prediction.
    \begin{equation}
        \hat{s}_{t+1} = \text{MLP}_3([s_t, \mathbf{a}_t]; \theta)
    \end{equation}

    \item \textbf{The FOCUS Baseline ($\text{FOCUS}^{\odot}_{\sigma=0}$):} Utilizes \citep{zhu2022offlinereinforcementlearningcausal} a binary causal tensor $\mathcal{C}$ derived from the PC algorithm. Masking is applied strictly at the input stage via the Hadamard product.
    \begin{equation}
        \hat{s}_{t+1} = \text{MLP}_3([s_t, \mathbf{a}_t] \odot \mathcal{C}; \theta)
    \end{equation}

    \item \textbf{Behavioral Cloning World Model ($\text{BC-WM}^{\odot}_{\sigma=0}$):} Employs the PCMCI+ algorithm to generate a binary tensor $\mathcal{C}$ but is trained on non-interventional expert demonstrations ($\sigma=0$).
    \begin{equation}
        \hat{s}_{t+1} = \text{MLP}_3(\mathbf{v}_t \odot \mathcal{C}; \theta)
    \end{equation}

    \item \textbf{Implicit Causal World Model ($\text{ICWM}^{\odot}_{\sigma}$):} The proposed framework. Integrates PCMCI+ structural masking via the soft gate $\Gamma$ with interventional data entropy $\sigma$ to recover invariant environmental physics.
    \begin{equation}
        \hat{s}_{t+1} = \text{MLP}_3(\mathbf{v}_t \odot \Gamma; \theta)
    \end{equation}

    \item \textbf{Recurrent Implicit Causal World Model ($\text{R-ICWM}_{\sigma}$):} Extends the proposed framework to natively handle temporal dependencies ($\tau \ge 1$) and historical conditioning ($H_k$) required by the MAS-SBC. By utilizing a Long Short-Term Memory (LSTM) architecture, the model unrolls the temporal sequence. Crucially, the statistical soft mask $\Gamma$ is applied strictly to the raw historical inputs $\mathbf{v}_k$, ensuring the hidden memory states ($h_k, c_k$) remain unmasked to fully preserve recurrent dynamics.
    \begin{align}
        \text{Masked Input:} \quad & \tilde{\mathbf{v}}_k = \mathbf{v}_k \odot \Gamma \quad \text{for } k \in \{t-\tau, \dots, t\} \\
        \text{Recurrent Update:} \quad & h_{k+1}, c_{k+1} = \text{LSTMCell}_{32}(\tilde{\mathbf{v}}_k, h_k, c_k; \theta_{\text{lstm}}) \\
        \text{State Prediction:} \quad & \hat{s}_{t+1} = \text{MLP}_{\text{2}}(h_{t+1}; \theta_{\text{head}})
    \end{align}

    \item \textbf{Injected Causal VAE ($\text{C-VAE}^{\text{inj}}_{\sigma}$):} Evaluates whether latent structural disentanglement can replace explicit input-level masking. To maintain equivalent representational capacity with the baseline architectures \citep{yang2023causalvaestructuredcausaldisentanglement}, both the inference (encoder) and generative (decoder) networks are parameterized as 3-layer Multi-Layer Perceptrons (MLPs). The discovered causal tensor $\mathcal{C}$ is injected strictly to constrain the forward transition dynamics in the latent space.
    \begin{align}
        \text{Encoder (3-layer):} \quad & h_t^{(e)} = \text{MLP}_{\phi}^{(3)}(\mathbf{v}_t) \nonumber \\
        & z_t \sim q_\phi(z \mid \mathbf{v}_t) = \mathcal{N}\big(\mu_\phi(h_t^{(e)}), \Sigma_\phi(h_t^{(e)})\big) \\
        \text{Causal Transition:} \quad & z_{t+1} = \text{ReLU}\big(\mathbf{W}_z (z_t \odot \mathcal{C}) + \mathbf{W}_a \mathbf{a}_t + b_z\big) + \epsilon, \quad \epsilon \sim \mathcal{N}(0, \mathbf{I}) \\
        \text{Decoder (3-layer):} \quad & h_{t+1}^{(d)} = \text{MLP}_{\theta}^{(3)}(z_{t+1}) \nonumber \\
        & \hat{s}_{t+1} \sim p_\theta(s \mid h_{t+1}^{(d)})
    \end{align}
\end{enumerate}

\subsection{Optimization and Loss Functions}

Training configurations are partitioned into deterministic models (MLP and LSTM variants) and generative latent models (C-VAE). All baselines use the Adam optimizer with a standardized learning rate of $1 \times 10^{-3}$. Deterministic architectures minimize Mean Squared Error (MSE) for direct state prediction, whereas the C-VAE maximizes the Evidence Lower Bound (ELBO), combining an MSE reconstruction term with a Kullback-Leibler (KL) divergence penalty.

\begin{table}[htbp]
    \centering
    \setlength{\tabcolsep}{4pt}
    \begin{minipage}[t]{0.48\textwidth}
        \centering
        \caption{Optimization for MLP \& LSTM Models}
        \label{tab:opt_mlp}
        \begin{tabular}{@{}l>{\raggedright\arraybackslash}p{0.55\textwidth}@{}}
        \toprule
        \textbf{Hyperparameter} & \textbf{Value} \\ \midrule
        Target Models & AWM, FOCUS, BC-WM, ICWM, R-ICWM \\
        Loss Function & MSE \\
        Optimizer & Adam \\
        Learning Rate ($\eta$) & $1 \times 10^{-3}$ \\
        Epochs & $50$ \\ \bottomrule
        \end{tabular}
    \end{minipage}\hfill
    \begin{minipage}[t]{0.48\textwidth}
        \centering
        \caption{Optimization for Causal VAE}
        \label{tab:opt_vae}
        \begin{tabular}{@{}l>{\raggedright\arraybackslash}p{0.55\textwidth}@{}}
        \toprule
        \textbf{Hyperparameter} & \textbf{Value} \\ \midrule
        Target Models & $\text{C-VAE}^{\text{inj}}_{\sigma}$ \\
        Loss Function & ELBO (MSE + KL Div.) \\
        Optimizer & Adam \\
        Learning Rate ($\eta$) & $1 \times 10^{-3}$ \\
        Epochs & $100$ \\ \bottomrule
        \end{tabular}
    \end{minipage}
\end{table}

\section{Evaluation Metrics}
\label{sec::metricWM}

\begin{itemize}[leftmargin=*,noitemsep,topsep=0pt]
    
    \item \textbf{Mean Squared Error (MSE):} 
    Intuitively, MSE measures the absolute physical accuracy of the internal simulator of the world model by heavily penalizing large deviations from the true future state. Formally, it quantifies predictive precision over an evaluation dataset $\mathcal{D}$ consisting of $N$ transitions. Given a state vector with dimensionality $D_{dim}$, MSE is the average squared $L_2$ norm between the predicted successor states $\hat{s}$ and the ground truth successor states $s$:
    \begin{equation}
        \text{MSE} = \frac{1}{N} \sum_{k=1}^N \sum_{d=1}^{D_{dim}} \left( \hat{s}_{k, t+1}^{(d)} - s_{k, t+1}^{(d)} \right)^2
    \end{equation}

    \item \textbf{Relative System Error (RSE):} 
    Intuitively, RSE quantifies the exact value of causal constraints by measuring how much our proposed model reduces prediction errors compared to a specific baseline. Formally, it isolates the predictive gain attributable to structural masking by normalizing the MSE of the candidate model against the baseline model trained with an interventional strength of $\sigma = 0.5$. A negative RSE denotes strict predictive improvement:
    \begin{equation}
        \text{RSE} = \frac{\text{MSE}_{\text{model}} - \text{MSE}_{\sigma=0.5}}{\text{MSE}_{\sigma=0.5}}
    \end{equation}

    \item \textbf{Continuous Structural Hamming Distance (SHD):} 
    Intuitively, SHD acts as a wiring diagram check, counting the exact number of structural mistakes such as hallucinated or missed connections the algorithm made when reverse engineering the true physical laws of the environment. Because the ICWM integrates structural constraints via a continuous soft mask the AMLS tensor $\mathcal{C}$ rather than a strict binary adjacency matrix, we extend the discrete SHD to measure continuous structural divergence. Formally, it calculates the $L_1$ Manhattan distance between the predicted continuous causal tensor $\mathcal{C}_{pred}$ and the ground truth graph $\mathcal{C}_{true}$ across all vertices $V$:
    \begin{equation}
        \text{SHD}(\mathcal{C}_{pred}, \mathcal{C}_{true}) = \sum_{i,j \in V} \left| \mathcal{C}_{pred}^{(i,j)} - \mathcal{C}_{true}^{(i,j)} \right|
    \end{equation}
\item \textbf{Spectral Distance ($\Delta_{spec}$)}
Intuitively, while SHD checks local connectivity, the Spectral Distance measures global topological equivalence. In Multi Agent Systems, agent roles are often interchangeable or equivariant under permutation. Spectral Distance captures this equivariance by comparing the fundamental vibrational modes of the causal graph, represented by the eigenvalues of the adjacency matrix. This metric proves if the model correctly identifies the underlying system dynamics regardless of specific agent permutations. Formally, given the sorted eigenvalues $\lambda$ of the predicted causal tensor $\mathcal{C}_{pred}$ and the ground truth $\mathcal{C}_{true}$, it calculates the $L_2$ distance:
\begin{equation}
    \Delta_{spec}(\mathcal{C}_{pred}, \mathcal{C}_{true}) = \sqrt{\sum_{i=1}^{|V|} \left( \lambda_i(\mathcal{C}_{pred}) - \lambda_i(\mathcal{C}_{true}) \right)^2}
\end{equation}
where $\lambda_i(\cdot)$ denotes the $i$-th eigenvalue of the adjacency representation. A low spectral distance ensures that the learned world model reproduces the long range propagation characteristics and role symmetry of the true environment dynamics.
    \item \textbf{Policy Invariance Gap ($\Delta_{inv}$):} 
    Intuitively, this gap is the ultimate test of causal understanding, proving whether the model actually learned invariant physics or simply memorized the specific habits of the expert. Formally, it verifies structural identifiability by measuring the divergence in prediction error when the model is evaluated on the expert training distribution $\mathcal{D}_{exp}$ versus an out of distribution policy regime $\mathcal{D}_{OOD}$. A gap approaching zero mathematically confirms the successful decoupling of invariant environmental physics from confounded strategic agent intent:
    \begin{equation}
        \Delta_{inv} = \text{MSE}_{\text{OOD}} - \text{MSE}_{\text{exp}}
    \end{equation}
    \item \textbf{Variance-Normalized Relative Error ($\Delta_{\text{variance}}$):}
    Intuitively, $\Delta_{\text{variance}}$ measures how much predictive information the world model has extracted about the successor state, relative to the trivial baseline of always predicting the mean successor state. Formally, it normalizes the MSE of the candidate model by the target-variance floor $\sigma^2_{\mathcal{D}}$, the error of the input-agnostic predictor $\bar{s} = \frac{1}{N}\sum_{k} s_{k,t+1}$:
    \begin{equation}
        \Delta_{\text{variance}} = \frac{\text{MSE}}{\sigma^2_{\mathcal{D}}}, \qquad \sigma^2_{\mathcal{D}} = \frac{1}{N} \sum_{k=1}^N \sum_{d=1}^{D_{dim}} \left( s_{k,t+1}^{(d)} - \bar{s}^{(d)} \right)^2
    \end{equation}
    Equivalently, $\Delta_{\text{variance}} = 1 - R^2$, where $R^2$ is the coefficient of determination \citep{pearl2009causality}; values near $0$ indicate near-perfect prediction, $\Delta_{\text{variance}} = 1$ indicates the model has learned no usable dynamics, and $\Delta_{\text{variance}} > 1$ indicates the model is worse than the mean-prediction baseline, signaling overfitting or distributional shift.

    \item \textbf{Kernel Two-Sample Test (RBF-MMD):} 
    Intuitively, MMD acts as a statistical diagnostic tool to determine if two sets of next-state samples $X=\{x_i\}_{i=1}^m$ and $Y=\{y_j\}_{j=1}^n$ originate from the same underlying distribution. Formally, we quantify this discrepancy using the squared Maximum Mean Discrepancy with an RBF kernel $k(u,v)=\exp(-\gamma\lVert u-v\rVert^2)$ \citep{gretton2012kernel}, estimated via an unbiased U-statistic:
    \begin{equation}
        \widehat{\mathrm{MMD}}^2(X,Y) = \frac{1}{m(m-1)}\sum_{i \neq i'} k(x_i, x_{i'}) + \frac{1}{n(n-1)}\sum_{j \neq j'} k(y_j, y_{j'}) - \frac{2}{mn}\sum_{i,j} k(x_i, y_j)
    \end{equation}
    The bandwidth $\gamma^{-1}$ is set using the median heuristic \citep{gretton2012kernel, scholkopf2002learning}. Significance is assessed via a permutation test; a location is flagged as faulty when the resulting $p$-value falls below $\alpha=0.05$.

    \item \textbf{Kernel Density Estimation (KDE) Difference Test:} 
    Intuitively, in continuous settings, this metric summarizes the discrepancy between full probability density functions rather than relying on scalar statistics. Formally, for a sample $\{z_i\}_{i=1}^n$, we estimate the density at query point $z$ using an isotropic Gaussian kernel \citep{rosenblatt1956, parzen1962, silverman1986density}:
    \begin{equation}
        \hat{p}(z) = \frac{1}{n}\sum_{i=1}^n \frac{1}{2\pi h^2}\exp\left(-\frac{\lVert z - z_i\rVert^2}{2h^2}\right)
    \end{equation}
    with bandwidth $h$ fixed for comparability. The global discrepancy between normal density $\hat{p}_{\text{norm}}$ and adversarial density $\hat{p}_{\text{adv}}$ is quantified by the Total Variation (TV) distance evaluated over a query grid $\mathcal{G}$:
    \begin{equation}
        \mathrm{TV}(\hat{p}_{\text{norm}}, \hat{p}_{\text{adv}}) = \frac{1}{2}\sum_{z \in \mathcal{G}} \big|\hat{p}_{\text{norm}}(z) - \hat{p}_{\text{adv}}(z)\big|\, \Delta z
    \end{equation}
    This metric identifies localized shape changes in the next-state distribution that variance-based error metrics often overlook.

\subsection{Toy Study: Validating $\Delta_{\text{variance}}$}
\label{sec:toy-delta-variance}

\subsubsection{Design}
We test $\Delta_{\text{variance}}$ on two synthetic environments where we know the true dynamics
exactly. The first is a simple linear system, $s_{t+1} = As_t + \epsilon$, where $A$ is a fixed
transition matrix and $\epsilon$ is Gaussian noise. Because we control $A$ and the noise directly, we
can cleanly test how $\Delta_{\text{variance}}$ behaves under scale changes, a degrading model, and
distribution shift, without any confounds from actual training. The second is a small two-agent
navigation task: agent 1 moves toward a goal, and agent 2's move is its own goal-directed step
\emph{plus a term that copies agent 1's move}, $\Delta p_2 = \text{step}\cdot(g_2-p_2) +
c\cdot\Delta p_1$. This coupling term is a direct causal link from agent 1 to agent 2, similar to the
inter-agent dependencies our full pipeline is meant to recover. We then build an adversarial version of
this task by changing both agents' goals, changing the step size, and \emph{flipping the sign of the
coupling} ($c \rightarrow -c$). A model that learned the original coupling direction is now simply
wrong on this adversarial data, not just noisier -- mirroring the \texttt{adversarial} evaluation set
used in our main experiments.

\subsubsection{Experiments}
All experiments use a single fixed seed; we do not average over seeds because every result here is
either an exact mathematical identity or a direct consequence of how we constructed the data, not
something that varies with random sampling. We run four experiments:
\begin{enumerate}
    \item \textbf{Scale invariance:} the same model, at the same quality, is evaluated on state
    dimensions at their normal scale and again after multiplying those dimensions by $100$.
    \item \textbf{Three regimes:} we blend between three predictors using a single knob $q$ from
    $0$ to $2$ -- the true dynamics at $q=0$, the mean predictor at $q=1$, and a deliberately
    wrong (sign-flipped) predictor at $q=2$.
    \item \textbf{Distribution shift (linear system):} a model is fit on one data distribution and
    then evaluated both on that same distribution and on a shifted one (rotated, stretched, and with
    slightly different dynamics).
    \item \textbf{Distribution shift (navigation):} a model is fit on the normal two-agent task and
    evaluated both on that task and on its adversarial, flipped-coupling version.
\end{enumerate}

\begin{figure}[htbp]
    \centering
    \includegraphics[trim={0pt 230pt 50pt 0pt}, clip, width=\linewidth]{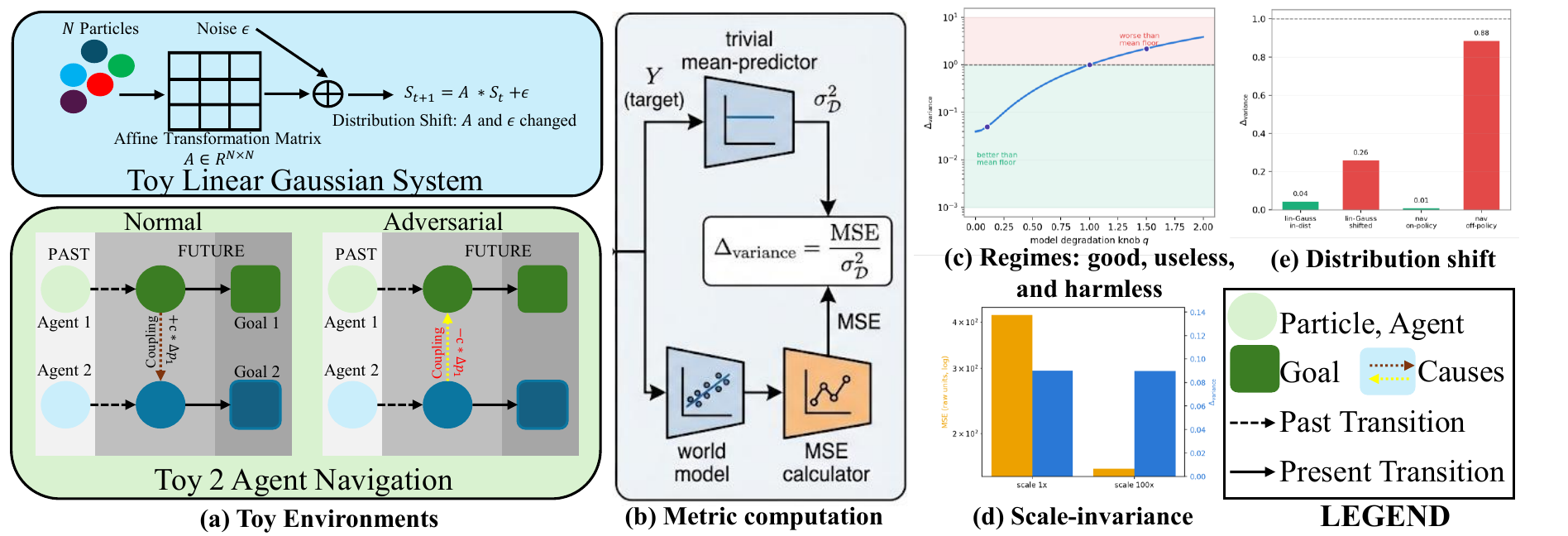}
    \caption{Overview of the $\Delta_{\text{variance}}$ testing framework. 
    \textbf{(a) Toy Environments}: Controlled linear-Gaussian and two-agent navigation systems used to isolate dynamics. 
    \textbf{(b) Metric Computation}: Pipeline normalizing model MSE by the variance floor $\sigma_{\mathcal{D}}^2$ to derive $\Delta_{\text{variance}}$. 
    \textbf{(c) Regimes}: Performance analysis tracking $\Delta_{\text{variance}}$ as models degrade from useful to harmful. 
    \textbf{(d) Scale-Invariance}: Demonstration of metric stability under target unit scaling. 
    \textbf{(e) Distribution Shift}: Quantitative evaluation of model robustness under adversarial or out-of-distribution conditions.}
    \label{fig:delta-variance-toy}
\end{figure}

\subsubsection{Reading the Figure}
Figure~\ref{fig:delta-variance-toy} has five panels. 

\textbf{Panel (a)} introduces the toy datasets for evaluation.

\textbf{Panel (b)} features the design of the metric.

\textbf{Panel (c)} shows $\Delta_{\text{variance}}$ as the model quality knob $q$ increases. The curve
starts low, in the green region ($\Delta_{\text{variance}} < 1$, model better than guessing the mean),
crosses the dashed line at exactly $\Delta_{\text{variance}} = 1$ when the model becomes the mean
predictor, and keeps rising into the red region ($\Delta_{\text{variance}} > 1$) once the model is
actively wrong. This shows that $1.0$ is not just a normalization choice -- it is a real threshold that
separates "learned something useful" from "learned nothing" or "learned something harmful."

\textbf{Panel (d)} shows the scale-invariance result. The orange bars (raw MSE, log scale) are wildly
different between the normal-scale and $100\times$-scale versions of the same problem, because MSE
depends on the units of the target. The blue bars ($\Delta_{\text{variance}}$) are essentially
identical for both, because the metric divides out the scale. This is the main practical benefit of
$\Delta_{\text{variance}}$: you can compare it across settings where raw MSE numbers would otherwise be
meaningless side by side.

\textbf{Panel (e)} shows both distribution-shift experiments together. The green bars
(in-distribution / on-policy) stay comfortably below the $\Delta_{\text{variance}} = 1$ line in both
the linear system and the navigation task. The red bars (shifted / adversarial) rise to or above that
line, showing exactly how much worse the model becomes once the data no longer matches what it was
trained on. This is the same effect we rely on when using $\Delta_{\text{variance}}$ to evaluate world
models on the \texttt{adversarial} split in our main results.

\subsection{Toy evaluation: density metrics vs.\ the variance metric}
\label{app:error-density-toy}
\begin{figure}[htbp]
    \centering
    \includegraphics[trim={0pt 200pt 50pt 0pt}, clip, width=\linewidth]{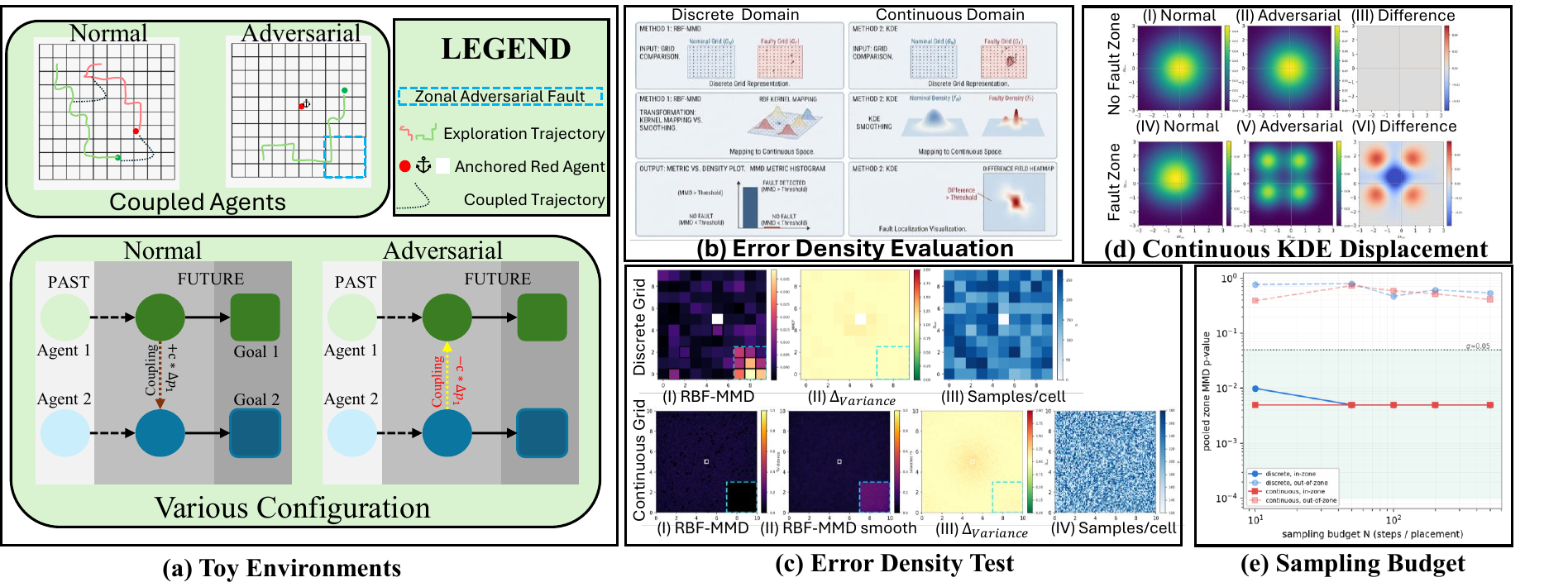}
    \caption{Toy evaluation of density-based fault detection vs.\ the variance
    metric. (a) Two-agent coupled toy world, normal vs.\ adversarial dynamics,
    with a zonal fault confined to one region. (b) Evaluation pipeline for the
    discrete and continuous grids. (c) RBF-MMD, $\Delta_{\text{variance}}$, and
    sample-coverage maps for both grids; only the density metric localizes the
    fault zone (cyan box). (d) KDE of secondary displacement, in- vs.\
    out-of-zone: the fault reshapes the distribution without changing its mean
    or variance. (e) Pooled in-zone MMD $p$-value vs.\ sampling budget for both
    grids, against the $\alpha=0.05$ significance line and the flat
    out-of-zone control.}
    \label{fig:error-density-toy}
\end{figure}

We use a small toy world to check a specific failure mode: a fault that
changes the \emph{shape} of an agent's next-step distribution while leaving
its mean and variance unchanged. Such a fault is invisible to any
variance-based metric by construction, and this toy study confirms that a
distributional metric catches it instead. Figure~\ref{fig:error-density-toy}
summarizes the setup and results.

\paragraph{Toy world (Fig.~\ref{fig:error-density-toy}a).}
Two agents share a bounded 2D area: a \emph{primary} agent, fixed at the
center, and a \emph{secondary} agent that is causally coupled to it. In the
\emph{normal} world the secondary is pulled toward the primary at every step.
In the \emph{adversarial} world this coupling is unchanged everywhere except
inside one small \emph{zonal fault} region, where the secondary's per-step
noise switches from a unimodal (Gaussian) law to a bimodal, two-point law with
the \emph{same mean and same variance}. Data is collected by anchoring the
secondary at every reachable cell in turn and sampling short random-action
trajectories, so every location is covered independently of the coupling
dynamics.

\paragraph{Evaluation protocol (Fig.~\ref{fig:error-density-toy}b).}
We instantiate the same fault in a \emph{discrete} grid and a
\emph{continuous} area. Per grid cell (discrete) or spatial bin (continuous),
we compare the normal and adversarial next-step samples with two families of
metric: (i) a distributional two-sample test — RBF-kernel MMD in the discrete
case, kernel density estimation (KDE) in the continuous case — and (ii) the
variance-normalized error metric $\Delta_{\text{variance}}$ used elsewhere in
this paper. A location is flagged as faulty when its metric exceeds a
significance threshold.

\paragraph{Spatial localization (Fig.~\ref{fig:error-density-toy}c).}
For both the discrete and continuous grids, the RBF-MMD / KDE map lights up
only inside the true fault zone (cyan dashed box), while $\Delta_{\text{variance}}$
stays flat and near its baseline value everywhere, including inside the fault
zone. The samples-per-cell/bin panels confirm coverage is uniform, so the
difference is not a sampling artifact.

\paragraph{Why $\Delta_{\text{variance}}$ is blind (Fig.~\ref{fig:error-density-toy}d).}
We isolate one bin outside the fault zone and one bin inside it and plot the
KDE of the secondary's displacement. Outside the zone, normal and adversarial
densities coincide and their difference is flat. Inside the zone, the
adversarial density splits into four lobes while the normal density stays
unimodal — the mean and spread are unchanged, only the shape differs. This is
exactly the change $\Delta_{\text{variance}}$ cannot see, and exactly what MMD/KDE
detect.

\paragraph{Detection power vs.\ sampling budget (Fig.~\ref{fig:error-density-toy}e).}
We repeat the evaluation across sampling budgets (steps per location). The
pooled in-zone MMD $p$-value falls below $\alpha=0.05$ once enough samples are
collected, for both the discrete and continuous setups, while the
out-of-zone $p$-value (negative control) stays flat near $0.5$ throughout.
This shows the detection is a genuine, budget-dependent statistical effect
rather than a fixed threshold artifact.

\paragraph{Takeaway.}
A localized, variance-preserving shape change in agent dynamics is
undetectable by $\Delta_{\text{variance}}$ at any sample size, but is reliably
caught by a distributional metric (MMD/KDE) given sufficient samples. This
motivates including a distributional check alongside variance-based error
metrics when auditing a learned world model for localized dynamics faults.

\subsection{Causal Metric Stress Tests}To validate the robustness of the proposed metrics, we conducted an empirical stress test across three graph scales ($NODE \in \{4, 10, 50\}$). We introduced independent perturbations: soft edges (confidence $0.2$) and random structural flips, then measured the metrics against a ground truth baseline $A_1$ from figure \ref{fig::demo_graph_metric} . As shown in Figure \ref{fig::demo_graph_metric}, the metrics demonstrate distinct behaviors under the triangle inequality $\Delta(A_1, A^{big=small_1 +small_2}_2) \le \Delta(A_1, A^{small_1}_2) + \Delta(A_1, A^{small_2}_2)$, where $A_2$ stands for perturbed causal structure.   \newline

\begin{figure}[htbp]
  \centering
  \includegraphics[trim={0pt 0pt 00pt 0pt}, clip, width=\linewidth]{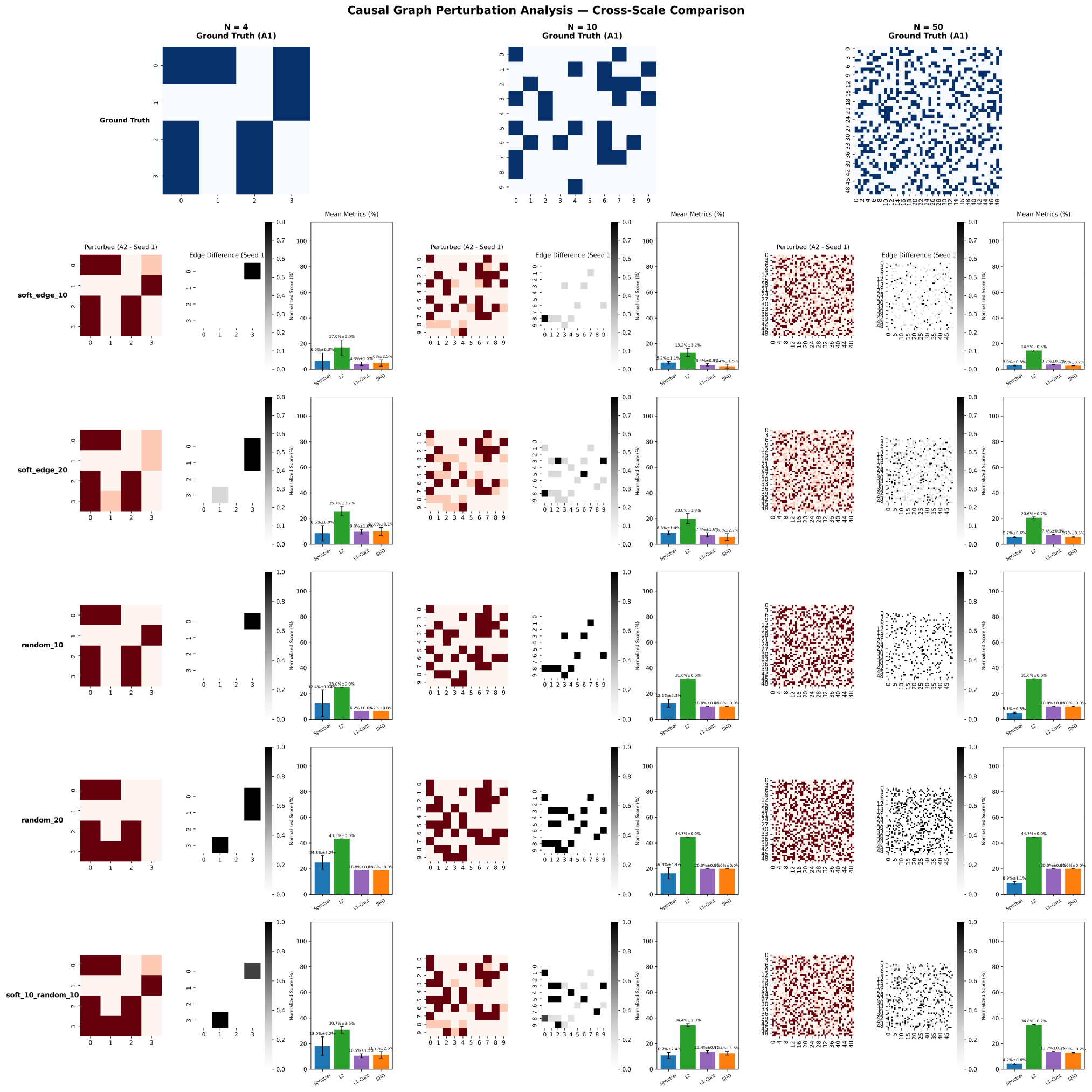}
  \caption{ Causal Graph Perturbation Analysis. We evaluate the robustness of causal metrics against systematic graph stress tests across multiple scales ($N \in \{4, 10, 50\}$). \textbf{TOP}: The Ground Truth baseline ($A_1$). \textbf{LEFT} (Scenarios $A_2$): We introduce controlled perturbations: Soft $Edge_{10}$ (randomly attenuation 10 percentage of edges), $Random_{20}$ (randomly flip 20 percentage of edges), and their Composite (a blend of both). \textbf{MIDDLE}: highlights the edge difference, \textbf{RIGHT}: impact on metrics. These scenarios test metric sensitivity to both local weight fluctuations and global structural shifts over 5 seeds.}\label{fig::demo_graph_metric}
\end{figure}

\begin{table}[htbp]
    \centering
    \caption{Triangle Inequality Verification for Causal Graph Metrics.}
    \label{tab:triangle_inequality}
    \small
    \begin{tabular}{llcccc}
        \toprule
        \textbf{NODE} & \textbf{Metric} & \textbf{$\Delta_1(A_1, A^{small_1}_2)$} & \textbf{$\Delta_2(A_1, A^{small_2}_2)$} & \textbf{$\Delta(A_1, A^{big=small_1+small_2}_2)$} & \textbf{Sum ($\Delta_1 + \Delta_2$)} \\
        \midrule
        10 
        & \cellcolor{blue!10} Spectral   & \cellcolor{blue!10} 0.263  & \cellcolor{blue!10} 0.496  & \cellcolor{blue!10} 0.720  & \cellcolor{blue!10} 0.759 \\
        & \cellcolor{green!10} $L_2$ Norm & \cellcolor{green!10} 0.680  & \cellcolor{green!10} 1.000  & \cellcolor{green!10} 1.229  & \cellcolor{green!10} 1.680 \\
        & \cellcolor{orange!10} $L_1$-Cont & \cellcolor{orange!10} 0.680  & \cellcolor{orange!10} 1.000  & \cellcolor{orange!10} 1.680  & \cellcolor{orange!10} 1.680 \\
        & \cellcolor{purple!10} SHD        & \cellcolor{purple!10} 0.800  & \cellcolor{purple!10} 1.000  & \cellcolor{purple!10} 1.800  & \cellcolor{purple!10} 1.800 \\
        \midrule
        100 
        & \cellcolor{blue!10} Spectral   & \cellcolor{blue!10} 0.517  & \cellcolor{blue!10} 1.259  & \cellcolor{blue!10} 1.072  & \cellcolor{blue!10} 1.776 \\
        & \cellcolor{green!10} $L_2$ Norm & \cellcolor{green!10} 1.318  & \cellcolor{green!10} 3.162  & \cellcolor{green!10} 3.439  & \cellcolor{green!10} 4.481 \\
        & \cellcolor{orange!10} $L_1$-Cont & \cellcolor{orange!10} 3.440  & \cellcolor{orange!10} 10.000 & \cellcolor{orange!10} 13.440 & \cellcolor{orange!10} 13.440 \\
        & \cellcolor{purple!10} SHD        & \cellcolor{purple!10} 2.400  & \cellcolor{purple!10} 10.000 & \cellcolor{purple!10} 12.400 & \cellcolor{purple!10} 12.400 \\
        \midrule
        500 
        & \cellcolor{blue!10} Spectral   & \cellcolor{blue!10} 1.483  & \cellcolor{blue!10} 2.544  & \cellcolor{blue!10} 2.092  & \cellcolor{blue!10} 4.027 \\
        & \cellcolor{green!10} $L_2$ Norm & \cellcolor{green!10} 7.274  & \cellcolor{green!10} 15.811 & \cellcolor{green!10} 17.406 & \cellcolor{green!10} 23.085 \\
        & \cellcolor{orange!10} $L_1$-Cont & \cellcolor{orange!10} 92.960 & \cellcolor{orange!10} 250.000 & \cellcolor{orange!10} 342.960 & \cellcolor{orange!10} 342.960 \\
        & \cellcolor{purple!10} SHD        & \cellcolor{purple!10} 71.600 & \cellcolor{purple!10} 250.000 & \cellcolor{purple!10} 321.600 & \cellcolor{purple!10} 321.600 \\
        \bottomrule
    \end{tabular}
\end{table}

\paragraph{Analysis of Metric Properties}
Our stress tests reveal (Table \ref{tab:triangle_inequality}) distinct mathematical behaviors regarding the triangle inequality  $\Delta(A_1, A^{big=small_1 +small_2}_2) \le \Delta(A_1, A^{small_1}_2) + \Delta(A_1, A^{small_2}_2)$. 
\begin{itemize}
    \item \textbf{Exact Equality ($L_1$ and SHD):} These metrics honor the triangle inequality as an exact equality. Because our perturbations ($A$ and $B$) are generated using disjoint supports, the cumulative structural damage is perfectly additive.
    \item \textbf{Strict Inequality ($L_2$ Norm):} The Frobenius norm exhibits strict inequality ($\Delta(A_1, A^{big=small_1 +small_2}_2) < \Delta(A_1, A^{small_1}_2) + \Delta(A_1, A^{small_2}_2)$). This is expected behavior for Euclidean distances in orthogonal perturbation spaces.
    \item \textbf{Non-Metric Behavior (Spectral Distance):} While mean values honor the triangle inequality, individual samples occasionally violate it. This confirms that the spectral spectrum is sensitive to global topological regime shifts. Small local perturbations can sometimes cancel within the eigenvalue spectrum due to symmetry, while composite perturbations break that symmetry entirely, demonstrating that the spectral representation of causal environments is non-linear.
\end{itemize}

\end{itemize}

\section{MAPPO Training and Evaluation Performance of Expert Policy}
\label{sec::expert_training}
The following Figures \ref{fig:expert}, log the mean training and testing rewards for the MAPPO experts across all experimental domains. These metrics verify that the experts achieved sufficient convergence to provide high-quality, goal-directed demonstrations for causal discovery. We select the expert policy marked with an X in the reward plots and provide a reward comparison about the policy's strength.  

\begin{figure}[htbp]
  \centering
  \includegraphics[trim={0pt 0pt 0pt 0pt}, clip, width=\linewidth]{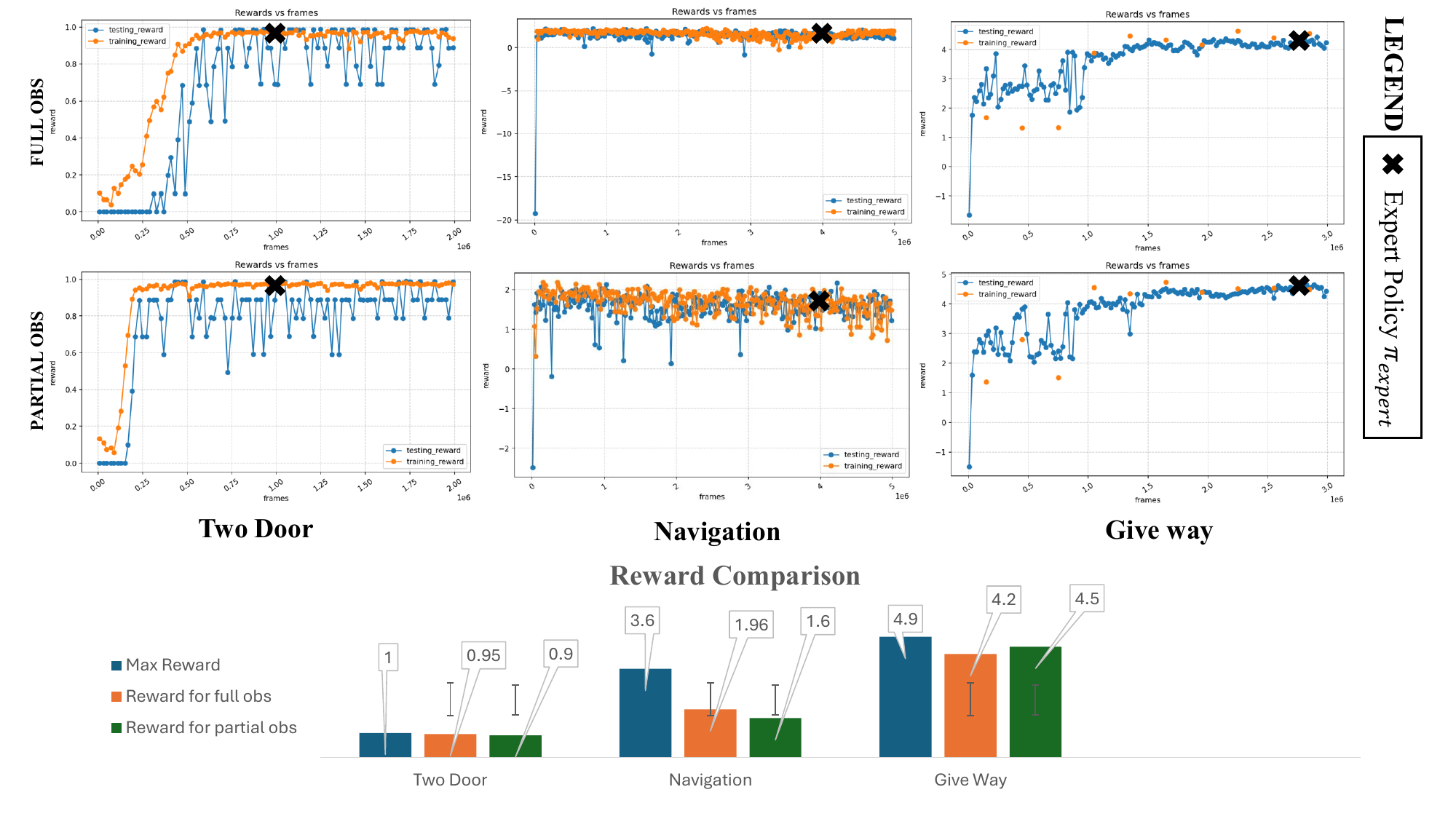}
  \caption{Quantitative evaluation of the expert policy in various domains}\label{fig:expert}
\end{figure}

\definecolor{lightgray}{gray}{0.8}
\definecolor{groupcolor}{RGB}{240, 248, 255} 

\begin{table}[htbp]
\centering
\renewcommand{\arraystretch}{1.2}
\small
\begin{tabular}{l|cc|cc|cc}
\toprule
\rowcolor{lightgray}
\textbf{Parameter} & \multicolumn{2}{c|}{\textbf{Two-Door}} & \multicolumn{2}{c|}{\textbf{Give-way}} & \multicolumn{2}{c}{\textbf{Navigation}} \\
\rowcolor{lightgray}
 & \textbf{Full} & \textbf{Part} & \textbf{Full} & \textbf{Noisy} & \textbf{Full} & \textbf{Part} \\
\midrule

\rowcolor{lightgray} \multicolumn{7}{c}{\textbf{[Env Specs]}} \\
\rowcolor{groupcolor} max\_steps & 500 & 500 & 500 & 500 & 500 & 500 \\
\rowcolor{groupcolor} fully\_obs & True & False & True & False & True & False \\
\rowcolor{groupcolor} obs\_noise & N/A & N/A & 0 & 0.03 & N/A & N/A \\
\rowcolor{groupcolor} n\_agents & 2 & 2 & 2 & 2 & 2 & 2 \\
\rowcolor{groupcolor} action\_support & 7 & 7 & $[-1,+1]$ & $[-1,+1]$ & $[-1,+1]$ & $[-1,+1]$ \\
\rowcolor{groupcolor} action\_dimension & 1 & 1 & 2 & 2 & 2 & 2 \\
\rowcolor{groupcolor} observation\_space & Discrete & Discrete & Continuous & Continuous & Continuous & Continuous \\
\rowcolor{groupcolor} action\_space & Discrete & Discrete & Continuous & Continuous & Continuous & Continuous \\

\rowcolor{lightgray} \multicolumn{7}{c}{\textbf{[Model: Actor] - Architecture Overview}} \\
\rowcolor{groupcolor} \multicolumn{7}{p{13.5cm}|}{\raggedright \textbf{Design:} Decentralized policy network. Maps local agent observations to a probability distribution over actions.} \\
\rowcolor{groupcolor} input\_channel & 8 & 8 & 6 & 6 & 20 & 18 \\
\rowcolor{groupcolor} mlp\_cells & 64x3 & 64x3 & 16x2 & 16x2 & 64x2 & 64x2 \\
\rowcolor{groupcolor} centralized & False & False & False & False & False & False \\
\rowcolor{groupcolor} share\_params & False & False & False & False & False & False \\
\rowcolor{groupcolor} activation & ReLU & ReLU & ReLU & ReLU & ReLU & ReLU \\

\rowcolor{lightgray} \multicolumn{7}{c}{\textbf{[Model: Critic] - Architecture Overview}} \\
\rowcolor{groupcolor} \multicolumn{7}{p{13.5cm}|}{\raggedright \textbf{Design:} Centralized state-value estimator. Concatenates all agent observations to map global state features to a scalar value.} \\
\rowcolor{groupcolor} input\_channel & 8 & 8 & 6 & 6 & 20 & 18 \\
\rowcolor{groupcolor} mlp\_cells & 64x3 & 64x3 & 16x2 & 16x2 & 64x2 & 64x2 \\
\rowcolor{groupcolor} centralized & True & True & True & True & True & True \\
\rowcolor{groupcolor} share\_params & False & False & False & False & False & False \\
\rowcolor{groupcolor} activation & ReLU & ReLU & ReLU & ReLU & ReLU & ReLU \\

\rowcolor{lightgray} \multicolumn{7}{c}{\textbf{[Training]}} \\
\rowcolor{groupcolor} learning\_rate & 0.0005 & 0.0005 & 0.0005 & 0.0005 & 0.0005 & 0.0005 \\
\rowcolor{groupcolor} max\_grad\_norm & 0.5 & 0.5 & 1.0 & 1.0 & 1.0 & 1.0 \\
\rowcolor{groupcolor} value\_loss\_coeff & 0.5 & 0.5 & 0.5 & 0.5 & 0.5 & 0.5 \\

\rowcolor{lightgray} \multicolumn{7}{c}{\textbf{[PPO]}} \\
\rowcolor{groupcolor} clip\_epsilon & 0.2 & 0.2 & 0.2 & 0.2 & 0.2 & 0.2 \\
\rowcolor{groupcolor} gamma & 0.9 & 0.9 & 0.99 & 0.99 & 0.99 & 0.99 \\
\rowcolor{groupcolor} lmbda & 0.8 & 0.8 & 0.9 & 0.9 & 0.9 & 0.9 \\
\rowcolor{groupcolor} entropy\_coeff & 0.001 & 0.001 & 0.01 & 0.01 & 0.01 & 0.01 \\

\rowcolor{lightgray} \multicolumn{7}{c}{\textbf{[Runtime]}} \\
\rowcolor{groupcolor} device & cuda & cuda & cuda & cuda & cuda & cuda \\
\rowcolor{groupcolor} checkpoint\_int & 20k & 20k & 20k & 20k & 20k & 20k \\

\bottomrule
\end{tabular}
\caption{Comprehensive PPO Parameter Configuration}
\label{tab:ppo_params}
\end{table}

\section{Dataset Cardinality and Evaluation }
\label{sec:appendix_cardinality}
\subsection{Scaling Dynamics and Cardinality Bounds}
The transition tensor dimension $dim$ comprises flattened endogenous state variables ($\mathbf{V}_{S}$) and the joint action vector ($\mathbf{A}_{joint}$). Covariance parameters $\Sigma_{\Theta}$ equal $\frac{dim(dim+1)}{2}$ for continuous spaces. For discrete Two-Door Coordination, $\Sigma_{\Theta}$ is scaled by 50 to account for the $2^{dim}$ complexity penalty. 

To bound the estimation error of causal dependencies within the structural causal model $\mathcal{M}$ at a 95\% confidence level ($Z \approx 1.96$), we calculate the required sample complexity. Given a target tolerance threshold ($\epsilon$ in std-div), the cardinality is determined by $|\mathcal{D}| = \lceil (1.96 / \epsilon)^2 \rceil \times \Sigma_{\Theta}$. These datasets were generated using either a pre-trained TorchRL MAPPO expert \citep{bou2023torchrldatadrivendecisionmakinglibrary, yu2022surprisingeffectivenessppocooperative} or hard-coded greedy heuristics. Dataset requirements are summarized in Table \ref{tab:sample_complexity_refined}.

\definecolor{lightgray}{RGB}{230, 230, 230}
\definecolor{lightbrown}{RGB}{245, 222, 179}
\definecolor{lightlavender}{RGB}{224, 176, 255}

\begin{table}[H]
\centering
\caption{Sample Complexity ($|\mathcal{D}|$) across Tolerance Thresholds}
\label{tab:sample_complexity_refined}
\resizebox{\textwidth}{!}{
\begin{tabular}{lcccccc}
\toprule
\textbf{Environment} & \textbf{$dim$} & \textbf{$\Sigma_{\Theta}$} & \textbf{$|\mathcal{D}|$ ($0.5$ std-div)} & \textbf{$|\mathcal{D}|$ ($0.3$ std-div)} & \textbf{$|\mathcal{D}|$ ($0.1$ std-div)} & \textbf{$|\mathcal{D}|$ ($0.05$ std-div)} \\
\midrule
Giveway Corridor   & 12 & 78        & 1,248    & 3,354   & \cellcolor{lightgray} 30,030$^*$   & \cellcolor{lightbrown} 119,886$^{**}$ \\
Two-Door (Disc.)   & 14 & $105 \times 50$ &  \cellcolor{lightgray} 84,000$^*$ & 225,750 & \cellcolor{lightbrown} 2,021,250$^{**}$ & 8,069,250 \\
Two Agent Navigation (POMDP) & 38 & 741       & 11,856   & 31,863  &  \cellcolor{lightgray} 285,285$^*$  & 1,138,917 \\
Two Agent Navigation (MDP)   & 40 & 820       & \cellcolor{lightlavender}13,120$^{***}$    & 35,260  &  \cellcolor{lightgray} 315,700$^*$  & 1,260,340 \\
Three Agent Navigation (MDP) & 66 & 2,211 & \cellcolor{lightlavender}35,376$^{***}$  & 95,074 & 851,032 & 3,397,312 \\
Four Agent Navigation (MDP) & 96 & 4,656 & \cellcolor{lightlavender}74,496$^{***}$  & 200,213 & 1,792,307 & 7,153,536 \\
\bottomrule
\end{tabular}
}
\end{table}

\begin{tcolorbox}[colback=gray!12, boxrule=0pt, sharp corners, left=3mm, right=3mm]
\small
\textbf{Selection Rationale:} Starred entries ($^*$) represent selected cardinalities. The Two-Door (Discrete) environment utilized a 0.5 std-div threshold for feasibility, while all other environments (Giveway Corridor, Navigation POMDP/MDP) utilized a 0.1 std-div threshold to ensure high-fidelity resolution of environmental causal physics. Results for these cardinalities are reported for causal discovery in Appendix~\ref{sec::discovery} (Figure~\ref{fig:discovery_metric_result}), for world-model error in Appendix~\ref{app:normal3-wm-sweep}, and for neural densities in Appendix~\ref{app:neural-density}.
\end{tcolorbox}

\begin{tcolorbox}[colback=lightbrown, boxrule=0pt, sharp corners, left=3mm, right=3mm]
\small
\textbf{High Sample Complexity Results:} Starred entries ($^{**}$) represent selected cardinalities for higher sample complexities and results are provided in the appendix: causal discovery in Figures~\ref{fig:discovery_metric_result_high_complexicity} and~\ref{fig:discovery_result_high_complexicity}, world-model error in Appendix~\ref{app:highsample-wm-sweep}, and neural densities in Appendix~\ref{app:neural-density}.
\end{tcolorbox}

\begin{tcolorbox}[colback=lightlavender, boxrule=0pt, sharp corners, left=3mm, right=3mm]
\small
\textbf{Multi-Agent Scaling:} We evaluated scalability by increasing agent count from the following choice $ \in \{2,3,4\}$. Selected configurations are marked with triple stars ($^{***}$). Full experimental results are available in Appendix~\ref{app:multiagent-lidar-wm} for world-model error and Appendix~\ref{app:nd-nolidar} for neural densities.
\end{tcolorbox}

\subsection{Evaluation Dataset Configuration}

\begin{figure}[htbp]
  \centering
  \includegraphics[trim={0pt 400pt 350pt 0pt}, clip, width=\linewidth]{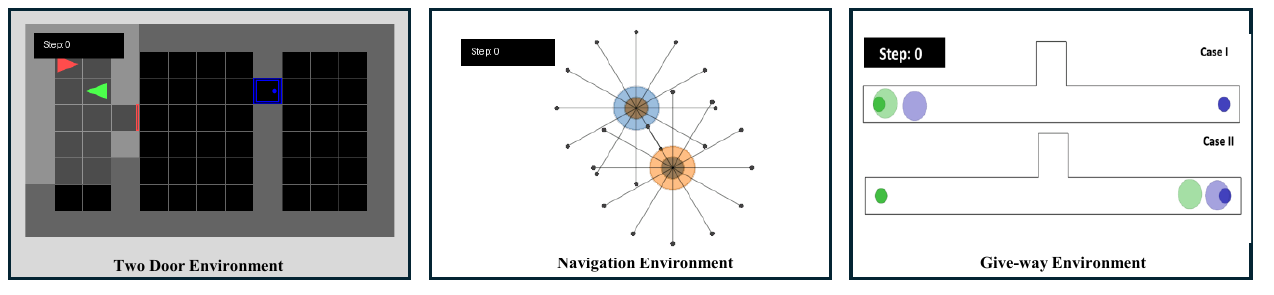}
  \caption{Starting State of environments used to capture an adversarial agents intent.}\label{fig:adv_env}
\end{figure}
Figure \ref{fig:adv_env} is shown to be starting states of the respective enviorments that facilitate the capture of the adversarial dataset ($\mathcal{D^{\ddagger\ddagger}}$). The rational has been explained below:
\begin{enumerate}
    \item \textbf{Two-Door Sequential Environment}
    \begin{itemize}
        \item \textit{Adversarial Starting Configuration:} Both agents are initialized at random, non-overlapping coordinates and orientations inside the leftmost room. The red door (which controls access from the left to the middle room) is initialized in an already open state, while the blue door (controlling access from the middle to the right room) is initialized as closed.
        \item \textit{Conceptual Objective:} This configuration evaluates whether agents can recognize the environmental shortcut (an already open red door) and immediately navigate through the middle room to unlock the blue door, rather than redundantly attempting to trigger the red door's mechanism.
    \end{itemize}

    \item \textbf{Multi-Agent Navigation}
    \begin{itemize}
        \item \textit{Adversarial Starting Configuration:} The starting position of each agent is set to be exactly coincident with the target destination of the opposing agent. That is, Agent 0 begins at Goal 1, and Agent 1 begins at Goal 0.
        \item \textit{Conceptual Objective:} This configuration represents a complete positional swap. It tests the agents' ability to coordinate path selection and resolve conflicts from the very first timestep, as their direct, straight-line trajectories to their goals are in immediate, head-on opposition.
    \end{itemize}

    \item \textbf{Giveaway Environment}
    \begin{itemize}
        \item \textit{Adversarial Starting Configuration:} Both agents are initialized at the same extreme side of the corridor (either both on the far left or both on the far right). We enforce a relative spatial constraint where the Green agent starts strictly to the left of the Blue agent.
        \item \textit{Conceptual Objective:} Placing both agents on the same side of the bottleneck with the Green agent positioned to the left of the Blue agent prevents direct, independent path execution. Because the agents' starting coordinates and respective destinations create an asymmetric blockage before they even reach the central passage, they are forced to negotiate the bottleneck from a highly constrained, asymmetric initial state.
    \end{itemize}
\end{enumerate}

Testing utilizes three OOD 50-episode datasets (Table \ref{tab:dataset-overview}) to isolate causal grounding:

\begin{table}[h]
\centering
\caption{Overview of dataset configurations used for stress-testing the Causal World Model dynamics.}
\label{tab:dataset-overview}
\begin{tabular}{@{}ll@{}}
\toprule
\textbf{Dataset} & \textbf{Purpose} \\ 
\midrule
\rowcolor{green!20} $\mathcal{D}^\ddagger_{\sigma=0.5}$ (Biased-Optimal) & Measures predictive precision under moderate distributional drift. \\
\rowcolor{orange!30} $\mathcal{D}^\ddagger_{\sigma=0.9}$ (Random) & Forces reliance on invariant physics by removing strategic intent. \\
\rowcolor{red!20} $\mathcal{D}^{\ddagger\ddagger}$ (Adversarial) & Stress test the modeling dynamics under malicious agent's intent. \\
\bottomrule
\end{tabular}
\end{table}

\subsection{Quality of the Collected datasets ($\mathcal{D}^N_\sigma$)}
Please refer to table \ref{tab:demo_metrics} for the quality of demonstrations. We emphasis when similar datasets will be collected in different user cases of economic data, robotics system data each agent will have some cost associated with it and 

\definecolor{EnvHeader}{RGB}{230, 230, 230}
\definecolor{GiveWayRow}{RGB}{245, 245, 245}  
\definecolor{TwoDoorRow}{RGB}{230, 240, 255}   
\definecolor{NavRow}{RGB}{230, 255, 230}       

\begin{table}[htbp]
    \centering

    \renewcommand{\arraystretch}{1.2}
    \small
    \begin{tabular}{l | cc | cc | cc}
    \toprule
     & \multicolumn{2}{c|}{\textbf{Reward (Avg $\pm$ Std)}} & \multicolumn{2}{c|}{\textbf{Ep. Length (Avg $\pm$ Std)}} & \multicolumn{2}{c}{\textbf{Episodes (Avg $\pm$ Std)}} \\
    \textbf{$\sigma$} & \textbf{Full} & \textbf{Part} & \textbf{Full} & \textbf{Part} & \textbf{Full} & \textbf{Part} \\
    \midrule
    
    \rowcolor{EnvHeader} \multicolumn{7}{l}{\textbf{Give Way}} \\
    \rowcolor{GiveWayRow} 0.99 & -0.0011±0.0001 & -0.0003±0.0001 & 500.0±0.0 & 500.0±0.0 & 61.0±0.0 & 61.0±0.0 \\
    \rowcolor{GiveWayRow} 0.75 & -0.0018±0.0003 & -0.0014±0.0002 & 500.0±0.0 & 500.0±0.0 & 61.0±0.0 & 61.0±0.0 \\
    \rowcolor{GiveWayRow} 0.50 & 0.0025±0.0004 & 0.0013±0.0006 & 500.0±0.0 & 500.0±0.0 & 61.0±0.0 & 61.0±0.0 \\
    \rowcolor{GiveWayRow} 0.25 & -0.0020±0.0004 & -0.0019±0.0002 & 500.0±0.0 & 500.0±0.0 & 61.0±0.0 & 61.0±0.0 \\
    \rowcolor{GiveWayRow} 0.01 & -0.0051±0.0004 & -0.0049±0.0001 & 500.0±0.0 & 500.0±0.0 & 61.0±0.0 & 61.0±0.0 \\
    
    \rowcolor{EnvHeader} \multicolumn{7}{l}{\textbf{Two-Door}} \\
    \rowcolor{TwoDoorRow} 0.99 & 0.900±0.006 & 0.852±0.013 & 50.5±2.9 & 74.5±7.2 & 1683.0±103.5 & 1128.8±93.6 \\
    \rowcolor{TwoDoorRow} 0.75 & 0.327±0.334 & 0.655±0.285 & 247.9±53.6 & 158.6±93.3 & 734.6±715.1 & 870.2±511.0 \\
    \rowcolor{TwoDoorRow} 0.50 & 0.270±0.054 & 0.428±0.348 & 314.1±34.2 & 250.7±108.9 & 270.8±32.8 & 540.8±404.9 \\
    \rowcolor{TwoDoorRow} 0.25 & 0.170±0.076 & 0.205±0.114 & 249.2±5.6 & 290.5±47.9 & 327.8±38.3 & 286.0±45.4 \\
    \rowcolor{TwoDoorRow} 0.01 & 0.111±0.017 & 0.109±0.008 & 243.6±10.6 & 239.3±8.8 & 347.8±11.2 & 353.0±13.9 \\
    
    \rowcolor{EnvHeader} \multicolumn{7}{l}{\textbf{Navigation}} \\
    \rowcolor{NavRow} 0.99 & 0.0038±0.0001 & 0.0026±0.0001 & 181.3±2.8 & 234.7±6.1 & 1756.2±28.7 & 1264.4±34.6 \\
    \rowcolor{NavRow} 0.75 & 0.0062±0.0001 & 0.0058±0.0001 & 129.9±0.7 & 135.7±1.7 & 2434.6±14.3 & 2107.0±26.5 \\
    \rowcolor{NavRow} 0.50 & 0.0080±0.0002 & 0.0077±0.0001 & 159.2±1.2 & 174.0±1.8 & 1986.6±15.6 & 1642.4±17.4 \\
    \rowcolor{NavRow} 0.25 & 0.0087±0.0004 & 0.0083±0.0002 & 300.1±5.3 & 319.9±2.5 & 1054.4±19.2 & 892.4±6.8 \\
    \rowcolor{NavRow} 0.01 & 0.0001±0.0011 & -0.0028±0.0022 & 499.4±0.4 & 499.6±0.3 & 632.2±0.4 & 571.2±0.2 \\
    \bottomrule
    \end{tabular}
    \caption{Performance metrics organized by environment and $\sigma$ configuration.}
    \label{tab:demo_metrics}
\end{table}

\subsection{Evaluation Dataset Configuration: two, three, four agent navigation}

\begin{figure}[htbp]
  \centering
  \includegraphics[trim={0pt 175pt 320pt 00pt}, clip, width=0.6\linewidth]{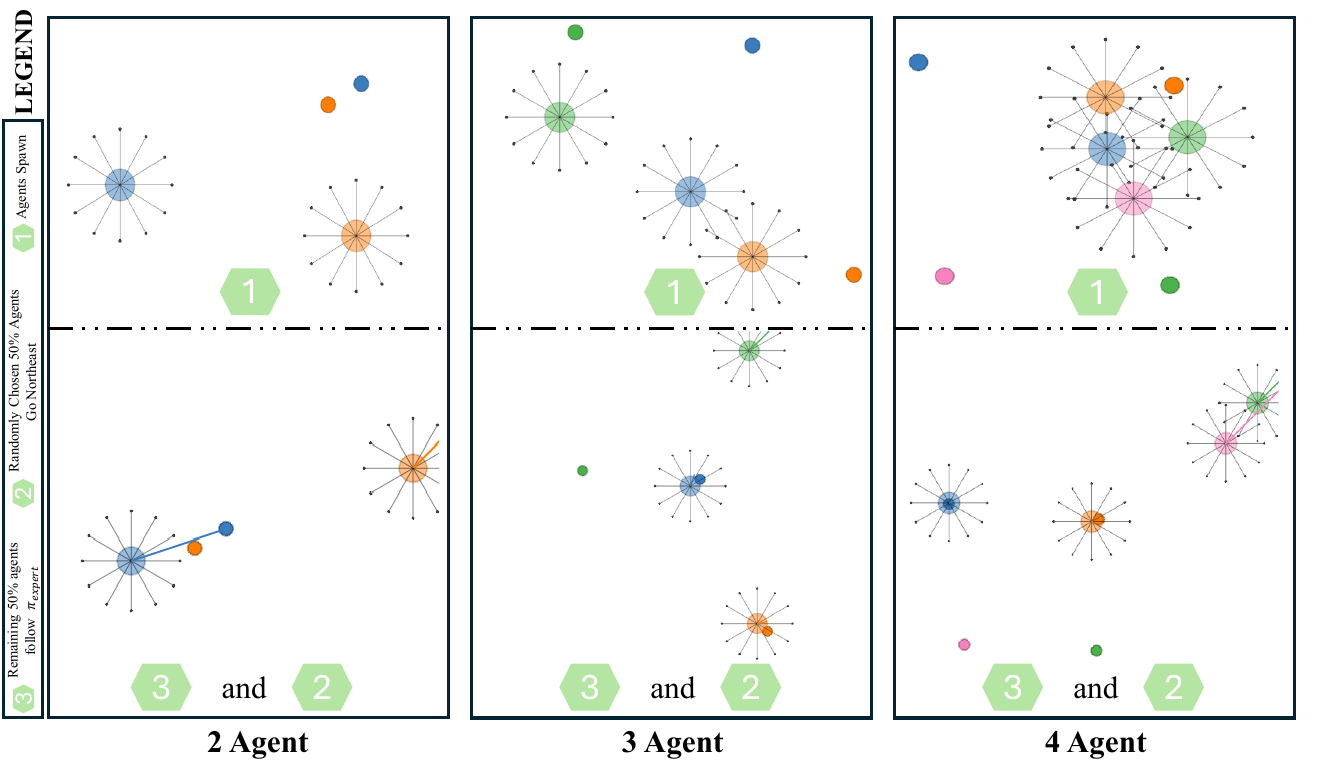}
  \caption{Evaluation dataset configuration for multi-agent navigation scenarios involving two, three, and four agents. The visualization illustrates the operational protocol: (1) agent spawn, (2) random trajectory assignment where 50\% of agents execute a Northeast move, and (3) adherence to the expert policy ($\pi_{\text{expert}}$) for the remaining 50\% of agents.}\label{fig:adv_env_multiple}
\end{figure}
To evaluate the environment configurations illustrated in Figure \ref{fig:adv_env_multiple}, we collected the dataset $\mathcal{D}^{\ddagger\ddagger}$, comprising 50 episodes with a 500-step rollout length. This configuration is calibrated to model robust transition dynamics; effectively capturing the dependencies inherent in coordinated across-agents choices is essential for the world model to perform better on this dataset.
\section{Causal Discovery Setup and Results}

\subsection{Causal Discovery Experimental Parameters}
Table~\ref{tab:causal_discovery_params} summarizes the configuration parameters used for the PC and PCMCI causal discovery algorithms across all experimental environments and observation modes.

\definecolor{lightgray}{gray}{0.8}
\definecolor{groupcolor}{RGB}{240, 248, 255} 

\begin{table}[htbp]
\centering
\renewcommand{\arraystretch}{1.2}
\small
\begin{tabular}{l|cc|cc|cc}
\toprule
\rowcolor{lightgray}
\textbf{Parameter} & \multicolumn{2}{c|}{\textbf{Two-Door}} & \multicolumn{2}{c|}{\textbf{Give-way}} & \multicolumn{2}{c}{\textbf{Navigation}} \\
\rowcolor{lightgray}
 & \textbf{Full} & \textbf{Part} & \textbf{Full} & \textbf{Noisy} & \textbf{Full} & \textbf{Part} \\
\midrule
\rowcolor{lightgray} \multicolumn{7}{c}{\textbf{[Causal Graph Parameters]}} \\
\rowcolor{groupcolor} discovery\_algorithm & PC / PCMCI & PC / PCMCI & PC / PCMCI & PC / PCMCI & PC / PCMCI & PC / PCMCI \\
\rowcolor{groupcolor} CI\_test\_name & Gsquared & Gsquared & robust\_parcorr & robust\_parcorr & robust\_parcorr & robust\_parcorr \\
\rowcolor{groupcolor} Ci\_test\_significance & analytic & analytic & analytic & analytic & analytic & analytic \\
\rowcolor{groupcolor} alpha\_level & 0.005 & 0.005 & 0.005 & 0.005 & 0.005 & 0.005 \\
\rowcolor{groupcolor} pc\_alpha & 0.1 & 0.1 & 0.1 & 0.1 & 0.1 & 0.1 \\
\rowcolor{groupcolor} tau\_min & 1 & 1 & 1 & 1 & 1 & 1 \\
\rowcolor{groupcolor} tau\_max & 1 & 1 & 1 & 1 & 1 & 1 \\
\rowcolor{groupcolor} max\_cond\_parent\_count & 3 & 3 & 3 & 3 & 3 & 3 \\
\rowcolor{groupcolor} missing\_value & -99 & -99 & -99 & -99 & -99 & -99 \\
\rowcolor{groupcolor} missing\_padding\_length & 5 & 5 & 5 & 5 & 5 & 5 \\
\rowcolor{groupcolor} discreate\_reward & true & true & false & false & false & false \\
\bottomrule
\end{tabular}
\caption{Causal Discovery Configuration Parameters for PC and PCMCI Algorithms}
\label{tab:causal_discovery_params}
\end{table}

\subsection{Evaluation Results: on three primary enviorments}
Figure~\ref{fig:discovery_metric_result} presents the evaluation results for the recovered causal graphs, illustrating how discovery algorithm performance varies across different interventional strengths. A crucial misinterpretation would be to view Figure~\ref{fig:discovery_metric_result} as evidence that the PC algorithm is superior to the PCMCI algorithm; this appearance is an artifact of the experimental design, where the number of transitions is fixed and normalized to isolate the effects of information gain relative to interventional strength.

\begin{figure}[htbp]
  \centering
  \includegraphics[trim={0pt 0pt 0pt 30pt}, clip, width=\linewidth]{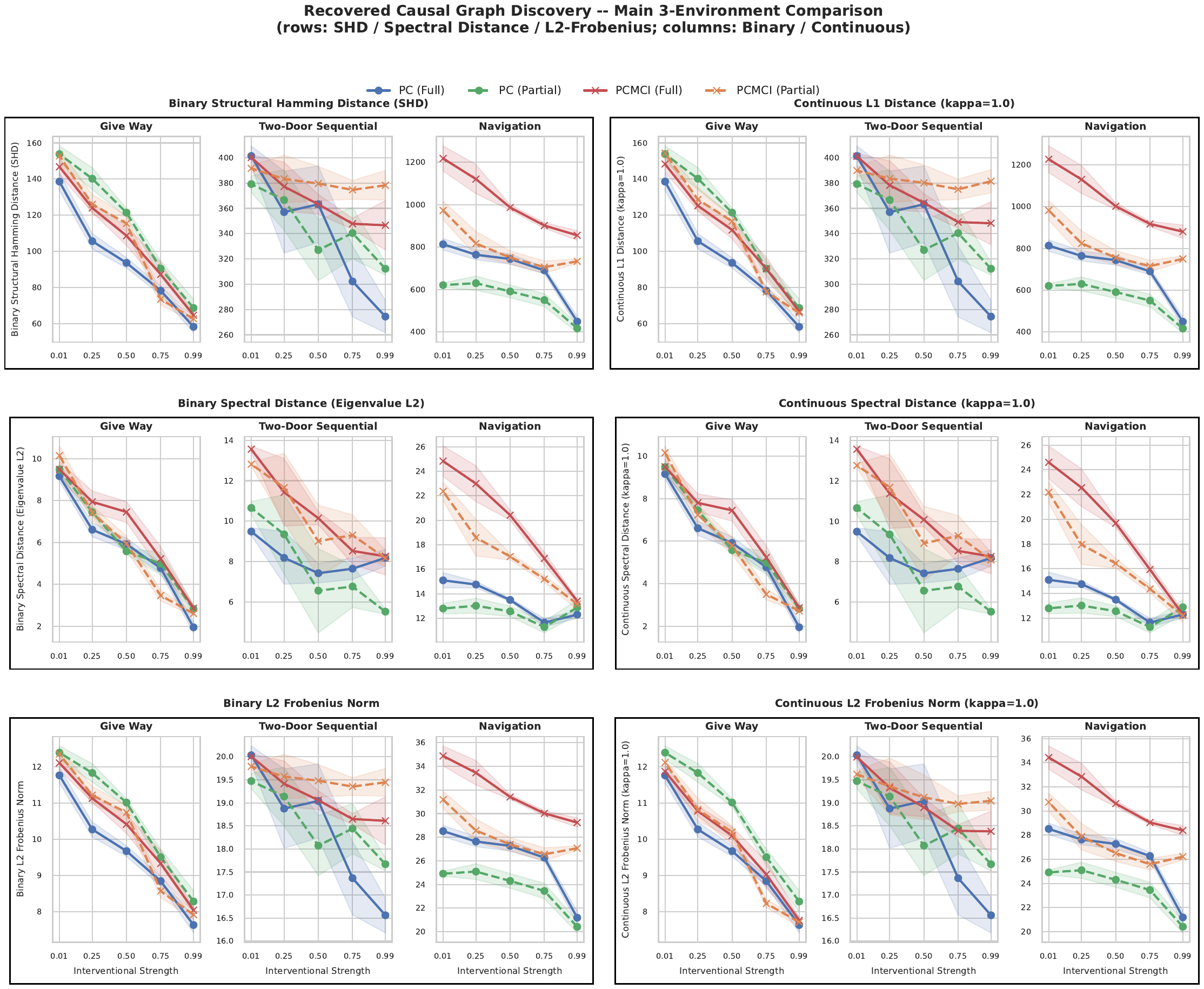}
  \caption{\textbf{Recovered causal-graph error falls as interventional strength grows, across every environment, metric, and observation mode.}
  Each of the six bordered panels reports one graph-error metric against the interventional strength $\sigma$ of the training data, for the three coordination domains (give-way, two-door, navigation).
  Rows: Structural Hamming Distance (SHD), spectral distance (eigenvalue $L_2$), and $L_2$ Frobenius norm; columns: the binary adjacency and the continuous soft-mask ($\kappa\!=\!1.0$) formulations of each.
  Curves compare the PC and PCMCI discovery algorithms under full and partial/noisy observation; shaded bands are $\pm1$ standard deviation over five seeds.
  Graph error decreases monotonically with $\sigma$ in nearly every cell: stronger soft interventions turn observational demonstrations into an identifiable causal structure, and the effect is consistent whether error is measured locally (SHD, Frobenius) or spectrally, and whether the graph is read as a binary skeleton or a continuous mask.}\label{fig:discovery_metric_result}
\end{figure}

\subsection{Evaluation Results: High sample complexity experiments}

\begin{figure}[htbp]
  \centering
  \includegraphics[trim={0pt 0pt 0pt 30pt}, clip, width=\linewidth]{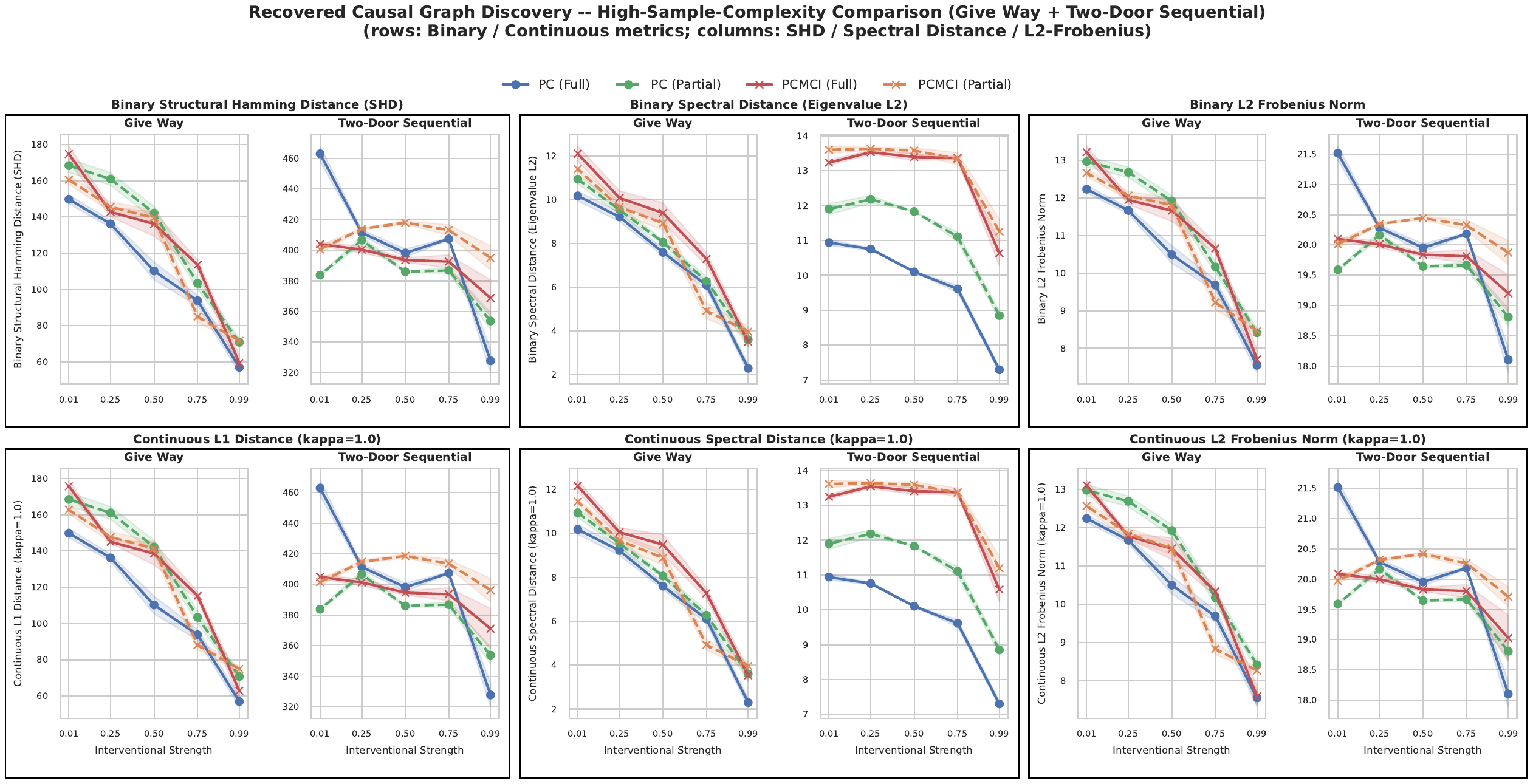}
  \caption{\textbf{At high sample complexity, the intervention--recovery trend holds and the discovered graphs sharpen further.}
  Recovered-graph error versus interventional strength $\sigma$ for the give-way and two-door domains trained on the high-sample datasets (give-way $\approx$30k--120k transitions, two-door $\approx$84k--2M).
  The six bordered panels span two metric families (top row: binary adjacency; bottom row: continuous soft-mask, $\kappa\!=\!1.0$) across three metrics per row (Structural Hamming Distance, spectral distance, and $L_2$ Frobenius norm).
  Curves compare the PC and PCMCI discovery algorithms under full and partial/noisy observation; shaded bands are $\pm1$ standard deviation over five seeds, and are visibly tighter than in the standard-sample regime.
  The downward trend with $\sigma$ persists in every cell, and partial-observation recovery closely tracks full observation because both domains share the same state-space dimension---so the parent--child dependencies the algorithm must recover are invariant across the two modes.}\label{fig:discovery_metric_result_high_complexicity}
\end{figure}

The results in Figure \ref{fig:discovery_metric_result_high_complexicity} show that increasing sample complexity significantly improves recovery stability. The thin uncertainty bands in Figure confirm that the algorithms produce consistent results across different seeds.

Under partial observation settings in Figure, the models successfully recover the causal structure. This is possible because both environments share the same state space dimension. According to causal discovery theory, this recovery is effective as long as the parent-child relationships remain invariant across environment steps. By relying on these fixed dependencies, the framework generalizes the causal skeleton across different tasks despite incomplete state observability.

\begin{figure}[htbp]
  \centering
  \includegraphics[trim={0pt 0pt 0pt 30pt}, clip, width=\linewidth]{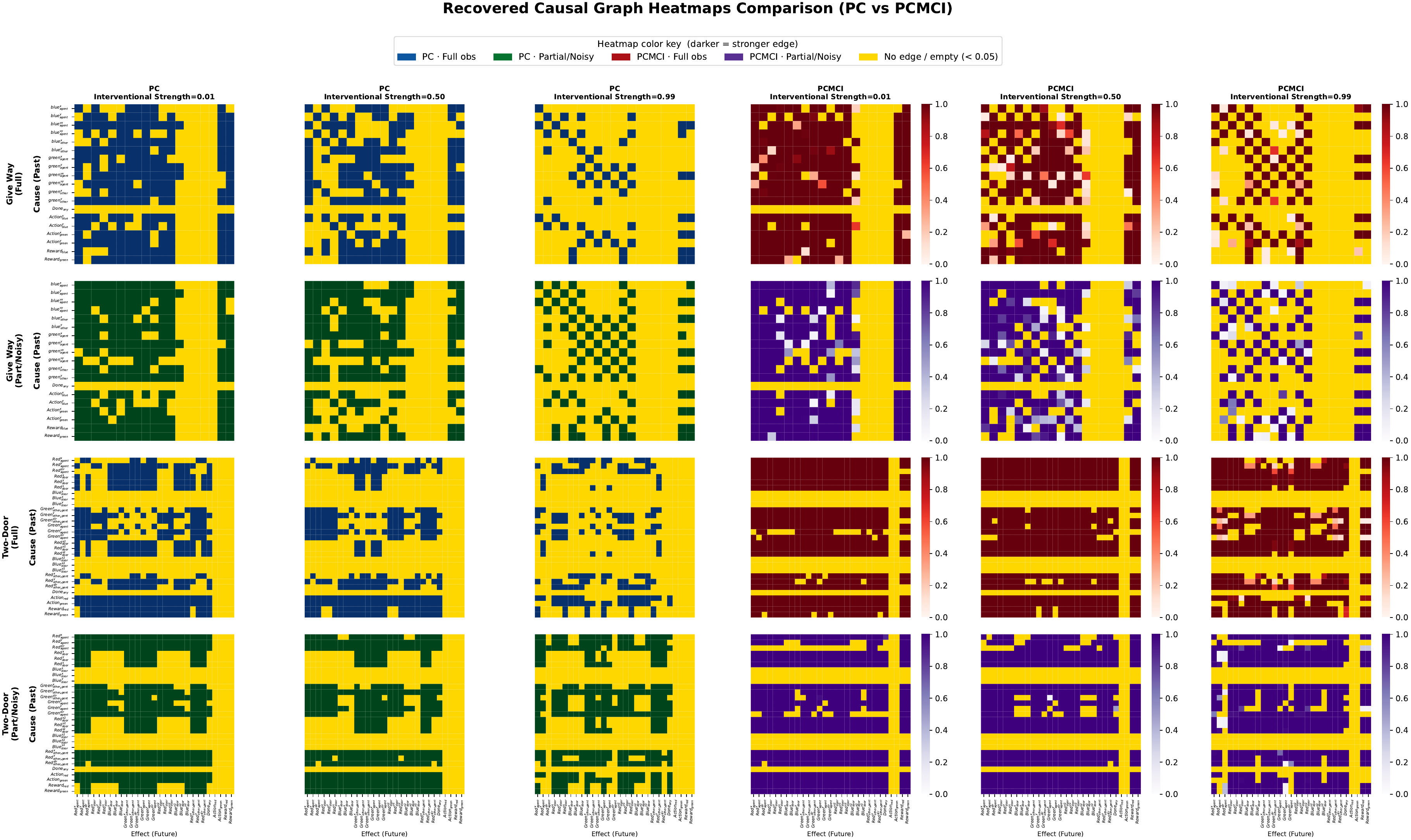}
  \caption{This figure displays recovered causal graph structures for Give-Way and Two-Door environments under full and partial/noisy observation settings. The heatmaps demonstrate algorithm performance across varying intervention strengths (0.01, 0.50, 0.99), with the color key indicating edge strength and yellow denoting empty or non-significant entries (< 0.05).}\label{fig:discovery_result_high_complexicity}
\end{figure}
The recovered causal graphs in Figure \ref{fig:discovery_result_high_complexicity} demonstrate that stronger interventional strength consistently leads to a graph structure that aligns more closely with the underlying environment physics. By systematically perturbing the system, we reveal causal mechanisms that remain obscured under purely observational data.

We adopt distinct interpretations for the two discovery approaches. The PC algorithm outputs are best suited for testing discrete edge hypotheses, providing a clear, binary structure of causal dependencies.  In contrast, PCMCI coupled with soft masking offers a fuzzy perspective of the causal graph. This probabilistic representation is advantageous, as it allows the subsequent parameter learning phase to absorb and mitigate minor masking errors, resulting in a more robust and flexible world model.  By assigning continuous weights to potential edges, the model avoids the brittleness of hard-thresholding; if a specific edge is misidentified, the model retains the ability to learn away this inaccuracy during the training of the dynamics model rather than being forced to rely on a fundamentally incorrect structural assumption.

\subsection{Evaluation Results: two, three, four agent navigation}
\begin{figure}[htbp]
    \centering
    \begin{subfigure}[b]{0.48\textwidth}
        \centering
        \fbox{\includegraphics[trim={0pt 0pt 0pt 30pt}, clip, width=\textwidth]{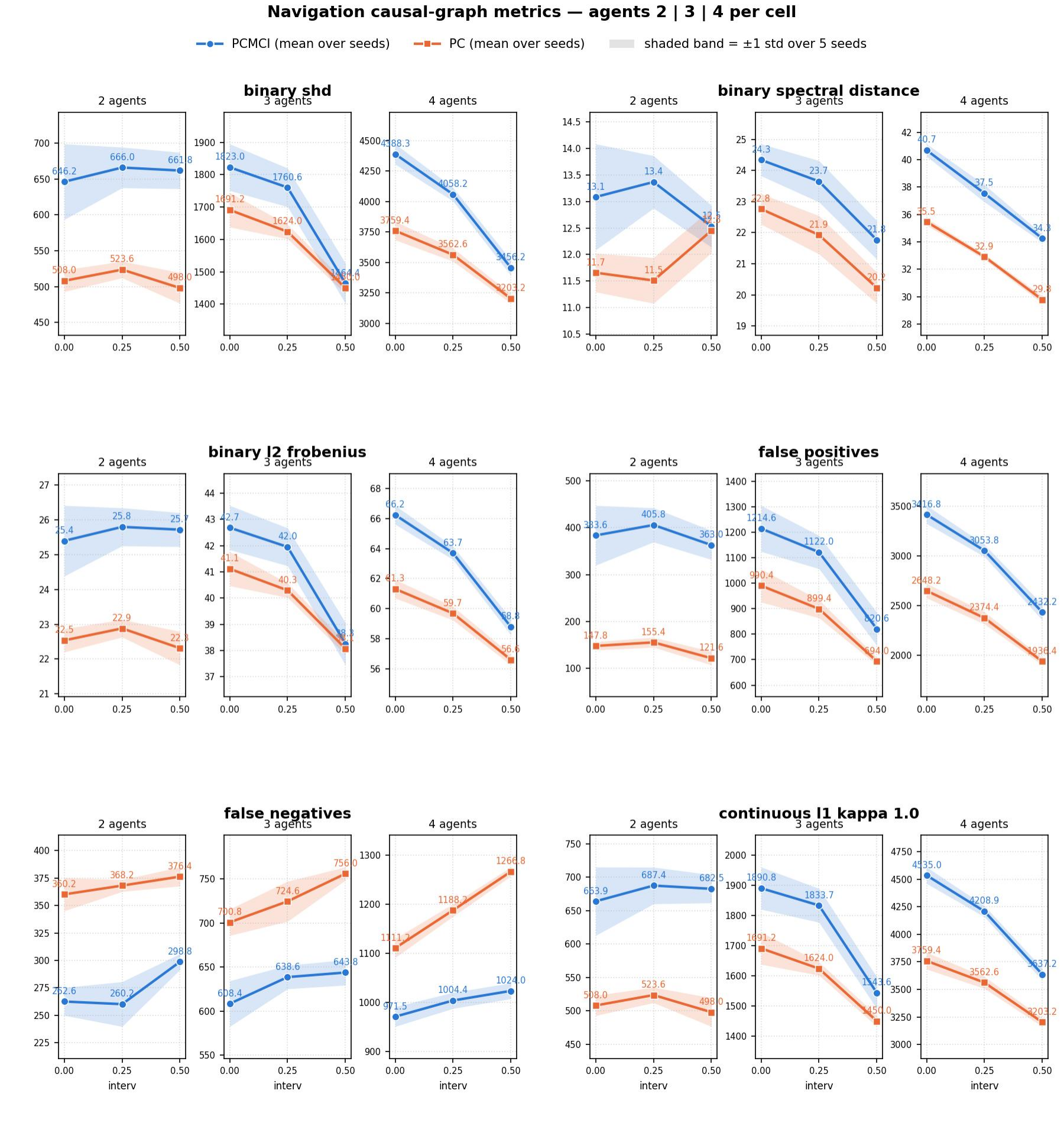}}
        \caption{Has LiDAR dimensions}
    \end{subfigure}
    \hfill
    \begin{subfigure}[b]{0.48\textwidth}
        \centering
        \fbox{\includegraphics[trim={0pt 0pt 0pt 30pt}, clip, width=\textwidth]{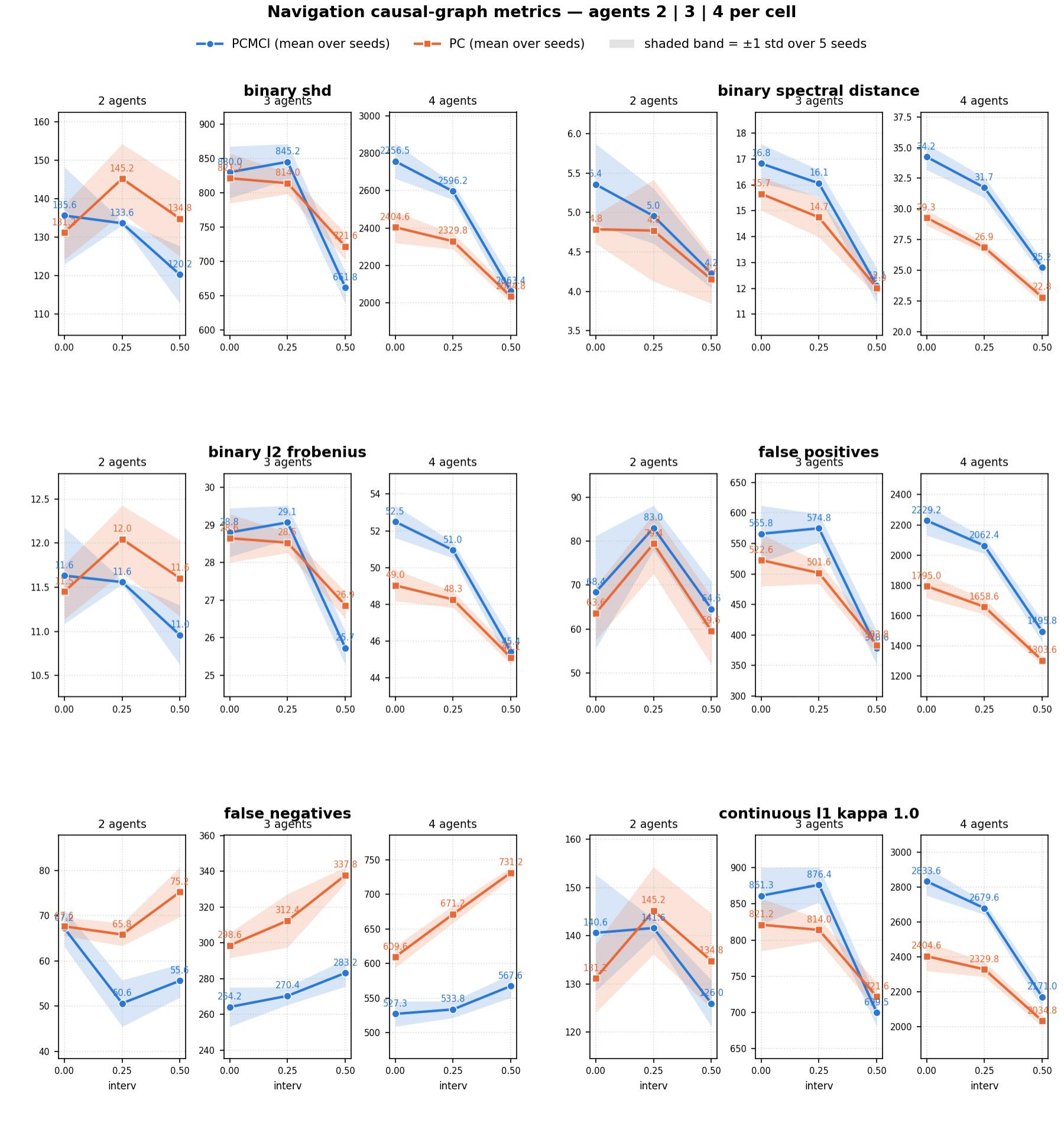}}
        \caption{No LiDAR dimensions}
    \end{subfigure}
   \caption{Performance metrics for PCMCI and PC algorithm's recovered graph in multi-agent navigation. Metrics include Binary SHD, Spectral Distance, Frobenius Norm, and error rates testing for intervention strength.}
    \label{fig:graph_metric_navigation_2_3_4}
\end{figure}

From figure \ref{fig:graph_metric_navigation_2_3_4}(a) the graph recovery metric error does not drop to zero because the environment is deterministic. Using 12 LiDAR scans per agent violates the faithfulness assumption, which requires independent noise. The tests mistake physical constraints for causal dependencies. A simpler way to describe the issue would be a near proximity agent will likely showup accross multiple LIDAR scan angles for the ego agent we are interested in. Interventions cannot fix these errors from finite data because the underlying physics remains deterministic (One LiDAR scan dimension can be represented as a linear combination of other LiDAR scan dimensions and system variables).

Despite these challenges, our modeling performs well on the remaining features. This success calls us to represent the state space of this task in a more rank-sufficient way in the future. By compressing redundant LiDAR information into a lower-dimensional representation, we can mitigate the impact of deterministic correlations on the causal discovery process (figure \ref{fig:graph_metric_navigation_2_3_4}(b)).

We formalize this rank-sufficient state space $s_t$ at time step $t$ as:
\begin{equation}
    s_t = \{p_t, \phi_t, \tilde{\mathbf{L}}_t, \mathbf{A}_t\}
\end{equation}
where:
\begin{itemize}
    \item $p_t = (x_t, y_t)$: Agent spatial coordinates.
    \item $\phi_t$: Agent orientation.
    \item $\tilde{\mathbf{L}}_t$: The compressed, rank-sufficient representation of the raw LiDAR vector.
    \item $\mathbf{A}_t$: Relative state information of other agents.
\end{itemize}

\section{World Model Setup and Results}
\label{app:wm_results}
\subsection{Evaluation Results: on three primary environments}

\label{app:normal3-wm-sweep}

\begin{figure}[htbp]
    \centering
    \includegraphics[trim={0pt 0pt 0pt 33pt}, clip,width=0.95\textwidth]{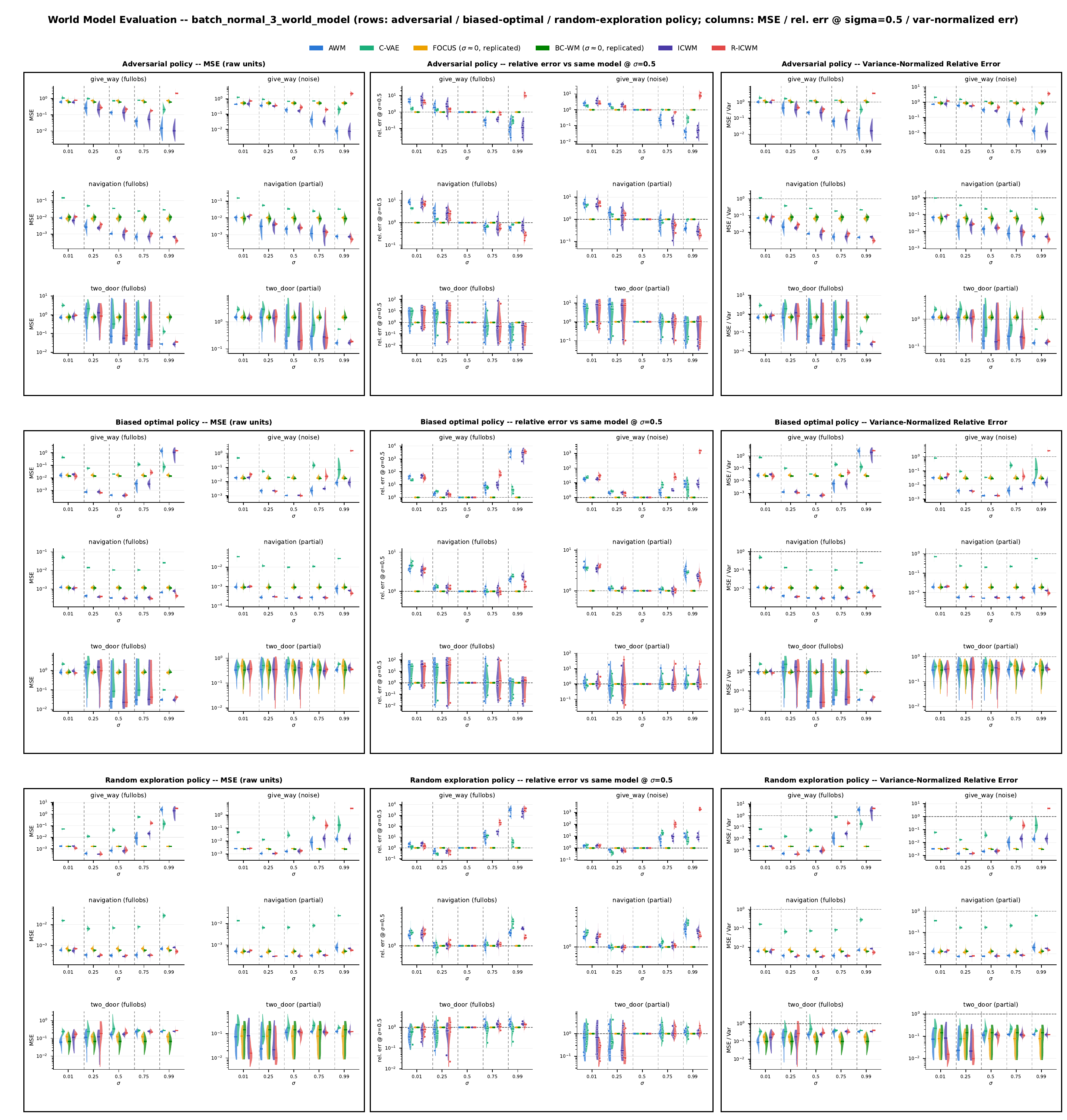}
    \caption{%
\textbf{World-model one-step prediction error across the canonical three-environment
sweep.} Six architectures (AWM, C-VAE, FOCUS, BC-WM, ICWM, R-ICWM) trained on
\texttt{give\_way}, \texttt{navigation}, and \texttt{two\_door} under full and partial
observability, at intervention strengths $\sigma \in \{0.01,0.25,0.5,0.75,0.99\}$
($\sigma = 1-\alpha$). Rows: evaluation policy (adversarial, biased-optimal, random).
Columns: metric (raw MSE, relative to $\sigma=0.5$, variance-normalized). Each panel is a
$3\times2$ grid of half-violins over environment and observability. FOCUS and BC-WM are
trained only at $\sigma \approx 0$ and shown as a flat reference. Error drops one to two
orders of magnitude with increasing $\sigma$ for \texttt{give\_way} and
\texttt{navigation}, where ICWM and R-ICWM attain the lowest error; \texttt{two\_door}
shows no such trend, consistent with its sparsity residing in the input to output
transition matrix rather than input marginals.}
\label{fig:app-normal3-wm-sweep}
\end{figure}
Figure~\ref{fig:app-normal3-wm-sweep} reports one-step prediction error for six world-model architectures: AWM, C-VAE, FOCUS, BC-WM, ICWM, and R-ICWM. Each model is trained on three environments (\texttt{give\_way}, \texttt{navigation}, \texttt{two\_door}) under full and partial observability, at five training intervention strengths $\sigma \in \{0.01, 0.25, 0.5, 0.75, 0.99\}$, where $\sigma $ is the probability that an agent takes a random (intervention in our setting) rather than optimal action. Each panel shows a $3\times2$ grid of half-violin distributions, one for each environment and observability combination.

The three rows of panels correspond to the policy used to generate evaluation trajectories: adversarial $\mathcal{D}^{\ddagger\ddagger}$ (top), biased-optimal $\mathcal{D}^\ddagger_{\sigma=0.5}$ matched to each model's training $\sigma$ (middle), and random exploration $\mathcal{D}^\ddagger_{\sigma=0.9}$ (bottom). The three columns correspond to the error metric: raw MSE (left), MSE relative to the same model's error at $\sigma=0.5$ (center), and variance-normalized relative error (right).

\textbf{Reading the trends.} Error falls by one to two orders of magnitude as $\sigma$ increases toward $0.99$ for \texttt{give\_way} and \texttt{navigation}. This shows that training on more strongly intervened trajectories produces better one-step dynamics models in these two environments. \texttt{two\_door} does not follow this pattern. Its error stays within roughly one order of magnitude across the full $\sigma$ range, and the ranking of architectures is less stable. We attribute this to \texttt{two\_door}'s causal sparsity living in the input by output transition matrix rather than in individual input variables, so the soft input gate used by ICWM and R-ICWM captures little of the useful structure in this environment.

ICWM and R-ICWM achieve the lowest error at high $\sigma$ under the adversarial and biased-optimal evaluation policies. This is clearest in \texttt{give\_way} and \texttt{navigation}, where they fall roughly one order of magnitude below AWM by $\sigma=0.99$. FOCUS and BC-WM are trained only near $\sigma \approx 0$, so they behave as flat, $\sigma$-invariant reference models. They are competitive only at low $\sigma$, and their relative error grows by one to three orders of magnitude as the evaluation policy moves further from the deterministic-optimal regime. C-VAE has the highest error and highest variance across nearly every panel, particularly in \texttt{navigation}. This suggests its stochastic latent variable does not improve one-step prediction accuracy in this setting.

The center and right columns confirm the same pattern under two different normalizations. Relative error against each model's own $\sigma=0.5$ checkpoint isolates how error scales with distributional shift. Variance normalization further controls for the fact that next-state variance itself grows with $\sigma$. Both normalizations preserve the same architecture ranking and the same \texttt{two\_door} exception, so the effect is not an artifact of the normalization choice. Under random exploration (bottom row), differences between architectures shrink at low to moderate $\sigma$ for \texttt{give\_way} and \texttt{navigation}, since random rollouts are less sensitive to the fine-grained transition structure that graph-informed gating exploits. The ICWM and R-ICWM advantage reappears at $\sigma \geq 0.75$, as random exploration increasingly resembles the high-intervention training regime.

\begin{figure}[htbp]
    \centering
    \includegraphics[trim={0pt 0pt 0pt 33pt}, clip,width=0.95\textwidth]{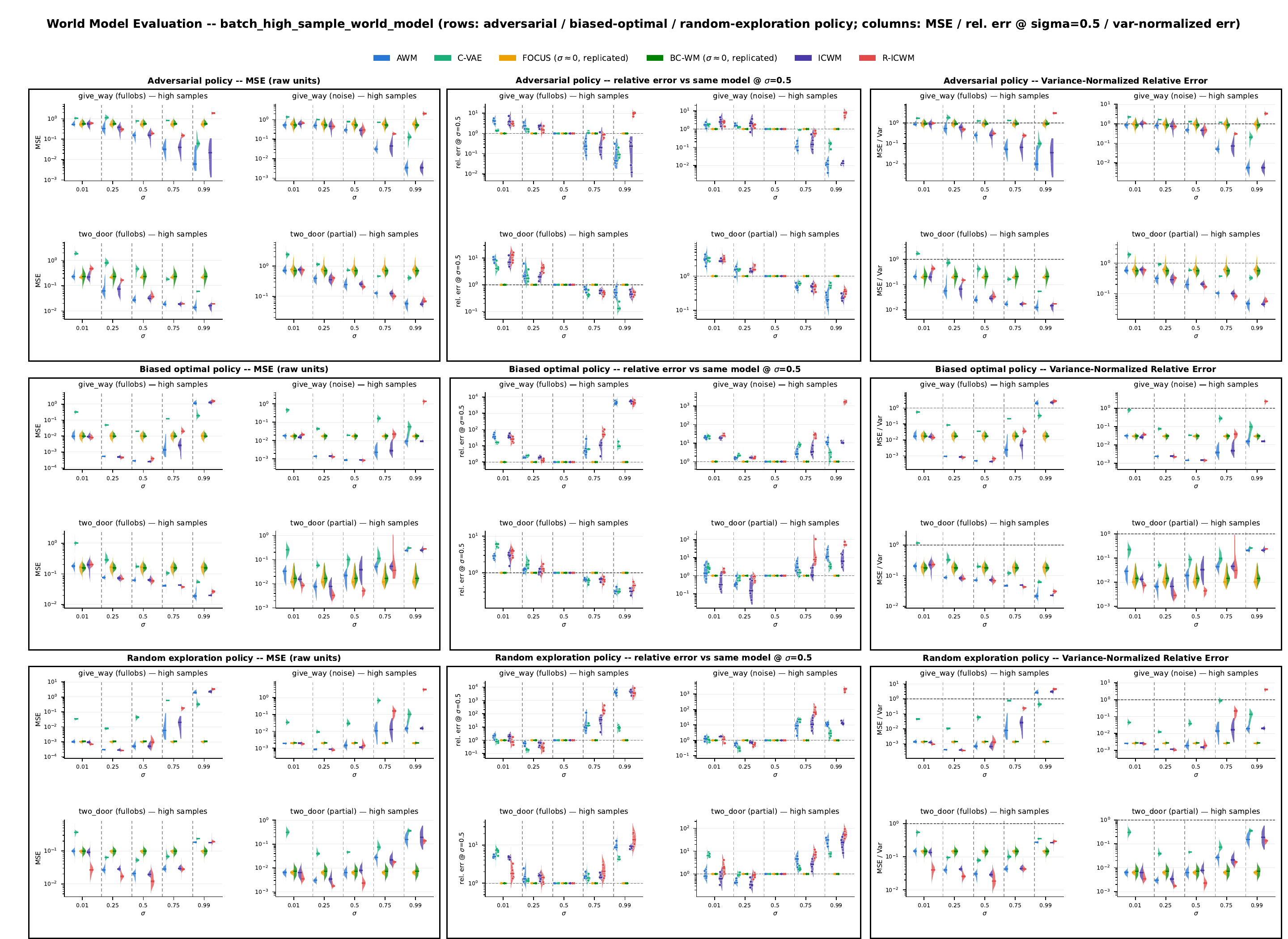}
    \caption{%
    \textbf{World-model one-step prediction error in the high-sample regime.}
    Six architectures trained on enlarged datasets for \texttt{give\_way}
    (roughly 120k transitions) and \texttt{two\_door} (roughly 2M transitions) under
    full and partial observability, at intervention strengths
    $\sigma \in \{0.01,0.25,0.5,0.75,0.99\}$ ($\sigma = 1-\alpha$). Layout matches
    Figure~\ref{fig:app-normal3-wm-sweep}. Rows: evaluation policy (adversarial,
    biased optimal, random). Columns: metric (raw MSE, relative to $\sigma=0.5$,
    variance-normalized). FOCUS and BC-WM are trained only at $\sigma \approx 0$ and
    shown as a flat reference. With more data, \texttt{two\_door} gains the monotone
    improvement with $\sigma$ that it lacked at normal sample sizes and its
    distributions tighten sharply, while \texttt{give\_way} reproduces the earlier
    trends, with ICWM lowest at high $\sigma$ under adversarial evaluation.}
    \label{fig:app-highsample-wm-sweep}
\end{figure}

\subsection{Evaluation Results: High sample complexity experiments}
\label{app:highsample-wm-sweep}

Figure~\ref{fig:app-highsample-wm-sweep} repeats the evaluation of
Appendix~\ref{app:normal3-wm-sweep} on high-sample training datasets. The
\texttt{give\_way} datasets grow from roughly 30k to roughly 120k transitions, and the
\texttt{two\_door} datasets grow from roughly 84k to roughly 2M transitions.
\texttt{navigation} is not included in this sweep. All six architectures (AWM, C-VAE,
FOCUS, BC-WM, ICWM, R-ICWM) are retrained from scratch at the same five intervention
strengths $\sigma \in \{0.01, 0.25, 0.5, 0.75, 0.99\}$, under full and partial
observability. The figure layout matches Appendix~\ref{app:normal3-wm-sweep}. Rows show
the evaluation rollout policy: adversarial (top), biased optimal (middle), and random
exploration (bottom). Columns show the error metric: raw MSE (left), MSE relative to the
same model at $\sigma=0.5$ (center), and variance-normalized relative error (right).
Each panel now contains a $2\times2$ grid of half-violins over the four
(environment, observability) groups.

\textbf{The main purpose of this sweep is to test whether the \texttt{two\_door}
exception observed in Appendix~\ref{app:normal3-wm-sweep} is a sample-complexity
artifact. It largely is.} With roughly 24 times more training data, \texttt{two\_door}
acquires the monotone trend it previously lacked. Under the adversarial evaluation,
its MSE falls by about one order of magnitude from $\sigma=0.01$ to $\sigma=0.99$ in
both observability settings, the violins tighten sharply, and the architecture ranking
stabilizes, with ICWM, R-ICWM, and AWM forming the low-error group at high $\sigma$.
The wide, unstable distributions seen at normal sample sizes are gone. This indicates
that the earlier flat profile reflected undertrained models rather than a property of
the environment, although the gap between graph-informed models and AWM remains small
in \texttt{two\_door}, consistent with its causal sparsity residing in the input to
output transition matrix rather than in the input marginals that the soft gate can
express.

The \texttt{give\_way} results reproduce the normal-sample trends. Under the
adversarial evaluation, error falls by two to three orders of magnitude as $\sigma$
increases, and ICWM reaches the lowest error of any model at $\sigma=0.99$, roughly
one order of magnitude below AWM. FOCUS and BC-WM remain flat $\sigma$-invariant
references, competitive only at low $\sigma$, and their relative error grows by up to
three orders of magnitude at high $\sigma$. C-VAE remains the weakest and most
variable model overall. More data therefore does not close the gap created by
training-policy coverage: what the trajectories contain matters more than how many
there are.

The middle and bottom rows show a complementary policy-mismatch effect in
\texttt{give\_way}. When evaluation trajectories are generated by near-optimal
policies, error is lowest for models trained at moderate intervention strength
($\sigma$ between 0.25 and 0.5) and rises steeply at $\sigma=0.99$, where the training
data contains almost no optimal behavior. This mirror image of the adversarial trend
confirms that each model is most accurate on the region of state space its training
policy visited, and that no single training $\sigma$ dominates across all evaluation
policies. The relative-error and variance-normalized columns preserve every one of
these patterns, so none of them is an artifact of the error normalization.

\subsection{Evaluation Results: two, three, four agent navigation}
\label{app:multiagent-lidar-wm}

\begin{figure}[htbp]
    \centering
    \includegraphics[trim={0pt 0pt 0pt 30pt}, clip,width=0.95\textwidth]{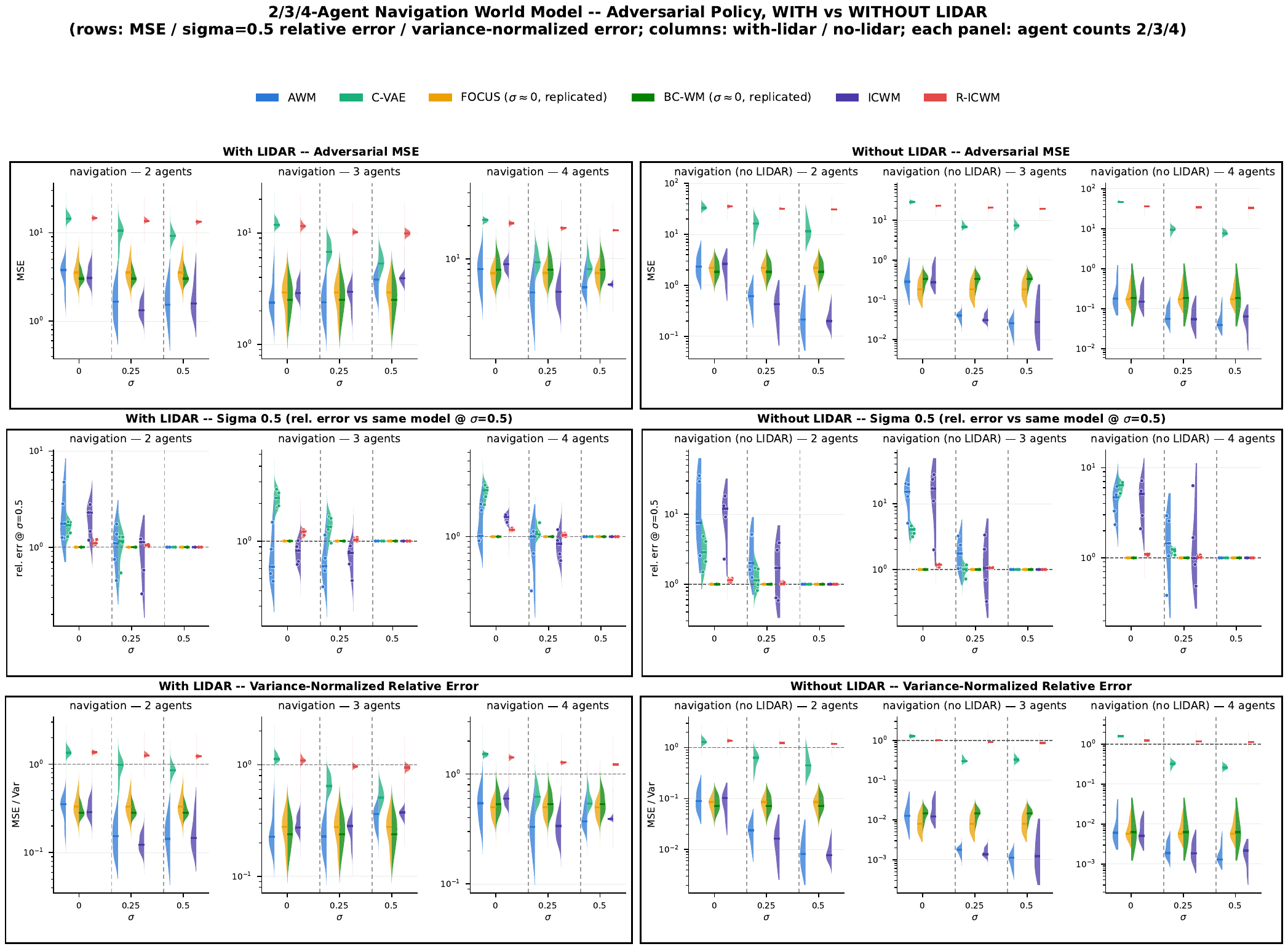}
    \caption{%
    \textbf{Multi-agent world-model error with and without LIDAR.}
    Six architectures trained on \texttt{navigation} with 2, 3, and 4 agents at
    intervention strengths $\sigma \in \{0, 0.25, 0.5\}$, evaluated on adversarial
    rollouts. Each metric (raw MSE, error relative to the same model at $\sigma=0.5$,
    variance-normalized error) is shown for the full observation including LIDAR and
    for a no-LIDAR variant with those dimensions removed from the state and the causal
    graph. FOCUS and BC-WM are trained only at $\sigma \approx 0$ and shown as a flat
    reference. With LIDAR, variance-normalized error stays near 0.1 to 1 and
    architectures are hard to separate. Without LIDAR it falls to $10^{-2}$ to
    $10^{-3}$, the gain from interventional data grows to roughly tenfold, and ICWM
    matches or beats AWM at every agent count.}
    \label{fig:app-multiagent-lidar-wm}
\end{figure}
Figure~\ref{fig:app-multiagent-lidar-wm} evaluates the six world-model architectures on
the multi-agent \texttt{navigation} environment with 2, 3, and 4 agents, at training
intervention strengths $\sigma \in \{0, 0.25, 0.5\}$. Hidden width scales with agent
count. All models are evaluated on adversarial rollouts only. Each condition is shown
twice: once for models trained on the full observation, which includes each agent's
LIDAR returns, and once for a no-LIDAR variant in which the LIDAR dimensions are removed
from both the state space and the causal graph before training. The no-LIDAR variant is
motivated by the discovery stage: LIDAR readings produce dense, low-quality causal
masks, and they are intrinsically hard to predict. The figure reports three metrics for
both variants: raw MSE, error relative to the same model at $\sigma=0.5$, and
variance-normalized relative error.

\textbf{With LIDAR, prediction error is dominated by the LIDAR dimensions themselves.}
Variance-normalized error sits between 0.1 and 1 for every architecture and every agent
count, meaning the models recover only a modest fraction of the predictable variance.
Differences between architectures are compressed, and increasing $\sigma$ yields only a
small improvement. This is consistent with the LIDAR returns acting as high-variance,
largely unpredictable targets that swamp the error signal from the agent dynamics.

\textbf{Dropping LIDAR reveals the structure that the full-state results obscure.}
Variance-normalized error falls to between $10^{-2}$ and $10^{-3}$ for AWM and ICWM,
one to two orders of magnitude below the with-LIDAR setting. The benefit of
interventional training data also becomes much larger: models trained at $\sigma=0$
are up to ten times worse than the same architecture trained at $\sigma=0.5$, whereas
the corresponding ratio with LIDAR is closer to two or three. ICWM matches or slightly
outperforms AWM across agent counts in this setting, most visibly at 3 and 4 agents.
The trends are stable as the number of agents grows from 2 to 4, indicating that the
approach scales with agent count once the uninformative sensor dimensions are removed.

Two architectures underperform throughout this sweep. C-VAE and R-ICWM sit roughly an
order of magnitude above the other models in raw MSE in both variants and remain near
a variance-normalized error of 1 even without LIDAR. For R-ICWM this contrasts with
its competitive performance in the single-agent sweeps of
Appendix~\ref{app:normal3-wm-sweep}, suggesting that its recurrent history mechanism
does not transfer well to the larger multi-agent state spaces. FOCUS and BC-WM,
trained only at $\sigma \approx 0$, again act as flat references and are outperformed
by the $\sigma$-matched models as intervention strength grows.

\section{Causal-Mask Ablation Study}
\label{app:mask-ablation}
\begin{figure}[htbp]
    \centering
    \includegraphics[width=0.9\linewidth]{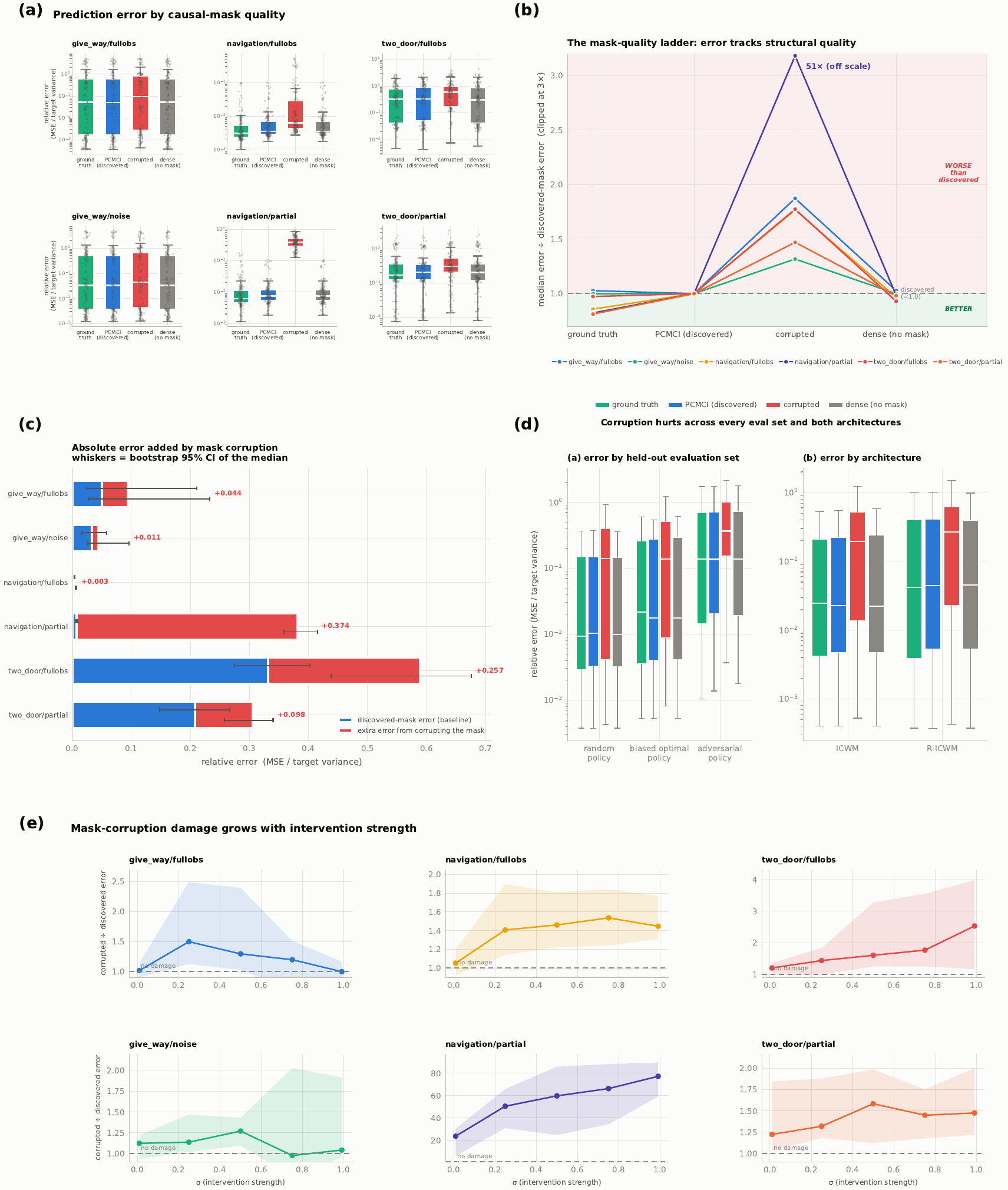}
    \caption{%
    \textbf{Causal-mask ablation study} (1{,}200 runs: 6 groups $\times$ 4 masks
    $\times$ 2 architectures $\times$ 5 intervention strengths $\times$ 5 seeds).
    (a) Relative error by mask for all six groups; row 1 full observability, row 2
    partial or noisy observability. (b) Median error normalized to each group's
    discovered-mask error; navigation/partial reaches 51$\times$ under corruption.
    (c) Absolute error added by corruption; whiskers are bootstrap 95\% CIs of the
    medians. (d) Breakdown by held-out evaluation policy and by architecture.
    (e) Corruption damage versus training intervention strength $\sigma$; band is the
    IQR over seeds, architectures, and evaluation sets. Error follows structural
    quality in every group, a corrupted mask is worse than no mask, and the damage
    grows with intervention strength.}
    \label{fig:app-mask-ablation}
\end{figure}

This study measures how sensitive a causal world model's accuracy is to the quality of
the causal mask it is given. For every combination of environment, observation mode,
architecture, intervention strength, and seed, we train the same model under four
masks: \emph{ground truth} (the environment's true causal adjacency), \emph{discovered}
(the mask an actual PCMCI causal-discovery run produces, the realistic default),
\emph{corrupted} (the discovered mask plus structural noise, with 30\% of gates softened
and 30\% flipped, simulating discovery errors), and \emph{dense} (an all-ones gate,
equivalent to a standard MLP or GRU with no structural prior). The sweep covers 6
(environment, observation) groups, 2 architectures (ICWM, R-ICWM), 5 intervention
strengths, and 5 seeds, for 1{,}200 runs in total. Every run is evaluated on three
held-out sets generated by a random policy, a biased optimal policy, and an adversarial
policy. The metric throughout is relative error, defined as MSE divided by the variance
of the prediction target.

\textbf{Mask quality orders prediction error.} Across all six groups the error ordering
follows the structural-quality ordering: ground truth $\leq$ discovered $\leq$
corrupted (Figure~\ref{fig:app-mask-ablation}a,b). Corrupting the mask inflates error
in every group and significantly in five (paired Wilcoxon, $p \leq 3\times10^{-4}$;
give\_way under full observability is marginal at $p=0.10$). The median corruption cost
is 1.14$\times$ to 1.48$\times$ in five groups. In navigation under partial
observability it reaches 51$\times$, because the discovered mask there gives a
near-perfect baseline (relative error 0.008) that corruption shatters to 0.38.

\textbf{A corrupted prior is worse than no prior.} In every group the corrupted mask is
worse than the dense baseline (Figure~\ref{fig:app-mask-ablation}b). A wrong structural
prior actively misleads the model, whereas the absence of one merely forgoes a benefit.

\textbf{Ground truth helps exactly where discovery is imperfect.} The ground-truth mask
significantly lowers error over the discovered mask in navigation (0.75$\times$ full
observability, 0.90$\times$ partial) and in two\_door full observability
(0.96$\times$), all $p \leq 6\times10^{-4}$; this is the ground-truth rung of the
ladder in Figure~\ref{fig:app-mask-ablation}b. In give\_way, where PCMCI already
recovers essentially the whole graph, ground truth matches discovered, as expected.

\textbf{Ratio and absolute magnitude disagree, so we report both.} By ratio the extreme
case is navigation/partial (51$\times$). By absolute error added, the extreme is
two\_door under full observability ($+0.257$), because it is the hardest environment to
begin with (Figure~\ref{fig:app-mask-ablation}c; whiskers are bootstrap 95\%
confidence intervals of the medians).

\textbf{The effect is consistent across evaluation policies and architectures.}
Corruption hurts on all three held-out evaluation sets, most strongly under the
adversarial policy, and in both ICWM and R-ICWM
(Figure~\ref{fig:app-mask-ablation}d).

\textbf{The damage grows with intervention strength.} The corrupted-to-discovered error
ratio rises with the intervention strength $\sigma$ of the training data in five of six
groups (Figure~\ref{fig:app-mask-ablation}e): in two\_door from about 1.2$\times$ near
$\sigma=0$ to about 2.5$\times$ at $\sigma \approx 1$, and in navigation/partial from
24$\times$ to 77$\times$. A wrong structural prior is most harmful precisely in the
interventional regime that causal world models are meant to win in.

\textbf{Mechanism.} The size of the corruption effect tracks how much the discovered
mask actually gates. In give\_way, PCMCI keeps 94\% of inputs open, so corrupting it is
a small perturbation and the effect is modest. In navigation, PCMCI is near-dense (98\%
open) but the true structure is sparse (39\% to 44\% open), so corruption flips gates
the model relies on and ground truth removes real distractors. In two\_door, PCMCI
closes about 27\% of inputs, the most active gating, and the corruption effect is large
and highly significant.

\textbf{Two caveats.} First, the ground-truth arm is degenerate for two\_door: reducing
the true adjacency to an input-vector gate collapses to all ones there, because
two\_door's sparsity lives in the input by output transition matrix, which a vector
gate cannot express. The ground-truth benefit result therefore rests on navigation, and
two\_door's ground-truth column should be read as equivalent to dense. Second, the
marginal give\_way full-observability result is expected from the mechanism above,
since its discovered mask is nearly dense and leaves little structure to corrupt.

\section{Uncertainty and Neural density measuring Setup and Results}
\label{app:neural-density}

The spine-and-cloud construction is defined in the motivation experiments
(Figure ~\ref{fig:motivation_env5}b: the \emph{spine} is the frozen world model's mean
next-state prediction, and the \emph{cloud} is a conditional normalizing flow trained
on the spine's residuals, so that sampling spine plus cloud yields a full predictive
distribution per state dimension (C-VAE uses its own latent posterior instead). For one
in-distribution and one out-of-distribution transition per environment we draw 1{,}000
samples and summarize the prediction density along every named state axis. The value
of this construction is explainability: the model's uncertainty is visualized,
isolated, and quantified per interpretable state variable, so a practitioner can read
off not only how uncertain the model is out of distribution but which physical
quantities that uncertainty attaches to. All figures use peak-height normalization
(every density rescaled to a maximum of one), which makes the width of each cloud the
single carrier of uncertainty, both visually and in the metrics of
Section~\ref{app:nd-metrics}. The sweep covers six experiments and 182 spine-cloud
fits over the six architectures at three training intervention strengths.

\subsection{Grid World: Two-Door}
\label{app:nd-twodoor}
\begin{figure}[htbp]
    \centering
    \includegraphics[width=0.85\textwidth]{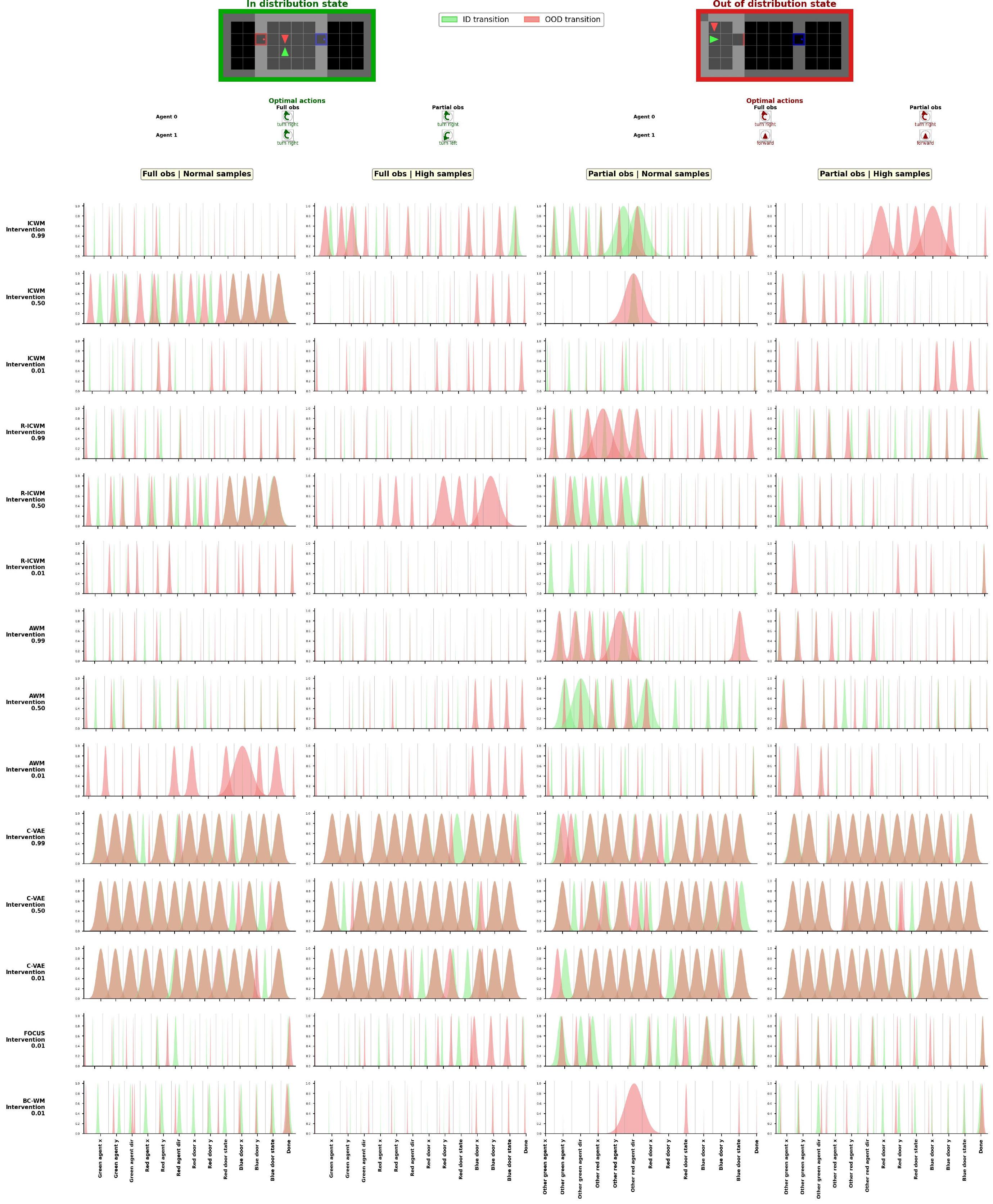}
    \caption{\textbf{Spine-and-cloud densities, two\_door} (peak-height
    normalization: every density is rescaled to a maximum of one, so the width of
    each curve is the model's uncertainty on that axis). \emph{Header:} the green and
    red framed snapshots show the evaluated in-distribution and out-of-distribution
    states; below each snapshot, the optimal-actions table lists one row per agent
    and one column per observability mode, each cell drawing the agent's optimal
    discrete action as a glyph (curved arrow for a turn, straight arrow for forward,
    hand for an interaction) with the action name underneath. \emph{Body:} rows are
    (model, training intervention strength $\sigma$) pairs, with ICWM, R-ICWM, AWM,
    and C-VAE at $\sigma \in \{0.99, 0.5, 0.01\}$ and FOCUS and BC-WM only at
    $\sigma \approx 0$ by construction; columns are observability (full, partial)
    crossed with training sample regime (normal, high). Each cell overlays, for
    every named state dimension on the x axis (agent positions and directions, door
    coordinates and door state, episode-done flag), the Gaussian summary of 1{,}000
    spine-plus-cloud samples: green for the in-distribution transition, red for the
    out-of-distribution one; brown regions are their overlap. Narrow green with
    selectively widened red indicates calibrated, localized out-of-distribution
    uncertainty; C-VAE's uniformly wide clouds in every cell indicate saturation.}
    \label{fig:app-nd-twodoor}
\end{figure}
Figure~\ref{fig:app-nd-twodoor} shows the composite for \texttt{two\_door} under full
and partial observability, at normal and high training sample counts. Two patterns
stand out. First, C-VAE's densities are saturated: near-uniform clouds spanning most
of each axis, in and out of distribution alike, so its uncertainty carries no
diagnostic information. Second, the regressor spines produce narrow in-distribution
clouds whose out-of-distribution widening concentrates on the axes the shift actually
touches, most visibly the door-state and agent-pose dimensions. The high-sample
columns show visibly tighter in-distribution clouds than the normal-sample columns,
which is what later makes their out-of-distribution inflation quantitatively decisive.

\subsection{Continuous Control: Give-Way}
\label{app:nd-giveway}

\begin{figure}[htbp]
    \centering
    \includegraphics[width=\textwidth]{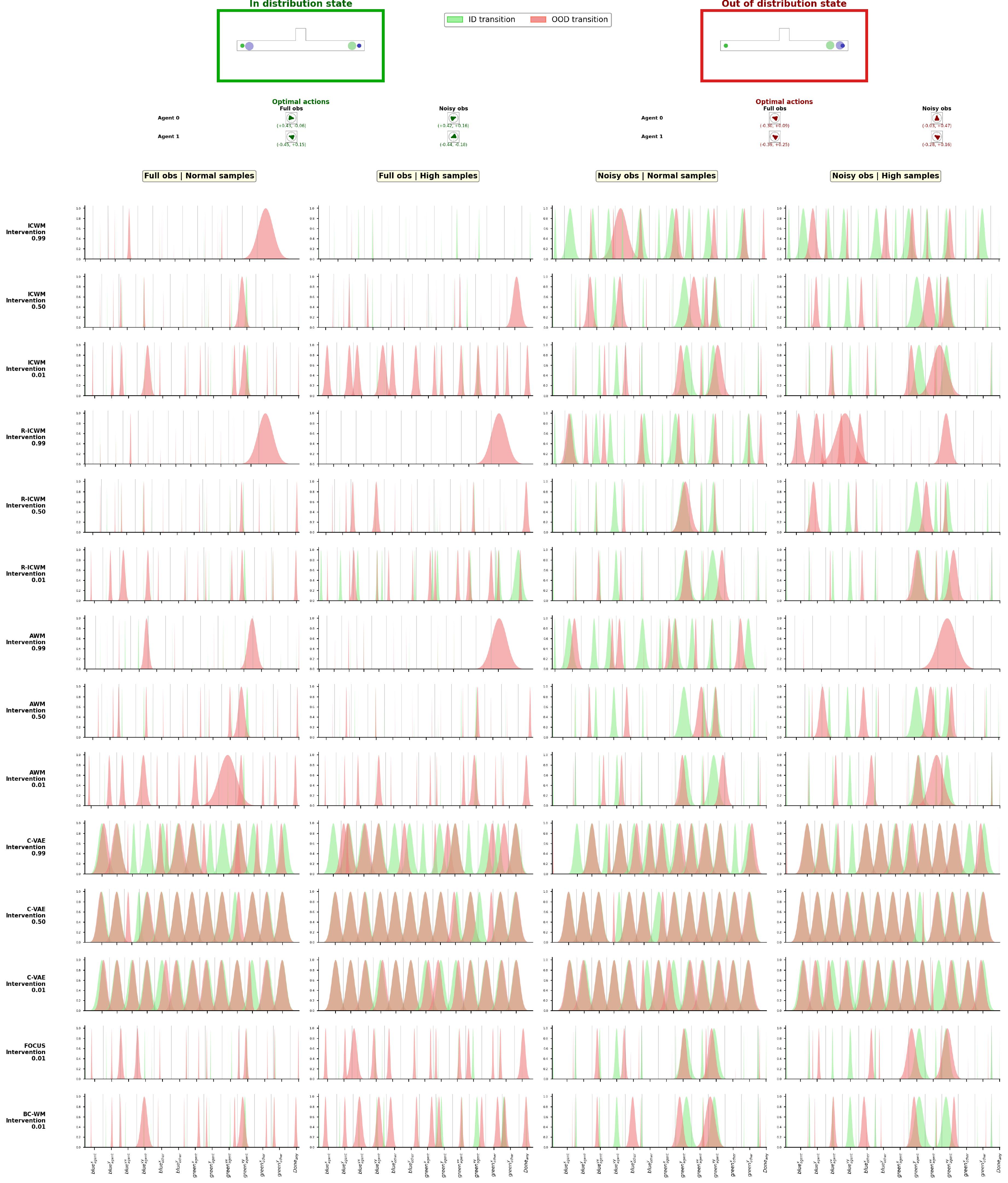}
    \caption{\textbf{Spine-and-cloud densities, give\_way} (peak-height
    normalization). Layout as in Figure~\ref{fig:app-nd-twodoor}, with two
    differences. Columns cross observability (full, noisy) with sample regime
    (normal, high), and actions are continuous: each optimal-action cell draws the
    agent's 2-D velocity command as an arrow on the unit circle with its numeric
    value underneath. State dimensions cover both agents' positions, velocities, and
    goal-relative quantities. The full-observability columns show narrow green
    in-distribution clouds that widen sharply on the red out-of-distribution
    transition, while in the noisy-observability columns the green clouds are already
    wide and the red ones add little: training noise absorbs the shift.}
    \label{fig:app-nd-giveway}
\end{figure}

Figure~\ref{fig:app-nd-giveway} shows the same layout for \texttt{give\_way}, with a
noisy-observation variant in place of partial observability. Under full observability
every regressor spine reacts sharply to the out-of-distribution transition, while
under noisy observations the in-distribution clouds are already wide and the
out-of-distribution transition adds nothing: training noise absorbs the shift. This
contrast is the clearest example of the calibration signal being a property of the
training data rather than of the architecture.

\subsection{Two-Agent Navigation}
\label{app:nd-nav}

\begin{figure}[htbp]
    \centering
    \includegraphics[width=0.6\textwidth]{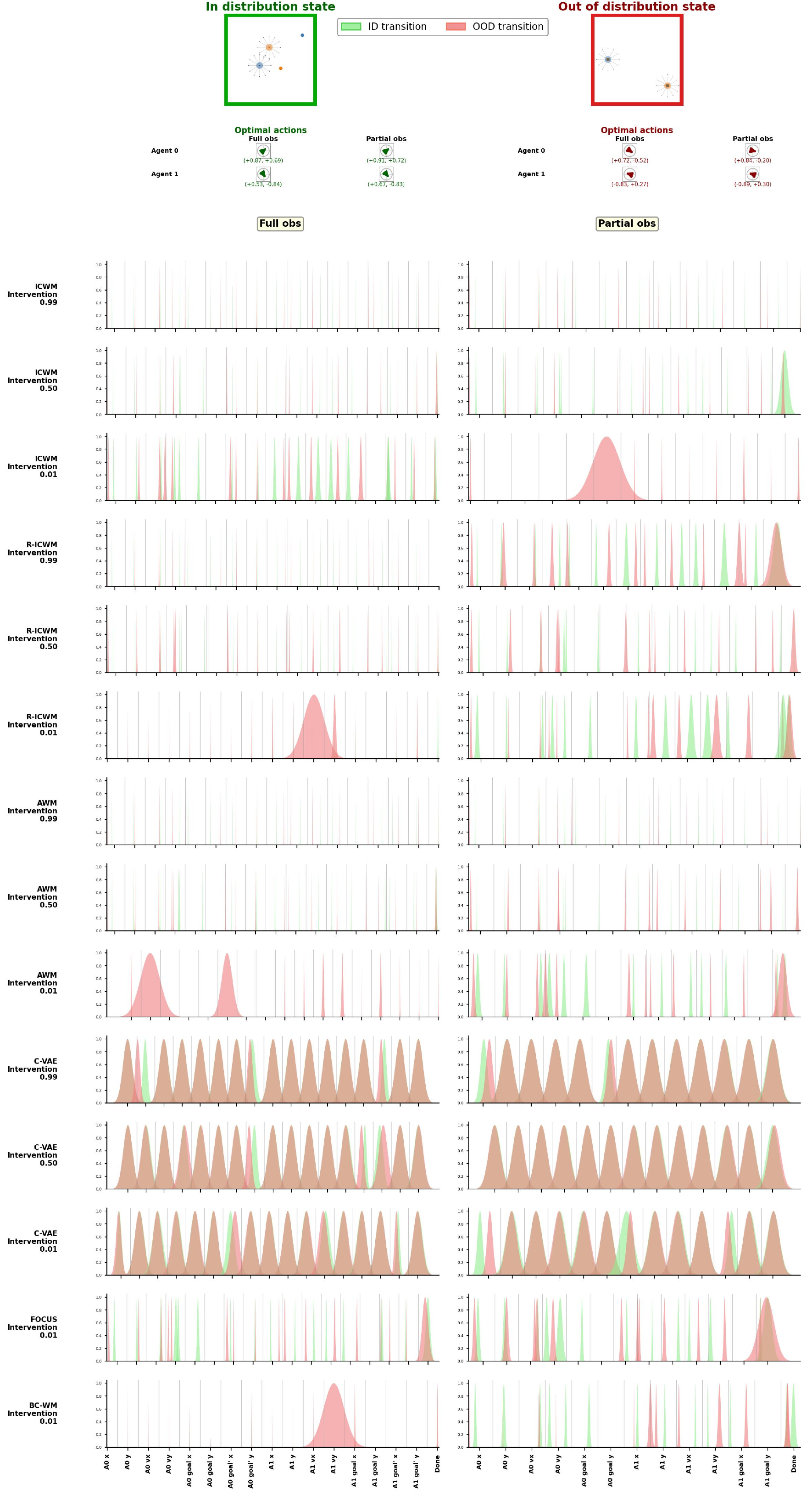}
    \caption{\textbf{Spine-and-cloud densities, navigation with 2 agents}
    (peak-height normalization). Layout as in Figure~\ref{fig:app-nd-twodoor}, with
    columns for full and partial observability at normal sample counts only, and
    continuous 2-D velocity actions drawn on the unit circle in the header. The
    state also contains many LIDAR range readings, omitted from the plotted axes
    for readability; they are intrinsically hard to predict, which keeps the green
    in-distribution clouds comparatively wide and mutes the contrast with the red
    out-of-distribution clouds. BC-WM, trained only on near-optimal data, produces
    extreme widths once the transition leaves that regime.}
    \label{fig:app-nd-navigation}
\end{figure}
Figure~\ref{fig:app-nd-navigation} shows the two-agent \texttt{navigation}
environment, where actions are continuous 2-D vectors drawn on the unit circle in the
header. Inflation here is mild for the main models; the state contains many LIDAR
dimensions that are intrinsically hard to predict, which keeps in-distribution clouds
wide and mutes the out-of-distribution contrast. BC-WM is the outlier: trained only on
near-optimal data, its residual flow assigns extreme widths once the evaluated
transition leaves that regime. LIDAR dimensions are not shown in the figure for a clean diagram. 

\subsection{Multi-Agent Navigation without LIDAR}
\label{app:nd-nolidar}
\begin{figure}[htbp]
    \centering
    \includegraphics[width=0.7\textwidth]{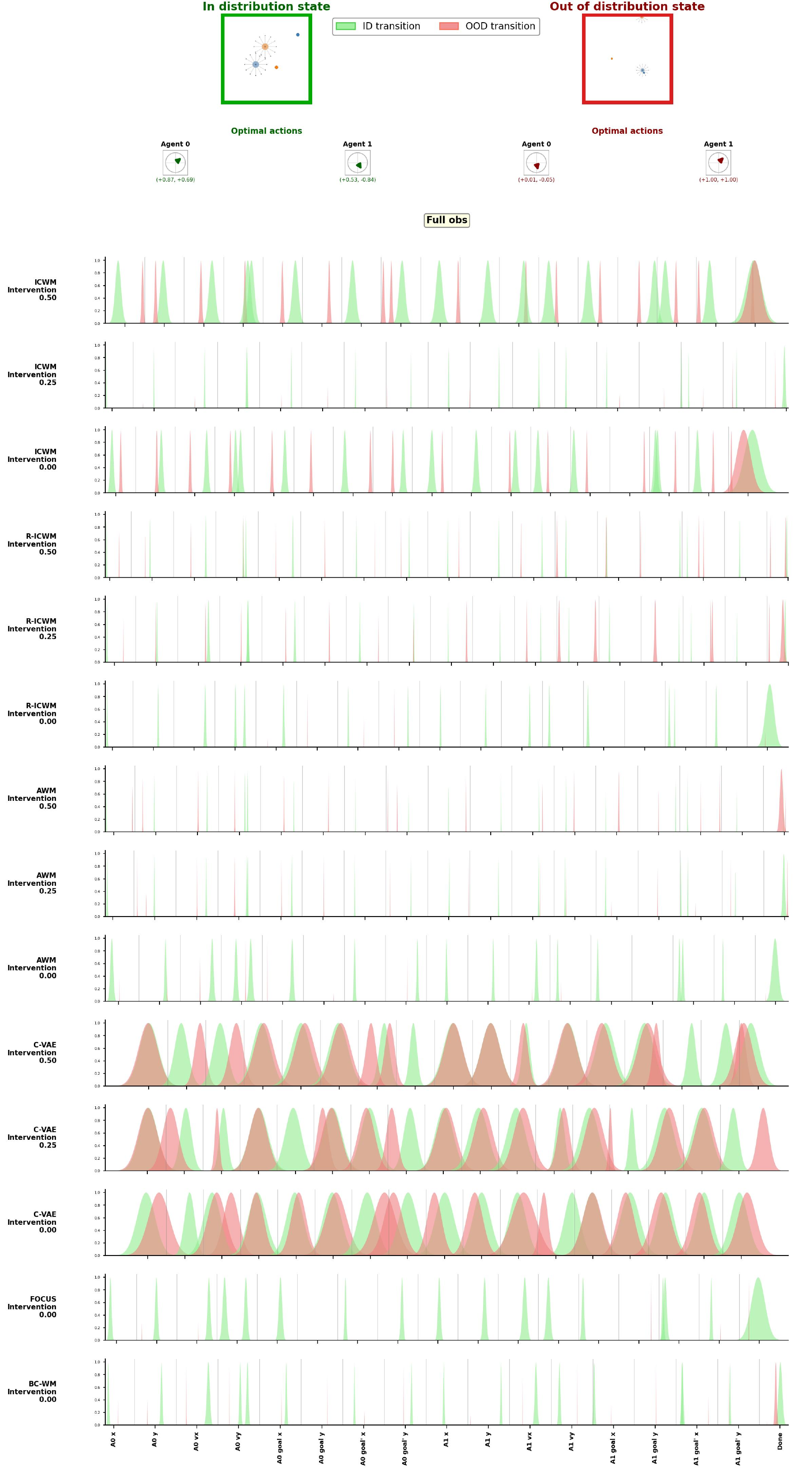}
    \caption{\textbf{Spine-and-cloud densities, navigation with 2 agents, no LIDAR}
    (peak-height normalization). Same construction as
    Figure~\ref{fig:app-nd-navigation}, but the LIDAR dimensions are removed from
    the state and the causal graph before training, leaving 17 interpretable pose
    and goal axes. There is a single observability mode, so the optimal-actions
    table lays agents out as columns. The red out-of-distribution clouds are as
    narrow as the green in-distribution ones or narrower, showing that the inflation
    seen with the full state was carried largely by the LIDAR axes.}
    \label{fig:app-nd-nav2a}
\end{figure}

\begin{figure}[htbp]
    \centering
    \includegraphics[width=0.7\textwidth]{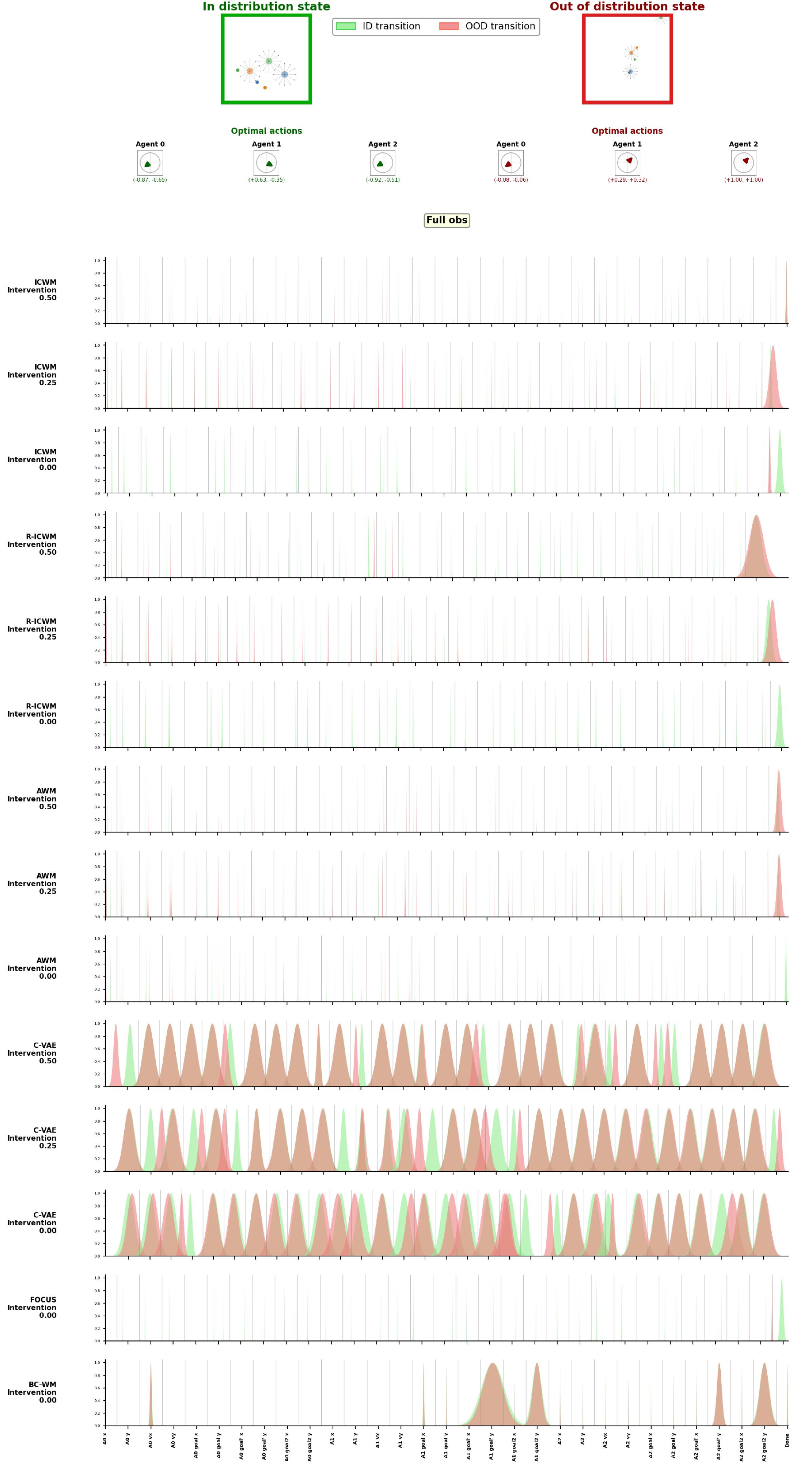}
    \caption{\textbf{Spine-and-cloud densities, navigation with 3 agents, no LIDAR}
    (peak-height normalization). Layout as in Figure~\ref{fig:app-nd-nav2a}, with the
    state and action tables extended to three agents. The absence of
    out-of-distribution widening on the pose and goal axes persists at this agent
    count.}
    \label{fig:app-nd-nav3a}
\end{figure}

\begin{figure}[htbp]
    \centering
    \includegraphics[width=0.7\textwidth]{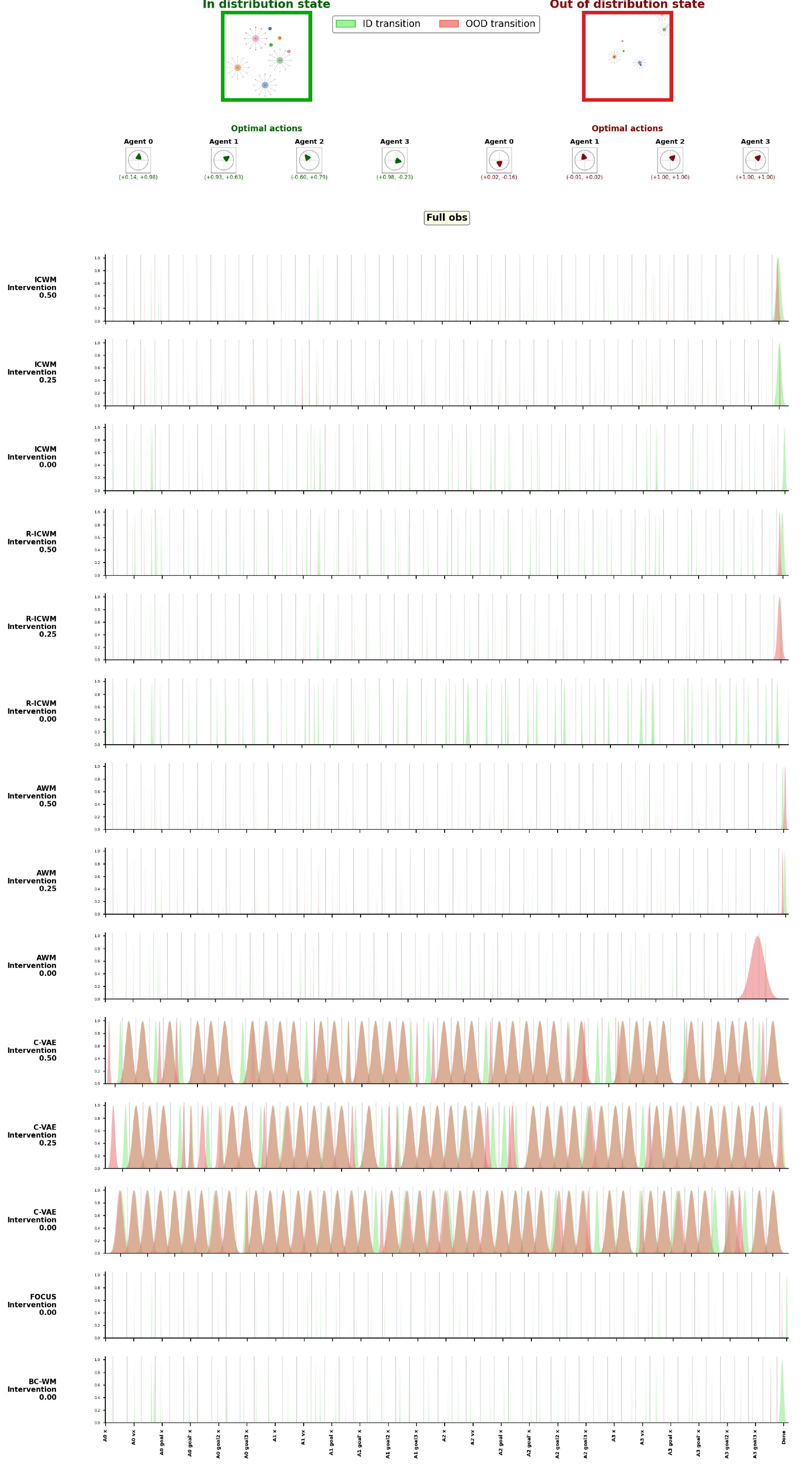}
    \caption{\textbf{Spine-and-cloud densities, navigation with 4 agents, no LIDAR}
    (peak-height normalization). Layout as in Figure~\ref{fig:app-nd-nav2a}, extended
    to four agents. Together with Figures~\ref{fig:app-nd-nav2a}
    and~\ref{fig:app-nd-nav3a}, this shows the no-LIDAR behavior is stable from 2 to
    4 agents, so it is a property of dropping the sensor dimensions rather than of
    scale.}
    \label{fig:app-nd-nav4a}
\end{figure}

Figures~\ref{fig:app-nd-nav2a}, \ref{fig:app-nd-nav3a}, and \ref{fig:app-nd-nav4a}
show the no-LIDAR variants with 2, 3, and 4 agents, where the LIDAR dimensions are
removed from the state and the causal graph before training. What remains are only
the interpretable pose and goal axes. On these, the out-of-distribution clouds are as
narrow as the in-distribution ones or narrower, showing that the inflation observed
with the full state was carried largely by the LIDAR axes. The pattern holds at every
agent count, so it is a property of dropping the sensor dimensions, not of scale.

\subsection{Quantitative Uncertainty Metrics}
\label{app:nd-metrics}

Because peak normalization leaves the width of each cloud as the only signal, all
quantitative analysis uses one metric family: the cloud width $\sigma$ per dimension
and its out-of-distribution inflation, the OOD width divided by the ID width of the
same model. The inflation ratio is dimensionless and self-baselined, so it is the only
quantity pooled across environments; raw widths are compared only within an
environment. We summarize with medians and IQRs because BC-WM's off-distribution
widths (up to $10^{5}$ normalized units) would make any mean meaningless.

Figure~\ref{fig:app-nd-metrics} pools all 182 fits. Panel (a): C-VAE's mean cloud
width is about 0.9 normalized units both in and out of distribution (median inflation
0.98), confirming saturation, while the regressor spines sit at 0.04 to 0.05 in
distribution and 0.10 to 0.11 out of distribution. Panel (b): ICWM's inflation falls
monotonically with training intervention strength, from 3.1$\times$ at $\sigma=0.01$
to 2.0$\times$ at $\sigma=0.5$ and 1.0$\times$ at $\sigma=0.99$; with enough
interventional data its out-of-distribution uncertainty matches its in-distribution
uncertainty. AWM is non-monotone and R-ICWM moves the opposite way. Panel (c): the
inflation is localized, not diffuse; the median share of dimensions inflating by more
than 1.5$\times$ is 14\% for ICWM and 20\% for AWM, and which dimensions these are can
be read directly off the composites above.

Figure~\ref{fig:app-nd-metrics-cond} breaks the same three metrics down per
environment, overlaying observation mode, sample regime, and agent count as
conditions, with C-VAE drawn as a flat dashed reference at ratio one, the visual
signature of an uninformative uncertainty estimate. It localizes the pooled findings:
in \texttt{two\_door} the inflation grows with sample count (ICWM 1.25$\times$ at
normal samples, 3.34$\times$ full and 4.68$\times$ partial observability at high
samples), because a better-trained spine has tight in-distribution clouds against
which the widening stands out; in \texttt{give\_way} the full-observability conditions
inflate four to seven fold while the noisy conditions stay at one; in
\texttt{navigation} the effect is mild throughout; and in the no-LIDAR family the
ratios sit at or below one at every agent count. Together the two figures give the
headline of this appendix: whether a world model's uncertainty is trustworthy is
decided by the interventional richness and coverage of its training data, and when it
is, the uncertainty is calibrated and localized to a few interpretable axes.

\begin{figure}[htbp]
    \centering
    \includegraphics[width=\textwidth]{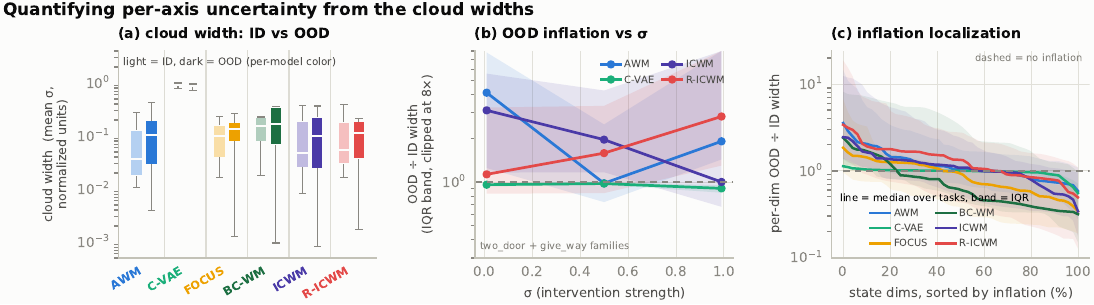}
    \caption{\textbf{Quantifying per-axis uncertainty from the cloud widths}
    (all 182 spine-cloud fits pooled; medians and IQRs throughout, since BC-WM's
    off-distribution widths reach $10^{5}$ normalized units and would dominate any
    mean). (a) Mean per-dimension cloud width per model on a fixed log range:
    in-distribution boxes in a light shade of each model's color, out-of-distribution
    boxes in the full color, with model zones separated by vertical rules and the
    model names colored to match; a box whose median exceeds the range is replaced by
    an arrow with its value. (b) Width inflation (OOD $\div$ ID) versus training
    intervention strength $\sigma$, one line per model with an IQR band clipped at
    8$\times$; restricted to the two\_door and give\_way families, which share the
    $\sigma \in \{0.01, 0.5, 0.99\}$ grid (FOCUS and BC-WM train at a single $\sigma$
    and therefore cannot appear as trends). ICWM converges to the dashed no-inflation
    line as $\sigma$ grows. (c) Per-dimension inflation profile: for each fit the
    per-axis ratios are sorted in decreasing order and mapped to a common rank scale,
    then summarized per model by the median (line) and IQR (band) over fits. The
    profiles show the inflation is carried by a minority of axes while most
    dimensions sit at or below one.}
    \label{fig:app-nd-metrics}
\end{figure}

\begin{figure}[htbp]
    \centering
    \includegraphics[trim={0pt 0pt 0pt 10pt}, clip,width=\textwidth]{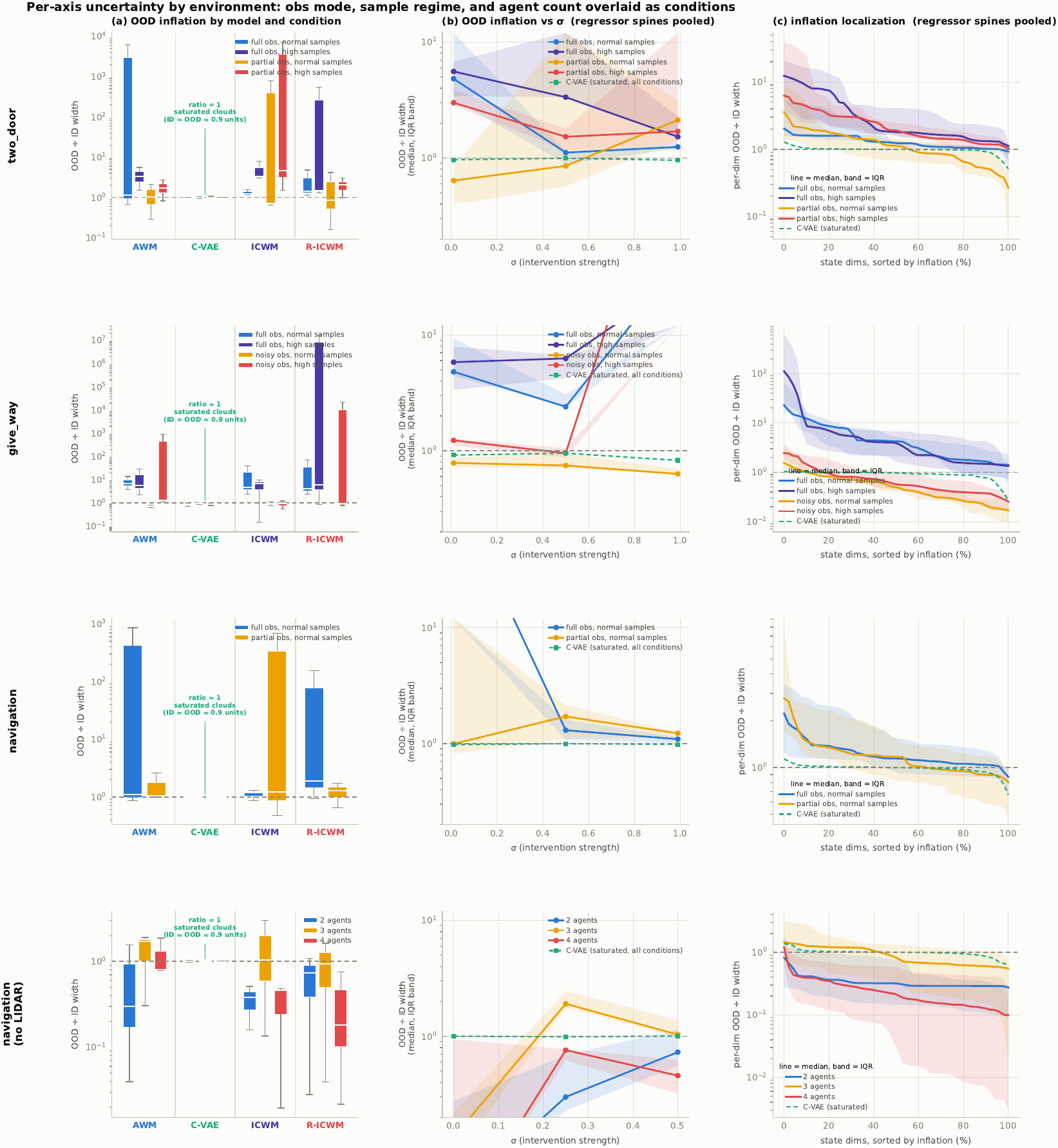}
    \caption{\textbf{Per-axis uncertainty by environment with conditions overlaid.}
    Rows: two\_door, give\_way, navigation, and navigation without LIDAR; note the
    per-row y scales and the different $\sigma$ grids ($\{0.01, 0.5, 0.99\}$ for the
    first two rows, $\{0, 0.25, 0.5\}$ for the multi-agent rows). Within each panel,
    colors distinguish the conditions and are explained in that panel's legend:
    observation mode crossed with sample regime for the main environments (cool
    colors for full observability, warm for partial or noisy), agent count for the
    no-LIDAR family. (a) OOD $\div$ ID width inflation by model and condition, one
    box per combination, with the dashed line marking no inflation; C-VAE's boxes
    collapse onto that line and are annotated in place, since its clouds are
    saturated at roughly 0.9 normalized units in and out of distribution and its
    ratio therefore carries no signal. (b) Inflation versus training intervention
    strength and (c) sorted per-dimension inflation profile, both pooled over the
    regressor spines (AWM, ICWM, R-ICWM) with medians as lines and IQRs as bands;
    C-VAE appears as the flat dashed reference at one, the visual signature of an
    uninformative uncertainty estimate. Sample count sharpens the two\_door signal,
    observation noise erases the give\_way signal, and the no-LIDAR conditions sit at
    or below one at every agent count.}
    \label{fig:app-nd-metrics-cond}
\end{figure}

\section{World-Model Spatial Error Density}
\label{app:error-density}

\subsection{Setup}

The case for looking at \emph{where} a world model makes mistakes, rather than only how
big they are on average, is made in Appendix~\ref{app:error-density-toy}. Here we run the
same idea on the three real environments. We build each map by pinning one agent on its
goal and sweeping the other agent over the space it can reach, with both agents taking
random actions. This gives two datasets per environment, one for each choice of which
agent is pinned, and we average them. Every valid cell receives the same number of
transitions ($200$ for the discrete \texttt{two\_door}, $100$ for the two continuous
environments), so the coverage is even everywhere.

The error in each cell is the mean squared residual on the dynamic (position and
velocity) dimensions $\mathcal{D}$. We normalise each dimension by its own standard
deviation so environments of different physical scale stay comparable,
\begin{equation}
\label{eq:err-dyn}
\mathrm{err} \;=\; \frac{1}{|\mathcal{D}|}\sum_{d\in\mathcal{D}}
\left(\frac{\hat{s}'_d - s'_d}{\sigma_d + \varepsilon}\right)^{\!2},
\qquad \varepsilon = 10^{-6},
\end{equation}
where $\sigma_d$ is the marginal standard deviation of dimension $d$. A value of
$\mathrm{err}=1$ means the model is off by about one standard deviation on that
dimension. For each environment we fix one \emph{a-priori} zone, the region where the two
agents interact, and test whether its mean error is higher than a size-matched pool of
cells outside the zone. We run the significance battery of Appendix~\ref{app:error-density-toy} to mimic how a practitioner would audit a model during real development, and we report
the in/out ratio $r_z$. A ratio above $1$ means the model is worse inside the zone.

The figures are laid out with model and intervention level down the rows, and
observation mode and sample regime across the columns. The intervention level is how
often the data-collection policy is forced off its goal-seeking action and made to
explore instead. A high intervention level means the data is more exploratory and covers
more of the space; a low one means the data mostly follows near-optimal, goal-directed
behaviour. The models are ICWM, R-ICWM, AWM, C-VAE, and the FOCUS and BC-WM baselines.

\subsection{Figures}

Figures~\ref{fig:ed-two-door} to~\ref{fig:ed-navigation} show the composite error maps
for the three environments. In each figure the rows are model and intervention level, and
the columns are observation mode and sample regime, so one figure covers every model
under every data condition. The colour scale is shared across all panels of a figure, so
brightness can be compared directly between them. The header row of each figure shows the
two fixed-agent layouts that the panels below it average over.

\begin{figure}[htbp]
  \centering
  \includegraphics[width=0.65\linewidth]{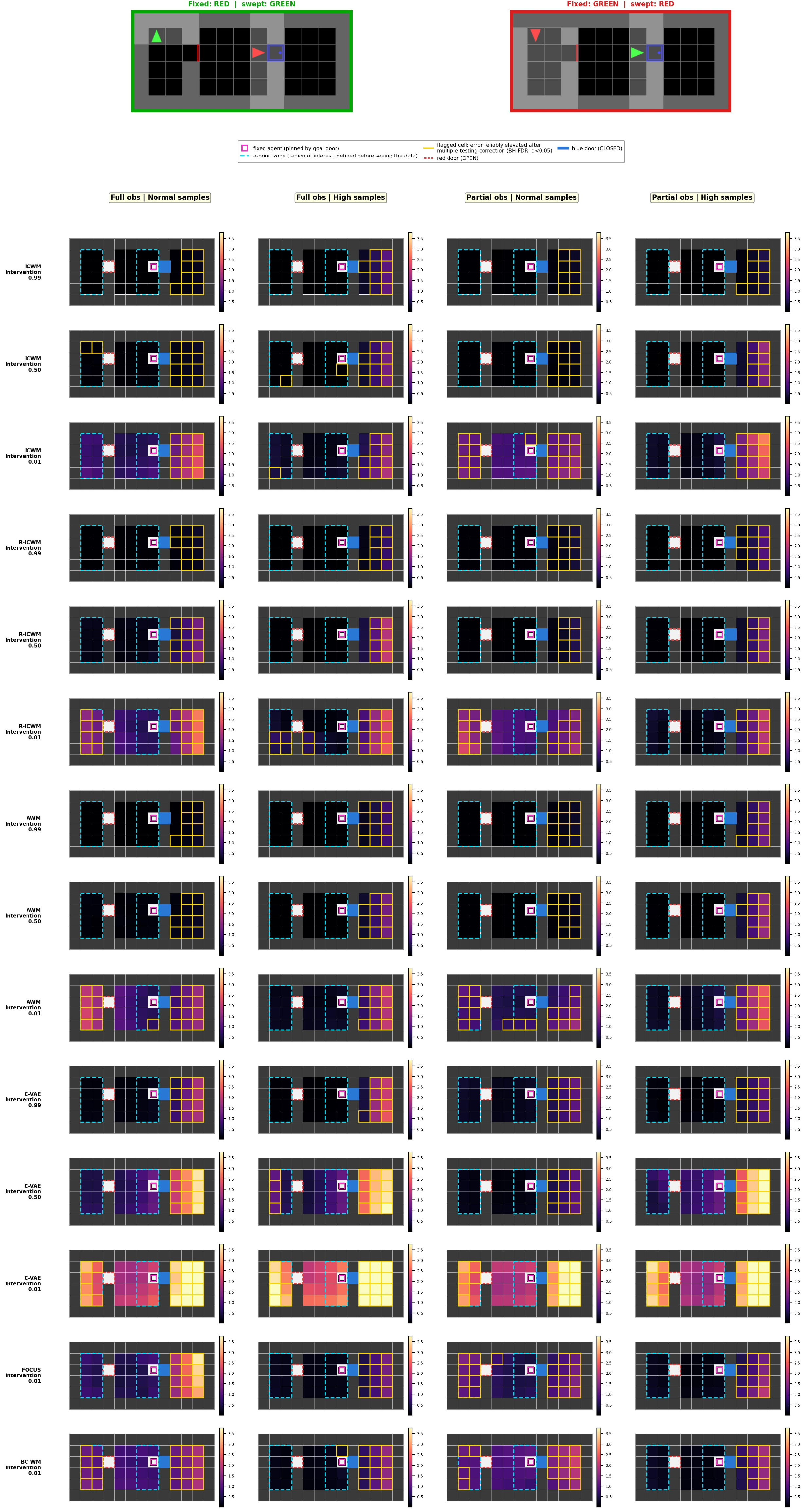}
  \caption{Spatial error density for \texttt{two\_door}, a discrete grid with one agent
  fixed next to the goal (blue) door. Each panel is a $13\times6$ error heatmap drawn over
  the room. Darker means lower error and brighter means higher error, on a colour scale
  shared across all panels. Magenta squares mark the fixed agent. The dashed cyan outline
  is the a-priori interaction zone (the corridor by the closed door). Gold cell borders
  mark cells whose error is significantly high after BH-FDR correction ($q<0.05$). The
  rows sweep the six models at intervention levels $\{0.99, 0.50, 0.01\}$, and the columns
  sweep full against partial observation and a normal against a high sample budget. Grey
  cells are walls and are never sampled.}
  \label{fig:ed-two-door}
\end{figure}

\begin{figure}[htbp]
  \centering
  \includegraphics[width=0.8\linewidth]{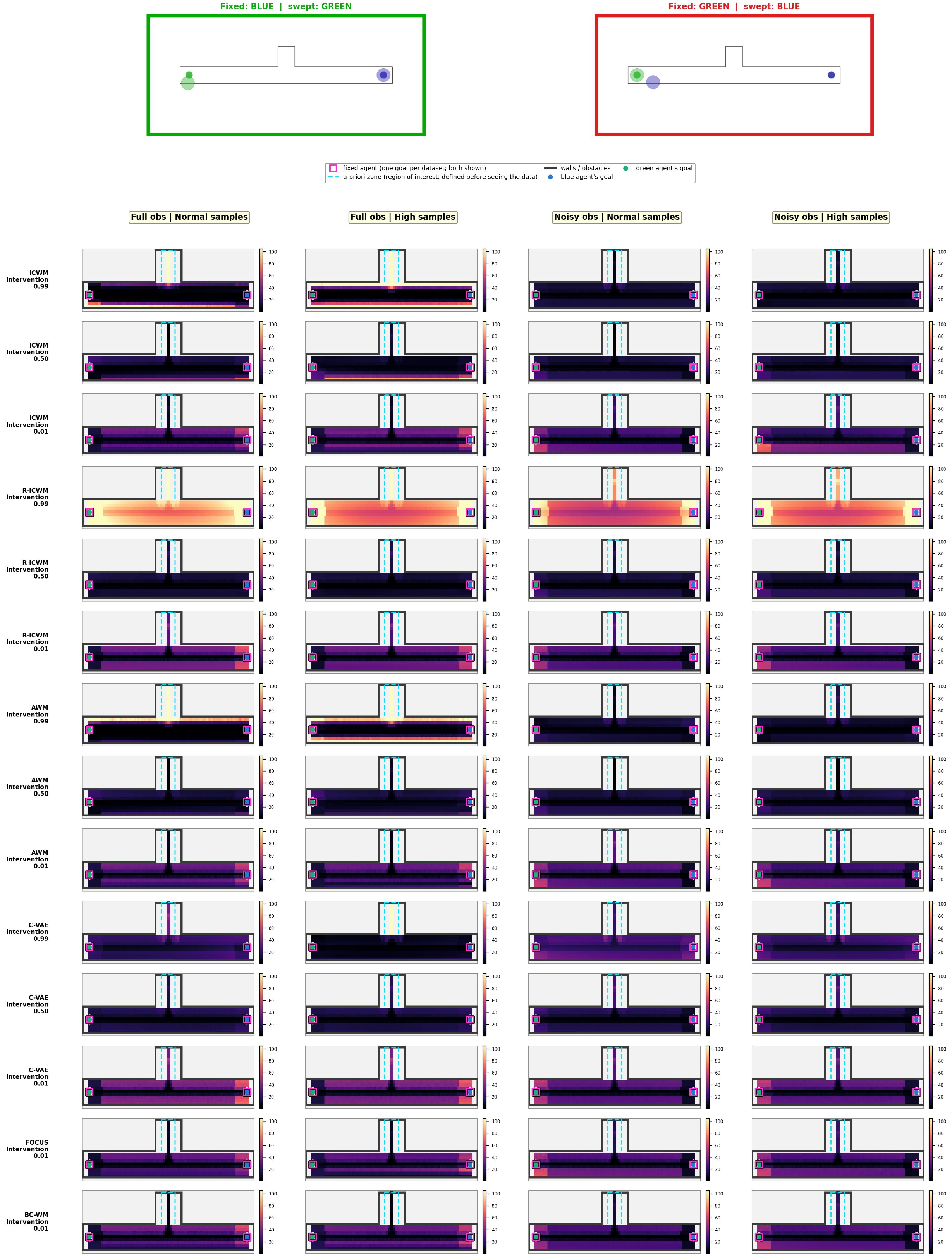}
  \caption{Spatial error density for \texttt{give\_way}, a continuous environment. It is a
  horizontal corridor with a give-way bay, the vertical notch, where one agent can pull
  aside and let the other pass. The fixed agent's goal sits inside the corridor. The
  corridor and the bay are sampled separately, because an agent cannot physically sit in
  the thin strip at the mouth of the notch. That strip is filled in for the drawing so the
  notch looks connected to the corridor, but every statistic uses only the cells that were
  actually sampled. The dashed cyan outline marks the bay, which is the interaction zone.
  At a high intervention level the causal models are almost perfect along the busy corridor
  but visibly bright inside the bay, which is the local blind spot measured in
  Table~\ref{tab:ed-all}. The row and column layout matches Fig.~\ref{fig:ed-two-door}.}
  \label{fig:ed-give-way}
\end{figure}

\begin{figure}[htbp]
  \centering
  \includegraphics[width=\linewidth]{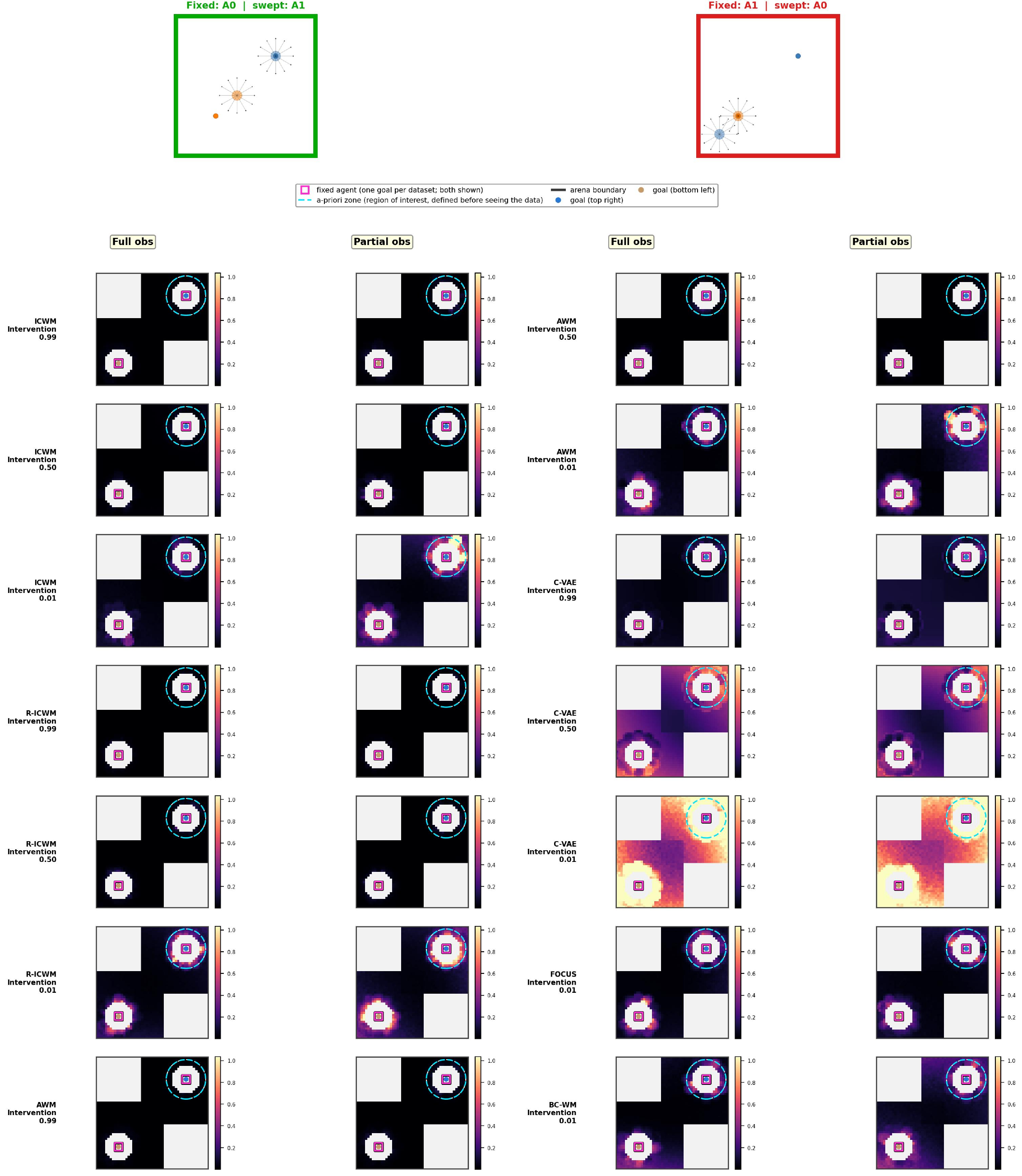}
  \caption{Spatial error density for \texttt{navigation}, a continuous $[-1,1]^2$ arena.
  The fixed agent sits on its goal and the moving agent is swept over a square region
  around it. The dashed white circle is the fixed agent's LIDAR range, the only region
  where the two agents actually sense each other. The error is close to the numerical
  floor across the open arena and gathers almost entirely inside the LIDAR circle, so each
  panel looks nearly black except for a bright ring around the fixed agent. This is the
  expected signature of a model that has learned to focus on the interaction, not a
  defect. The row and column layout matches Fig.~\ref{fig:ed-two-door}.}
  \label{fig:ed-navigation}
\end{figure}

\subsection{Quantitative summary}

Table~\ref{tab:ed-all} reports, for each model and intervention level, the global mean
error $\overline{\mathrm{err}}$ from Eq.~\eqref{eq:err-dyn} and the interaction-zone
in/out ratio $r_z$ in each environment (full observation, normal samples, both datasets
averaged). The zone is the closed-door corridor for \texttt{two\_door}, the give-way bay
for \texttt{give\_way}, and the fixed agent's LIDAR disc for \texttt{navigation}. Ratios
above $1$, where the model is worse inside the interaction zone, are shown in bold.

\begin{table}[htbp]
  \centering
  \scriptsize
  \setlength{\tabcolsep}{3.2pt}
  \caption{Spatial error density across all three environments and every data condition.
  For each condition we report the global mean error $\overline{\mathrm{err}}$ (``e'') and
  the interaction-zone in/out ratio $r_z$ (``$r$''), both averaged over the two
  fixed-agent datasets. Columns are grouped by environment, then by observation mode
  (full vs.\ degraded), then by sample budget (normal ``N'' vs.\ high ``H''). Higher
  intervention level means more exploratory data. Ratios above $1$ are in bold.}
  \label{tab:ed-all}
  \begin{tabular}{ll *{4}{cc} *{4}{cc} cc}
    \toprule
    & & \multicolumn{8}{c}{\texttt{two\_door}} & \multicolumn{8}{c}{\texttt{give\_way}}
      & \multicolumn{2}{c}{\texttt{navigation}} \\
    \cmidrule(lr){3-10}\cmidrule(lr){11-18}\cmidrule(lr){19-20}
    & & \multicolumn{4}{c}{full obs} & \multicolumn{4}{c}{partial obs}
      & \multicolumn{4}{c}{full obs} & \multicolumn{4}{c}{noisy obs}
      & \multicolumn{2}{c}{full obs} \\
    \cmidrule(lr){3-6}\cmidrule(lr){7-10}\cmidrule(lr){11-14}\cmidrule(lr){15-18}\cmidrule(lr){19-20}
    & & \multicolumn{2}{c}{N} & \multicolumn{2}{c}{H} & \multicolumn{2}{c}{N}
      & \multicolumn{2}{c}{H} & \multicolumn{2}{c}{N} & \multicolumn{2}{c}{H}
      & \multicolumn{2}{c}{N} & \multicolumn{2}{c}{H} & \multicolumn{2}{c}{N} \\
    \cmidrule(lr){3-4}\cmidrule(lr){5-6}\cmidrule(lr){7-8}\cmidrule(lr){9-10}%
    \cmidrule(lr){11-12}\cmidrule(lr){13-14}\cmidrule(lr){15-16}\cmidrule(lr){17-18}%
    \cmidrule(lr){19-20}
    Model & Int. & e & $r$ & e & $r$ & e & $r$ & e & $r$ & e & $r$ & e & $r$
      & e & $r$ & e & $r$ & e & $r$ \\
    \midrule
    \rowcolor{blue!5}  ICWM   & 0.99 & 0.050 & 0.20 & 0.221 & 0.01 & 0.072 & 0.26 & 0.100 & 0.10 & 28.37 & \textbf{48.96} & 60.66 & \textbf{29.16} & 9.28 & 0.87 & 7.67 & \textbf{1.74} & 0.006 & \textbf{11.03} \\
    \rowcolor{blue!5}  ICWM   & 0.50 & 0.116 & 0.56 & 0.281 & 0.34 & 0.055 & 0.32 & 0.297 & 0.06 & 13.55 & 0.06 & 21.40 & 0.36 & 10.50 & 0.07 & 11.03 & 0.34 & 0.012 & \textbf{5.07} \\
    \rowcolor{blue!5}  ICWM   & 0.01 & 1.04 & 0.59 & 0.537 & 0.41 & 1.09 & 0.78 & 0.761 & 0.35 & 19.93 & 0.29 & 22.20 & 0.24 & 19.59 & 0.67 & 22.61 & 0.74 & 0.045 & \textbf{5.53} \\
    \addlinespace
    \rowcolor{green!5}  R-ICWM & 0.99 & 0.078 & 0.48 & 0.155 & 0.09 & 0.114 & 0.40 & 0.196 & 0.06 & 93.88 & \textbf{1.57} & 79.88 & \textbf{1.56} & 66.37 & \textbf{1.31} & 77.58 & \textbf{1.21} & 0.005 & \textbf{10.75} \\
    \rowcolor{green!5}  R-ICWM & 0.50 & 0.408 & 0.39 & 0.406 & 0.05 & 0.128 & 0.21 & 0.292 & 0.09 & 8.88 & \textbf{1.79} & 8.08 & \textbf{2.21} & 10.41 & \textbf{1.03} & 11.56 & 0.89 & 0.010 & \textbf{11.48} \\
    \rowcolor{green!5}  R-ICWM & 0.01 & 1.25 & 0.46 & 0.756 & 0.19 & 1.17 & 0.70 & 0.538 & 0.31 & 26.56 & 0.79 & 24.12 & 0.80 & 20.12 & 0.80 & 20.90 & 0.75 & 0.094 & \textbf{6.18} \\
    \addlinespace
    \rowcolor{orange!5} AWM    & 0.99 & 0.048 & 0.20 & 0.170 & 0.02 & 0.082 & 0.22 & 0.253 & 0.01 & 32.77 & \textbf{81.97} & 56.98 & \textbf{22.21} & 9.09 & 0.77 & 7.86 & \textbf{2.17} & 0.005 & \textbf{11.26} \\
    \rowcolor{orange!5} AWM    & 0.50 & 0.114 & 0.67 & 0.271 & 0.13 & 0.060 & 0.38 & 0.334 & 0.05 & 10.06 & 0.18 & 11.30 & 0.66 & 10.89 & 0.10 & 10.99 & 0.20 & 0.009 & \textbf{13.99} \\
    \rowcolor{orange!5} AWM    & 0.01 & 1.17 & 0.57 & 0.593 & 0.34 & 0.774 & 0.74 & 0.772 & 0.43 & 21.09 & 0.19 & 18.77 & 0.16 & 20.34 & 0.97 & 19.95 & 0.65 & 0.085 & \textbf{4.77} \\
    \addlinespace
    \rowcolor{purple!5} C-VAE  & 0.99 & 0.495 & 0.26 & 0.507 & 0.07 & 0.417 & 0.49 & 0.295 & 0.23 & 17.16 & \textbf{2.46} & 7.89 & \textbf{24.19} & 24.13 & 0.88 & 17.89 & \textbf{1.27} & 0.050 & \textbf{1.86} \\
    \rowcolor{purple!5} C-VAE  & 0.50 & 1.41 & 0.64 & 1.57 & 0.55 & 0.341 & 0.37 & 1.64 & 0.57 & 11.73 & \textbf{1.09} & 11.12 & \textbf{1.47} & 11.75 & \textbf{1.71} & 11.49 & \textbf{1.74} & 0.335 & \textbf{1.57} \\
    \rowcolor{purple!5} C-VAE  & 0.01 & 2.99 & 0.64 & 3.63 & 0.65 & 2.71 & 0.66 & 2.83 & 0.58 & 29.24 & 0.91 & 28.61 & 0.98 & 22.41 & 0.88 & 22.92 & 0.88 & 0.823 & \textbf{2.03} \\
    \bottomrule
  \end{tabular}
\end{table}

\subsection{Discussion}
\label{app:ed-discussion}

Three points stand out in Table~\ref{tab:ed-all}. First, the error gathers on the
interaction states in every environment, and this is expected rather than a fault. In
\texttt{navigation} the ratio $r_z$ is above $1$ for every model, and above $10$ for the
causal models at high intervention. This is because away from the other agent the
dynamics are simple and the error drops to the floor (about $5\times10^{-3}$), so all of
the difficulty is packed into the LIDAR disc. A flat map would be the surprising outcome.
The right way to read $r_z$ is therefore as a comparison between models at the same global
error, not as an absolute score.

Second, a low average can hide a blind spot. In \texttt{give\_way} the causal and
autoregressive models look fine on the global error ($28.4$ and $32.8$) but carry very
large bay ratios ($48.9$ and $82.0$). They predict the busy corridor well and fail on the
rare give-way move. C-VAE behaves in the opposite way, with a much flatter map (bay ratio
$2.5$) at a similar global error. So two models that are within a factor of two on the
single number can spread their error in completely different ways, and only one of them
leaves a gap right where the agents coordinate.

Third, more exploration flattens the map but raises the whole floor. Lowering the
intervention level from $0.99$ to $0.01$ raises the global error by roughly a factor of
$20$ while the zone ratios even out, and the give-way bay ratio for the causal models
falls from about $50$ to about $0.2$. In other words, exploratory data is what teaches the
model the interaction region, but the price is lower overall sharpness.

R-ICWM stands out in \texttt{give\_way}, where its map is bright almost everywhere and its
global error is the highest in the table ($93.88$ at high intervention). The reason is the
way we build the samples. To cover the space we teleport the moving agent to each cell and
set both agents to zero velocity, so every sampled state is static. R-ICWM is a recurrent
model trained on real demonstrations, in which an agent is almost always in motion, so
these frozen, zero-velocity states are far from anything it saw during training and it
does not generalise to them. The other models are less sensitive to this because they do
not carry a recurrent state that expects motion to continue.

Two limits are worth stating. The ratios describe the pattern but do not explain it,
because the intervention level is tied to both how hard the data is and how well it covers
the space. And the global error is normalised per dataset, so it can be compared within an
environment but not across environments.

\section{Generalization under layout shift}
\label{app:shifted-generalization}

\subsection{Setup}

The shifted environments of Appendix~\ref{app:shifted-envs} let us ask a question the
in-distribution error maps cannot: does a world model trained on one layout still predict
correctly once the geometry changes, or has it only memorised the training room? For each
shift we take the same frozen, never-retrained world models from Appendix~\ref{app:error-density},
run them on the shifted layout using the same grid-sampling protocol (one agent fixed on
its goal, the other swept over the reachable space, both taking random actions), and
compare the error against the model's own error on the original, unshifted layout under
the matched observation mode and sample regime. We report two error readings per
condition: the dynamics error $\overline{\mathrm{err}}_{\mathrm{dyn}}$ of
Eq.~\eqref{eq:err-dyn}, computed over all dynamic state dimensions, and the position-only
error $\overline{\mathrm{err}}_{\mathrm{pos}}$, the same statistic restricted to the
moving agent's position dimensions. The two rarely move together: a model can localise the
agent correctly while still getting its velocity wrong, or vice versa, so reporting only
one would hide half the failure mode. The generalization ratio
$\rho = \overline{\mathrm{err}}_{\mathrm{shifted}} / \overline{\mathrm{err}}_{\mathrm{source}}$
is the headline statistic: $\rho \approx 1$ means the model transfers to the new layout as
well as it performs on the layout it was trained on, and $\rho \gg 1$ means the shift
exposes real extrapolation failure.

We group the seven shifts into two families that probe different things. The first four
\emph{rearrange space}: they rotate, mirror, or stretch the room, so a model is fed
coordinates and configurations it has partly (rotation, mirror) or never (stretch) seen.
These confound two failure modes a heatmap alone cannot separate: a model can fail on a
rotated grid either because it never learned the dynamics abstractly, or simply because it
is being handed coordinates outside its training support. The last three
(Appendix~\ref{app:shifted-envs}) are designed to break that confound: each holds the
geometry and the visited coordinates \emph{fixed} and perturbs exactly one thing: the
transition rule (\texttt{Wrap Two Door}), the joint state distribution
(\texttt{Anti-Causal Two Door}), or one feature's extent (\texttt{Easy Give Way}). Read
together, the two families let us attribute a transfer failure to a specific cause rather
than to unseen coordinates in general.

Figures~\ref{fig:shifted-rotated}--\ref{fig:shifted-easy-gw} show the full composite
error maps; Table~\ref{tab:shift-all} gives the matching numbers for all seven shifts,
broken out by observation class (full vs.\ degraded) and sample complexity
(normal vs.\ high).

\subsection{Figures: space-rearranging shifts}

\begin{figure}[htbp]
  \centering
  \includegraphics[trim={0pt 0pt 0pt 50pt}, clip,width=0.9\linewidth]{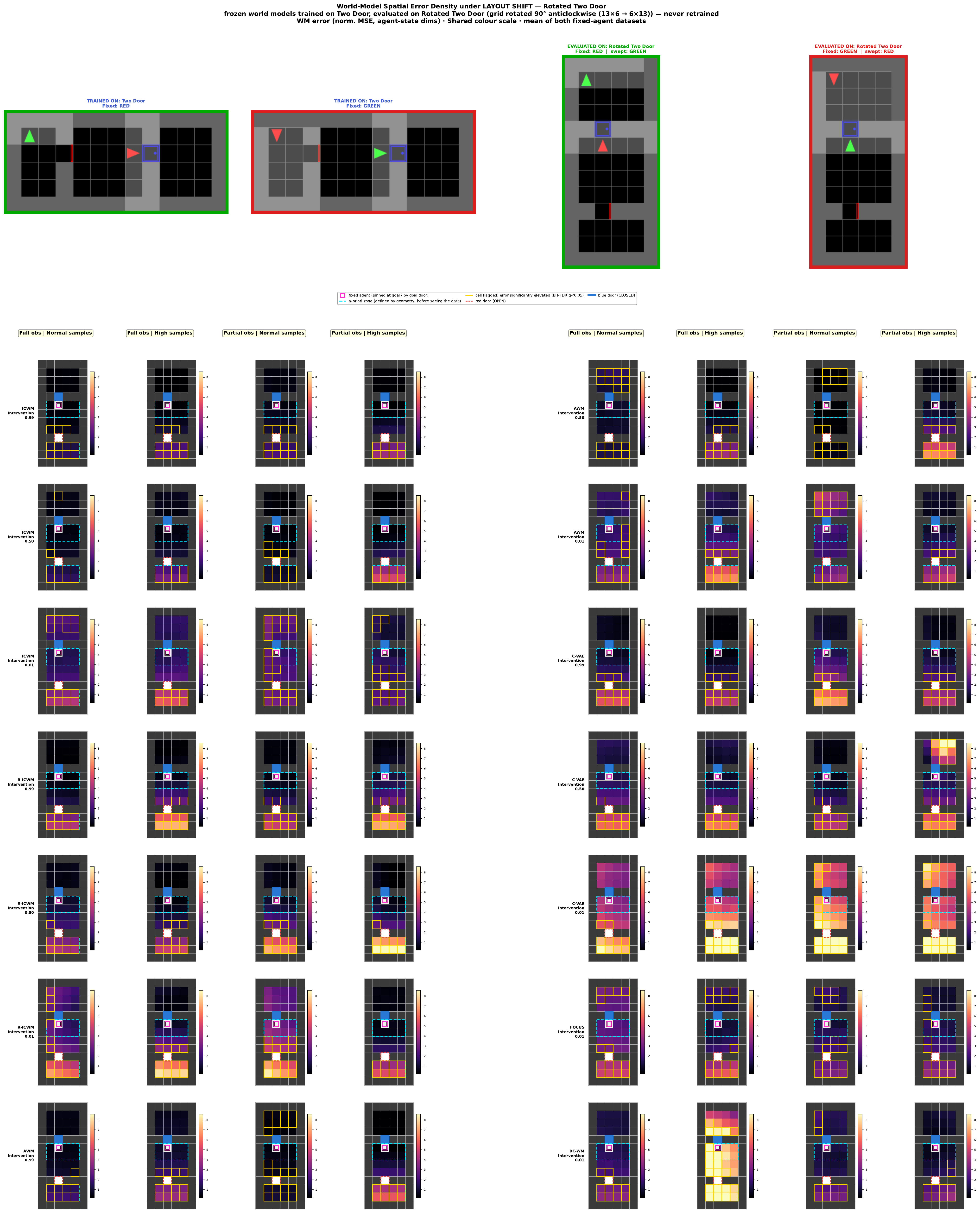}
  \caption{Dynamics error density on \texttt{Rotated Two Door}. The header shows the two
  layouts side by side: the training room (left two panels, blue title border) and the same
  room turned $90^\circ$ anticlockwise, so its left-right corridors now run top-bottom
  (right two panels, green/red title borders, one per fixed-agent choice). Each body panel
  is a spatial heatmap of the frozen world model's one-step prediction error, bucketed by
  the swept agent's position and drawn over the shifted room; darker is lower error,
  brighter is higher, on a colour scale shared across every panel so brightness is directly
  comparable. Rows are model $\times$ intervention level, split into two side-by-side blocks
  (models 1--7 left, 8--14 right) so all fourteen fit on one page without a tall vertical
  stack; columns are observation mode $\times$ sample regime. Overlays: magenta squares mark
  the fixed agent (pinned at its goal), the dashed cyan box is the a-priori interaction zone
  fixed before seeing the data, and gold cell borders flag cells whose error is
  significantly elevated after BH-FDR correction ($q<0.05$). This is a discrete grid, so no
  smoothing is applied and every cell is the raw per-cell statistic.}
  \label{fig:shifted-rotated}
\end{figure}

\begin{figure}[htbp]
  \centering
  \includegraphics[trim={0pt 0pt 0pt 50pt}, clip,width=\linewidth]{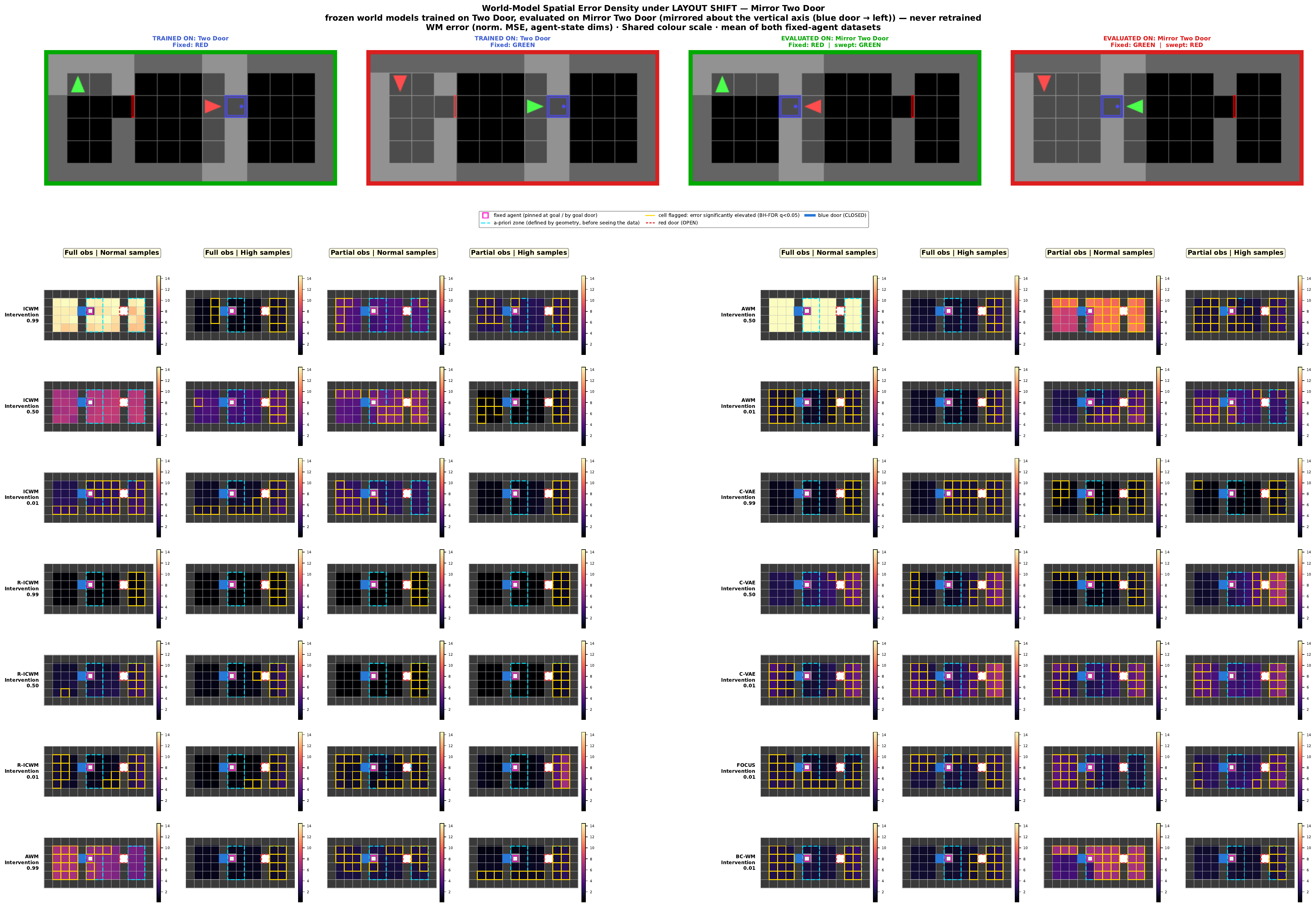}
  \caption{Dynamics error density on \texttt{Mirror Two Door}: the room reflected about the
  vertical axis, so the closed goal (blue) door moves from the right side of the corridor
  to the left and the open (red) door moves the opposite way. Same row/column layout,
  overlays (magenta fixed agent, dashed cyan zone, gold FDR flags), and shared colour scale
  as Fig.~\ref{fig:shifted-rotated}; discrete grid, so no smoothing. The mirror flips
  left-right without changing the axes' meaning, which makes it the sharpest test of
  whether a model has learned a direction-specific shortcut (its error spikes on the
  swapped side) instead of the true, symmetric corridor dynamics (its error stays low).}
  \label{fig:shifted-mirror-td}
\end{figure}

\begin{figure}[htbp]
  \centering
  \includegraphics[trim={0pt 0pt 0pt 47pt}, clip,width=\linewidth]{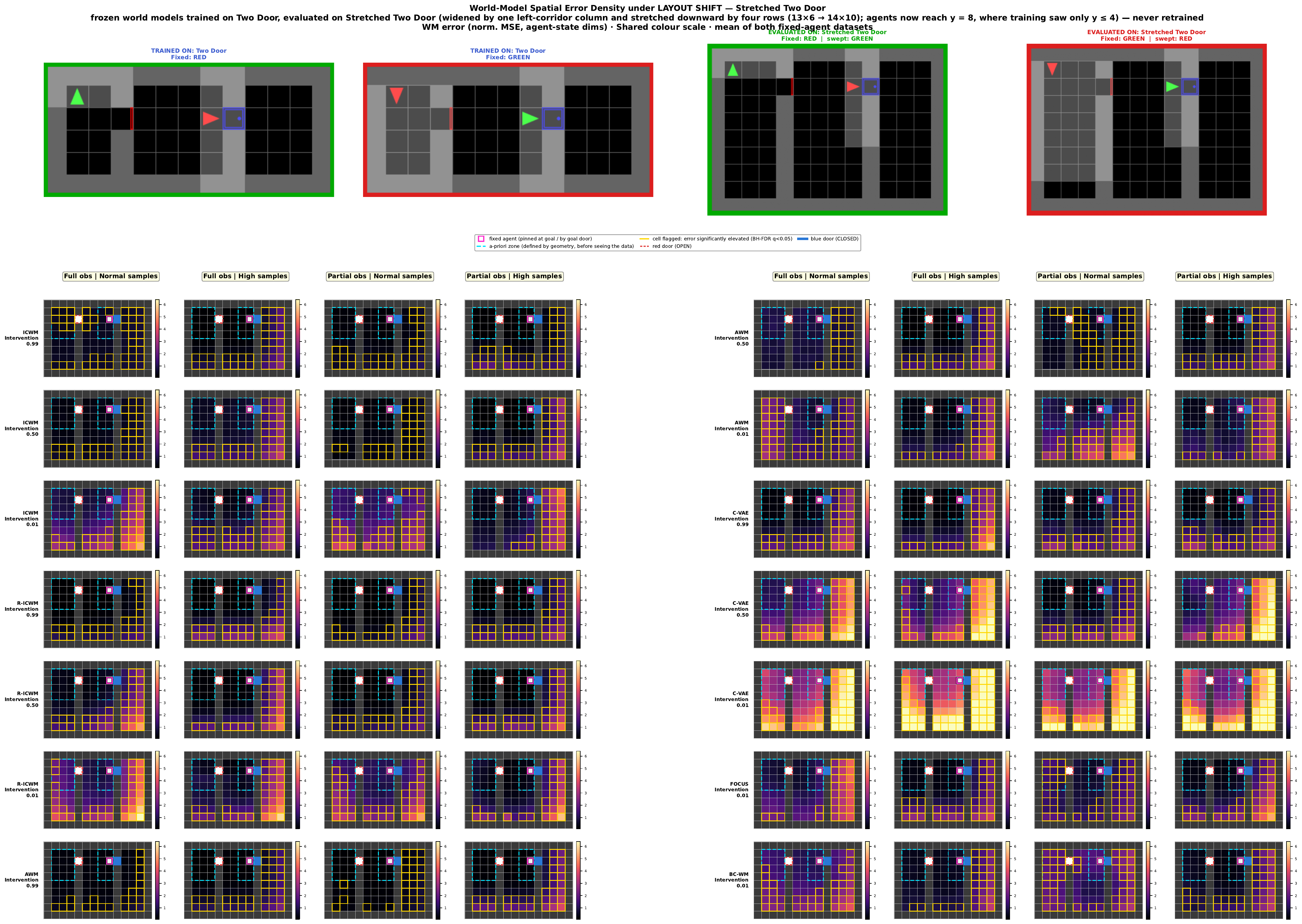}
  \caption{Dynamics error density on \texttt{Stretched Two Door}: the room widened by one
  corridor column and stretched by four extra rows, so the agents now reach
  $y$-coordinates the training layout never visited. Unlike the rotation and the mirror,
  this shift adds genuinely unseen space rather than rearranging seen space, so the lower
  rows of each panel test true extrapolation, and the error there is the clearest signal
  of whether a model has learned transferable dynamics or only a map of the training room.
  Same row/column layout, overlays (magenta fixed agent, dashed cyan zone, gold FDR
  flags), and shared colour scale as Fig.~\ref{fig:shifted-rotated}; discrete grid, so no
  smoothing.}
  \label{fig:shifted-stretched}
\end{figure}

\begin{figure}[htbp]
  \centering
  \includegraphics[trim={0pt 0pt 0pt 50pt}, clip,width=\linewidth]{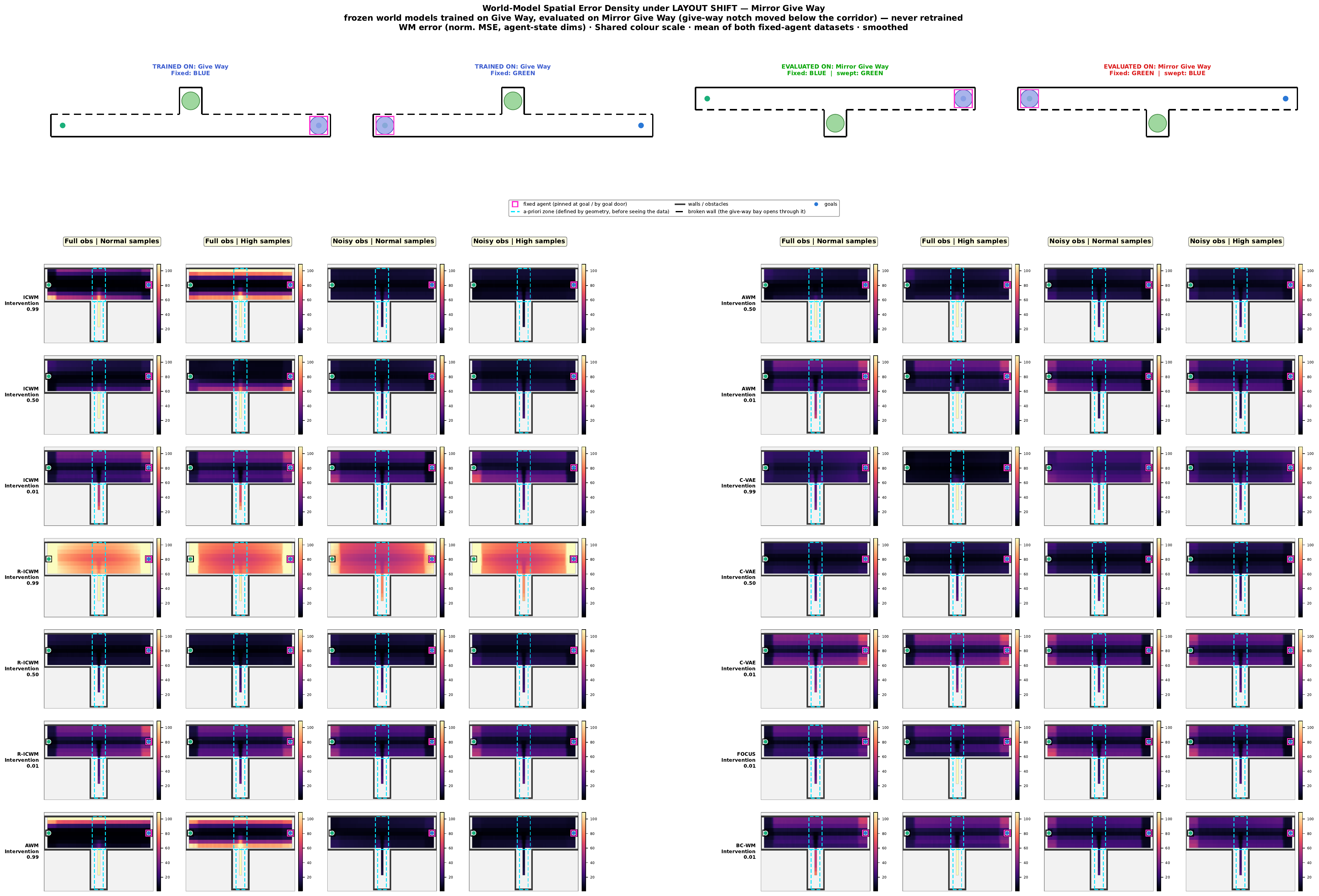}
  \caption{Dynamics error density on \texttt{Mirror Give Way}: the give-way bay (the notch
  off the corridor where one agent waits for the other to pass) flipped from above the
  corridor to below it, while the corridor itself, the goals, the agent radius, and the
  action range are unchanged. Same row/column layout and overlays (magenta fixed agent,
  dashed cyan bay zone) as Fig.~\ref{fig:shifted-rotated}, with the colour scale shared
  across every panel. Because so little of the sampled geometry actually moves, the error
  maps here look essentially identical to the un-shifted Give Way maps, the visual
  counterpart of the $\rho \approx 1.0$ ratios in Table~\ref{tab:shift-all}. This is a
  continuous environment, so (unlike the three discrete two-door shifts above) the
  sample-count-weighted Gaussian smoothing is applied, trading a little per-cell precision
  for a map that is not dominated by sampling noise in the sparsely-covered bay.}
  \label{fig:shifted-mirror-gw}
\end{figure}

\subsection{Figures: controlled single-factor shifts}

\begin{figure}[htbp]
  \centering
  \includegraphics[trim={0pt 0pt 0pt 50pt}, clip,width=\linewidth]{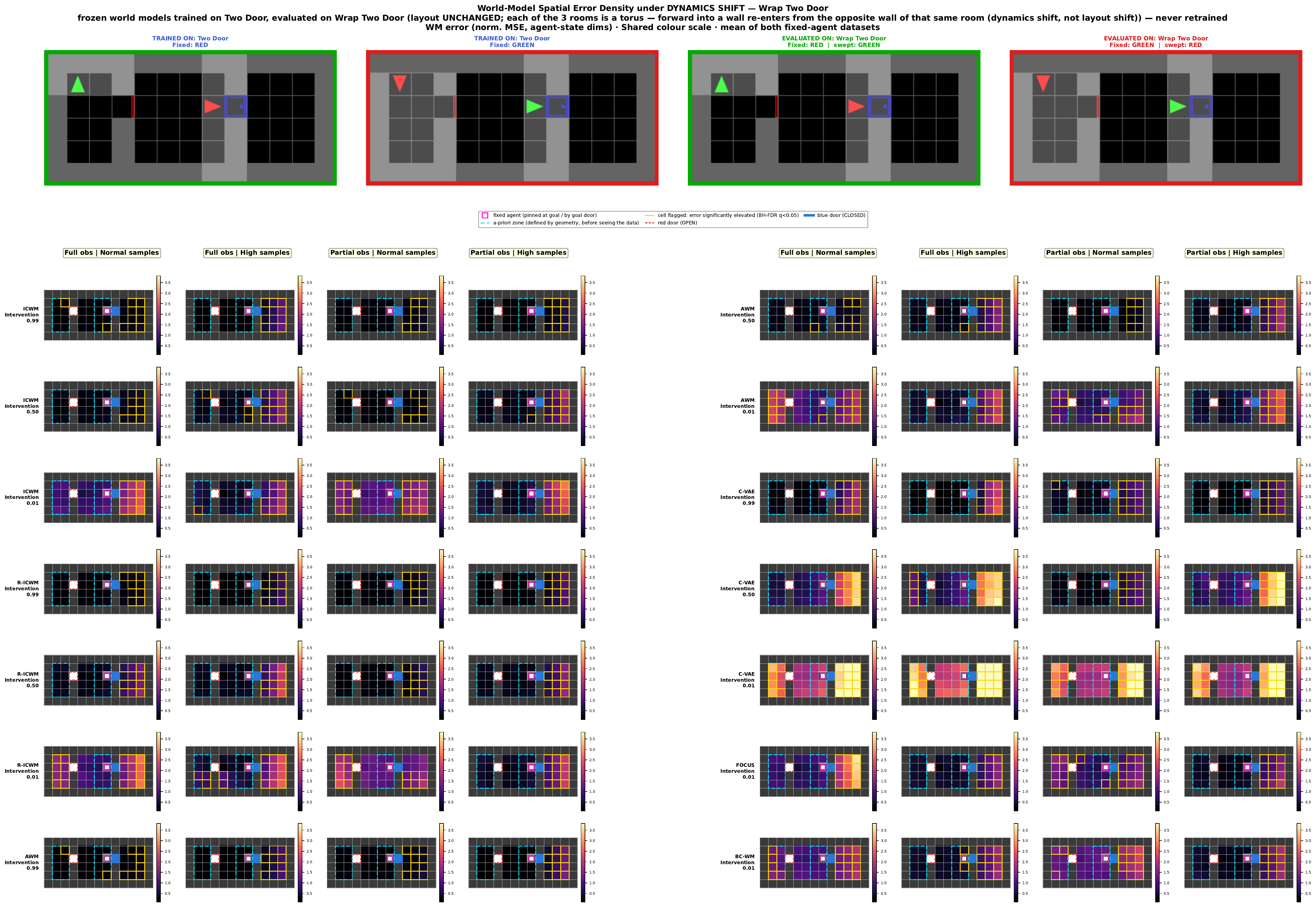}
  \caption{Dynamics error density on \texttt{Wrap Two Door}: the two-door layout is
  byte-identical to the training room (same walls, doors, legal cells, and grid size),
  but the \emph{transition function} is changed so that each of the three rooms is a torus,
  and a \texttt{forward} step into a wall re-enters from the opposite wall of that same
  room. The model is therefore fed only coordinates it has seen; the sole novelty is
  \emph{what happens next} at the room edges. This isolates the transition rule from the
  input distribution: a model that transfers here at $\rho \approx 1$ (as all do,
  Table~\ref{tab:shift-all}) is one whose error is driven by unfamiliar coordinates
  rather than by the dynamics themselves. Same row/column layout, overlays, and shared
  colour scale as Fig.~\ref{fig:shifted-rotated}; discrete grid, so no smoothing.}
  \label{fig:shifted-wrap}
\end{figure}

\begin{figure}[htbp]
  \centering
  \includegraphics[trim={0pt 0pt 0pt 50pt}, clip,width=\linewidth]{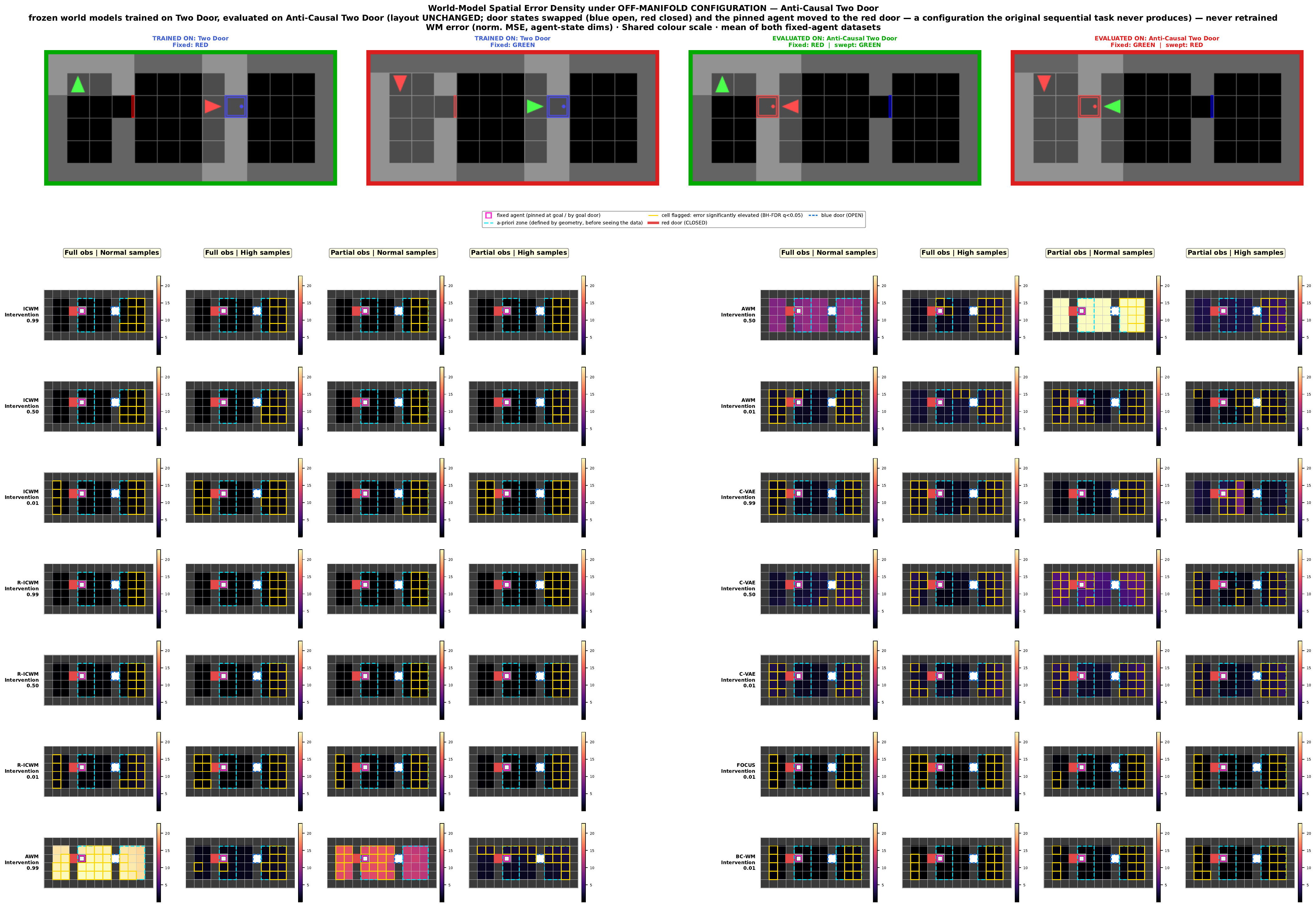}
  \caption{Dynamics error density on \texttt{Anti-Causal Two Door}: the geometry is
  unchanged, but the \emph{joint state} is placed off-manifold: the door states are
  swapped (blue open, red closed) and the pinned agent is moved to the red door, a
  sequentially impossible configuration since the task requires opening red before blue.
  Every coordinate and door column has been seen individually; only their joint setting is
  novel. The diagnostic is the split between the two error readings
  (Table~\ref{tab:shift-all}): a model that factored agent dynamics from door state
  should keep its \emph{position} error near the source ($\rho_{\mathrm{pos}} \approx 1$)
  while a model that memorised the joint configuration should see error rise
  \emph{everywhere at once}. The maps show the latter: error lifts across the whole room
  rather than concentrating at a door. Same row/column layout, overlays, and shared colour
  scale as Fig.~\ref{fig:shifted-rotated}; discrete grid, so no smoothing.}
  \label{fig:shifted-anti-causal}
\end{figure}

\begin{figure}[htbp]
  \centering
  \includegraphics[trim={0pt 0pt 0pt 50pt}, clip,width=\linewidth]{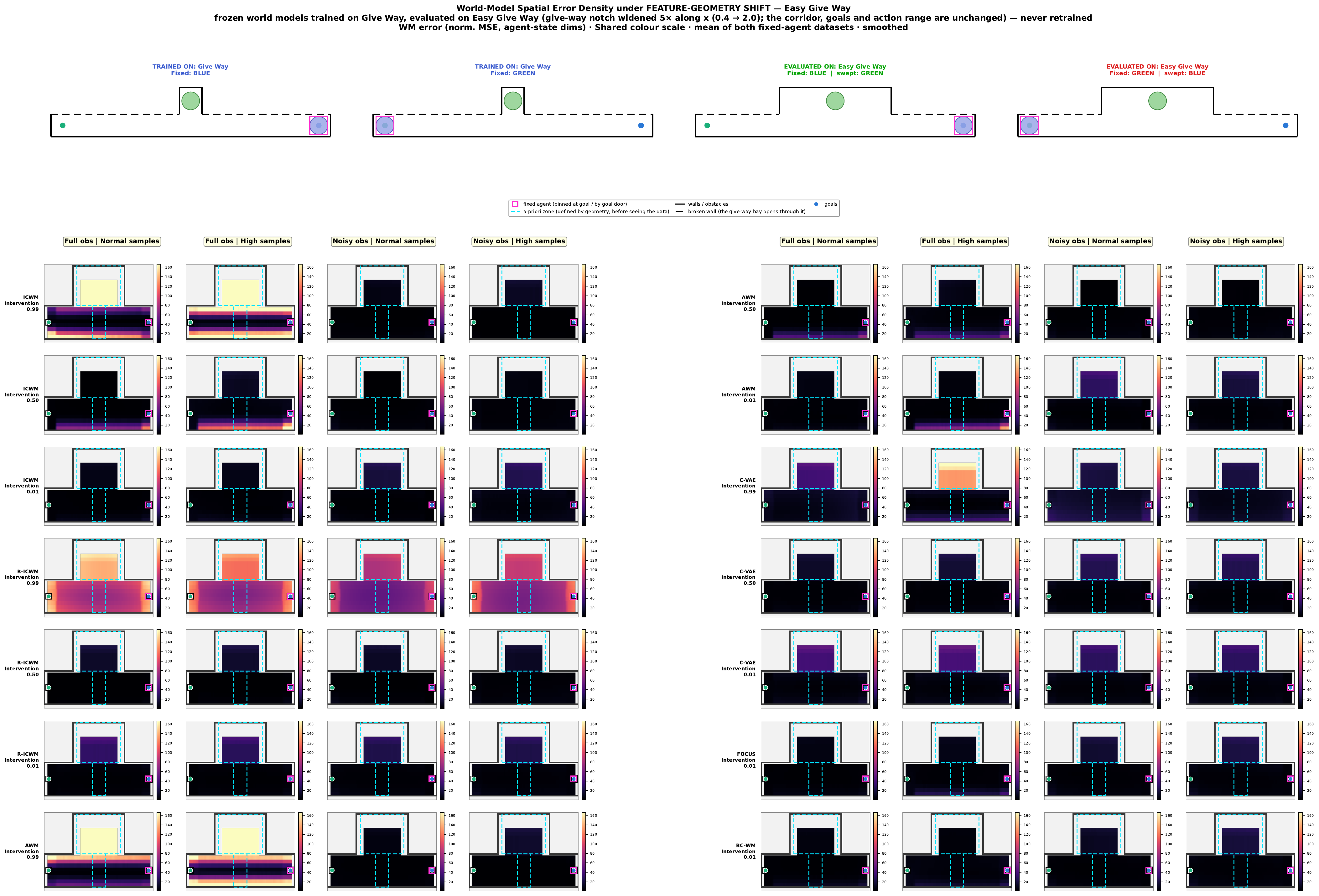}
  \caption{Dynamics error density on \texttt{Easy Give Way}: exactly one feature's geometry
  moves (the give-way notch is widened five-fold along $x$, passage $0.4 \to 2.0$),
  while the corridor, the goals, the agent radius, and the action range are all held fixed.
  The widened bay places the swept agent at $x$-extremes that were physically unreachable in
  the training layout, so this is a targeted probe of \emph{positional} extrapolation
  against a backdrop of unchanged dynamics. The result is the mirror image of
  Anti-Causal: dynamics error stays at or below the source ($\rho_{\mathrm{dyn}} \le 1$),
  while position error rises two- to five-fold in the new bay region
  (Table~\ref{tab:shift-all}). Continuous environment, so sample-count-weighted
  Gaussian smoothing is applied as in Fig.~\ref{fig:shifted-mirror-gw}. Same row/column
  layout and overlays as Fig.~\ref{fig:shifted-rotated}.}
  \label{fig:shifted-easy-gw}
\end{figure}

\subsection{Quantitative summary}

For each shifted environment we report the error at the lowest intervention level
($\mathrm{interv.}=0.01$), the condition where every model's data most closely resembles
near-optimal, goal-directed behaviour. This is also, as the discussion below shows, the
condition under which generalization is best, so it is the fair place to look for evidence
of transfer. $\rho_{\mathrm{dyn}}$ and $\rho_{\mathrm{pos}}$ are the generalization ratios
for the two error readings; a bold entry marks $\rho \le 1.5$, i.e.\ a shifted-layout error
within 50\% of the model's own source-layout error under the matched condition.

\begin{table*}[htbp]
  \centering
  \footnotesize
  \setlength{\tabcolsep}{3.5pt}
  \caption{Generalization ratios for all seven shifted layouts at $\mathrm{interv.}=0.01$,
  grouped by source environment. For each model, observation mode, and sample budget we
  report $\rho_{\mathrm{dyn}}$ and $\rho_{\mathrm{pos}}$, the dynamics- and position-error
  ratios against the model's own error on its un-shifted source layout under the matched
  condition. A bold entry marks $\rho \le 1.5$. The five two-door shifts are the three
  \emph{space-rearranging} ones (Rotated, Mirror, Stretched) plus the two
  \emph{controlled single-factor} ones (Wrap changes only the transition rule; Anti-Causal
  changes only the joint state); the two give-way shifts are the space-rearranging Mirror
  and the controlled Easy (one feature's extent widened). The controlled shifts isolate
  distinct signatures: Wrap transfers cleanly, Anti-Causal lifts \emph{dynamics} error while
  leaving position untouched, and Easy does the opposite. ``Degraded obs'' is partial
  observation for the two-door shifts and observation noise for the two give-way layouts.}
  \label{tab:shift-all}
  \resizebox{\textwidth}{!}{%
  \begin{tabular}{ll l cc cc cc cc cc cc cc}
    \toprule
    & & & \multicolumn{10}{c}{Two-Door} & \multicolumn{4}{c}{Give-Way} \\
    \cmidrule(lr){4-13}\cmidrule(lr){14-17}
    & & & \multicolumn{2}{c}{Rotated} & \multicolumn{2}{c}{Mirror}
      & \multicolumn{2}{c}{Stretched} & \multicolumn{2}{c}{Wrap}
      & \multicolumn{2}{c}{Anti-Causal}
      & \multicolumn{2}{c}{Mirror} & \multicolumn{2}{c}{Easy} \\
    \cmidrule(lr){4-5}\cmidrule(lr){6-7}\cmidrule(lr){8-9}\cmidrule(lr){10-11}%
    \cmidrule(lr){12-13}\cmidrule(lr){14-15}\cmidrule(lr){16-17}
    Obs & Samples & Model
      & $\rho_{\mathrm{dyn}}$ & $\rho_{\mathrm{pos}}$
      & $\rho_{\mathrm{dyn}}$ & $\rho_{\mathrm{pos}}$
      & $\rho_{\mathrm{dyn}}$ & $\rho_{\mathrm{pos}}$
      & $\rho_{\mathrm{dyn}}$ & $\rho_{\mathrm{pos}}$
      & $\rho_{\mathrm{dyn}}$ & $\rho_{\mathrm{pos}}$
      & $\rho_{\mathrm{dyn}}$ & $\rho_{\mathrm{pos}}$
      & $\rho_{\mathrm{dyn}}$ & $\rho_{\mathrm{pos}}$ \\
    \midrule
    \rowcolor{blue!5} Full & Normal & ICWM   & 2.2x & 3.2x & 2.5x & 2.4x & \textbf{1.9x} & 2.9x & \textbf{1.0x} & \textbf{1.1x} & 2.7x & \textbf{0.3x} & \textbf{1.1x} & \textbf{1.0x} & \textbf{0.2x} & 3.1x \\
    \rowcolor{blue!5} Full & Normal & R-ICWM & 2.5x & 3.2x & \textbf{1.1x} & \textbf{1.2x} & \textbf{1.9x} & 2.5x & \textbf{1.0x} & \textbf{1.1x} & 3.1x & \textbf{0.7x} & \textbf{1.0x} & \textbf{1.0x} & \textbf{0.3x} & 4.6x \\
    \rowcolor{blue!5} Full & Normal & AWM    & 1.9x & 2.0x & \textbf{1.2x} & 1.8x & \textbf{1.3x} & \textbf{1.7x} & \textbf{1.0x} & \textbf{1.1x} & 2.4x & \textbf{1.4x} & \textbf{1.1x} & \textbf{1.0x} & \textbf{0.2x} & 2.4x \\
    \rowcolor{blue!5} Full & Normal & C-VAE  & \textbf{1.6x} & \textbf{1.4x} & \textbf{0.9x} & \textbf{1.0x} & \textbf{1.5x} & \textbf{1.7x} & \textbf{1.0x} & \textbf{1.0x} & 2.8x & \textbf{0.7x} & \textbf{1.0x} & \textbf{1.0x} & \textbf{0.4x} & 2.7x \\
    \addlinespace
    \rowcolor{green!5} Full & High   & ICWM   & 2.3x & 3.3x & \textbf{1.5x} & 1.8x & \textbf{1.2x} & 2.1x & \textbf{1.1x} & \textbf{1.1x} & 5.0x & \textbf{0.7x} & \textbf{1.2x} & \textbf{0.9x} & \textbf{0.2x} & 3.4x \\
    \rowcolor{green!5} Full & High   & R-ICWM & 2.1x & 4.1x & \textbf{0.6x} & \textbf{1.1x} & \textbf{1.3x} & 2.0x & \textbf{1.1x} & \textbf{1.1x} & 3.7x & \textbf{0.7x} & \textbf{0.9x} & \textbf{0.8x} & \textbf{0.3x} & 4.9x \\
    \rowcolor{green!5} Full & High   & AWM    & 2.4x & 3.4x & \textbf{1.1x} & \textbf{1.0x} & \textbf{0.9x} & 2.1x & \textbf{1.0x} & \textbf{1.1x} & 3.8x & \textbf{1.6x} & \textbf{1.3x} & \textbf{1.2x} & \textbf{0.7x} & 2.1x \\
    \rowcolor{green!5} Full & High   & C-VAE  & 2.4x & 2.6x & \textbf{1.2x} & 1.5x & \textbf{2.0x} & \textbf{2.0x} & \textbf{1.0x} & \textbf{1.0x} & 2.8x & \textbf{0.8x} & \textbf{1.0x} & \textbf{1.0x} & \textbf{0.4x} & 2.9x \\
    \addlinespace
    \rowcolor{orange!5} Degraded & Normal & ICWM   & 2.2x & 4.0x & 2.4x & \textbf{1.2x} & \textbf{1.7x} & \textbf{2.0x} & \textbf{1.0x} & \textbf{1.1x} & 3.2x & \textbf{0.4x} & \textbf{1.0x} & \textbf{1.0x} & \textbf{0.3x} & 2.9x \\
    \rowcolor{orange!5} Degraded & Normal & R-ICWM & 3.4x & 5.6x & \textbf{1.1x} & \textbf{1.5x} & \textbf{1.6x} & 2.9x & \textbf{1.0x} & \textbf{1.1x} & 4.0x & \textbf{0.7x} & \textbf{1.0x} & \textbf{1.0x} & \textbf{0.3x} & 2.9x \\
    \rowcolor{orange!5} Degraded & Normal & AWM    & 3.6x & 3.5x & 3.6x & 3.0x & 2.4x & 3.1x & \textbf{1.0x} & \textbf{1.2x} & 4.0x & 1.7x & \textbf{1.0x} & \textbf{1.0x} & \textbf{0.3x} & 2.9x \\
    \rowcolor{orange!5} Degraded & Normal & C-VAE  & 2.6x & 2.5x & \textbf{1.2x} & \textbf{1.2x} & \textbf{1.6x} & \textbf{1.9x} & \textbf{1.0x} & \textbf{1.0x} & 2.9x & \textbf{1.0x} & \textbf{1.0x} & \textbf{1.0x} & \textbf{0.4x} & 2.6x \\
    \addlinespace
    \rowcolor{purple!5} Degraded & High   & ICWM   & \textbf{1.5x} & 2.5x & \textbf{0.9x} & \textbf{1.1x} & \textbf{1.2x} & 2.1x & \textbf{1.0x} & \textbf{1.1x} & 3.8x & \textbf{0.5x} & \textbf{1.1x} & \textbf{1.3x} & \textbf{0.4x} & 2.7x \\
    \rowcolor{purple!5} Degraded & High   & R-ICWM & \textbf{1.4x} & 4.0x & \textbf{1.4x} & 2.6x & \textbf{1.1x} & 2.9x & \textbf{1.0x} & \textbf{1.1x} & 6.4x & \textbf{0.9x} & \textbf{1.0x} & \textbf{1.0x} & \textbf{0.3x} & 2.9x \\
    \rowcolor{purple!5} Degraded & High   & AWM    & 2.4x & 4.9x & 4.2x & \textbf{1.1x} & \textbf{1.7x} & 2.8x & \textbf{1.0x} & \textbf{1.1x} & 4.5x & \textbf{0.8x} & \textbf{1.0x} & \textbf{0.7x} & \textbf{0.3x} & 2.9x \\
    \rowcolor{purple!5} Degraded & High   & C-VAE  & 2.5x & 2.5x & \textbf{1.4x} & \textbf{1.4x} & \textbf{1.7x} & 2.0x & \textbf{1.0x} & \textbf{1.0x} & 4.2x & \textbf{0.7x} & \textbf{1.0x} & \textbf{1.0x} & \textbf{0.4x} & 2.6x \\
    \bottomrule
  \end{tabular}%
  }
\end{table*}

\subsection{Discussion}

\textbf{Generalization is intervention-dependent, and the effect is large.} Averaging the
dynamics-error ratio $\rho_{\mathrm{dyn}}$ over the three two-door shifts and both
observation modes, at $\mathrm{interv.}=0.99$ (near-optimal source data replaced by heavily
exploratory data) the causal and autoregressive models transfer badly: ICWM averages
$61\times$ worse on the shifted layout than on its own training layout, AWM $30\times$
worse, R-ICWM $7\times$ worse. Lowering the intervention level to $0.01$ collapses every
one of these ratios into a tight band of \textbf{1.1x--3.6x} (Table~\ref{tab:shift-all}),
with several cells at or below $1.5\times$ and R-ICWM occasionally transferring \emph{better}
than its own source-layout error ($\rho < 1$). The pattern is consistent enough to state
plainly: \textbf{the world models generalize to a shifted layout only in the regime where
their training data was itself close to near-optimal, goal-directed behaviour.} Exploratory
training data widens the states a model has seen, but it does so within the training
room's own geometry; that coverage does not transfer when the geometry itself changes, and
the model ends up more, not less, tied to the specific room it was fitted on.

\textbf{Mirror Give Way is the cleanest positive control among the space-rearranging
shifts.} Every model, at every observation mode, sample regime, and intervention level we
checked, transfers to the mirrored give-way layout at $\rho \approx 1.0$
(Table~\ref{tab:shift-all}), including at the highest intervention level where the two-door
shifts fail hardest. This is not a stronger claim about causal generalization than the
geometry supports: flipping the bay to the other side of the corridor changes almost
nothing about what the swept agent's position actually looks like relative to the fixed
agent and the corridor walls, so a model that has learned the corridor dynamics at all
should transfer near-perfectly regardless of intervention level. Mirror Give Way is
therefore best read as validating the measurement pipeline (a genuinely easy transfer reads
as easy) rather than as evidence that these models generalize broadly under layout shift.

\textbf{The controlled shifts separate ``unseen coordinates'' from ``wrong mechanism.''}
The space-rearranging family cannot say \emph{why} a model fails, because rotating or
stretching a room both moves the dynamics and hands the model coordinates outside its
support. The three single-factor shifts in Table~\ref{tab:shift-all} take those
causes apart, and each produces a distinct, readable signature.
\emph{(i)~Wrap Two Door} feeds the model only seen coordinates and changes only what happens
next at the room edges; \emph{every} model transfers at $\rho_{\mathrm{dyn}},
\rho_{\mathrm{pos}} \approx 1.0$ across all sixteen conditions. This is the decisive
negative control: it shows the large ratios in Table~\ref{tab:shift-all} are driven by
unfamiliar \emph{coordinates}, not by the presence of a shift as such: when the input
stays on-manifold, transfer is essentially perfect even though the transition rule has
changed underneath it.
\emph{(ii)~Anti-Causal Two Door} holds every coordinate fixed but places the joint
door/agent state off-manifold. Here dynamics error rises two- to six-fold
($\rho_{\mathrm{dyn}} = 2.4\text{--}6.4$) while position error is essentially unchanged and
frequently \emph{below} the source ($\rho_{\mathrm{pos}} \approx 0.3\text{--}1.0$). The
error lifts across the whole room rather than concentrating at a door: the models predict
where the agent goes just fine, but misread the door dimensions of the next state, exactly
the failure expected of a model that has entangled agent dynamics with the joint
configuration rather than factoring them.
\emph{(iii)~Easy Give Way} is the mirror image: widening one feature's extent leaves the
dynamics untouched ($\rho_{\mathrm{dyn}} \le 0.7$, often $\approx 0.2$) but drives
\emph{position} error up two- to five-fold ($\rho_{\mathrm{pos}} = 2.1\text{--}4.9$) in the
newly reachable bay region. The velocity model is correct; it is only the mapping to the
unseen $x$-extremes that breaks. Taken together, the two positive shifts (Anti-Causal and
Easy) demonstrate that the $\rho_{\mathrm{dyn}}$/$\rho_{\mathrm{pos}}$ split is not a
reporting artefact but a genuine, dissociable pair of failure modes: one shows up in
dynamics alone, the other in position alone, on demand.

\textbf{Position error is consistently harder to transfer than dynamics error, except
when the mechanism itself is what breaks.} Within Table~\ref{tab:shift-all}, for the three
two-door \emph{space-rearranging} shifts $\rho_{\mathrm{pos}}$ is larger than
$\rho_{\mathrm{dyn}}$ in most rows, sometimes by a wide margin (R-ICWM, rotated, full obs,
high samples: $2.1\times$ on dynamics but $4.1\times$ on position alone), because moving the
room's coordinates is primarily a positional challenge. Anti-Causal Two Door is the sole
condition that inverts this ordering ($\rho_{\mathrm{dyn}} \gg \rho_{\mathrm{pos}}$),
precisely because it is the one shift that leaves coordinates alone and corrupts the
mechanism instead. Reporting both readings is what makes these two regimes distinguishable;
a dynamics-only or position-only metric would collapse them.

\subsection{Where the shift actually bites: subtracting the source error surface}
\label{app:notch-distortion}

The give-way error maps above are hard to read by eye. The corridor is a thin
$5.0\times0.95$ sliver, and its one-step error is dominated by a flat level that does not
depend much on position, so the change the shift causes is a small signal sitting on top of
a large one, and the heatmap shows a smooth wash. The ratios in Table~\ref{tab:shift-all}
have the opposite problem: they turn the whole corridor into one number, so a real change
that only affects a few cells is averaged away to $\rho\approx1$.

One operation fixes both. For each frozen model we subtract, cell by cell, that same model's
error on its \emph{own} training layout, leaving the change the shift caused,
\begin{equation}
  \Delta(x,y) \;=\; \overline{\mathrm{err}}_{\text{shifted}}(x,y)
                 \;-\; \overline{\mathrm{err}}_{\text{source}}(x,y).
  \label{eq:notch-residual}
\end{equation}
The subtraction is exact rather than approximate: \texttt{Easy Give-Way} uses the source
grid unchanged, and \texttt{Mirror Give-Way}'s grid is the source grid reflected in $y$, so
after the flip the two edge arrays agree to machine precision and all $691$ valid cells
match one to one. The bay is sparsely visited, so we report medians over the two
role-swapped datasets; no single cell can drive a number below.

Figure~\ref{fig:notch-aggregate} lays out every give-way condition in one grid. We split the
corridor into the three regions the sampler actually resolves, read off the reachable cells
rather than picked by hand: the \textbf{Corridor Wall} (the straight stretch away from the
bay), the \textbf{Corridor Connector} (the same band where it passes the bay's mouth), and
the \textbf{Notch} (the bay interior). Rows are the three layouts, columns are the regions
and then the models. Two things stand out.

\textbf{Only the region whose geometry moved changes.} The Wall and Connector groups look
the same in all three rows: same medians, same spread. Neither shift touches that geometry,
and the error agrees. The Notch is the only region that moves, and it moves differently for
the two shifts. Widening the notch five-fold in \texttt{Easy Give-Way} opens $96$ cells per
dataset in the throat $x\in[-0.83,0.83]$ that have \emph{no counterpart} in the training
layout, so they never enter Table~\ref{tab:shift-all} at all. Those cells are where the
damage is. At $\sigma=0.01$ the median error there is $75\times$ the shared-cell median for
AWM ($991$ vs.\ $13$) and $36\times$ for ICWM ($443$ vs.\ $12$), and every one of the $96$
cells is more than five times the shared-cell median, while the shared cells transfer
normally. \texttt{Mirror Give-Way} instead shows a mild, roughly even lift in the Notch,
matching the fact that its geometry moved but nothing new opened up.

\textbf{The damage is gated by intervention, and only visible under clean observation.}
Raising $\sigma$ from $0.01$ to $0.50$ collapses the AWM and ICWM novel/shared ratio from
$75\times$ and $36\times$ to about $1\times$; at $\sigma=0.99$ it is $2$--$3\times$. This is
the localized, spatial version of the headline result of this appendix: a model trained on
near-optimal data ties its predictions to the exact reach of the training room, so enlarging
that reach breaks it badly and \emph{only there}, while exploratory data removes the cliff.
R-ICWM and C-VAE do not show the same clean gating, since their error is set by other
bottlenecks. The effect also depends on observation quality: under observation noise the
same $\sigma=0.01$ fold-change is only $3.3\times$ (AWM) and $4.0\times$ (ICWM), because the
noise floor swamps it. Clean-observation maps are the sensitive probe here, and a study run
only on noisy rollouts would under-report this failure.

\begin{figure}[htbp]
  \centering
  \includegraphics[width=\linewidth]{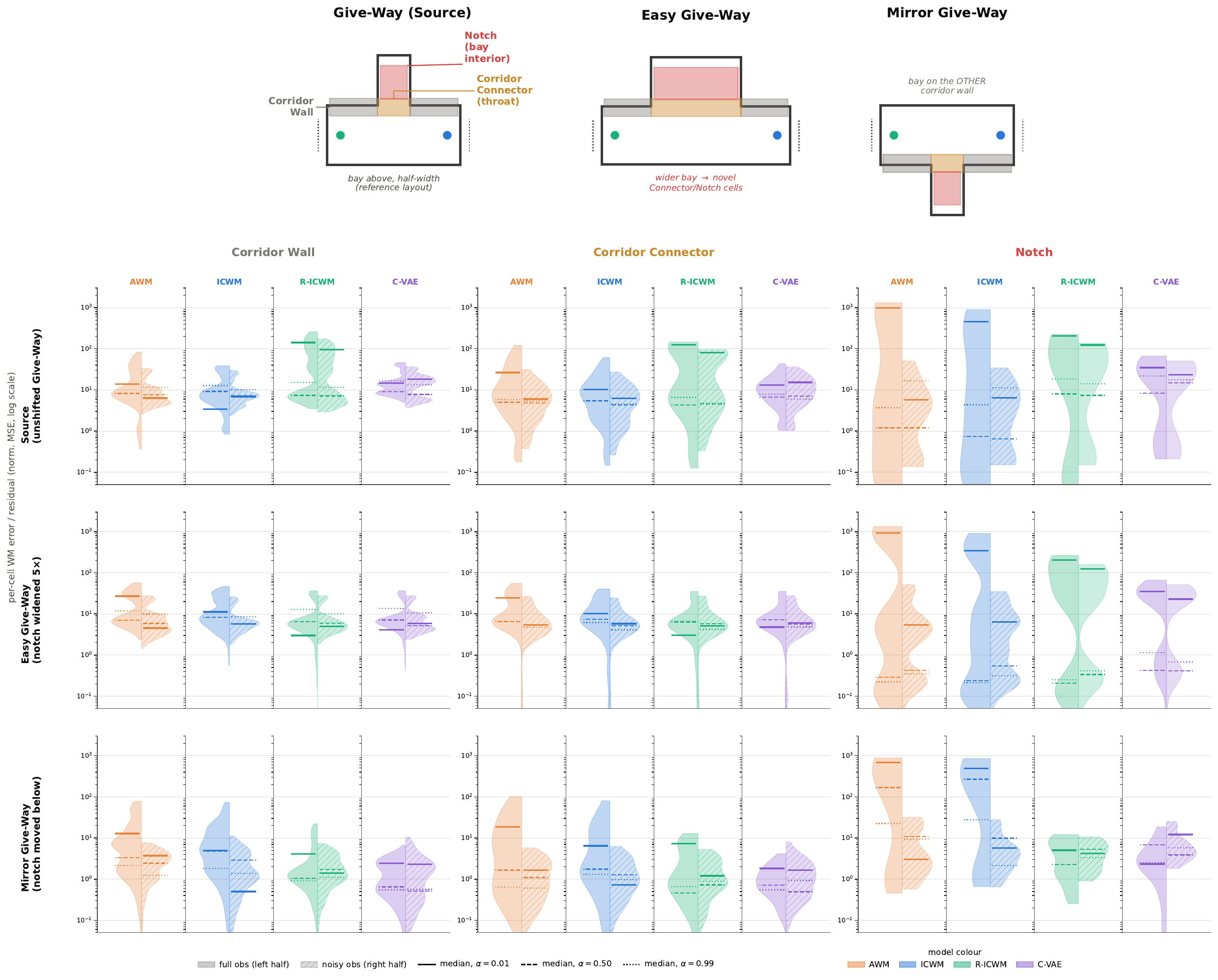}
  \caption{\textbf{Give-way distortion by region, layout, model and intervention.}
  \emph{Top:} the three regions marked on the real geometry -- Corridor Wall (grey),
  Corridor Connector (gold), Notch (red). \emph{Grid:} rows are the three layouts (Source =
  unshifted reference, Easy, Mirror); column groups are the regions, each split into the
  four models. Every violin pools all three intervention levels but keeps separate median
  ticks (solid $\sigma=0.01$, dashed $\sigma=0.5$, dotted $\sigma=0.99$; the panel labels
  them $\alpha$); the left half is
  full observation, the hatched right half is noisy observation; $y$ is log. For Easy and
  Mirror the plotted value is $|\Delta|$ of Eq.~\eqref{eq:notch-residual} on shared cells,
  together with the raw error on cells only the shifted layout reaches. Wall and Connector
  are stable everywhere; the Notch is the only region that moves, and only Easy shows the
  sharp split between $\sigma=0.01$ and the higher levels.}
  \label{fig:notch-aggregate}
\end{figure}

\subsection{Splitting the velocity channel: where recurrence pays off}
\label{app:velocity-companion}

The dynamics error $\overline{\mathrm{err}}_{\mathrm{dyn}}$ used above averages over all
agent-state dimensions, position and velocity together. That is the right summary for
transfer, but it can hide an effect that lives in one channel only. This subsection reports
one such case, and supports Figure~\ref{fig:ricwm-edge} of the main text.

The question is simple: R-ICWM is ICWM with a recurrent state, so where should memory
help? A single frame already tells you almost everything when an agent is moving fast and
straight. It tells you much less when the agent is nearly stationary, because a still frame
cannot say whether the agent is about to move, and in which direction. So we expect
recurrence to pay off where the training-time speed was \emph{low}. The give-way corridor
gives us a natural gradient of exactly that: Figure~\ref{fig:velocity-companion}B shows the
median training speed falls from the Corridor Wall, through the Connector, to the Notch,
which agents only enter to wait.

To test it we recompute the error restricted to the velocity dimensions alone
($\mathrm{err}_{\mathrm{velo}}$, the two agents' velocity components), re-running the frozen
models' forward pass on the same cached grid probes so the normalization path is identical
to the reports above. The result is a clean, monotonic trend
(Figure~\ref{fig:velocity-companion}D). In the fast Corridor Wall the two models are roughly
even ($\Delta\log_{10}$ between $+0.34$ and $-0.11$). Through the Connector, R-ICWM pulls
ahead ($-0.22$ at $\sigma=0.5$). In the slow Notch its advantage is largest, $-0.55$ at
$\sigma=0.5$ and $-1.26$ at $\sigma=0.99$, the two conditions with the lowest trained speeds
of all. This is exactly the predicted direction, and it holds in all three give-way layouts;
under \texttt{Mirror} the Notch gap reaches $-3.05$, so the layout shift does not disrupt
the mechanism and if anything sharpens it.

Two cautions. First, this trend does \emph{not} appear in the mixed
$\overline{\mathrm{err}}_{\mathrm{dyn}}$: there, R-ICWM carries a large, region-independent
handicap at $\sigma=0.01$ that comes from how its recurrent history is padded at grid-probe
time, which swamps the real velocity-channel signal and even reverses the ordering in the
Notch. Isolating the velocity dimensions removes that confound. Second, the Connector and
Notch have no $\sigma=0.01$ entry, because at near-optimal behaviour agents never visit them
at all. The practical lesson generalizes past this one result: when a hypothesis concerns a
\emph{subset} of the state dimensions, an aggregate error that mixes all of them can hide or
even invert the effect.

\begin{figure}[htbp]
  \centering
  \includegraphics[width=\linewidth]{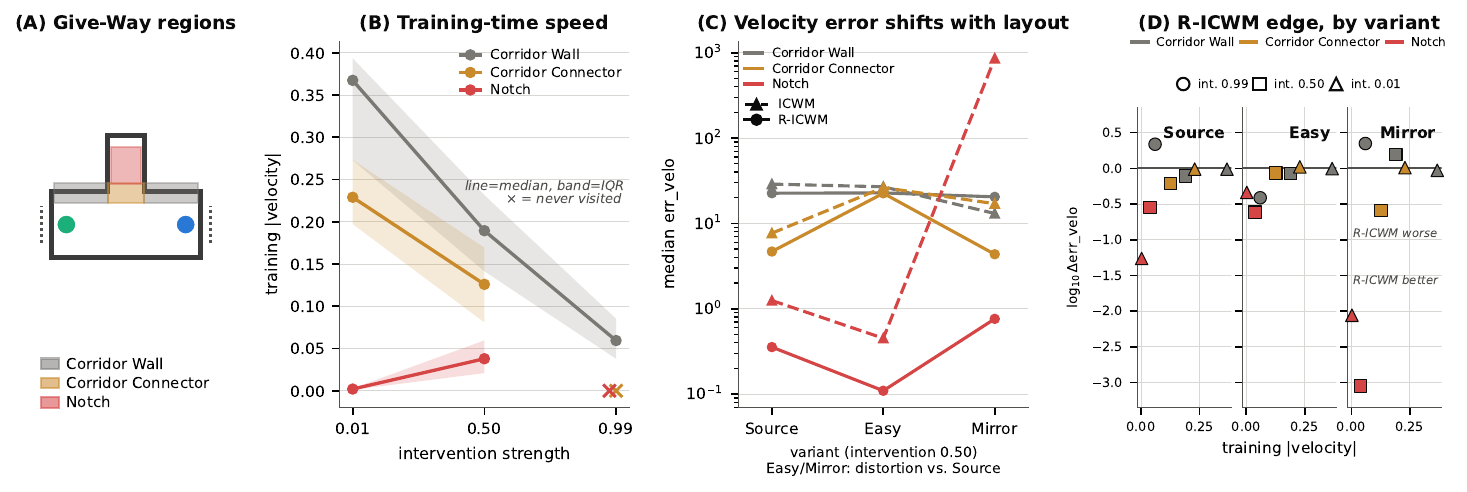}
  \caption{\textbf{Recurrence helps where the training speed was low.}
  \textbf{(A)} The three give-way regions.
  \textbf{(B)} Median training-time speed per region against intervention level (band: IQR;
  $\times$ marks a region never visited at that level). Wall is fastest, Notch slowest.
  \textbf{(C)} Median velocity-only error as the layout changes, at $\sigma=0.5$; for Easy
  and Mirror the value is the distortion against the model's own Source error.
  \textbf{(D)} R-ICWM minus ICWM in $\log_{10}$ velocity error, against training speed, one
  panel per layout. Below zero means R-ICWM is better. The edge is near zero in the fast
  Wall and grows as the region's trained speed drops, in every layout.}
  \label{fig:velocity-companion}
\end{figure}

\subsection{Overall takeaway}

Read together, the three views above say something more precise than the ratios alone. The
aggregate $\rho$ of Table~\ref{tab:shift-all} answers ``does the model still work after the
shift'', and its answer is that transfer holds only when the training data was close to
near-optimal, and only within the training geometry. Subtracting the source surface then
answers ``\emph{where} does it break'', and the answer is: on exactly the geometry that
moved, and nowhere else, with a severity set by the intervention level of the training data
and visible only under clean observation. Splitting the velocity channel answers ``\emph{which
part} of the prediction breaks'', and shows that a model choice invisible in the mixed error recurrence has a systematic, direction-predictable effect once the channel it acts on
is isolated. The common thread is that averaging is what hides the structure: over cells, over
state dimensions, or over observation modes. Each of the three fixes removes one such average,
and each reveals a different, consistent part of the same picture.

See Appendix~\ref{app:error-density} for the source-layout error densities these ratios
are measured against, and Appendix~\ref{app:shifted-envs} for the shifted layouts
themselves.


\section{Towards Pearl’s framework to evaluate counterfactual queries}
\label{app:structure-correlation}

\subsection{Setup}

The world models of Appendix~\ref{app:error-density} split into two groups: five models
that are handed a discovered causal graph (ICWM, R-ICWM, FOCUS, BC-WM, and C-VAE, which
injects the graph into its latent transition) and one model that is handed no graph at
all (AWM). If the graph is doing real work, the models that use it should predict better
out of distribution, and that advantage should grow as the graph gets more accurate. We
test this by pooling every world-model run across three environments, six architectures,
and the full intervention sweep, and joining each run to the structural error of the
graph it was trained with. This gives 1{,}240 world-model runs over 295 datasets, with
complete coverage between the two.

\subsection{The pooled correlation is confounded}

Figure~\ref{fig:struct-main} shows the headline result. If we pool all environments
together, structural error and out-of-distribution error are \emph{negatively} correlated.
Taken at face value this would say that worse graphs make world models look better. It is
an artefact: the environments differ by an order of magnitude in both variable count and
error scale, so pooling them mixes two very different scales (this is Simpson's paradox).
We therefore compute the correlation within each cell instead, where a cell is one
combination of environment, observation mode, intervention level, and architecture. Inside
a cell the correlation flips sign and grows with the distribution shift: it is mildly
negative in distribution, near zero under a mild shift, and $+0.59$ under the adversarial,
out-of-distribution evaluation. This is exactly the pattern we would expect if the
structure mattered most where the model is stressed hardest.

That pattern fails its own control test. AWM never receives a graph, so its accuracy cannot
depend on how good the graph is. Yet under the same adversarial evaluation AWM shows a
within-cell correlation of $+0.69$, slightly stronger than the graph-informed models. A
Mann-Whitney test on the per-cell correlations finds no difference between the two groups
($p = 0.29$). So whatever produces the $+0.59$ correlation, it is not the graph: a model
that never sees the graph shows the same effect.

\subsection{Intervention strength is the common cause}

Figure~\ref{fig:struct-alpha} identifies the real driver. Intervention strength $\sigma$
measures how far the training policy is pushed away from pure goal-seeking. It correlates
with two things at once: the discovery graph's structural error ($\rho = -0.75$) and the
world model's own out-of-distribution error ($\rho = -0.48$). In other words, more
exploratory data makes both the discovered graph and the world model better, and it does so
independently. When we hold $\sigma$ fixed and re-measure, the correlation between
structural error and out-of-distribution error drops from $+0.59$ to $-0.01$, which is
indistinguishable from zero. So the data do not support recovered-graph $\to$
OOD-error (panel C). Instead $\sigma$ is a common cause: it drives both the graph and the
world model's accuracy, and neither one passes its effect through the other. A matched-pair
comparison against the graph-free AWM baseline confirms this directly: receiving a graph
gives no significant improvement to ICWM, FOCUS, or BC-WM, and makes R-ICWM measurably
worse. The benefit does not grow with graph quality because, on this evidence, there is no
benefit to grow.

\subsection{The null replicates under layout shift}

Figure~\ref{fig:struct-shift} repeats the same test under a much sharper distribution
shift: frozen world models evaluated on all seven shifted layouts of
Appendix~\ref{app:shifted-generalization}, never retrained. If the graph encoded real
mechanism, its advantage should be most visible exactly here, where the model must
extrapolate rather than merely predict a familiar next state. It is not. Panels A--G
show every architecture growing worse at roughly the same rate as intervention strength
increases, panel H shows the ICWM/AWM error ratio straddling parity at every $\sigma$
with no trend, and panel J shows the spread across all six architectures collapsing to
within about 1\% of each other at the highest intervention level, the opposite of what a
graph-driven advantage would predict. Panel I supplies the mechanism for why: the
two-door discovery graph's input gate keeps 21--22 of 29 inputs open at every
intervention level tested, essentially a dense model in disguise, because two-door's
real sparsity lives in the input-by-output coupling matrix, which a purely input-marginal
gate cannot express.

\subsection{A localized, conditional exception}

Figure~\ref{fig:struct-stretch} is the one place the graph earns its keep. It narrows the
null result above rather than overturning it. Stretched Two-Door adds a region of the room
(roughly the bottom half) that the training layout never visited, so a frozen model
evaluated there has to extrapolate. We split the results by observation mode. Under partial
observation, in the region the training data did cover, ICWM beats AWM by $0.53\times$
(panel A; it wins 87\% of paired cells, $p = 2.8\times 10^{-15}$). Every graph-gated
regressor beats the graph-free baseline there, while C-VAE, which receives the same
soft-gated PCMCI graph, is $2.66\times$ worse (panel C). So the structure pays off only
through an architecture that can use it. Under full observation the same mask instead
\emph{hurts} (panels A, B): when the whole state is visible, gating inputs away just removes
signal the model could have used. Panel F reads this as Pearl's three-stage counterfactual
query. Partial observation is the abduction step, where the model must infer the state it
cannot see, and this is exactly where a structural prior can stand in for what is missing.
Intervention strength is the action step: the benefit peaks at $\sigma = 0.5$, where the
graph is best identified and the shift is real but not so extreme that no prior transfers.
The stretched layout's unseen region is the prediction step, pushing the model past its
training support. The graph helps at the one stage it is suited to, and nowhere else.

\subsection{Why stretch benefits and rotation does not}

Figure~\ref{fig:struct-align} answers the natural follow-up. The five two-door shifts
(stretched, mirrored, rotated, wrap, anti-causal) give the frozen mask a wide range of
coordinate novelty (33--75\%, panel C), yet the benefit does not track that novelty. The
discovered graph is byte-identical across all five, because the ground-truth schema does
not depend on geometry, so the difference cannot be the graph. Two things decide it: whether
the shift keeps the frozen mask \emph{aligned} with the room's coordinates, and whether it
is the kind of shift a state-variable mask can act on at all. Stretching adds unseen values
along axes that keep their meaning (invalidity $0.0$; the mask helps, $0.68$--$0.92\times$
across intervention levels, panel A). Mirroring swaps which door sits at which column but
leaves the axes intact (invalidity $0.5$; the mask helps on average but fails at high
intervention, where the model has learned to rely on which door is which). Rotation swaps
the $x$ and $y$ axes themselves (invalidity $1.0$; no benefit, panel B, $p = 0.35$). Wrap
keeps the room byte-identical and changes only the transition rule (invalidity $0.0$, the
lowest novelty of any variant; still no benefit, $1.03\times$, $p = 0.95$). Anti-Causal
swaps the door \emph{states}, a corrupted variable the mask firewalls, and gives the largest
benefit anywhere ($0.007\times$). Novelty alone does not decide the outcome: panel C pairs
stretch and rotation at the same $75\%$ novelty with opposite results, and pairs wrap and
anti-causal at the same $33\%$ novelty with 150-fold-different results. What decides the
outcome is what the shift changes: the
meaning of a coordinate, the transition rule, or the value of a variable. It is not how
many new coordinate values appear. The single-factor Wrap and Anti-Causal shifts are
developed in full in Appendix~\ref{app:struct-newenvs}.

\subsection{Figures}

\begin{figure}[htbp]
  \centering
  \includegraphics[trim={0pt 0pt 00pt 20pt}, clip,width=0.9\linewidth]{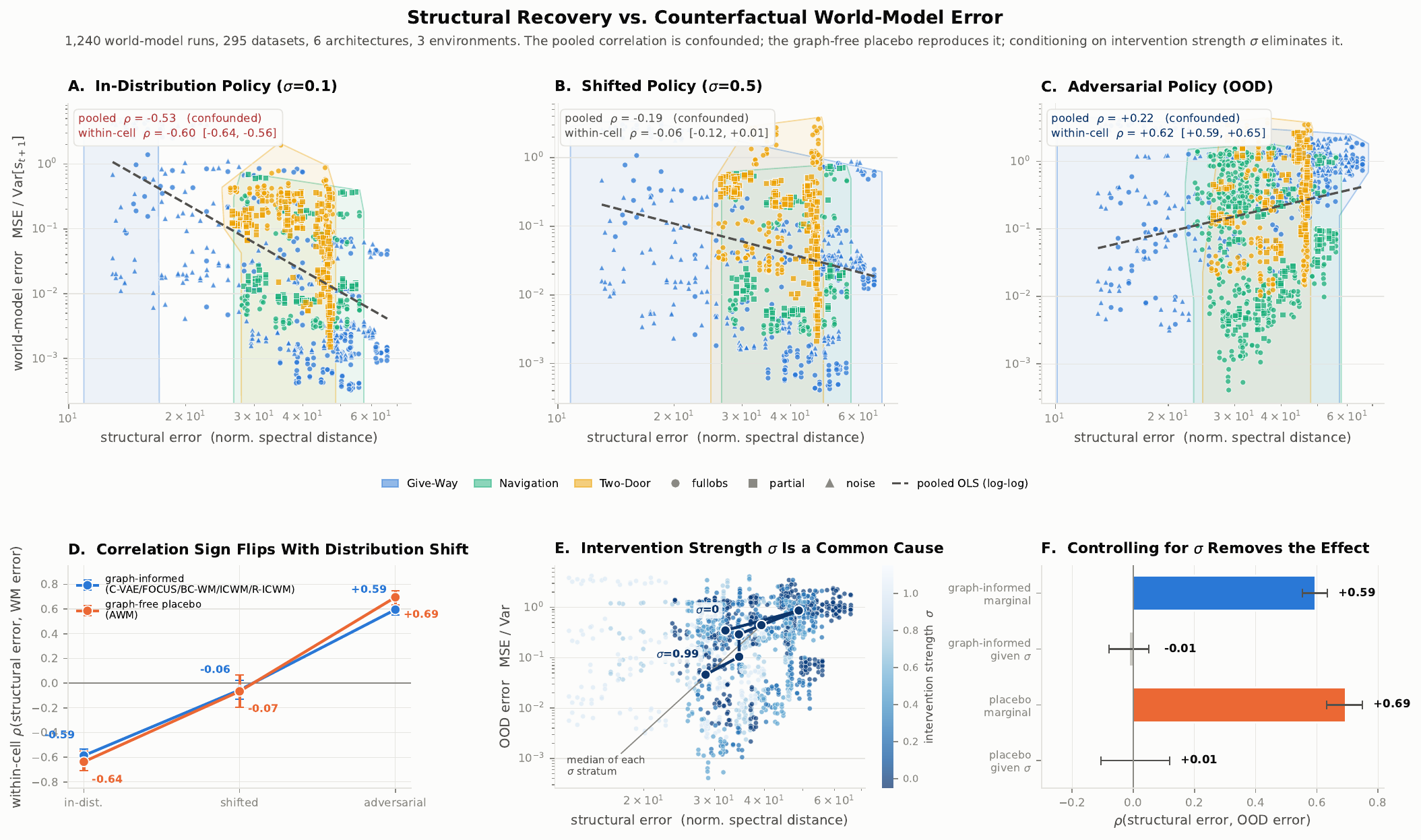}
  \caption{Structural recovery versus counterfactual world-model error, pooled over
  1{,}240 world-model runs and 295 datasets. The naive pooled correlation is confounded
  by environment; the within-cell correlation rotates from negative in distribution to
  $+0.59$ under adversarial, out-of-distribution evaluation (panel D); conditioning on
  intervention strength $\sigma$ removes the effect entirely (panel F).}
  \label{fig:struct-main}
\end{figure}

\begin{figure}[htbp]
  \centering
  \includegraphics[trim={0pt 0pt 00pt 30pt}, clip,width=0.9\linewidth]{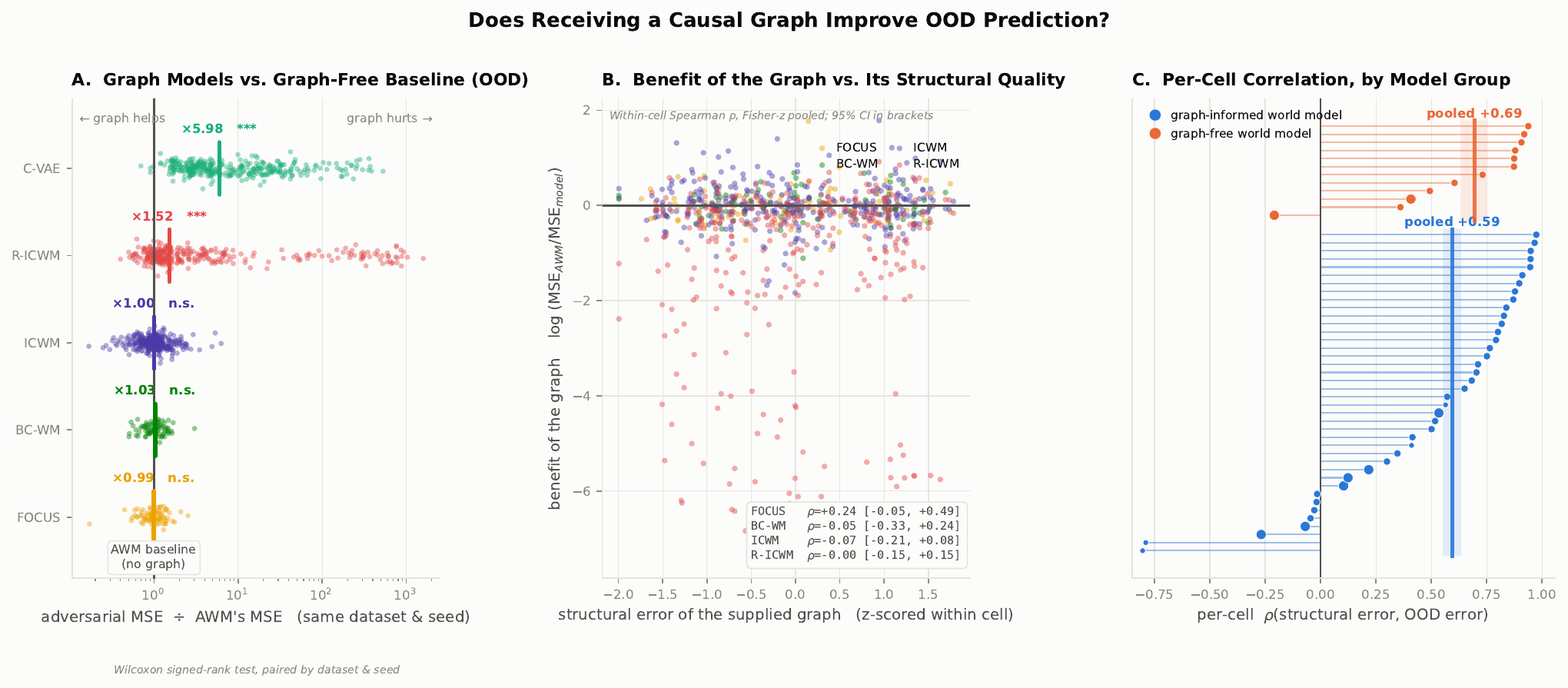}
  \caption{Does receiving a causal graph improve out-of-distribution prediction? Panel A
  compares graph-informed against graph-free models directly; panel B tests whether the
  benefit of the graph scales with the graph's own structural quality (it does not);
  panel C breaks the per-cell correlation down by model group, the placebo test that
  motivates Figure~\ref{fig:struct-main}'s common-cause reading.}
  \label{fig:struct-models}
\end{figure}

\begin{figure}[htbp]
  \centering
  \includegraphics[trim={0pt 0pt 00pt 30pt}, clip,width=0.9\linewidth]{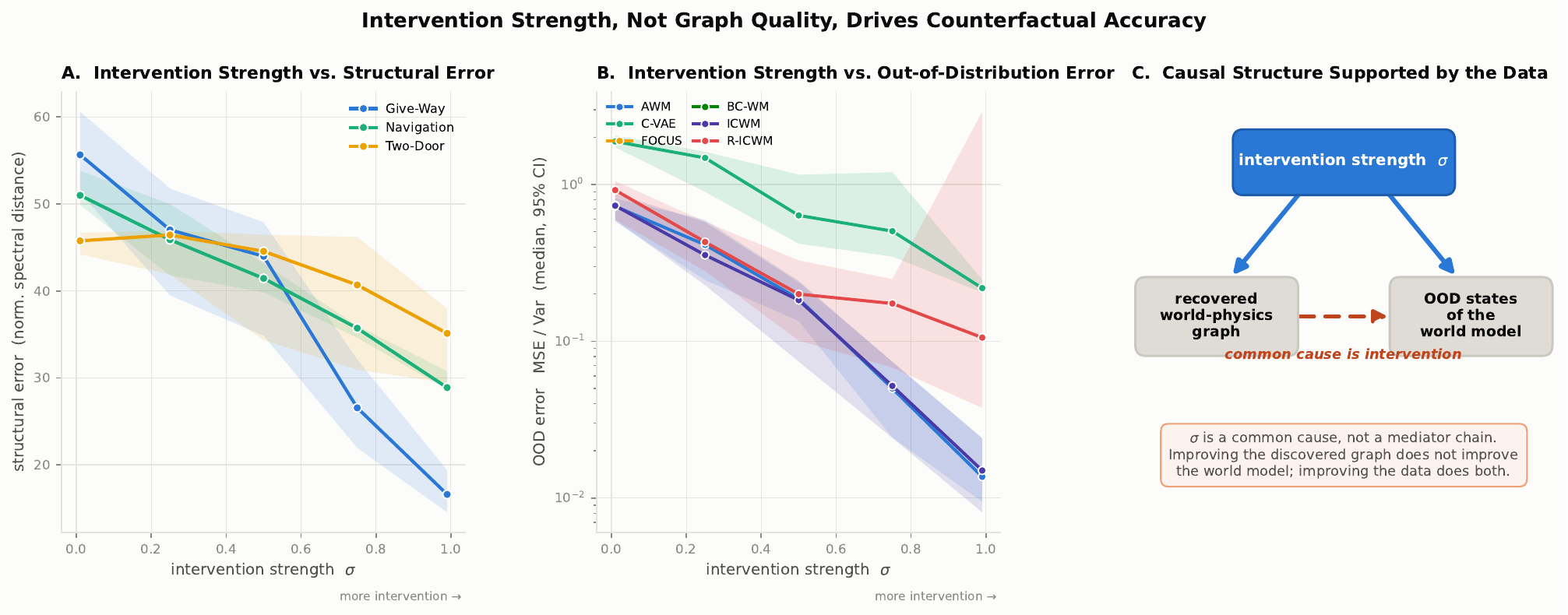}
  \caption{Intervention strength, not graph quality, drives counterfactual accuracy.
  Panel A: $\sigma$ against structural error. Panel B: $\sigma$ against
  out-of-distribution error, by architecture. Panel C: the causal structure the data
  actually support, $\sigma$ as a common cause of both the recovered graph and the
  world model's downstream error, with the direct graph-to-error edge tested and not
  found.}
  \label{fig:struct-alpha}
\end{figure}

\begin{figure}[htbp]
  \centering
  \includegraphics[trim={0pt 0pt 00pt 50pt}, clip,width=0.9\linewidth]{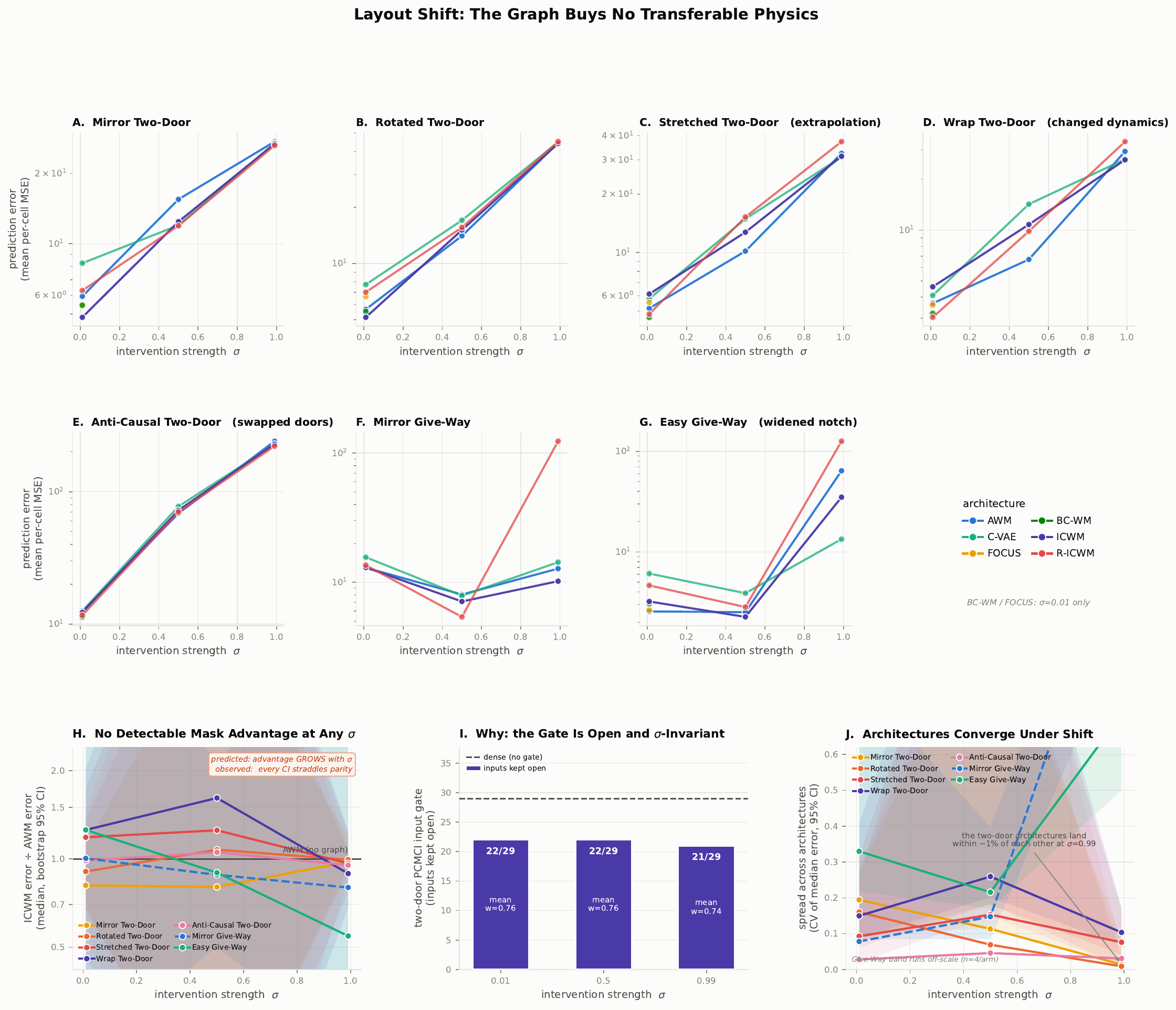}
  \caption{Layout shift: the recovered graph buys no transferable physics. Frozen world
  models, never retrained, evaluated on seven shifted layouts (five two-door variants and
  two give-way variants, panels A--G); the ICWM/AWM error ratio straddles parity at every
  intervention level with no trend (panel H); the two-door discovery graph's input gate stays
  open and $\sigma$-invariant, which is why (panel I); cross-architecture spread collapses
  under the strongest shift (panel J), the opposite of what a graph-driven advantage
  predicts.}
  \label{fig:struct-shift}
\end{figure}

\begin{figure}[htbp]
  \centering
  \includegraphics[trim={0pt 0pt 00pt 30pt}, clip,width=0.9\linewidth]{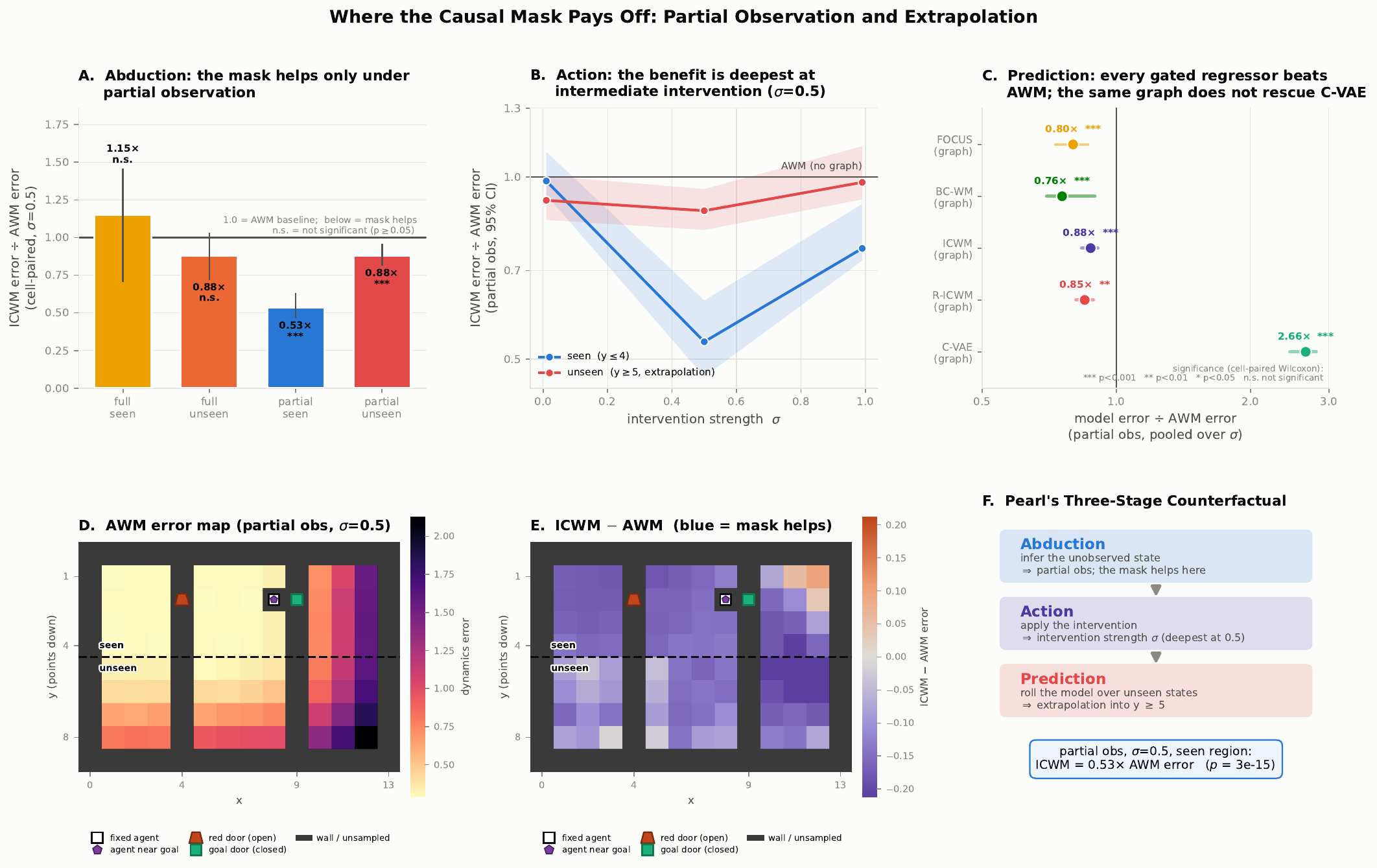}
  \caption{Where the causal mask pays off: partial observation and extrapolation on the
  Stretched Two-Door layout. Panel A: the mask helps only under partial observation.
  Panel B: the benefit is deepest at intermediate intervention strength
  ($\sigma = 0.5$). Panel C: every graph-gated regressor beats the graph-free AWM
  baseline under partial observation, while C-VAE, which receives the same soft-gated
  graph, is far worse. Panels D--E: the spatial error maps behind panel A's
  headline number. Panel F: the result read as Pearl's three-stage counterfactual query,
  abduction, action, prediction, with the mask assisting specifically the abduction
  stage.}
  \label{fig:struct-stretch}
\end{figure}

\begin{figure}[htbp]
  \centering
  \includegraphics[trim={0pt 0pt 00pt 30pt}, clip,width=\linewidth]{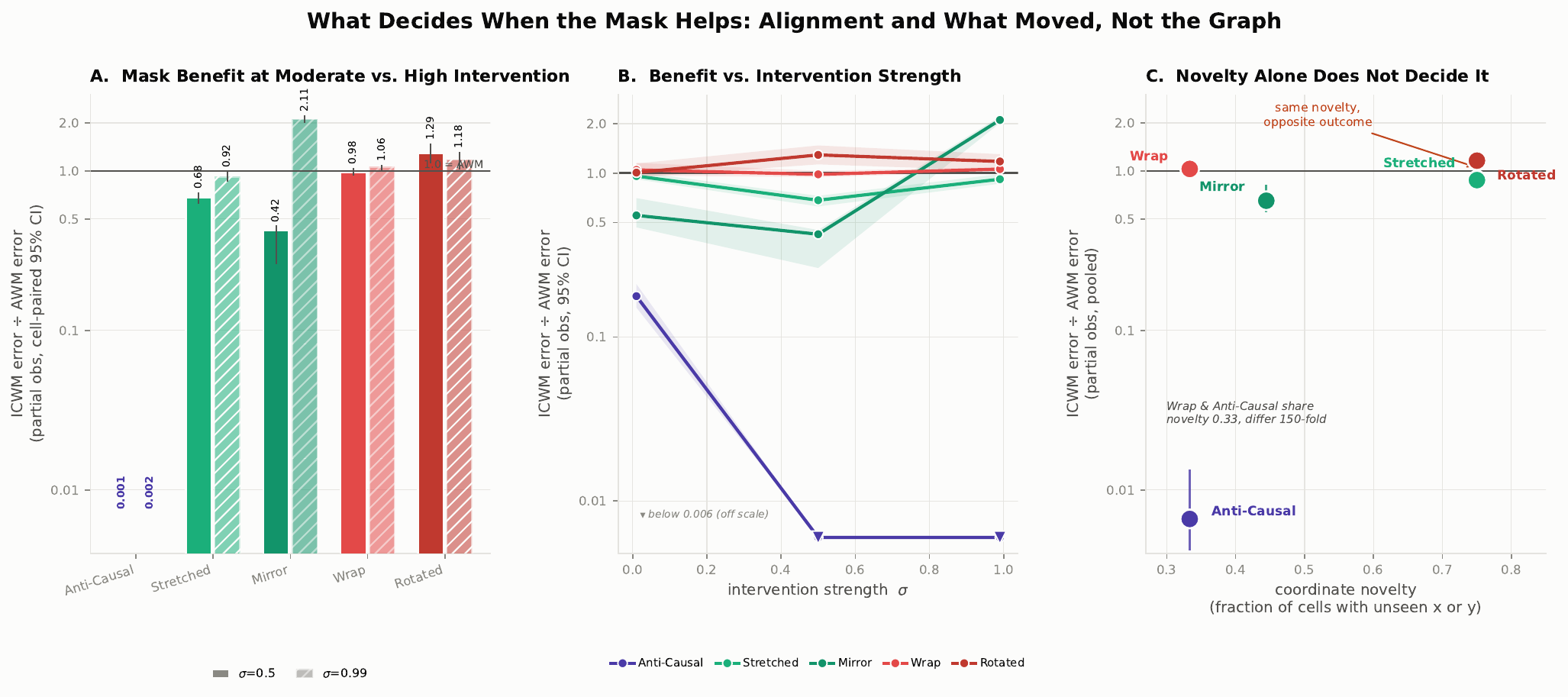}
  \caption{What decides when the mask helps: alignment and what moved, not the graph.
  The discovered graph is identical across all five two-door shifts; what differs is
  whether the shift keeps the frozen mask aligned with the room's coordinates and whether it
  is a shift the mask can act on at all. Panel A: benefit at moderate versus high
  intervention strength, by shift (Anti-Causal's sub-$0.01$ ratio is labelled at the axis
  floor). Panel B: benefit across the full intervention sweep. Panel C: Stretched and Rotated
  present the same $75\%$ coordinate novelty with opposite outcomes, and Wrap and Anti-Causal
  the same $33\%$ novelty with 150-fold-different outcomes, so novelty alone does not decide
  it.}
  \label{fig:struct-align}
\end{figure}

\subsection{Three controlled shifts that separate what earlier layouts confounded}
\label{app:struct-newenvs}

The alignment account of Fig.~\ref{fig:struct-align} rests on layouts that all
\emph{move the geometry}, so they confound three things a single heatmap cannot separate:
the room moved, the visited coordinates changed, and (implicitly) the state distribution
changed. Three further shifts hold all but one of those fixed, each perturbing exactly one
factor while keeping the discovered graph and the frozen checkpoints identical to the
runs above. \texttt{Wrap Two-Door} keeps the room byte-identical (same walls, doors, legal
cells, door states) and changes only the transition rule: a \texttt{forward} step into a
wall now re-enters from the opposite side. \texttt{Anti-Causal Two-Door} keeps the geometry
\emph{and} every door position, and swaps only the door \emph{states} (blue open, red
closed) with the pinned agent moved to match, an off-manifold joint configuration in a task
that requires opening red before blue. \texttt{Easy Give-Way} widens the give-way notch
five-fold in the continuous environment, where there is no grid mask to align at all.
Fig.~\ref{fig:struct-newenvs} reports all three; every ratio is the same cell-paired
median log-ratio, $4000$-replicate bootstrap CI, and Wilcoxon test used for
Fig.~\ref{fig:struct-align}, so the numbers are directly comparable.

\paragraph{Wrap bounds the alignment claim.} Wrap is the strongest possible case
\emph{for} the mask on Fig.~\ref{fig:struct-align}'s own axis: zero invalidity, the lowest
coordinate novelty of any variant ($33\%$), and doors in their trained positions and
states. Yet the mask does nothing, $\text{ICWM}/\text{AWM} = 1.03$ under partial
observation ($95\%$ CI $[1.00, 1.05]$, $p = 0.95$) and $0.98$ under full observation
($p = 0.51$), against $0.88$ for stretched at the same invalidity. Alignment is therefore
\emph{necessary but not sufficient}: a mask over state variables constrains which variables
inform a prediction, but says nothing about the transition rule that maps them forward.
When only that rule shifts, a correctly-aligned mask has no purchase. This does not
overturn the alignment result; it sharpens it from ``alignment explains the benefit'' to
``alignment is a precondition, and the shift must additionally be one the mask can act on.''

\paragraph{Anti-Causal reveals a distinct mechanism: containment, not accuracy.} The
swapped-door configuration produces the largest mask effect anywhere in the project:
$\text{ICWM}/\text{AWM} = 0.007$ on agent-state error, winning $100\%$ of $1{,}960$ cells
($p = 1.5\times10^{-70}$). Taken at face value this is a $140\times$ improvement, but
reporting it that way would mislead. On \emph{all} state dimensions (doors included), every
model, ICWM included, sits at a median error near $80$--$115$, because none of them predict
the swapped doors: the environment is off-manifold for all of them by design. The gap
appears only on the agent-state dimensions, where ICWM ($0.03$) is $90\times$ below AWM
($2.37$) on byte-identical geometry. So the mask is not making better door predictions. It
is stopping the corrupted door signal from leaking into the agent-position prediction. AWM
conditions on everything, so it lets a badly-wrong door variable contaminate a sub-problem
it could otherwise learn; the gate blocks that leak (Fig.~\ref{fig:struct-newenvs}B--C). No
earlier environment could show this, because layout shifts move every variable at once.
Anti-causal is the first to corrupt one group of variables while leaving another intact.

\paragraph{The firewall is learned from intervention, and differs by model family.} Two
further properties tie this mechanism back to the paper's central result and guard against
over-reading it (Fig.~\ref{fig:struct-newenvs-mech}). First, the containment itself depends
on $\sigma$. ICWM's agent-state advantage on anti-causal grows by about two orders of
magnitude as training intervention rises, from $0.18$ at $\sigma = 0.01$ to $0.001$ at
$\sigma = 0.5$ (partial observation), while wrap and the space-rearranging shifts stay flat
across $\sigma$. So the firewall is not a fixed property of the gate. The model
\emph{learns} it from interventional data, the same lever that drives every other result in
this appendix. Second, the two model families that seem to ``contain'' the door error do so
for different reasons. We tell them apart with prediction confidence $\Delta_{\mathrm{var}}$,
the ratio of the prediction error to the variance of a mean-predictor; a value well above
$1$ means the model is making a confident, dynamic prediction. ICWM and R-ICWM predict the
doors as confidently and as wrongly as AWM ($\Delta_{\mathrm{var}} \approx 270$,
all-dimension error $\approx 80$), but they keep that error out of the agent sub-problem.
FOCUS and BC-WM are the $\sigma = 0$ baselines and sit in a different regime
($\Delta_{\mathrm{var}} \approx 67$, all-dimension error $\approx 12$): they never learned a
dynamic door model at all, so their lower door error is a near-static guess, not active
containment. Telling the two apart is what stops the containment claim from reducing to
``the $\sigma=0$ models are just better here.'' They are not better; they are answering a
different question.

\paragraph{Easy Give-Way is a clean control.} In the continuous environment the mask stays
within $4\%$ of AWM at every intervention level ($\text{ICWM}/\text{AWM} = 0.96$ pooled
under observation noise), against the $0.65\times$--$2.1\times$ swings of the two-door
family. Widening a notch gives a grid-shaped mask nothing
to align to, exactly as the alignment account predicts for a shift with no grid geometry.
Several of these points are individually significant only because the continuous sweep has
$\sim\!44$k cells; this is a ``no layout-shift-scale effect'' result, not a strict null.

\begin{figure}[htbp]
  \centering
  \includegraphics[trim={0pt 0pt 00pt 22pt}, clip,width=\linewidth]{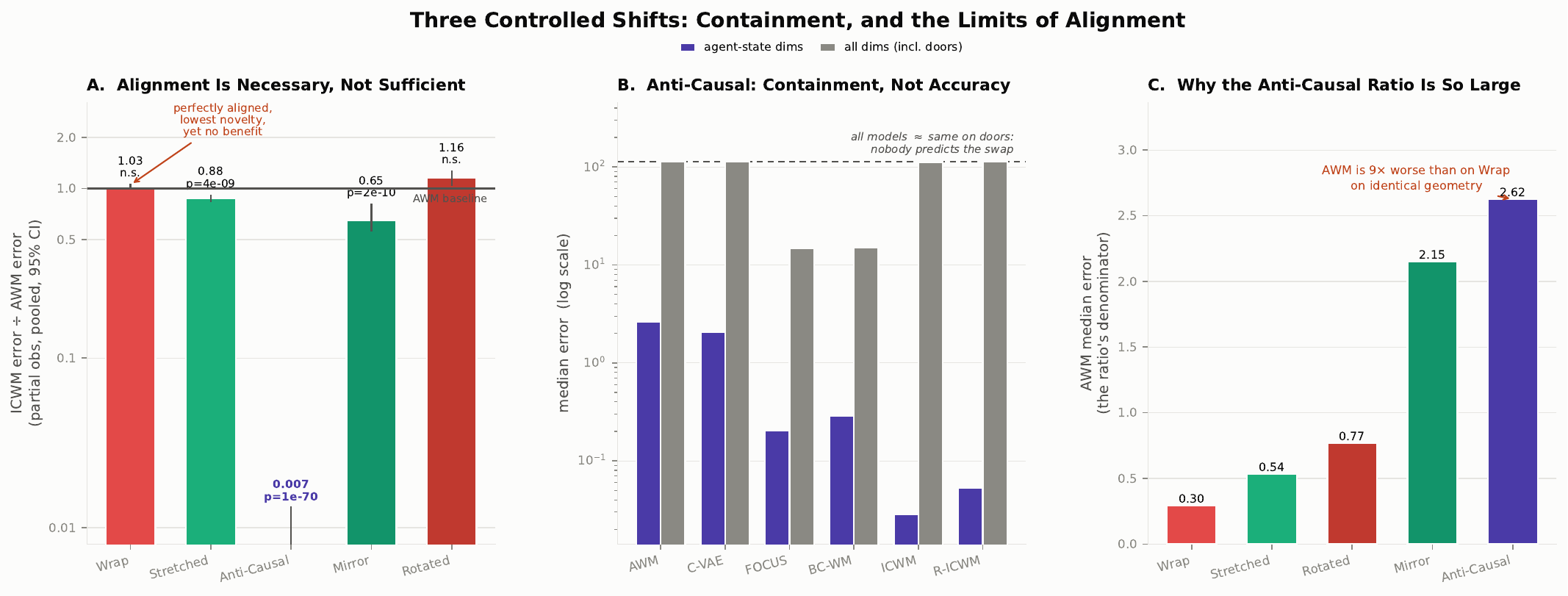}
  \caption{Three controlled shifts, each isolating one factor the earlier layouts moved
  together. \textbf{A}: mask benefit ($\text{ICWM}\div\text{AWM}$ agent-state error, pooled,
  partial observation) across all five two-door shifts; Wrap (perfectly aligned, lowest
  novelty) shows no benefit, bounding the alignment claim. \textbf{B}: on anti-causal, every
  model's error on \emph{all} dimensions is $\sim\!112$ (nobody predicts the swapped doors),
  but on agent-state dimensions ICWM/R-ICWM sit far below AWM, the signature of containment
  rather than accuracy. \textbf{C}: the large anti-causal ratio is mostly AWM blowing up
  ($8.9\times$ its own error on identical geometry), not ICWM improving, which is why
  absolute levels accompany every ratio.}
  \label{fig:struct-newenvs}
\end{figure}

\begin{figure}[htbp]
  \centering
  \includegraphics[trim={0pt 0pt 00pt 30pt}, clip,width=\linewidth]{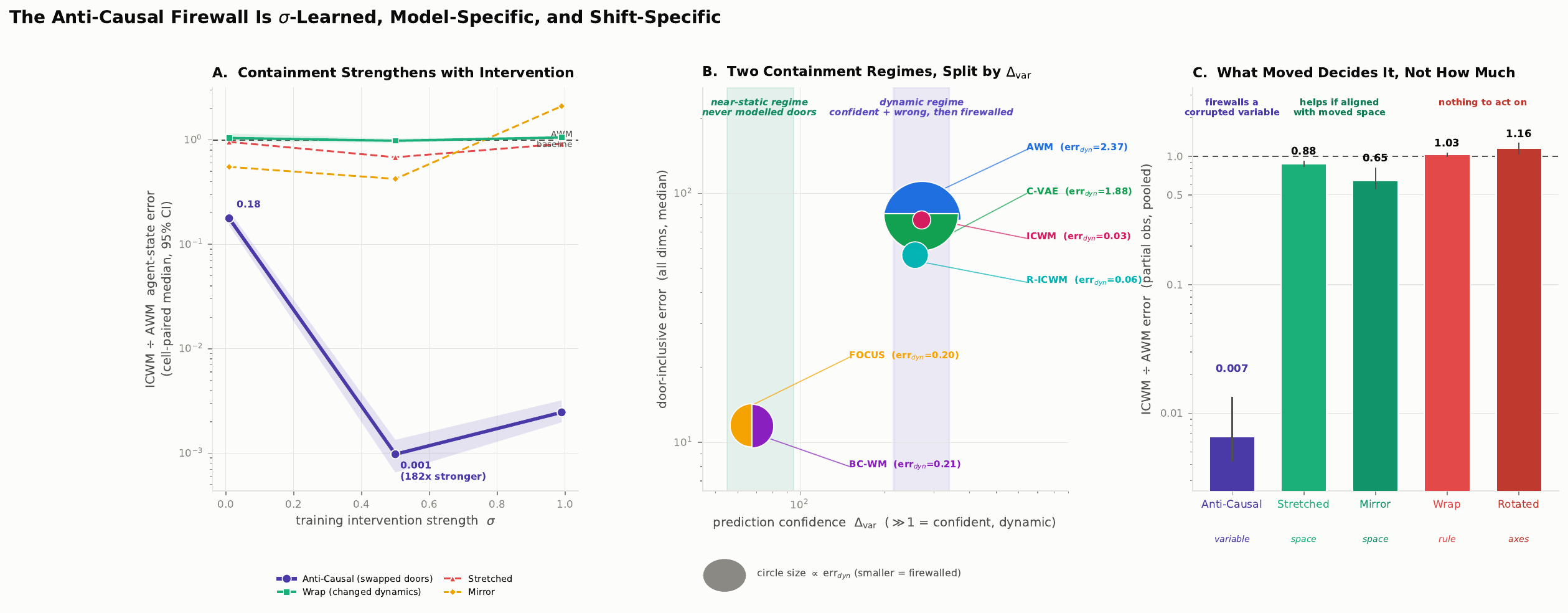}
  \caption{The anti-causal firewall is intervention-learned, model-specific, and
  shift-specific. \textbf{A}: ICWM's agent-state advantage over AWM on the swapped-door
  environment strengthens by $\sim\!180\times$ as training intervention rises from
  $\sigma=0.01$ to $\sigma=0.5$ (partial observation), while wrap and the space-rearranging
  shifts stay flat, folding the containment mechanism back into the paper's
  $\sigma$-is-the-lever result. \textbf{B}: two containment regimes separated by prediction
  confidence $\Delta_{\mathrm{var}}$, each model drawn as a filled circle at
  ($\Delta_{\mathrm{var}}$, all-dimension error) whose area scales with its agent-state error
  (err$_{\mathrm{dyn}}$), so a collapsed circle marks error that has been firewalled;
  coincident models are split into half-circles (AWM top / C-VAE bottom; FOCUS left / BC-WM
  right). ICWM/R-ICWM (and AWM/C-VAE) confidently mispredict the doors
  ($\Delta_{\mathrm{var}}\!\approx\!270$); the masked $\sigma$-trained models firewall that
  error out of the agent sub-problem (circles collapse to a dot) while AWM does not (large
  circle). FOCUS and BC-WM occupy a distinct near-static regime
  ($\Delta_{\mathrm{var}}\!\approx\!67$) whose lower door error reflects the absence of a
  dynamic door model, not active containment. \textbf{C}: across all five two-door shifts,
  coordinate novelty does not sort the mask's benefit; what \emph{moved} does, into three
  regimes: a corrupted variable is firewalled (large benefit), rearranged but
  meaning-preserving space helps if the mask stays aligned, and a changed transition rule or
  transposed axes leave the mask nothing to act on.}
  \label{fig:struct-newenvs-mech}
\end{figure}

\paragraph{Pulling the five two-door shifts together: what moved, not how much.} With all
five two-door shifts in hand, the rule for when the mask helps becomes clean
(Fig.~\ref{fig:struct-newenvs-mech}C). How much of the room is new does not sort the outcomes:
Wrap and Anti-Causal present the model with the same small fraction of unseen coordinates
($0.33$), yet the mask does nothing on one and helps 150-fold on the other. What sorts them
is the \emph{kind} of change. When a corrupted variable enters (Anti-Causal, swapped door
states), the mask firewalls it and the benefit is large. When space is rearranged along
axes that keep their meaning (Stretched, Mirror), the mask helps if it stays aligned with
the moved space. When the movement rule itself changes (Wrap) or the axes are transposed
(Rotated), the mask has nothing to act on and the benefit disappears. This is the same
conclusion the three-shift analysis of Fig.~\ref{fig:struct-align} reached, now stated as a
rule over the type of shift rather than resting on a two-point coincidence. We also note the
containment effect is not an artifact of degraded observation: it is essentially identical
under full and partial observation ($0.0069$ vs $0.0066$), so it is a property of the mask,
not of the observation bottleneck.

\subsection{Summary}

Across every test we ran, the quality of structural recovery does not predict
counterfactual world-model accuracy. This holds when we pool the data, when we work within
each cell after the control test, under layout shift, and when the graph is held
byte-identical across shifts. There is one narrow regime where a graph-informed model
reliably beats its graph-free counterpart: partial observation, intermediate intervention
strength, and a shift that keeps the meaning of each coordinate. Even there, the benefit
tracks how well the frozen structural prior stays aligned with the query's geometry, not
how accurate the graph is. This narrows the general claim rather than rescuing it. A
structural prior is not a free source of counterfactual robustness. It is a targeted tool
for one step of a counterfactual query, namely abduction under missing information, which is
the step it is actually suited to. The three controlled shifts of
Appendix~\ref{app:struct-newenvs} turn this into a mechanism. A correctly-aligned mask buys
nothing when the transition rule is what moved (Wrap). Where it helps most, on the
swapped-door configuration, it works by \emph{containment}: it keeps a corrupted group of
variables out of a sub-problem the model could otherwise learn, rather than predicting the
corrupted variables any better. That firewall is not a fixed property of the gate; it
strengthens with training intervention by two orders of magnitude, the same $\sigma$ lever
that drives every result above. This distinction, between structure as a general performance
lever and structure as an intervention-learned tool for one stage of a counterfactual
question, is what motivates the closing section.


\end{document}